\documentclass[runningheads]{llncs}
\usepackage{graphicx}

\usepackage{tikz}
\usepackage{comment}
\usepackage{amsmath,amssymb} 
\usepackage{color}

\usepackage{subcaption}
\usepackage{placeins}
\usepackage{verbatim}
\usepackage{booktabs}
\usepackage{multirow}
\usepackage{nicefrac}
\usepackage{mathtools}

\usepackage{adjustbox}
\usepackage{tabularx}
\usepackage{wrapfig} 

\usepackage{bm}
\usepackage{float}

\begin{document}
\pagestyle{headings}
\mainmatter
\def\ECCVSubNumber{2548}  

\title{How does Lipschitz Regularization \\ Influence GAN Training?} 

\titlerunning{How does Lipschitz Regularization Influence GAN Training?}
%
\author{Yipeng Qin\inst{1,3} \and
Niloy Mitra\inst{2} \and
Peter Wonka\inst{3}}
\authorrunning{Qin, Mitra, Wonka}
%
\institute{Cardiff University, \email{qiny16@cardiff.ac.uk}
\and
UCL/Adobe Research, \email{n.mitra@cs.ucl.ac.uk}
\and
KAUST,
\email{pwonka@gmail.com}}
\maketitle

\begin{abstract}

Despite the success of Lipschitz regularization in stabilizing GAN training, the exact reason of its effectiveness remains poorly understood.
The direct effect of $K$-Lipschitz regularization is to restrict the $L2$-norm of the neural network gradient to be smaller than a threshold $K$ (e.g., $K=1$) such that $\|\nabla f\| \leq K$.
In this work,~we uncover an even more important effect of Lipschitz regularization by examining its impact on the loss function:
\textit{It degenerates GAN loss functions to almost linear ones by restricting their domain and interval of attainable gradient values}.
Our analysis shows that loss functions are only successful if they are degenerated to almost linear ones. We also show that loss functions perform poorly if they are not degenerated and that a wide range of functions can be used as loss function as long as they are sufficiently degenerated by regularization.
Basically, Lipschitz regularization ensures that all loss functions \textit{effectively work in the same way.}
Empirically, we verify our proposition on the MNIST, CIFAR10 and CelebA datasets.


\end{abstract}


\section{Introduction}

Generative Adversarial Networks (GANs) are a class of generative models successfully applied to various applications, e.g., pose-guided image generation \cite{ma2017pose}, image-to-image translation \cite{CycleGAN2017,park2019semantic}, texture synthesis \cite{anna2019tilegan}, high resolution image synthesis \cite{wang2018pix2pixHD}, 3D model generation \cite{3dgan}, urban modeling~\cite{KellyGuerreroEtAl_FrankenGAN_SigAsia2018}. 
Goodfellow et al.~\cite{goodfellow2014generative} proved the convergence of GAN training by assuming that the generator is always updated according to the temporarily optimal discriminator at each training step.
In practice, this assumption is too difficult to satisfy and GANs remain  difficult to train.
To stabilize the training of the GANs, various techniques have been proposed regarding the choices of architectures \cite{radford2015unsupervised,he2016deep}, loss functions \cite{pmlr-v70-arjovsky17a,Mao_2017_ICCV}, regularization and normalization \cite{pmlr-v70-arjovsky17a,gulrajani2017improved,DBLP:conf/icml/MeschederGN18,miyato2018spectral}.
We refer interested readers to \cite{lucic2017gans,kurach2018gan} for extensive empirical studies.

Among them, the Lipschitz regularization \cite{gulrajani2017improved,miyato2018spectral} has shown great success in stabilizing the training of various GANs.
For example, \cite{mao2017effectiveness} and \cite{fedus*2018many} observed that the gradient penalty Lipschitz regularizer helps to improve the training of the LS-GAN \cite{Mao_2017_ICCV} and the NS-GAN \cite{goodfellow2014generative},  respectively; \cite{miyato2018spectral} observed that the NS-GAN, with their spectral normalization Lipschitz regularizer, works better than the WGAN \cite{pmlr-v70-arjovsky17a} regularized by gradient penalty (WGAN-GP) \cite{gulrajani2017improved}.

\begin{figure}[t]
\begin{center}
\includegraphics[width=.99\linewidth]{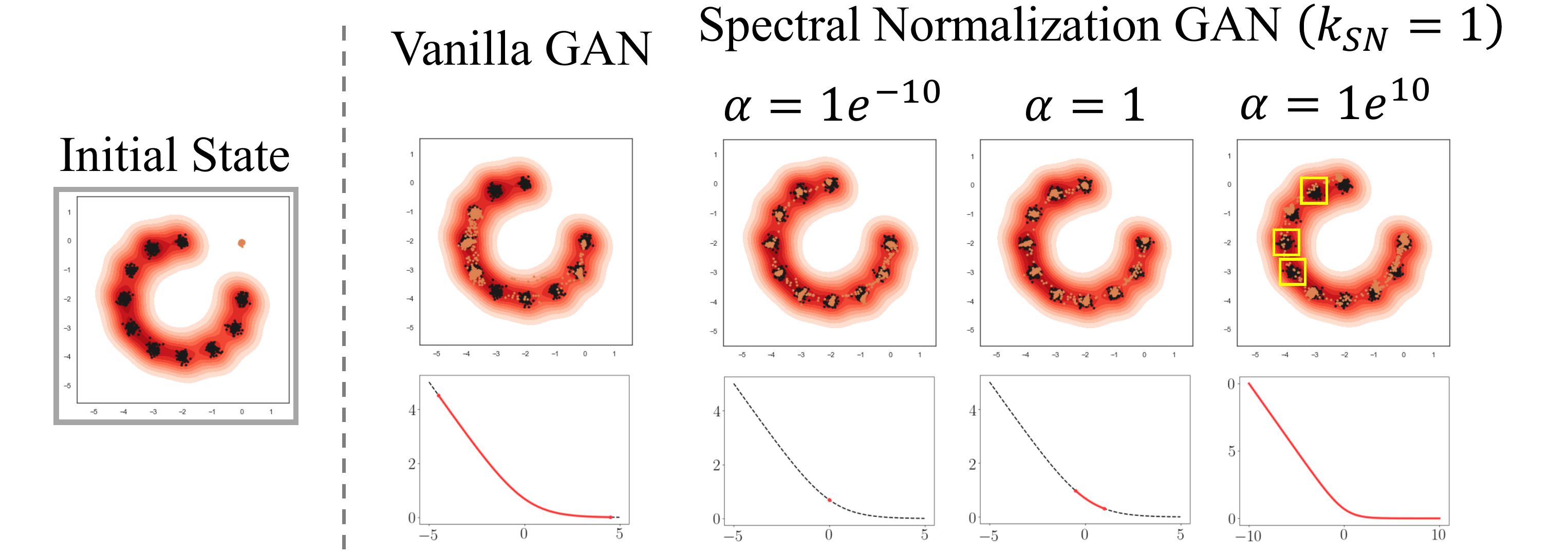}
\caption{An illustrative 2D example. First row: model distribution (orange point) vs. data distribution (black points). Second row: domains of the loss function (red curve). It can be observed that the performance of spectral normalized GANs \cite{miyato2018spectral} worsen when their domains are enlarged ($\alpha=1e^{10}$, yellow boxes). However, good performance can always be achieved when their domains are restricted to near-linear ones ($\alpha=1e^{-10}$ and $\alpha=1$). Please see Sections  \ref{sec:decoupling_degenerate_loss_functions} and \ref{sec:experiment_setup} for the definitions of $\alpha$ and $k_{SN}$,  respectively.}
\label{fig:toy_example}
\end{center}
\end{figure}

\begin{figure}[ht]
\begin{center}
\begin{subfigure}{0.9\textwidth}
    \includegraphics[width=.99\linewidth]{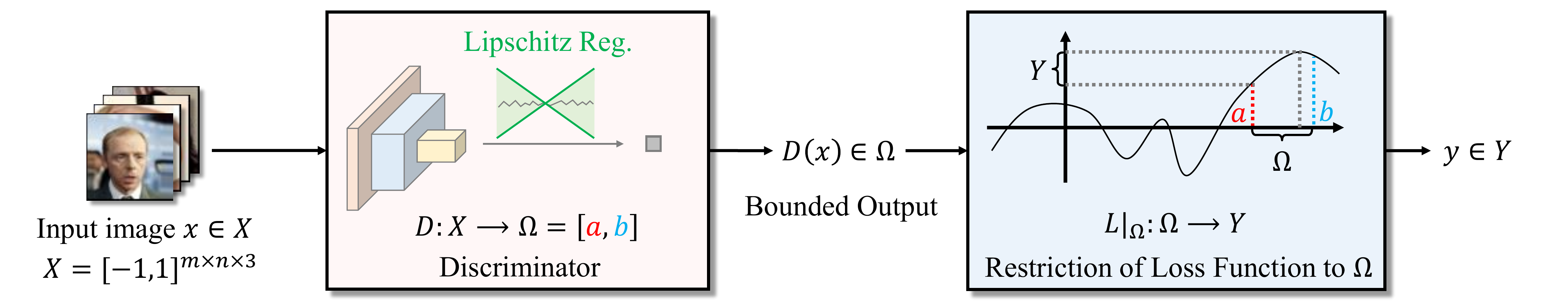}
\end{subfigure}
\begin{subfigure}{0.3\textwidth}
    \includegraphics[width=.99\linewidth]{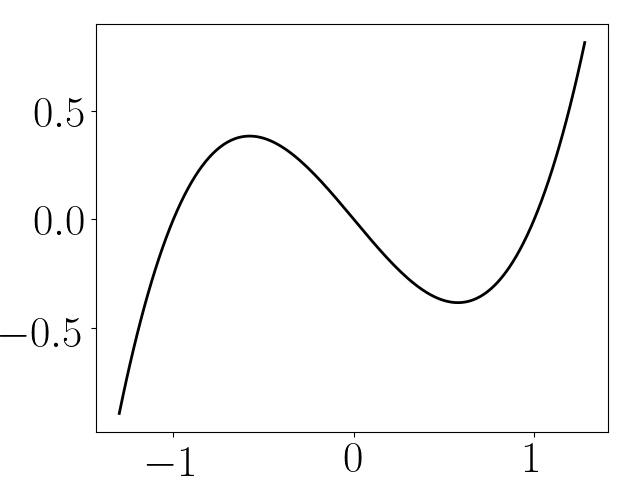}
    \subcaption{$f(x)=x^3-x$}
\end{subfigure}
\begin{subfigure}{0.3\textwidth}
    \includegraphics[width=.99\linewidth]{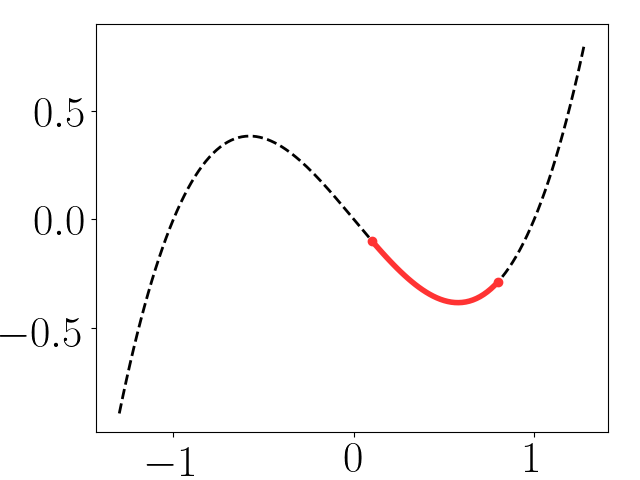}
    \subcaption{$[a,b] = [0.1,0.8]$}
\end{subfigure}
\begin{subfigure}{0.3\textwidth}
    \includegraphics[width=.99\linewidth]{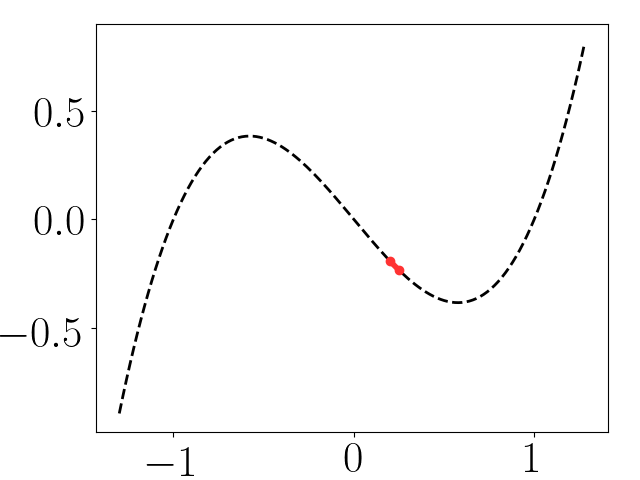}
    \subcaption{$[a,b] = [0.2,0.25]$}
\end{subfigure}
\caption{First row: Applying Lipschitz regularization restricts the domain of the loss function to an interval $\Omega = [a,b]$. Second row: Illustration of the domain restriction. We take a third-order polynomial loss function $f(x)=x^3-x$ as an example. Its restricted domain $[a,b]$ is shown in red. (a) without restriction $f(x)$ is non-convex. (b) Restricting the domain of $f(x)$ makes it convex. (c) $f(x)$ is almost linear when its domain is restricted to a very small interval.}
\label{fig:illustration}
\end{center}
\end{figure}

In this paper, we provide an analysis to better understand the \textit{coupling} of Lipschitz regularization and the choice of  loss function.
Our main insight is that the rule of thumb of using small Lipschitz constants (e.g., $K=1$) is degenerates the loss functions to almost linear ones by restricting their domain and interval of attainable gradient values (see Figure~\ref{fig:toy_example}). 
These degenerate losses improve GAN training.
Because of this, the exact shapes of the loss functions before degeneration do not seem to matter that much. 
We demonstrate this by two experiments. 
First, we show that when $K$ is sufficiently small, even GANs trained with non-standard loss functions (e.g., cosine) give comparable results to all other loss functions.
Second, we can directly degenerate loss functions by introducing domain scaling.
This enables successful GAN training for a wide range of $K$ for all loss functions, which only worked for the Wasserstein loss before. 
Our contributions include:
\begin{itemize}
    \item We discovered an important effect of Lipschitz regularization. It restricts the domain of the loss function (Figure \ref{fig:illustration}).
    \item Our analysis suggests that although the choice of loss functions matters, the successful ones currently being used are all near-linear within the chosen small domains and actually work in the same way.
\end{itemize}


\section{Related Work}

\subsection{GAN Loss Functions}

A variety of GAN loss functions have been proposed from the idea of understanding the GAN training as the minimization of statistical divergences.
Goodfellow et al. \cite{goodfellow2014generative} first proposed to minimize the Jensen-Shannon (JS) divergence between the model distribution and the target distribution.
In their method, the neural network output of the discriminator is first passed through a sigmoid function to be scaled into a probability in $[0,1]$. 
Then, the cross-entropy loss of the probability is measured.
Following \cite{fedus*2018many}, we refer to such loss as the \textit{minimax} (MM) loss since the GAN training is essentially a minimax game.
However, because of the saturation at both ends of the sigmoid function, the MM loss can lead to vanishing gradients and thus fails to update the generator.
To compensate for it, Goodfellow et al. \cite{goodfellow2014generative} proposed a variant of MM loss named the \textit{non-saturating} (NS) loss, which heuristically amplifies the gradients when updating the generator.

Observing that the JS divergence is a special case of the $f$-divergence, Nowozin et al. \cite{Nowozin2016nips} extended the idea of Goodfellow et al. \cite{goodfellow2014generative} and showed that any $f$-divergence can be used to train GANs.
Their work suggested a new direction of improving the performance of the GANs by employing ``better'' divergence measures.

Following this direction, Arjovsky et al. first pointed out the flaws of the JS divergence used in GANs \cite{arjovsky2017towards} and then proposed to use the Wasserstein distance instead (WGAN) \cite{pmlr-v70-arjovsky17a}.
In their implementation, the raw neural network output of the discriminator is directly used (i.e. the WGAN loss function is an identity function) instead of being passed through the sigmoid cross-entropy loss function.
However, to guarantee that their loss is a valid Wasserstein distance metric, the discriminator is required to be Lipschitz continuous.
Such requirement is usually fulfilled by applying an extra Lipschitz regularizer to the discriminator.
Meanwhile, Mao et al. \cite{Mao_2017_ICCV} proposed the Least-Square GAN (LS-GAN) to minimize the Pearson $\chi^2$ divergence between two distributions.
In their implementation, the sigmoid cross-entropy loss is replaced by a quadratic loss.

\subsection{Lipschitz Regularization}
The first practice of applying the Lipschitz regularization to the discriminator came together with the WGAN \cite{pmlr-v70-arjovsky17a}.
While at that time, it was not employed to improve the GAN training but just a requirement of the Kantorovich-Rubinstein duality applied.
In \cite{pmlr-v70-arjovsky17a}, the Lipschitz continuity of the discriminator is enforced by \textit{weight clipping}.
Its main idea is to clamp the weights of each neural network layer to a small fixed range $[-c,c]$, where $c$ is a small positive constant.
Although weight clipping guarantees the Lipschitz continuity of the discriminator, the choice of parameter $c$ is difficult and prone to invalid gradients.

To this end, Gulrajani et al. \cite{gulrajani2017improved} proposed the \textit{gradient penalty} (GP) Lipschitz regularizer to stabilize the WGAN training, i.e. WGAN-GP.
In their method, an extra regularization term of discriminator's gradient magnitude is weighted by parameter $\lambda$ and added into the loss function. 
In \cite{gulrajani2017improved}, the gradient penalty regularizer is one-centered, aiming at enforcing 1-Lipschitz continuity.
While Mescheder et al. \cite{DBLP:conf/icml/MeschederGN18} argued that the zero-centered one should be more reasonable because it makes the GAN training converge.
However, one major problem of gradient penalty is that it is computed with finite samples, which makes it intractable to be applied to the entire output space.
To sidestep this problem, the authors proposed to heuristically sample from the straight lines connecting model distribution and target distribution.
However, this makes their approach heavily dependent on the support of the model distribution \cite{miyato2018spectral}.

Addressing this issue, Miyato et al. \cite{miyato2018spectral} proposed the \textit{spectral normalization} (SN) Lipschitz regularizer which enforces the Lipschitz continuity of a neural network in the operator space.
Observing that the Lipschitz constant of the entire neural network is bounded by the product of those of its layers, they break down the problem to enforcing Lipschitz regularization on each neural network layer. These simplified sub-problems can then be solved by normalizing the weight matrix of each layer according to its largest singular value.


\section{Restrictions of GAN Loss Functions}

In this section, first we derive why a $K$-Lipschitz regularized discriminator \textit{restricts} the domain and interval of attainable gradient values of the loss function to intervals bounded by $K$ (Section \ref{sec:why_do_GAN_loss_functions_degenerate}).
Second, we propose a scaling method to restrict the domain of the loss function without changing $K$ (Section \ref{sec:decoupling_degenerate_loss_functions}).

\subsection{How does the Restriction Happen?}
\label{sec:why_do_GAN_loss_functions_degenerate}

Let us consider a simple discriminator $D(x) = L(f(x))$, where $x$ is the input, $f$ is a neural network with scalar output, $L$ is the loss function.
During training, the loss function $L$ works by backpropagating the gradient $\nabla L = \partial L(f(x)) / \partial f(x)$ to update the neural network weights:
\begin{equation}
    \frac{\partial D(x)}{\partial W^n} = \bm{\frac{\partial L(f(x))}{\partial f(x)}} \frac{\partial f(x)}{\partial W^n}
\end{equation}
where $W^n$ is the weight matrix of the $n$-th layer.
Let $X$ and $\Omega$ be the domain and the range of $f$ respectively (i.e., $f: X \to \Omega$), it can be easily derived that the attainable values of $\nabla L$ is determined by $\Omega$ (i.e., $\nabla L: \Omega \to \Psi$). 
Without loss of generality, we assume that $x \in X = [-1,1]^{m\times n\times 3}$ are normalized images and derive the bound of the size of $\Omega$ as follows:
\begin{theorem}
If the discriminator neural network $f$ satisfies the $k$-Lipschitz continuity condition, we have $f: X \to \Omega \subset \mathbb{R}$ satisfying $|\min(\Omega) - \max(\Omega)| \leq k\sqrt{12mn}$.
\label{theorem:degenerate_loss_function}
\end{theorem}
\begin{proof}
Given a $k$-Lipschitz continuous neural newtork $f$, for all $x_1,x_2 \in X$, we have:
\begin{equation}
    |f(x_1) - f(x_2)| \leq k \| x_1 - x_2\|.
\end{equation}
Let $x_b, x_w \in X$ be the pure black and pure white images that maximize the Euclidean distance:
\begin{equation}
  \|x_b - x_w\| = \sqrt{(-1 - 1)^2 \cdot m \cdot n \cdot 3}= \sqrt{12mn}.
\end{equation}
Thus, we have:
\begin{equation}
\begin{split}
    | f(x_1) - f(x_2) | &\leq k\|x_1 - x_2\| \\
                        &\leq k\|x_b - x_w\| = k\sqrt{12mn}.
\end{split}
    \label{eq:discriminator_range_bound}
\end{equation}
Thus, the range of $f$ is restricted to $\Omega$, which satisfies:
\begin{equation}
    |\min(\Omega) - \max(\Omega)| \leq k\sqrt{12mn}
\end{equation}
\end{proof}
Theorem \ref{theorem:degenerate_loss_function} shows that the size of $\Omega$ is bounded by $k$.
However, $k$ can be unbounded when Lipschitz regularization is not enforced during training, which results in an unbounded $\Omega$ and a large interval of attainable gradient values.
On the contrary, when $K$-Lipschitz regularization is applied (i.e., $k \leq K$), the loss function $L$ is restricted as follows:
\begin{corollary}[\textbf{Restriction of Loss Function}]
Assume that $f$ is a Lipschitz regularized neural network whose Lipschitz constant $k \leq K$, the loss function $L$ is $C^2$-continuous with $M$ as the maximum absolute value of its second derivatives in its domain. Let $\Psi$ be the interval of attainable gradient values that $\nabla L: \Omega \to \Psi$, we have
\begin{equation}
    \big\vert \min(\Omega) - \max(\Omega) \big\vert \leq K\sqrt{12mn}
    \label{eq:domain_bound}
\end{equation}
\begin{equation}
    \big\vert \min(\Psi) - \max(\Psi) \big\vert \leq M \cdot K\sqrt{12mn}
    \label{eq:gradient_range_bound}
\end{equation}
\label{corollary:gradient_range}
\end{corollary}

Corollary \ref{corollary:gradient_range} shows that under a mild condition ($C^2$-continuous), applying $K$-Lipschitz regularization restricts the domain $\Omega$ and thereby the interval of attainable gradient values $\Psi$ of the loss function $L$ to intervals bounded by $K$.
When $K$ is small, e.g., $K=1$ \cite{gulrajani2017improved,fedus*2018many,mao2017effectiveness,miyato2018spectral}, the interval of attainable gradient values of the loss function is considerably reduced, which prevents the backpropagation of vanishing or exploding gradients and thereby stabilizes the training.
Empirically, we will show that these restrictions are indeed significant in practice and strongly influence the training.

\setlength{\intextsep}{0pt}
\setlength{\columnsep}{0pt}
\begin{wrapfigure}{r}{0.5\textwidth}
\begin{center}
    \includegraphics[width=.8\linewidth]{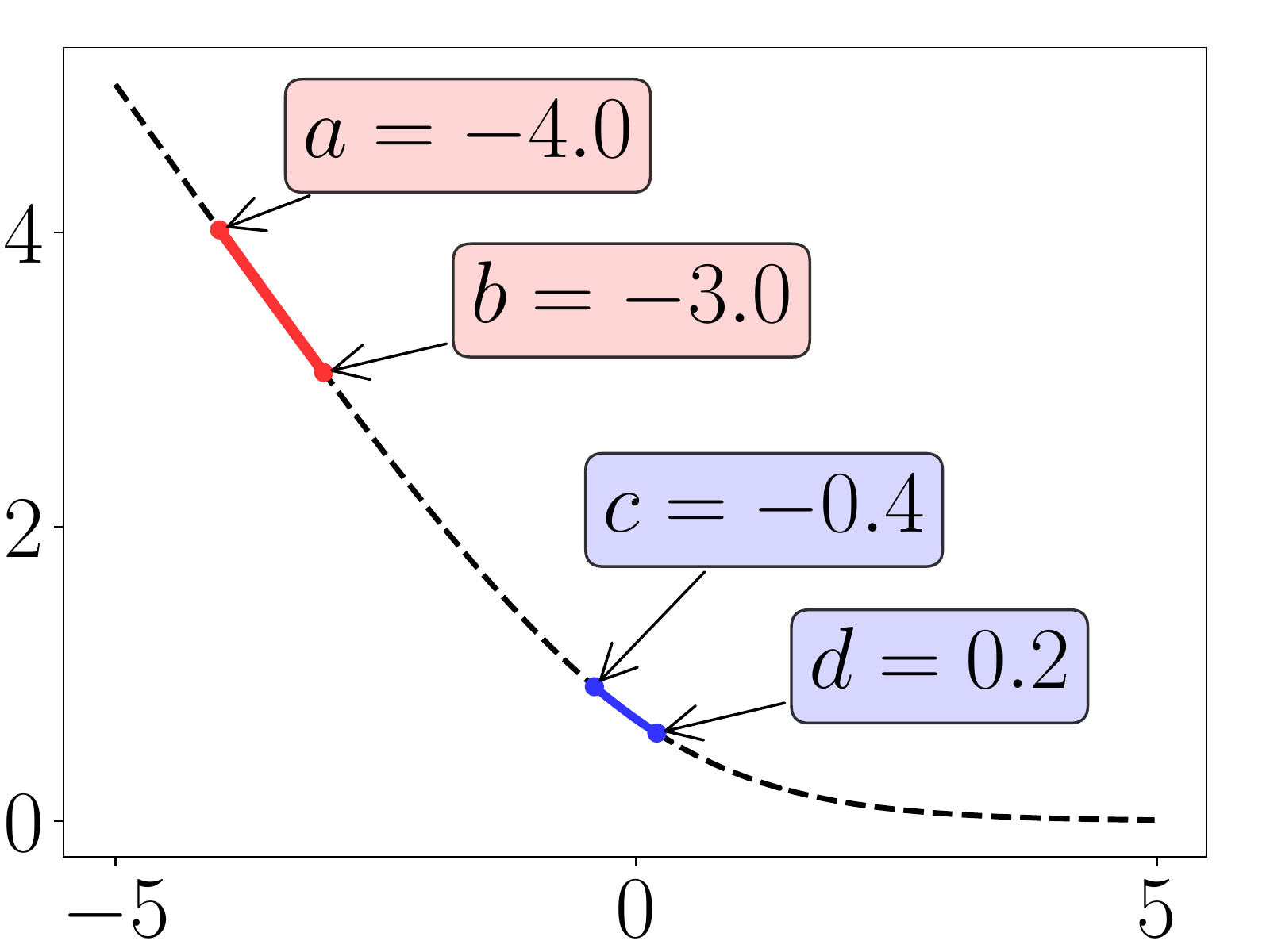}
\end{center}
  \caption{Domain $[a,b]$ shifts to $[c,d]$.}
\label{fig:shifting_domain}
\end{wrapfigure}
\noindent{\bf Change in $\Omega_i$ During Training.}
So far we analyzed the restriction of the loss function by a static discriminator. 
However, the discriminator neural network $f$ is dynamically updated during training and thus its range $\Omega^{\cup} = \cup_i \Omega_i$, where $\Omega_i$ is the discriminator range at each training step $i$. 
Therefore, we need to analyze two questions:\\
\indent (i) How does the size of $\Omega_i$ change during training?\\
\indent (ii) Does $\Omega_i$ shift during training (Figure \ref{fig:shifting_domain})?\\

For question (i), the size of $\Omega_i$ is always bounded by the Lipschitz constant $K$ throughout the training (Corollary \ref{corollary:gradient_range}).
For question (ii), the answer depends on the discriminator loss function:
\begin{itemize}
    \item The shifting of $\Omega_i$ is prevented if the loss function is strictly convex.
    For example, the discriminator loss function of NS-GAN \cite{goodfellow2014generative} (Table \ref{tb:GAN_losses}) is strictly convex and has a unique minimum when $f(x) = f(g(z)) = 0$ at convergence.
    Thus, minimizing it forces $\Omega_i$ to be positioned around $0$ and prevents it from shifting.
    The discriminator loss function of LS-GAN \cite{Mao_2017_ICCV} (Table \ref{tb:GAN_losses}) has a similar behavior.
    Its $\Omega_i$ is positioned around $0.5$, since its minimum is achieved when $f(x) = f(g(z)) = 0.5$ at convergence.
    In this scenario, the $\Omega_i$ is relatively fixed throughout the training. Thus, $\Omega^{\cup}$ is still roughly bounded by the Lipschitz constant $K$.
    \item When the discriminator loss functions is not strictly convex, $\Omega_i$ may be allowed to shift.
    For example, the WGAN \cite{pmlr-v70-arjovsky17a} discriminator loss function (Table \ref{tb:GAN_losses}) is linear and achieves its minimum when $f(x) = f(g(z))$ at convergence.
    Thus, it does not enforce the domain $\Omega_i$ to be fixed.
    However, the linear WGAN loss function has a constant gradient that is independent of $\Omega_i$.
    Thus, regarding to the interval of attainable gradient values (Eq.\ref{eq:gradient_range_bound}), we can view it as a degenerate loss function that still fits in our discussion.
    Interestingly, we empirically observed that the domain $\Omega_i$ of WGANs also get relatively fixed at late stages of the training (Figure \ref{fig:valid_interval_NS_SN}).
\end{itemize}

\begin{table*}[t]
    \captionof{table}{The GAN loss functions used in our experiments. $f(\cdot)$ is the output of the discriminator neural network; $g(\cdot)$ is the output of the generator; $x$ is a sample from the training dataset; $z$ is a sample from the noise distribution. LS-GAN$^{\#}$ is the zero-centered version of LS-GAN \cite{mao2017effectiveness}, and for the NS-GAN, $f^*(\cdot) = \mathrm{sigmoid}[f(\cdot)]$. The figure on the right shows the shape of the loss functions at different scales. The dashed lines show  non-standard loss functions: $\cos$ and $\exp$.}
    \label{tb:GAN_losses}

\begin{minipage}{.81\textwidth}
    \begin{center}
        \centering
            \renewcommand{\arraystretch}{1.25}
            \resizebox{\textwidth}{!}{
            \begin{tabular}{r  l  l}
            GAN types & Discriminator Loss  & Generator Loss \\ \toprule
            NS-GAN & $L_{D}=-\mathbb{E}[\log(f^*(x))] - \mathbb{E}[\log(1-f^*(g(z)))]$ & $L_{G}=-\mathbb{E}[\log(f^*(g(z)))]$ \\
            LS-GAN & $L_{D}=\mathbb{E}[(f(x) - 1)^2] + \mathbb{E}[f(g(z))^2]$ & $L_{G}=\mathbb{E}[(f(g(z)) - 1)^2]$ \\
            LS-GAN$^{\#}$ & $L_{D}=\mathbb{E}[(f(x) - 1)^2] + \mathbb{E}[(f(g(z))+1)^2]$ & $L_{G}=\mathbb{E}[(f(g(z)) - 1)^2]$ \\ 
            WGAN & $L_{D}=\mathbb{E}[f(x)] - \mathbb{E}[f(g(z))]$ & $L_{G}=\mathbb{E}[f(g(z))]$ \\ 
            COS-GAN & $L_{D}=-\mathbb{E}[\cos(f(x) - 1)] - \mathbb{E}[\cos(f(g(z)) + 1)]$ & $L_{G}=-\mathbb{E}[\cos(f(g(z)) - 1)]$ \\ 
            EXP-GAN & $L_{D}=\mathbb{E}[\exp(f(x))] + \mathbb{E}[\exp(-f(g(z)))]$ & $L_{G}=\mathbb{E}[\exp(f(g(z)))]$ \\
            \end{tabular}
            }
            \renewcommand{\arraystretch}{1}
    \end{center}
\end{minipage}
\begin{minipage}{.17\textwidth}
  \centering
  \includegraphics[width=.99\linewidth]{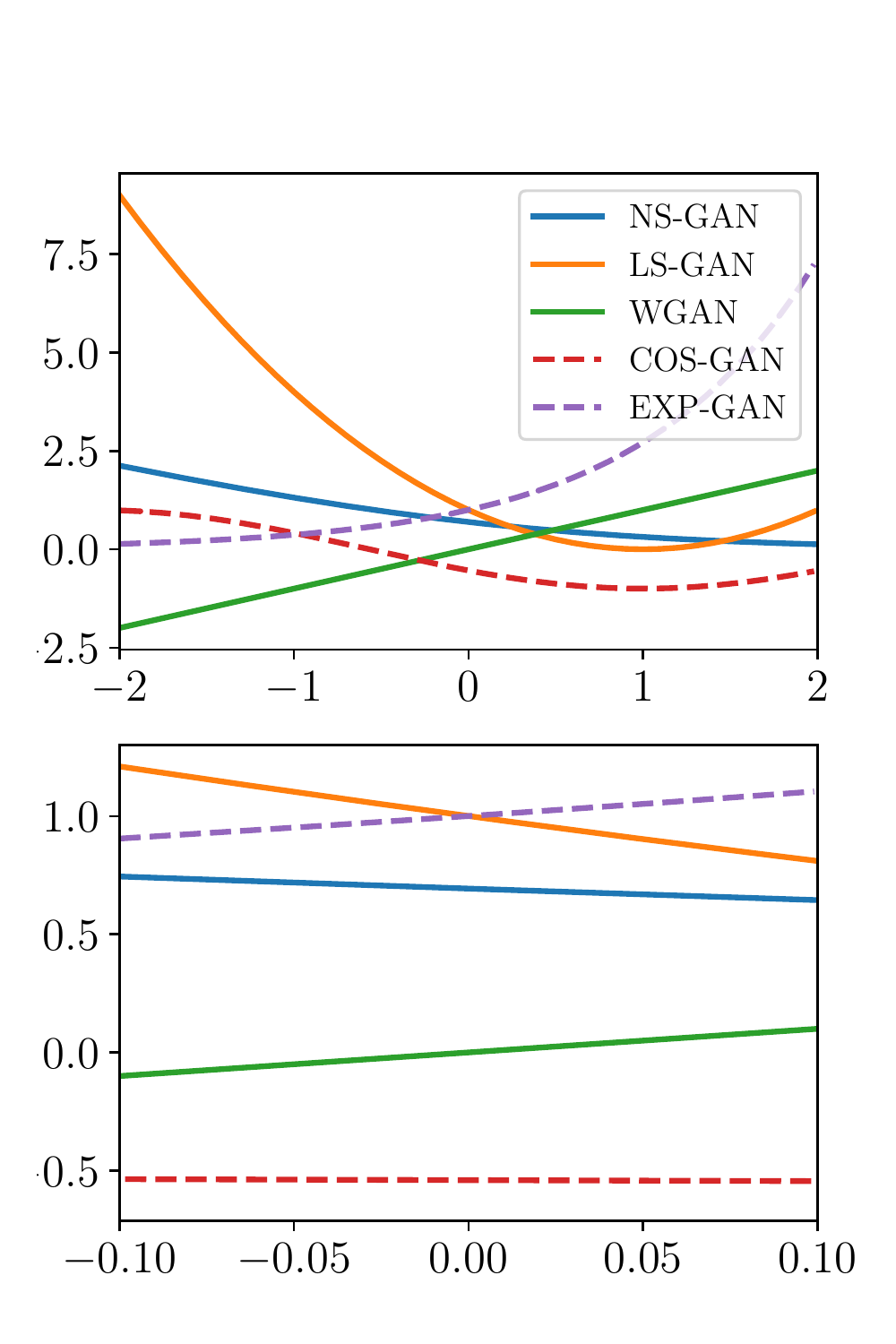}
  \label{fig:degenerated_loss_function}
\end{minipage}
\end{table*}

\subsection{Restricting Loss Functions by Domain Scaling}
\label{sec:decoupling_degenerate_loss_functions}
As discussed above, applying $K$-Lipschitz regularization not only restricts the gradients of the discriminator, but as a side effect also restricts the domain of the loss function to an interval $\Omega$.
However, we would like to investigate these two effects separately.
To this end, we propose to decouple the restriction of $\Omega$ from the Lipschitz regularization by scaling the domain of loss function $L$ by a positive constant $\alpha$ as follows, 
\begin{equation}
    L_{\alpha}(\Omega) = {L(\alpha \cdot \Omega)}/{\alpha}. 
    \label{eq:alpha_scaling}
\end{equation}
Note that the $\alpha$ in the denominator helps to preserve the gradient scale of the loss function.
With this scaling method, we can effectively restrict $L$ to an interval $\alpha \cdot \Omega$ without adjusting $K$.

\noindent {\bf Degenerate Loss Functions.}
To explain why this works, we observe  that any loss function degenerates as its domain $\Omega$ shrinks to a single value.
According to Taylor's expansion, let $\omega,\omega + \Delta \omega \in \Omega$, we have:
\begin{equation}
     L(\omega + \Delta \omega) = L(\omega) + \frac{L'(\omega)}{1!}\Delta \omega + \frac{L''(\omega)}{2!} (\Delta \omega)^2 + \cdots.
    \label{eq:taylor_expansion}
\end{equation}
As $|\max(\Omega) - \min(\Omega)|$ shrinks to zero, we have $L(\omega + \Delta \omega) \approx L(\omega) + L'(\omega)\Delta \omega$ showing that we can approximate any loss function by a linear function with constant gradient as its domain $\Omega$ shrinks to a single value.
Let $\omega \in \Omega$, we implement the degeneration of a loss function by scaling its domain $\Omega$ with an extremely small constant~$\alpha$:
\begin{equation}
    \lim_{\alpha \to 0} \frac{\partial L_{\alpha}(\omega)}{\partial \omega} = \frac{1}{\alpha} \cdot \frac{\partial L(\alpha \cdot \omega)}{\partial \omega} = \frac{\partial L(\alpha \cdot \omega)}{\partial (\alpha \cdot \omega)} = \nabla L(0).
    \label{eq:linear_baseline}
\end{equation}
In our work, we use $\alpha = 1e^{-25}$, smaller values are not used due to numerical errors ($NaN$).

\section{Experiments}

To support our proposition, first we empirically verify that applying $K$-Lipschitz regularization to the discriminator has the side-effect of restricting the domain and interval of attainable gradient values of the loss function.
Second, with the proposed scaling method (Section \ref{sec:decoupling_degenerate_loss_functions}), we investigate how the varying restrictions of loss functions influence the performance of GANs when the discriminator is regularized with a fixed Lipschitz constant.
Third, we show that restricting the domain of any loss function (using decreasing $\alpha$) converges to the same (or very similar) performance as WGAN-SN.

\subsection{Experiment Setup}
\label{sec:experiment_setup}

\noindent {\bf General Setup.}
In the following experiments, we use two variants of the standard CNN architecture \cite{radford2015unsupervised,pmlr-v70-arjovsky17a,miyato2018spectral} for the GANs to learn the distributions of the MNIST, CIFAR10 datasets at $32 \times 32$ resolution and the CelebA dataset \cite{liu2015faceattributes} at $64 \times 64$ resolution. 
Details of the architectures are shown in the supplementary material.
We use a batch size of $64$ to train the GANs.
Similar to \cite{pmlr-v70-arjovsky17a}, we observed that the training could be unstable with a momentum-based optimizer such as Adam, when the discriminator is regularized with a very small Lipschitz constant $K$.
Thus, we choose to use an RMSProp optimizer with learning rate $0.00005$.
To make a fair comparison, we fix the number of discriminator updates in each iteration $n_{dis}=1$ for all the GANs tested (i.e., we do not use multiple discriminator updates like \cite{arjovsky2017towards,pmlr-v70-arjovsky17a}). Unless specified, we stop the training after $10^5$ iterations.

\begin{figure*}[t]
\begin{center}
    \begin{subfigure}{0.32\textwidth}
        \centering
        \includegraphics[width=0.99\linewidth]{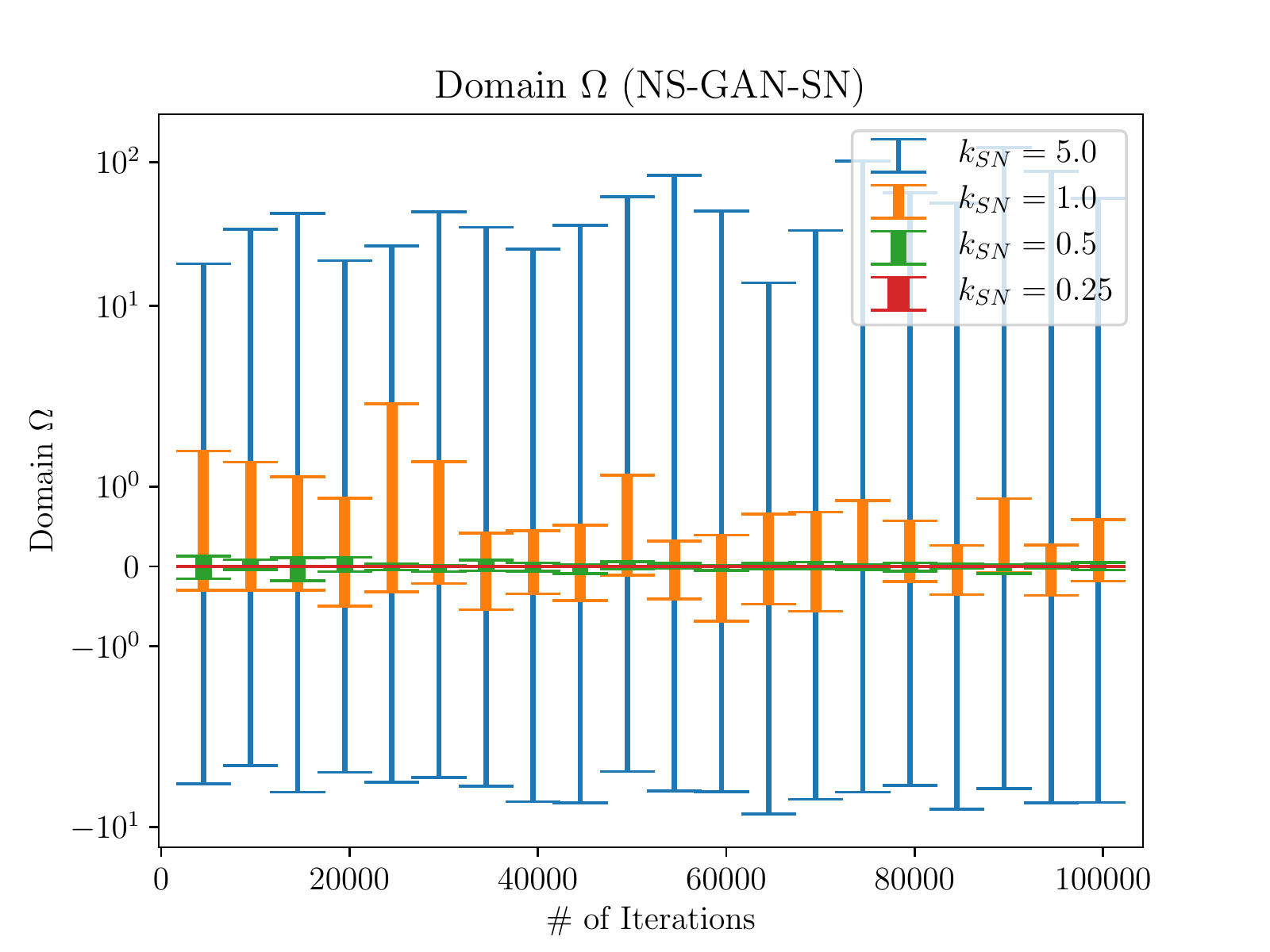}
        \subcaption{}
    \end{subfigure}
    \begin{subfigure}{0.32\textwidth}
        \centering
        \includegraphics[width=0.99\linewidth]{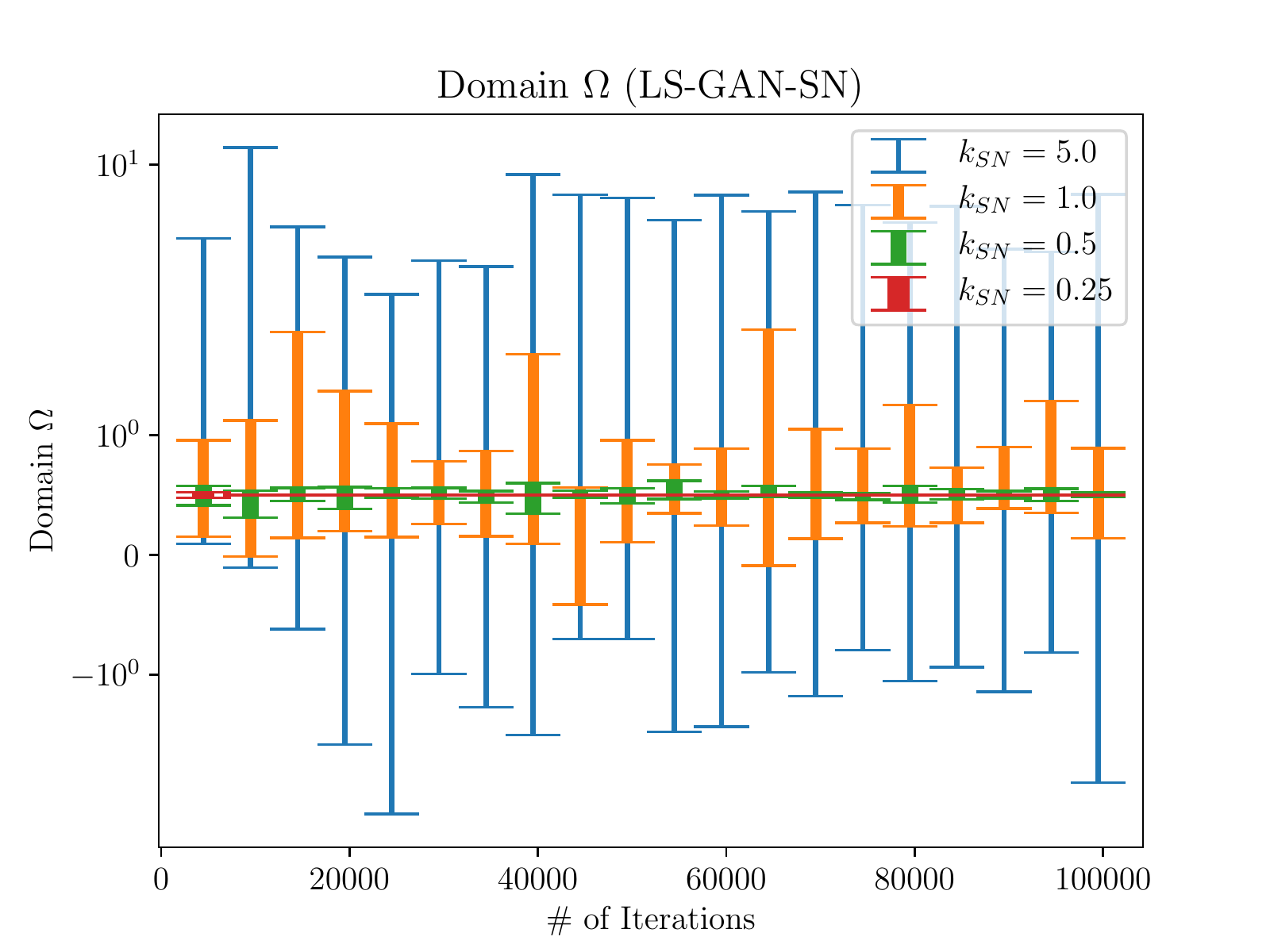}
        \subcaption{}
    \end{subfigure}
    \begin{subfigure}{0.32\textwidth}
        \centering
        \includegraphics[width=0.99\linewidth]{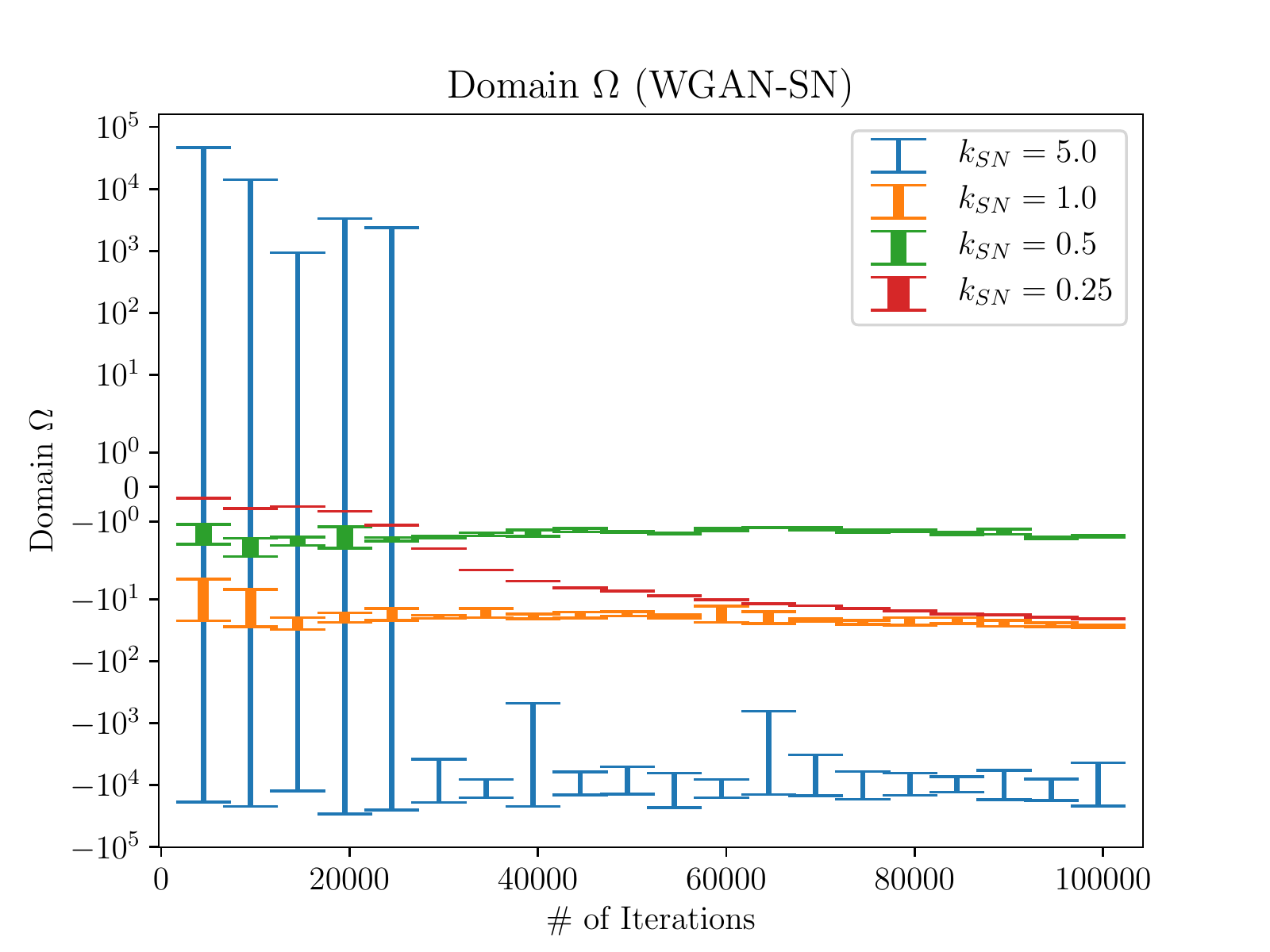}
        \subcaption{}
    \end{subfigure}
    \begin{subfigure}{0.32\textwidth}
        \centering
        \includegraphics[width=0.99\linewidth]{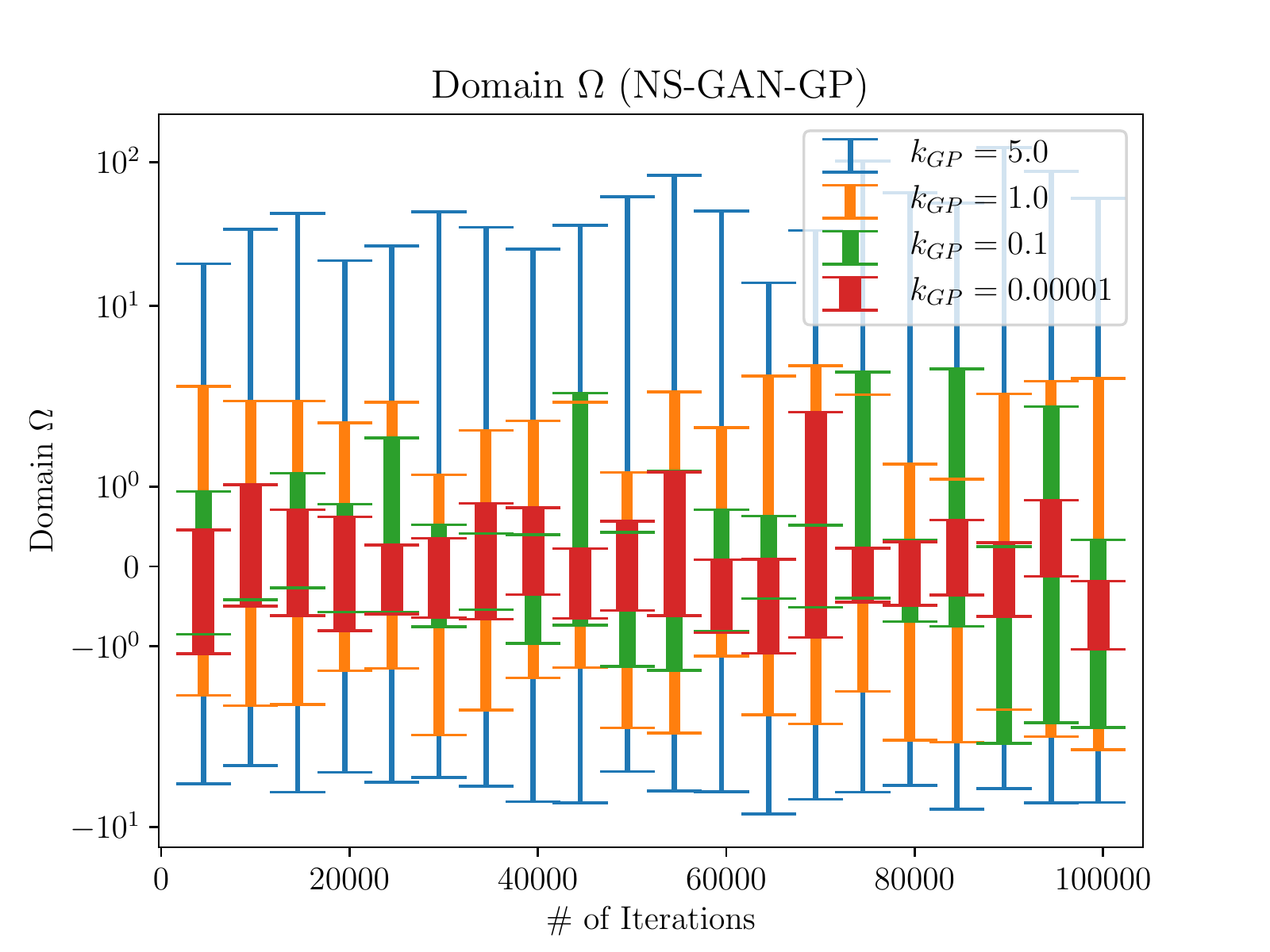}
        \subcaption{}
    \end{subfigure}
    \begin{subfigure}{0.32\textwidth}
        \centering
        \includegraphics[width=0.99\linewidth]{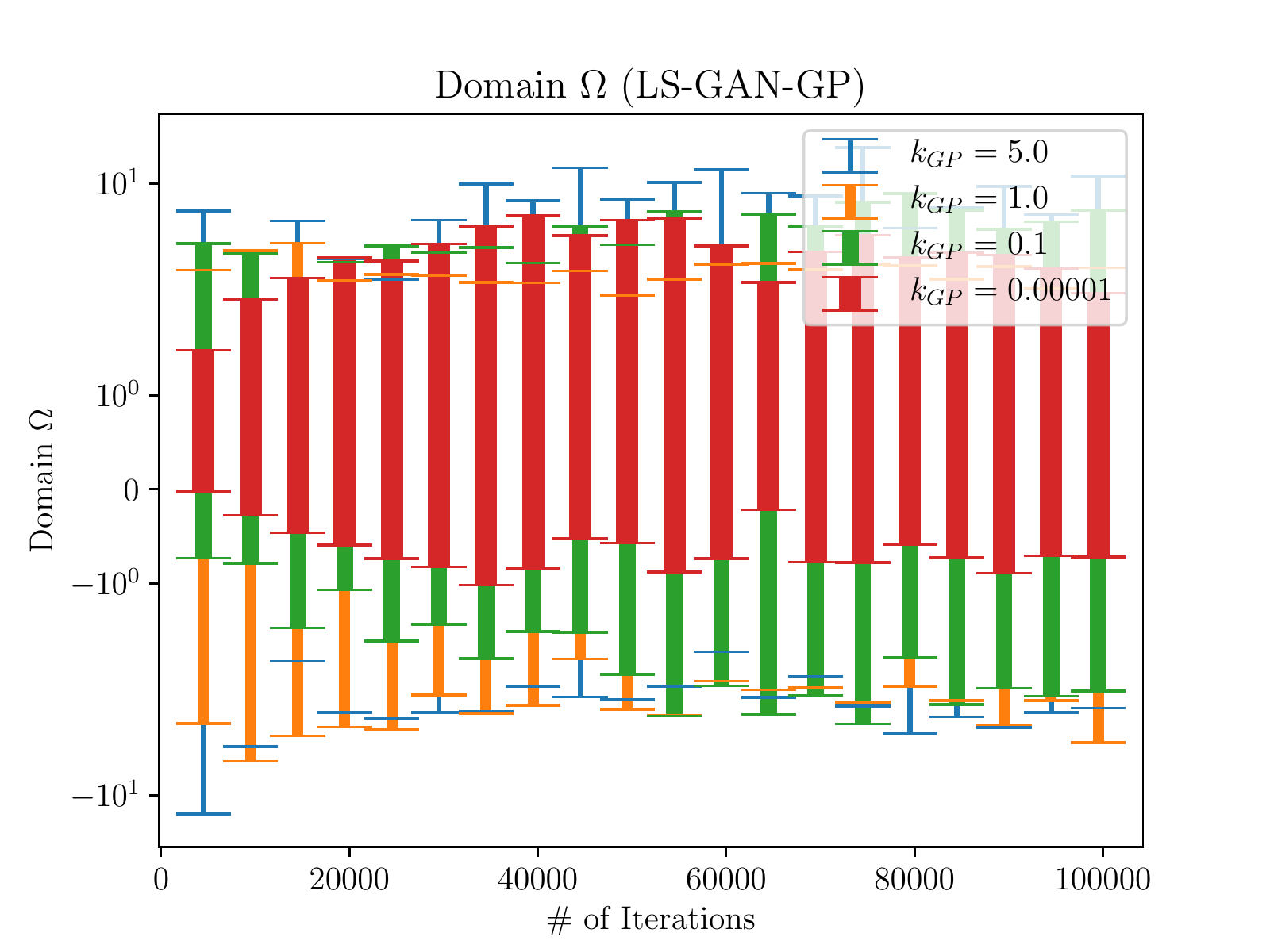}
        \subcaption{}
    \end{subfigure}
    \begin{subfigure}{0.32\textwidth}
        \centering
        \includegraphics[width=0.99\linewidth]{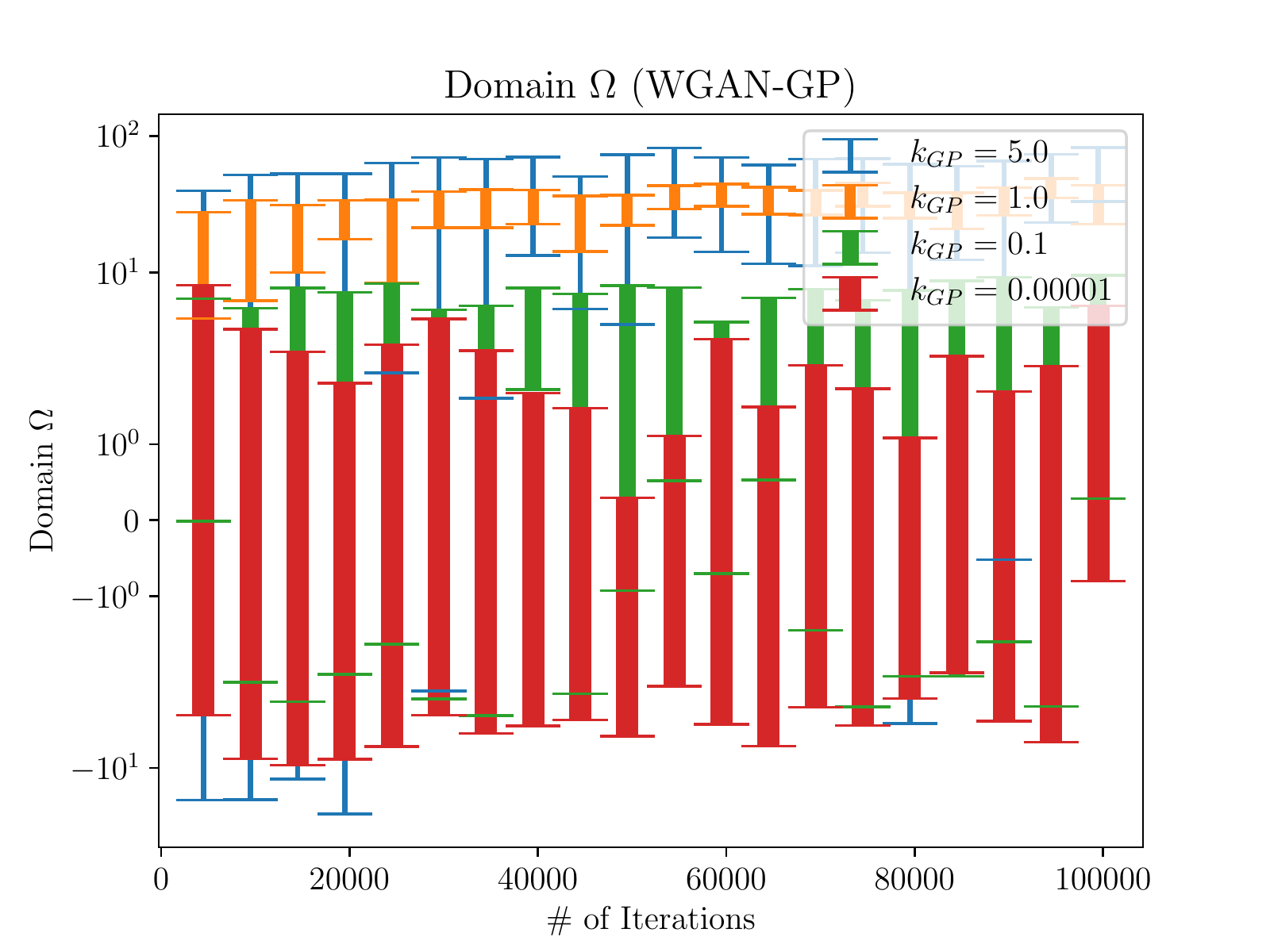}
        \subcaption{}
    \end{subfigure}
\end{center}
   \caption{Relationship between domain $\Omega$ and $k_{GP}$, $k_{SN}$ for different loss functions on CelebA dataset, where $k_{GP}$, $k_{SN}$ are the parameters controlling the strength of the Lipschitz regularizers.
   The domain $\Omega$ shrinks with decreasing $k_{GP}$ or $k_{SN}$.
   Each column shares the same loss function while each row shares the same Lipschitz regularizer. NS-GAN: Non-Saturating GAN \cite{goodfellow2014generative}; LS-GAN: Least-Square GAN \cite{Mao_2017_ICCV}; WGAN: Wasserstein GAN \cite{pmlr-v70-arjovsky17a}; GP: gradient penalty \cite{gulrajani2017improved}; SN: spectral normalization \cite{miyato2018spectral}. Note that the $y$-axis is in $\log$ scale.}
\label{fig:valid_interval_NS_SN}
\end{figure*}

\noindent {\bf Lipschitz Regularizers.}
In general, there are two state-of-the-art Lipschitz regularizers: the gradient penalty (GP) \cite{gulrajani2017improved} and the spectral normalization (SN) \cite{miyato2018spectral}.
In their original settings, both techniques applied only $1$-Lipschitz regularization to the discriminator.
However, our experiments require altering the Lipschitz constant $K$ of the discriminator. 
To this end, we propose to control $K$ for both techniques by adding parameters $k_{GP}$ and $k_{SN}$, respectively. 
\begin{itemize}
    \item For the gradient penalty, we control its strength by adjusting the target gradient norm $k_{GP}$, 
    \begin{equation}
        L = L_{GAN} + \lambda \mathop{\mathbb{E}}_{\hat{x}\in P_{\hat{x}}} [(\| \bigtriangledown_{\hat{x}} D(\hat{x}) \|- k_{GP} )^2], 
    \label{eq:gradient_penalty}
    \end{equation}
    where $L_{GAN}$ is the GAN loss function without gradient penalty, $\lambda$ is the weight of the gradient penalty term, $P_{\hat{x}}$ is the distribution of linearly interpolated samples between the target distribution and the model distribution \cite{gulrajani2017improved}. Similar to \cite{gulrajani2017improved,fedus*2018many}, we  use $\lambda=10$. 
    \item For the spectral normalization, we control its strength by adding a weight parameter $k_{SN}$ to the normalization of each neural network layer, 
    \begin{equation}
        \bar{W}_{SN}(W,k_{SN}) := k_{SN} \cdot W / \sigma(W), 
    \label{eq:spectral_normalization}
    \end{equation}
    where $W$ is the weight matrix of each layer, $\sigma(W)$ is its largest singular value.
\end{itemize}
The relationship between $k_{SN}$ and $K$ can be quantitatively approximated as $K \approx k_{SN}^n$ \cite{miyato2018spectral}, where $n$ is the number of neural network layers in the discriminator.
While for $k_{GP}$, we can only describe its relationship against $K$ qualitatively as: the smaller $k_{GP}$, the smaller $K$.
The challenge on finding a quantitative approximation resides in that the gradient penalty term $\lambda \mathop{\mathbb{E}}_{\hat{x}\in P_{\hat{x}}} [(\| \bigtriangledown_{\hat{x}} D(\hat{x}) \|- k_{GP} )^2]$ has no upper bound during training (Eq.\ref{eq:gradient_penalty}).
We also verified our claims using Stable Rank Normalization (SRN)+SN~\cite{Sanyal2020Stable} as the Lipschitz regularizer, whose results are shown in the supplementary material.

\noindent {\bf Loss Functions.}
In Table \ref{tb:GAN_losses} we compare the three most widely-used GAN loss functions: the Non-Saturating (NS) loss function \cite{goodfellow2014generative}, the Least-Squares (LS) loss function \cite{Mao_2017_ICCV} and the Wasserstein loss function \cite{pmlr-v70-arjovsky17a}.
In addition, we also test the performance of the GANs using some non-standard loss functions, $\cos(\cdot)$ and $\exp(\cdot)$, to support the observation that the restriction of the loss function is the dominating factor of Lipschitz regularization.
Note that both the $\cos(\cdot)$ and $\exp(\cdot)$ loss functions are (locally) convex at convergence, which helps to prevent shifting $\Omega_i$ (Section \ref{sec:why_do_GAN_loss_functions_degenerate}).

\noindent {\bf Quantitative Metrics.}
To quantitatively measure the performance of the GANs, we follow the best practice and employ the Fr\'{e}chet Inception Distance (FID) metric \cite{Heusel2017GANs} in our experiments.
The smaller the FID score, the better the performance of the GAN.
The results on other metrics, \textit{i.e.} Inception scores~\cite{barratt2018note} and Neural Divergence~\cite{gulrajani2018NND}, are shown in the supplementary material.

\subsection{Empirical Analysis of Lipschitz Regularization}
\label{sec:verification_degeneration}

In this section, we empirically analyze how varying the strength of the Lipschitz regularization impacts the domain, interval of attained gradient values, and performance (FID scores) of different loss functions (Section \ref{sec:why_do_GAN_loss_functions_degenerate}).

\noindent {\bf Domain vs. Lipschitz Regularization.}
In this experiment, we show how the Lipschitz regularization influences the domain of the loss function.
As Figure \ref{fig:valid_interval_NS_SN} shows, we plot the domain $\Omega$ as intervals for different iterations under different $k_{GP}$ and $k_{SN}$ for the gradient penalty and the spectral normalization regularizers respectively.
It can be observed that:
(i)~For both regularizers, the interval $\Omega$ shrinks as $k_{GP}$ and $k_{SN}$ decreases. 
However, $k_{SN}$ is much more impactful than $k_{GP}$ in restricting $\Omega$.
Thus, we use spectral normalization to alter the strength of the Lipschitz regularization in the following experiments.
(ii)~For NS-GANs and LS-GANs, the domains $\Omega_i$ are rather fixed during training. 
For WGANs, the domains $\Omega_i$ typically shift at the beginning of the training, but then get relatively fixed in later stages.

\begin{table}[t]
\caption{
Domain $\Omega$ and the interval of attained gradient values $\nabla L(\Omega)$ against $k_{SN}$ on the CelebA dataset.}
\begin{subtable}{0.49\linewidth}
    \begin{center}
    \begin{tabular}{l r r r r}
    \toprule
    $k_{SN}$ & \multicolumn{2}{c}{$\Omega$} & \multicolumn{2}{c}{$\nabla L(\Omega)$}  \\ \hline
    $5.0$  & $[-8.130,$ & $126.501]$ & $[-1.000,$ & $-0.000]$ \\ 
    $1.0$  & $[-0.683,$ & $2.091]$ & $[-0.664,$ & $-0.110]$ \\ 
    $0.5$  & $[-0.178,$ & $0.128]$ & $[-0.545,$ & $-0.468]$ \\ 
    $0.25$ & $[-0.006,$ & $0.006]$ & $[-0.502,$ & $-0.498]$ \\
    \bottomrule
    \end{tabular}
    \end{center}
    \subcaption{NS-GAN-SN, $L(\cdot)=-\log(\textrm{sigmoid}(\cdot))$}
\end{subtable}
\begin{subtable}{0.49\linewidth}
    \begin{center}
    \begin{tabular}{l r r r r}
    \toprule
    $k_{SN}$ & \multicolumn{2}{c}{$\Omega$} & \multicolumn{2}{c}{$\nabla L(\Omega)$}  \\ \hline
    $5.0$  & $[-2.460,$ & $12.020]$ & $[-4.921,$ & $24.041]$ \\ 
    $1.0$  & $[-0.414,$ & $1.881]$ & $[-0.828,$ & $3.762]$ \\ 
    $0.5$  & $[0.312,$ & $0.621]$  & $[0.623,$ & $1.242]$ \\ 
    $0.25$ & $[0.478,$ & $0.522]$  & $[0.956,$ & $1.045]$ \\
    \bottomrule
    \end{tabular}
    \end{center}
    \subcaption{LS-GAN-SN, $L(\cdot)=(\cdot)^2$}
\end{subtable}
\label{tab:gradient_numbers}
\end{table}

\noindent {\bf Interval of Attained Gradient Values vs. Lipschitz Regularization.}
Similar to the domain, the interval of attained gradient values of the loss function also shrinks with the increasing strength of Lipschitz regularization.
Table \ref{tab:gradient_numbers} shows the corresponding interval of attained gradient values of the NS-GAN-SN and LS-GAN-SN experiments in Figure \ref{fig:valid_interval_NS_SN}.
The interval of attained gradient values of WGAN-SN are not included as they are always zero.
It can be observed that the shrinking interval of attained gradient values avoids the saturating and exploding parts of the loss function.
For example when $k_{SN}=5.0$, the gradient of the NS-GAN loss function saturates to a value around $0$ while that of the LS-GAN loss function explodes to $24.041$.
However, such problems do not happen when $k_{SN} \leq 1.0$.
Note that we only compute the interval of attained gradient values on one of the two symmetric loss terms used in the discriminator loss function (Table \ref{tb:GAN_losses}). 
The interval of attained gradient values of the other loss term follows similar patterns.

\begin{table*}[t]
\caption{FID scores vs. $k_{SN}$ (typically fixed as $1$ \cite{miyato2018spectral}) on different datasets. 
When $k_{SN} \leq 1.0$, all GANs have similar performance except the WGANs (slightly worse).
For the line plots, $x$-axis shows $k_{SN}$ (in log scale) and $y$-axis shows the FID scores. From left to right, the seven points on each line have $k_{SN}=0.2$, $0.25$, $0.5$, $1.0$, $5.0$, $10.0$, $50.0$ respectively.
Lower FID scores are better.}

\centering
\begin{center}
    \centering
        \newcolumntype{R}{>{\raggedleft\arraybackslash}X}
        \begin{tabularx}{\linewidth}{l l R  R  R  R  R  R  R}
        \toprule
        \multirow{2}{*}{Dataset} & \multirow{2}{*}{GANs} & \multicolumn{7}{c}{FID Scores} \\
        & & $k_{SN}$~=~0.2 & 0.25 & 0.5 & 1.0 & 5.0 & 10.0 & 50.0 \\
        \hline
        \multirow{5}{*}{MNIST} & NS-GAN-SN & 5.41 & 3.99 & 4.20 & 3.90 & 144.28 & 156.60 & 155.41 \\
        & LS-GAN-SN & 5.14 & \textbf{3.96} & \textbf{3.90} & 4.42 & 36.26  & 59.04 & 309.35 \\
        & WGAN-SN & 6.35 & 6.06 & 4.44 & 4.70 & \textbf{3.58} & \textbf{3.50} & \textbf{3.71} \\
        & COS-GAN-SN & 5.41 & 4.83 & 4.05 & 3.86 & 291.44 & 426.62 & 287.23 \\
        & EXP-GAN-SN & \textbf{4.38} & 4.93 & 4.25 & \textbf{3.69} & 286.96 & 286.96 & 286.96 \\
        \hline
        \multirow{5}{*}{CIFAR10} & NS-GAN-SN & 29.23 & \textbf{24.37} & 23.29 & \textbf{15.81} & 41.04 & 49.67 & 48.03 \\
        & LS-GAN-SN & \textbf{28.04} & 26.85 & 23.14 & 17.30 & 33.53 & 39.90 & 349.35 \\
        & WGAN-SN & 29.20 & 25.07 & 26.61 & 21.75 & \textbf{21.63} & \textbf{21.45} & \textbf{23.36} \\
        & COS-GAN-SN & 29.45 & 25.31 & \textbf{20.73} & 15.88 & 309.96 & 327.20 & 370.13 \\
        & EXP-GAN-SN & 30.89 & 24.74 & 20.90 & 16.66 & 401.24 & 401.24 & 401.24 \\
        \hline
        \multirow{5}{*}{CelebA} & NS-GAN-SN & 18.59 & 12.71 & \textbf{8.04} & 6.11 & 18.95 & 17.04 & 184.06 \\
        & LS-GAN-SN & 20.34 & \textbf{12.14} & 8.85 & 5.69 & 12.40 & 13.14 & 399.39 \\
        & WGAN-SN & 23.26 & 17.93 & 8.48 & 9.41 & \textbf{9.03} & \textbf{7.37} & \textbf{7.82} \\
        & COS-GAN-SN & 20.59 & 13.93 & 8.88 & \textbf{5.20} & 356.70 & 265.53 & 256.44 \\
        & EXP-GAN-SN & \textbf{18.23} & 13.65 & 9.18 & 5.88 & 328.94 & 328.94 & 328.94 \\
        \bottomrule
        \end{tabularx}
\end{center}

\begin{subfigure}{0.32\textwidth}
    \includegraphics[width=.99\linewidth]{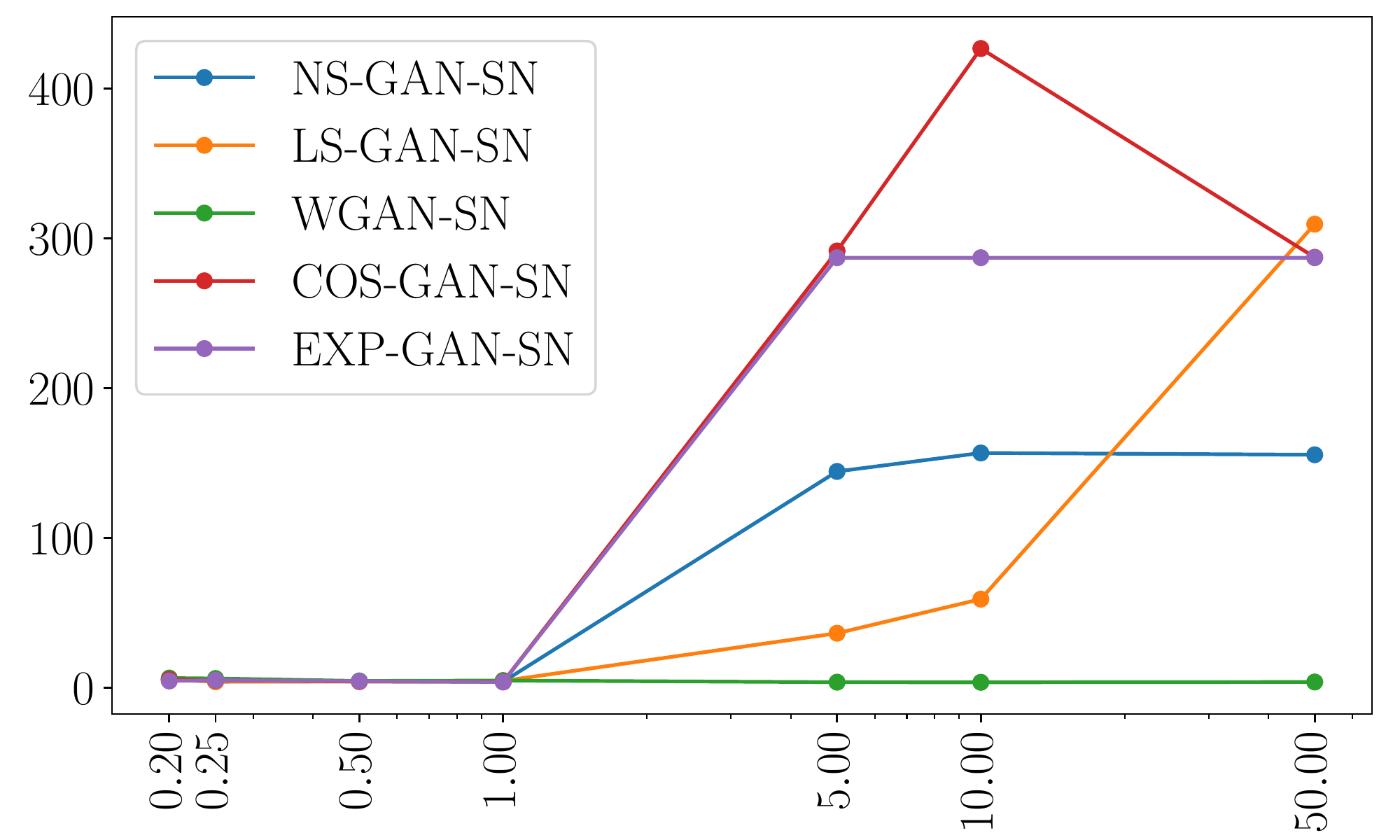}
    \subcaption{MNIST}
\end{subfigure}
\begin{subfigure}{0.32\textwidth}
    \includegraphics[width=.99\linewidth]{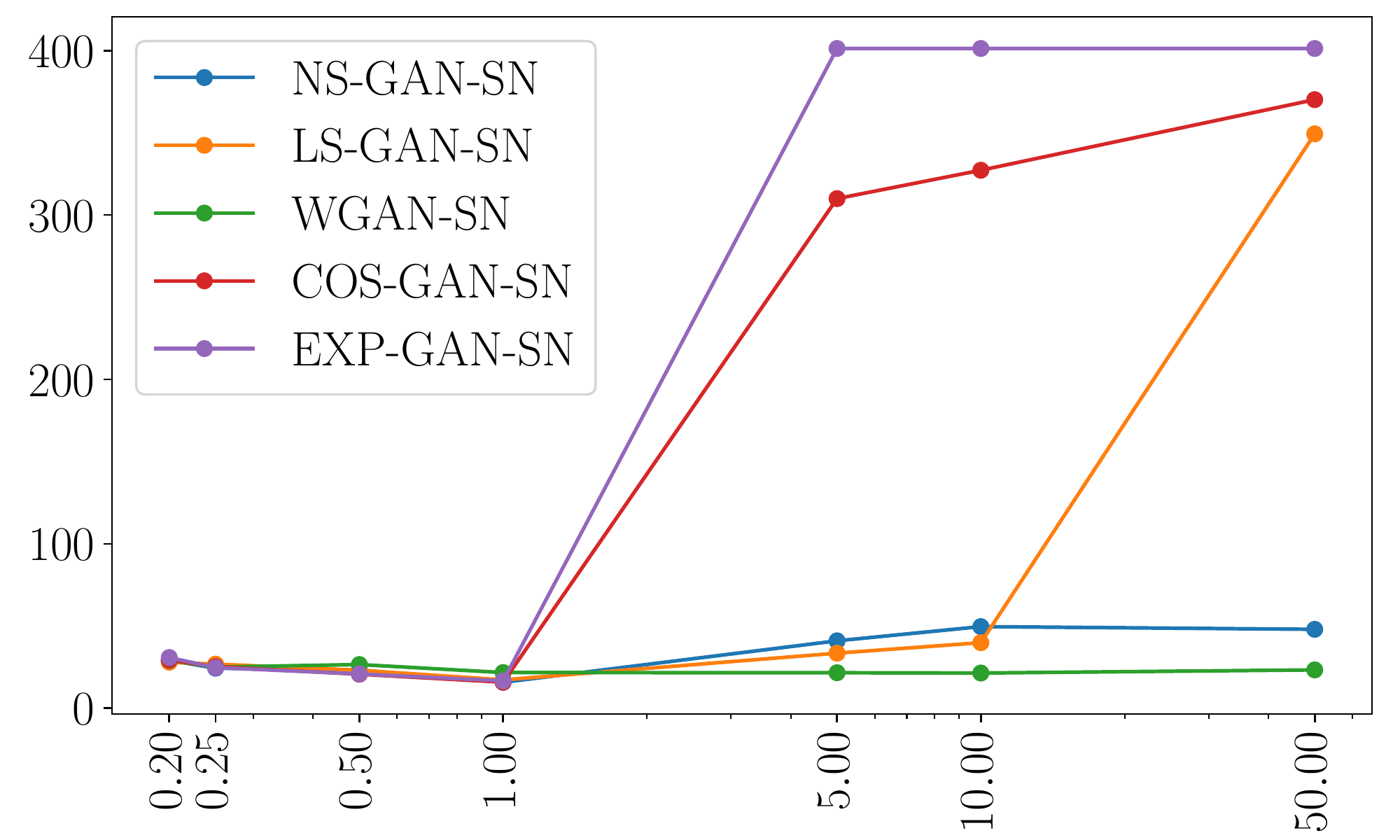}
    \subcaption{CIFAR10}
\end{subfigure}
\begin{subfigure}{0.32\textwidth}
    \includegraphics[width=.99\linewidth]{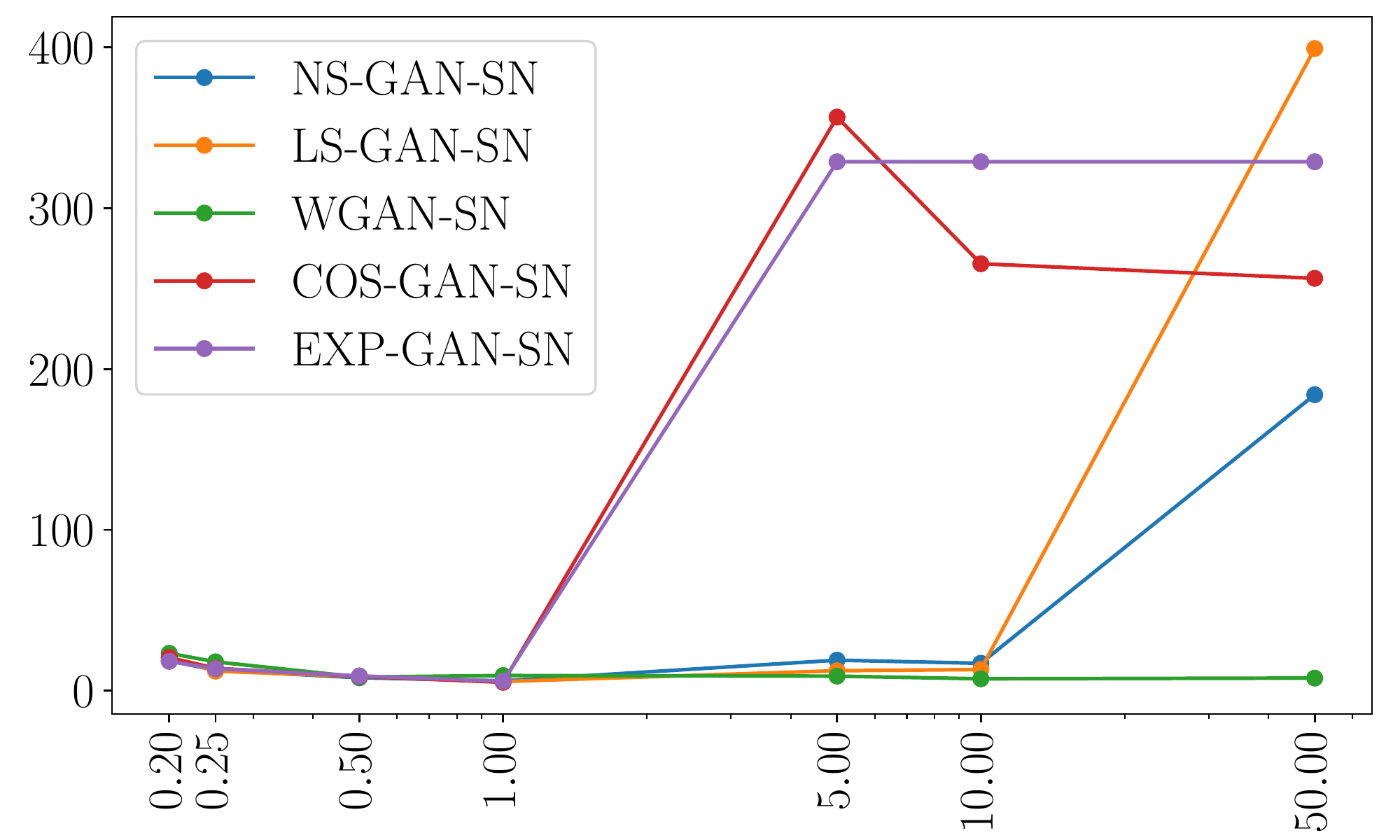}
    \subcaption{CelebA}
\end{subfigure}

\label{tab:FID_vs_ksn}
\end{table*}

\begin{table*}[t]
\caption{FID scores vs. $\alpha$. For the line plots, the $x$-axis shows $\alpha$ (in log scale) and the $y$-axis shows the FID scores. Results on other datasets are shown in the supplementary material. Lower FID scores are better.}
\begin{subtable}{0.99\textwidth}
    \begin{center}
        \centering
            \begin{tabular}{l l r r  r  r  r  r  c}
            \toprule
            \multirow{2}{*}{Dataset} & \multirow{2}{*}{GANs} & \multicolumn{6}{c}{FID Scores} & \multirow{2}{*}{Line Plot}\\
            & & $\alpha$ = $1e^{-11}$ & $1e^{-9}$ & $1e^{-7}$ & $1e^{-5}$ & $1e^{-3}$ & $1e^{-1}$ & \\
            \hline
            \multirow{5}{*}{CelebA} & NS-GAN-SN & 9.08 & 7.05 & 7.84 & 18.51 & 18.41 & 242.64 & \multirow{5}{*}{\raisebox{-.9\height}{\includegraphics[width=.21\linewidth]{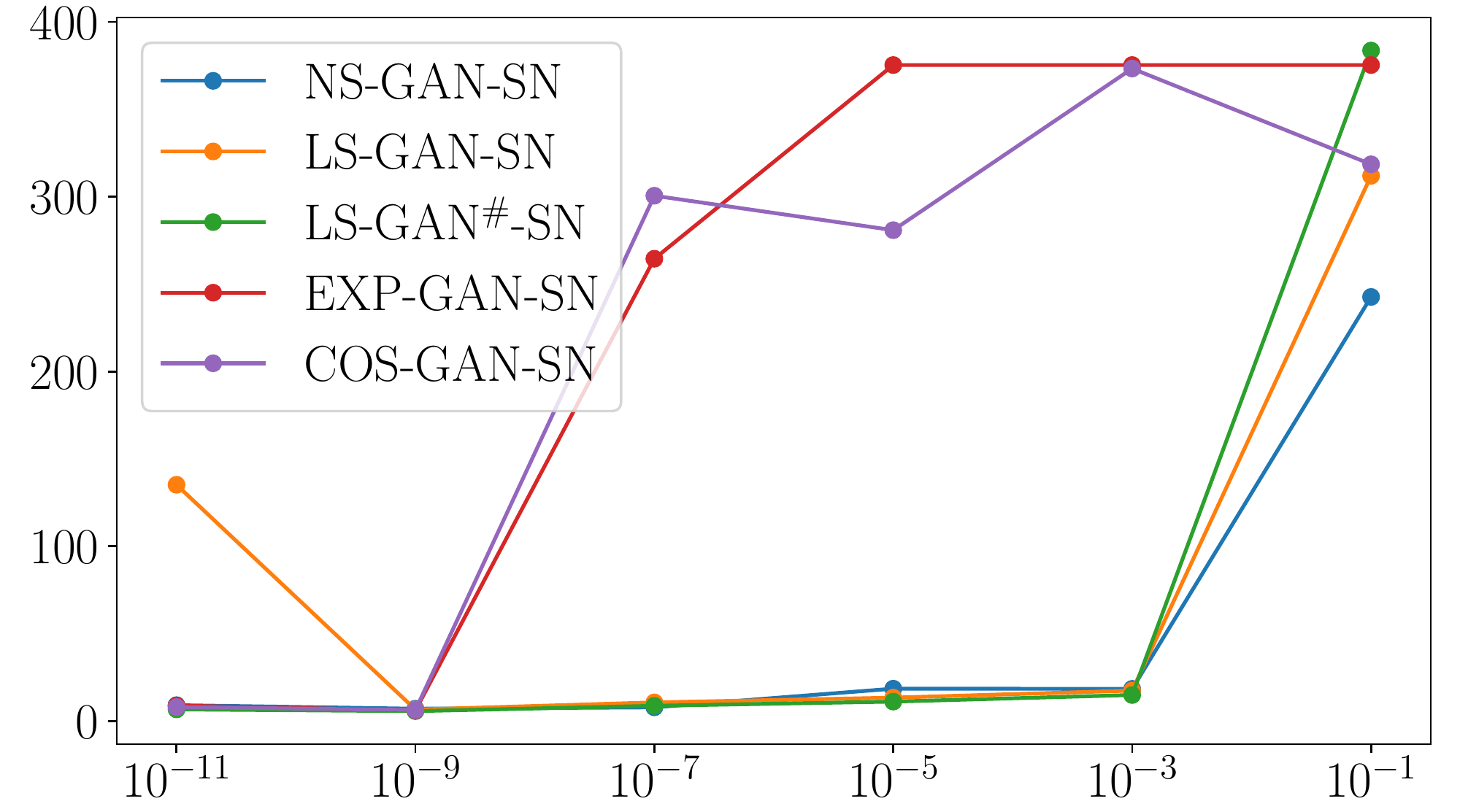}}}\\
            & LS-GAN-SN & 135.17 & 6.57 & 10.67 & 13.39 & 17.42 & 311.93 &\\
            & LS-GAN$^{\#}$-SN & 6.66 & 5.68 & 8.72 & 11.13 & 14.90 & 383.61 &\\
            & COS-GAN-SN & 8.00 & 6.31 & 300.55 & 280.84 & 373.31 & 318.53&\\
            & EXP-GAN-SN & 8.85 & 6.09 & 264.49 & 375.32 & 375.32 & 375.32 &\\
            \bottomrule
            \end{tabular}
    \end{center}
    \subcaption{$k_{SN}=50.0$}
\end{subtable}
\begin{subtable}{0.99\textwidth}
    \begin{center}
        \centering
            \begin{tabular}{l l r r  r  r  r  r  c}
            \toprule
            \multirow{2}{*}{Dataset} & \multirow{2}{*}{GANs} & \multicolumn{6}{c}{FID Scores} & \multirow{2}{*}{Line Plot}\\
            & & $\alpha$ = $1e^{1}$ & $1e^{3}$ & $1e^5$ & $1e^{7}$ & $1e^{9}$ & $1e^{11}$ & \\
            \hline
            \multirow{5}{*}{MNIST} & NS-GAN-SN & 6.55 & 148.97 & 134.44 & 133.82 & 130.21 & 131.87 & \multirow{5}{*}{\raisebox{-.9\height}{\includegraphics[width=.21\linewidth]{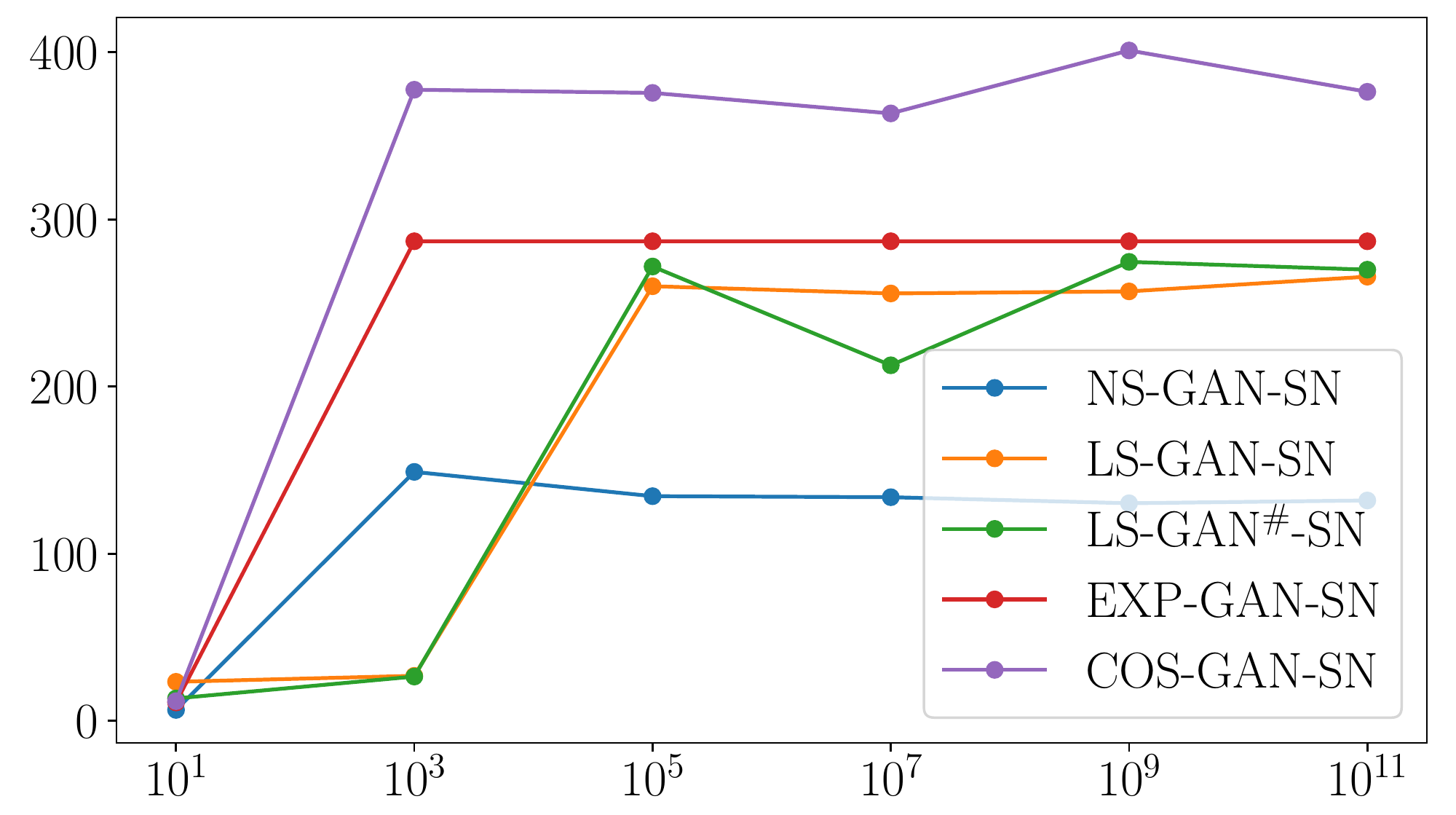}}}\\
            & LS-GAN-SN & 23.37 & 26.96 & 260.05 & 255.73 & 256.96 & 265.76 & \\
            & LS-GAN$^{\#}$-SN & 13.43 & 26.51 & 271.85 & 212.74 & 274.63 & 269.96 &\\
            & COS-GAN-SN & 11.79 & 377.62 & 375.72 & 363.45 & 401.12 & 376.39 &\\
            & EXP-GAN-SN & 11.02 & 286.96 & 286.96 & 286.96 & 286.96 & 286.96 &\\
            \bottomrule
            \end{tabular}
    \end{center}
    \subcaption{$k_{SN}=1.0$}
\end{subtable}
\label{tab:FID_vs_alpha}
\end{table*}

\noindent {\bf FID scores vs. Lipschitz Regularization.}
Table \ref{tab:FID_vs_ksn} shows the FID scores of different GAN loss functions with different $k_{SN}$.
It can be observed that:
\begin{itemize}
\item When $k_{SN} \leq 1.0$, all the loss functions (including the non-standard ones) can be used to train GANs stably.
However, the FID scores of all loss functions slightly worsen as $k_{SN}$ decreases.
We believe that the reason for such performance degradation comes from the trick used by modern optimizers to avoid divisions by zero.
For example in RMSProp~\cite{RMSProp2012}, the moving average of the squared gradients are kept for each weight.
In order to stabilize training, gradients are divided by the square roots of their moving averages in each update of the weights, where a small positive constant $\epsilon$ is included in the denominator to avoid dividing by zero.
When $k_{SN}$ is large, the gradients are also large and the effect of $\epsilon$ is negligible.
While when $k_{SN}$ is very small, the gradients are also small so that $\epsilon$ can significantly slow down the training and worsen the results.
\item When $k_{SN} \leq 1.0$, the performance of WGAN is slightly worse than almost all the other GANs.
Similar to the observation of \cite{ProgressiveGAN2018}, we ascribe this problem to the shifting domain of WGANs (Figure \ref{fig:valid_interval_NS_SN} (c)(f)). The reason for the domain shift is that the Wasserstein distance only depends on the difference between $\mathbb{E}[f(x)]$ and $\mathbb{E}[f(g(z)]$ (Table \ref{tb:GAN_losses}). For example, the Wasserstein distances $\mathbb{E}[f(x)] - \mathbb{E}[f(g(z)]$ are the same for i) $\mathbb{E}[f(x)]=0.5, \mathbb{E}[f(g(z))=-0.5$ and ii) $\mathbb{E}[f(x)] = 100.5, \mathbb{E}[f(g(z))=99.5$.
\item When $k_{SN} \geq 5.0$, the WGAN works normally while the performance of all the other GANs worsen and even break (very high FID scores, e.g. $\geq 100$).
The reasons for the stable performance of WGAN are two-fold: i) due to the KR duality, the Wasserstein distance is insensitive to the Lipschitz constant $K$. Let $W(\mathbb{P}_r, \mathbb{P}_g)$ be the Wasserstein distance between the data distribution $\mathbb{P}_r$ and the generator distribution $\mathbb{P}_g$. As discussed in~\cite{pmlr-v70-arjovsky17a}, applying $K$-Lipschitz regularization to WGAN is equivalent to estimating $K \cdot W(\mathbb{P}_r, \mathbb{P}_g)$, which shares the same solution as $1$-Lipschitz regularized WGAN.
ii) To fight the exploding and vanishing gradient problems, modern neural networks are intentionally designed to be scale-insensitive to the backpropagated gradients (e.g. ReLU~\cite{relu2011}, RMSProp~\cite{RMSProp2012}). This largely eliminates the scaling effect caused by $k_{SN}$.
This observation also supports our claim that the restriction of the loss function is the dominating factor of the Lipschitz regularization.
\item The best performance is obtained by GANs with strictly convex (e.g. NS-GAN) and properly restricted (e.g. $k_{SN}=1$) loss functions that address the shifting domain and exploding/vanishing gradient problems at the same time. However, there is no clear preference and even the non-standard ones (e.g.,  COS-GAN) can be the best. We believe that this is due to the subtle differences of the convexity among loss functions and propose to leave it to the fine-tuning of loss functions using the proposed domain scaling.
\end{itemize}
Qualitative results are in the supplementary material.

\subsection{Empirical Results on Domain Scaling}
\label{sec:FID_vs_scaled_loss_functions}

In this section, we empirically verify our claim that the restriction of the loss function is the dominating factor of the Lipschitz regularization.
To illustrate it, we decouple the restriction of the domain of the loss function from the Lipschitz regularization by the proposed domain scaling method (Section \ref{sec:decoupling_degenerate_loss_functions}).

\setlength{\intextsep}{5pt}
\setlength{\columnsep}{5pt}
\begin{wraptable}{r}{0.5\textwidth}
    \caption{FID scores of WGAN-SN and some extremely degenerate loss functions ($\alpha = 1e^{-25}$) on different datasets. We use $k_{SN} = 50$ for all our experiments.}
    \begin{center}
        \centering
        \resizebox{0.5\textwidth}{!}{
            \begin{tabular}{l l r  r  r  r}
            \toprule
            \multirow{2}{*}{GANs} & \multicolumn{3}{c}{FID Scores} \\
            & MNIST & CIFAR10 & CELEBA  \\
            \hline
            WGAN-SN & 3.71 & 23.36 & 7.82\\
            \hline
            NS-GAN-SN & 3.74 & 21.92 & 8.10\\
            LS-GAN$^{\#}$-SN & 3.81 & 21.47 & 8.51\\
            COS-GAN-SN & 3.96 & 23.65 & 8.30\\
            EXP-GAN-SN & 3.86 & 21.91 & 8.22\\
            \bottomrule
            \end{tabular}
            }
    \end{center}
    \label{tb:linear_baseline_experiment}
\end{wraptable}
Table \ref{tab:FID_vs_alpha} (a) shows that 
(i) the FID scores of different loss functions generally improve with decreasing $\alpha$. 
When $\alpha \leq 10^{-9}$, we can successfully train GANs with extremely large Lipschtiz constant ($K \approx k_{SN}^n = 50^4 = 6.25 \times 10^6$), whose FID scores are comparable to the best ones in Table \ref{tab:FID_vs_ksn}.
(ii) The FID scores when $\alpha \leq 10^{-11}$ are slightly worse than those when $\alpha \leq 10^{-9}$. 
The reason for this phenomenon is that restricting the domain of the loss function converges towards the performance of WGAN, which is slightly worse than the others due to its shifting domain.
To further illustrate this point, we scale the domain by $\alpha=1e^{-25}$ and show the FID scores of WGAN-SN and those of different loss functions in Table \ref{tb:linear_baseline_experiment}.
It can be observed that all loss functions have similar performance.
Since domain scaling does not restrict the neural network gradients, it does not suffer from the above-mentioned numerical problem of division by zero ($k_{SN} \leq 1.0$, Table \ref{tab:FID_vs_ksn}).
Thus, it is a better alternative to tuning $k_{SN}$.

Table \ref{tab:FID_vs_alpha} (b) shows that the FID scores of different loss functions generally worsen with less restricted domains. 
Note that when $\alpha \geq 10^{5}$, the common practice of $1$-Lipschitz regularization fails to stabilize the GAN training.
Note that the LS-GAN-SN has some abnormal behavior (e.g. $\alpha = 1e^{-11}$ in Table \ref{tab:FID_vs_alpha} (a) and $\alpha = 1e^{1}$ in Table \ref{tab:FID_vs_alpha} (b)) due to the conflict between its $0.5$-centered domain and our zero-centered domain scaling method (Eq.\ref{eq:alpha_scaling}).
This can be easily fixed by using the zero-centered LS-GAN$^\#$-SN (see Table \ref{tb:GAN_losses}).

\begin{wrapfigure}{r}{0.6\textwidth}
\begin{center}
    \begin{subfigure}{0.49\linewidth}
        \centering
        \includegraphics[width=0.8\linewidth]{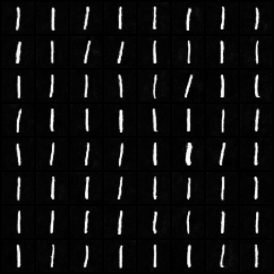}
        \subcaption{NS-GAN-SN}
    \end{subfigure}
    \begin{subfigure}{0.49\linewidth}
        \centering
        \includegraphics[width=0.8\linewidth]{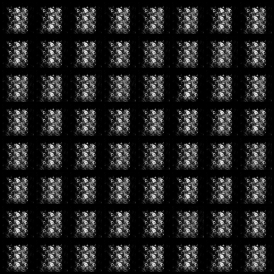}
        \subcaption{LS-GAN-SN}
    \end{subfigure}
\end{center}
  \caption{(a) Mode collapse and (b) crashed training on MNIST, $k_{SN}=50.0$, $\alpha=1e^{-1}$.}
\label{fig:mode_collapse}
\end{wrapfigure}
\noindent {\bf Bias over Input Samples.}
When weak Lipschitz regularization (large Lipschitz constant $K$) is applied, we observed mode collapse for NS-GAN and crashed training for LS-GAN, EXP-GAN and COS-GAN (Figure \ref{fig:mode_collapse}, more results in supplementary material).
We conjecture that this phenomenon is rooted in the inherent bias of neural networks over input samples: neural networks may ``prefer'' some input (class of) samples over the others by outputting higher/lower values, even though all of them are real samples from the training dataset.
When the above-mentioned loss functions are used, such different outputs result in different backpropagated gradients $\nabla L = \partial D(f(x)) / \partial f(x)$.
The use of weak Lipschitz regularization further enhances the degree of unbalance among backpropagated gradients and causes mode collapse or crashed training.
Note that mode collapse happens when $\nabla L$ is bounded (e.g. NS-GAN) and crashed training happens when $\nabla L$ is unbounded (e.g. LS-GAN, EXP-GAN) or ``random'' (e.g. COS-GAN).
However, when strong Lipschitz regularization is applied, all loss functions degenerate to almost linear ones and balance the backpropagated gradients, thereby improve the training.

\section{Conclusion}

In this paper, we studied the \textit{coupling} of Lipschitz regularization and the loss function.
Our key insight is that instead of keeping the neural network gradients small, the dominating factor of Lipschitz regularization is its restriction on the domain and interval of attainable gradient values of the loss function.
Such restrictions stabilize GAN training by avoiding the bias of the loss function over input samples, which is a new step in understanding the exact reason for Lipschitz regularization's effectiveness.
Furthermore, our finding suggests that although different loss functions can be used to train GANs successfully, they actually work in the same way because all of them degenerate to near-linear ones within the chosen small domains.

\noindent \textbf{Acknowledgement} This work was supported in part by the KAUST Office of Sponsored Research (OSR) under Award No. OSR-CRG2018-3730.

\clearpage
%
%
\bibliographystyle{splncs04}
\bibliography{mainbib}


\title{How does Lipschitz Regularization Influence GAN Training? \\- Supplementary Material -} 

\titlerunning{Supplementary Material}
%
\author{Yipeng Qin\inst{1,3} \and
Niloy Mitra\inst{2} \and
Peter Wonka\inst{3}}
\authorrunning{Qin, Mitra, Wonka}
%
\institute{Cardiff University, \email{qiny16@cardiff.ac.uk}
\and
UCL/Adobe Research, \email{n.mitra@cs.ucl.ac.uk}
\and
KAUST,
\email{pwonka@gmail.com}}
\maketitle

\begin{abstract}
   As supplementary material to the main paper we provide the detailed network architectures (Sec. \ref{sec:detailed_network_architectures}),   additional quantitative results (Sec. \ref{sec:additional_quantitative_results}) and   additional qualitative results (Sec. \ref{sec:additional_qualitative_results}).
\end{abstract}

\section{Detailed Network Architectures}
\label{sec:detailed_network_architectures}

The detailed network architectures of the generators and the discriminators are shown in Table \ref{tb:CNN_architecture}.

\section{Additional Quantitative Results}
\label{sec:additional_quantitative_results}
Table \ref{tab:FID_vs_alpha_supplement} shows the FID scores against $\alpha$ on all the three datasets (MNIST, CIFAR10 and CelebA), which is complementary to Table 4 in the main paper.
The discussion in the main paper also applies to the additional results in Table~\ref{tab:FID_vs_alpha_supplement}.

We further verified the influence of domain scaling on the CIFAR10 dataset with different Lipschitz regularizers, \textit{i.e.} Spectral Normalization (SN)~\cite{miyato2018spectral} and Stable Rank Normalization (SRN)+SN~\cite{Sanyal2020Stable}, using additional metrics, \textit{i.e.} Inception Scores (IS)~\cite{barratt2018note} and Neural Divergence (ND)~\cite{gulrajani2018NND}.
Specifically, we set $k_{SN}=1$ for SN, $c=0.7$ for SRN and reported both IS and ND scores with different $\alpha$. As Table~\ref{tab:IS_vs_alpha} and Table~\ref{tab:ND_vs_alpha} show,
\begin{itemize}
    \item When $\alpha$ is small, all loss functions degenerate to linear ones like WGAN and have similar IS and ND scores.
    \item When $\alpha$ is large, different loss functions behave differently and have different IS, ND scores (mostly worse).
\end{itemize}
Thus, our conclusions still hold in these settings.

\section{Additional Qualitative Results}
\label{sec:additional_qualitative_results}

In this section, we show the qualitative results (sample images) corresponding to the quantitative experiments in the main paper. 
The FID scores and line plots are shown together with the samples.

\subsection{Samples of FID scores v.s. $k_{SN}$ Experiment}
This subsection corresponds to the FID scores v.s. $k_{SN}$ experiment (Table 3 in the main paper).

With varying $k_{SN}$ on the MNIST dataset,
\begin{itemize}
    \item Figure \ref{fig:NSGANSN_MNIST} shows sample images of the \textcolor[RGB]{51,120,181}{\textbf{NS-GAN-SN}}; 
    \item Figure \ref{fig:LSGANSN_MNIST} shows sample images of the \textcolor[RGB]{246,126,0}{\textbf{LS-GAN-SN}};
    \item Figure \ref{fig:WGANSN_MNIST} shows sample images of the \textcolor[RGB]{64,160,32}{\textbf{WGAN-SN}};
    \item Figure \ref{fig:COSGANSN_MNIST} shows sample images of the \textcolor[RGB]{205,35,33}{\textbf{COS-GAN-SN}};
    \item Figure \ref{fig:EXPGANSN_MNIST} shows sample images of the \textcolor[RGB]{145,102,191}{\textbf{EXP-GAN-SN}}.
\end{itemize}

With varying $k_{SN}$ on the CIFAR10 dataset,
\begin{itemize}
    \item Figure \ref{fig:NSGANSN_CIFAR10} shows sample images of the \textcolor[RGB]{51,120,181}{\textbf{NS-GAN-SN}}; 
    \item Figure \ref{fig:LSGANSN_CIFAR10} shows sample images of the \textcolor[RGB]{246,126,0}{\textbf{LS-GAN-SN}};
    \item Figure \ref{fig:WGANSN_CIFAR10} shows sample images of the \textcolor[RGB]{64,160,32}{\textbf{WGAN-SN}};
    \item Figure \ref{fig:COSGANSN_CIFAR10} shows sample images of the \textcolor[RGB]{205,35,33}{\textbf{COS-GAN-SN}};
    \item Figure \ref{fig:EXPGANSN_CIFAR10} shows sample images of the \textcolor[RGB]{145,102,191}{\textbf{EXP-GAN-SN}}.
\end{itemize}

With varying $k_{SN}$ on the CelebA dataset,
\begin{itemize}
    \item Figure \ref{fig:NSGANSN_CelebA} shows sample images of the \textcolor[RGB]{51,120,181}{\textbf{NS-GAN-SN}}; 
    \item Figure \ref{fig:LSGANSN_CelebA} shows sample images of the \textcolor[RGB]{246,126,0}{\textbf{LS-GAN-SN}};
    \item Figure \ref{fig:WGANSN_CelebA} shows sample images of the \textcolor[RGB]{64,160,32}{\textbf{WGAN-SN}};
    \item Figure \ref{fig:COSGANSN_CelebA} shows sample images of the \textcolor[RGB]{205,35,33}{\textbf{COS-GAN-SN}};
    \item Figure \ref{fig:EXPGANSN_CelebA} shows sample images of the \textcolor[RGB]{145,102,191}{\textbf{EXP-GAN-SN}}.
\end{itemize}

\subsection{Samples of the FID scores v.s. $\alpha$ Experiment}
This subsection corresponds to the FID scores v.s. $\alpha$ experiment (Table \ref{tab:FID_vs_alpha_supplement}).
The results show that instead of the restriction of the neural network gradients, the restriction of the loss function is the dominating factor of Lipschitz regularization.

\paragraph{Results for Table \ref{tab:FID_vs_alpha_supplement} (a)}
With varying $\alpha$, $k_{SN}=50.0$ and on MNIST dataset,
\begin{itemize}
    \item Figure \ref{fig:Scale_NSGANSN_MNIST_k_50} shows sample images of the \textcolor[RGB]{51,120,181}{\textbf{NS-GAN-SN}}; 
    \item Figure \ref{fig:Scale_LSGANSN_MNIST_k_50} shows sample images of the \textcolor[RGB]{246,126,0}{\textbf{LS-GAN-SN}};
    \item Figure \ref{fig:Scale_LSGANSN_zero_centered_MNIST_k_50} shows sample images of the \textcolor[RGB]{64,160,32}{\textbf{LS-GAN$^\#$-SN}};
    \item Figure \ref{fig:Scale_EXPGANSN_MNIST_k_50} shows sample images of the \textcolor[RGB]{205,35,33}{\textbf{EXP-GAN-SN}};
    \item Figure \ref{fig:Scale_COSGANSN_MNIST_k_50} shows sample images of the \textcolor[RGB]{145,102,191}{\textbf{COS-GAN-SN}}.
\end{itemize}

With varying $\alpha$, $k_{SN}=50.0$ and on CIFAR10 dataset,
\begin{itemize}
    \item Figure \ref{fig:Scale_NSGANSN_CIFAR10_k_50} shows sample images of the \textcolor[RGB]{51,120,181}{\textbf{NS-GAN-SN}}; 
    \item Figure \ref{fig:Scale_LSGANSN_CIFAR10_k_50} shows sample images of the \textcolor[RGB]{246,126,0}{\textbf{LS-GAN-SN}};
    \item Figure \ref{fig:Scale_LSGANSN_zero_centered_CIFAR10_k_50} shows sample images of the \textcolor[RGB]{64,160,32}{\textbf{LS-GAN$^\#$-SN}};
    \item Figure \ref{fig:Scale_EXPGANSN_CIFAR10_k_50} shows sample images of the \textcolor[RGB]{205,35,33}{\textbf{EXP-GAN-SN}};
    \item Figure \ref{fig:Scale_COSGANSN_CIFAR10_k_50} shows sample images of the \textcolor[RGB]{145,102,191}{\textbf{COS-GAN-SN}}.
\end{itemize}

With varying $\alpha$, $k_{SN}=50.0$ and on CelebA dataset,
\begin{itemize}
    \item Figure \ref{fig:Scale_NSGANSN_CelebA_k_50} shows sample images of the \textcolor[RGB]{51,120,181}{\textbf{NS-GAN-SN}}; 
    \item Figure \ref{fig:Scale_LSGANSN_CelebA_k_50} shows sample images of the \textcolor[RGB]{246,126,0}{\textbf{LS-GAN-SN}};
    \item Figure \ref{fig:Scale_LSGANSN_zero_centered_CelebA_k_50} shows sample images of the \textcolor[RGB]{64,160,32}{\textbf{LS-GAN$^\#$-SN}};
    \item Figure \ref{fig:Scale_EXPGANSN_CelebA_k_50} shows sample images of the \textcolor[RGB]{205,35,33}{\textbf{EXP-GAN-SN}};
    \item Figure \ref{fig:Scale_COSGANSN_CelebA_k_50} shows sample images of the \textcolor[RGB]{145,102,191}{\textbf{COS-GAN-SN}}.
\end{itemize}

\paragraph{Results for Table \ref{tab:FID_vs_alpha_supplement} (b)}
With varying $\alpha$, $k_{SN}=1.0$ and on MNIST dataset,
\begin{itemize}
    \item Figure \ref{fig:Scale_NSGANSN_MNIST_k_1} shows sample images of the \textcolor[RGB]{51,120,181}{\textbf{NS-GAN-SN}}; 
    \item Figure \ref{fig:Scale_LSGANSN_MNIST_k_1} shows sample images of the \textcolor[RGB]{246,126,0}{\textbf{LS-GAN-SN}};
    \item Figure \ref{fig:Scale_LSGANSN_zero_centered_MNIST_k_1} shows sample images of the \textcolor[RGB]{64,160,32}{\textbf{LS-GAN$^\#$-SN}};
    \item Figure \ref{fig:Scale_EXPGANSN_MNIST_k_1} shows sample images of the \textcolor[RGB]{205,35,33}{\textbf{EXP-GAN-SN}};
    \item Figure \ref{fig:Scale_COSGANSN_MNIST_k_1} shows sample images of the \textcolor[RGB]{145,102,191}{\textbf{COS-GAN-SN}}.
\end{itemize}

With varying $\alpha$, $k_{SN}=1.0$ and on CIFAR10 dataset,
\begin{itemize}
    \item Figure \ref{fig:Scale_NSGANSN_CIFAR10_k_1} shows sample images of the \textcolor[RGB]{51,120,181}{\textbf{NS-GAN-SN}}; 
    \item Figure \ref{fig:Scale_LSGANSN_CIFAR10_k_1} shows sample images of the \textcolor[RGB]{246,126,0}{\textbf{LS-GAN-SN}};
    \item Figure \ref{fig:Scale_LSGANSN_zero_centered_CIFAR10_k_1} shows sample images of the \textcolor[RGB]{64,160,32}{\textbf{LS-GAN$^\#$-SN}};
    \item Figure \ref{fig:Scale_EXPGANSN_CIFAR10_k_1} shows sample images of the \textcolor[RGB]{205,35,33}{\textbf{EXP-GAN-SN}};
    \item Figure \ref{fig:Scale_COSGANSN_CIFAR10_k_1} shows sample images of the \textcolor[RGB]{145,102,191}{\textbf{COS-GAN-SN}}.
\end{itemize}

With varying $\alpha$, $k_{SN}=1.0$ and on CelebA dataset,
\begin{itemize}
    \item Figure \ref{fig:Scale_NSGANSN_CelebA_k_1} shows sample images of the \textcolor[RGB]{51,120,181}{\textbf{NS-GAN-SN}}; 
    \item Figure \ref{fig:Scale_LSGANSN_CelebA_k_1} shows sample images of the \textcolor[RGB]{246,126,0}{\textbf{LS-GAN-SN}};
    \item Figure \ref{fig:Scale_LSGANSN_zero_centered_CelebA_k_1} shows sample images of the \textcolor[RGB]{64,160,32}{\textbf{LS-GAN$^\#$-SN}};
    \item Figure \ref{fig:Scale_EXPGANSN_CelebA_k_1} shows sample images of the \textcolor[RGB]{205,35,33}{\textbf{EXP-GAN-SN}};
    \item Figure \ref{fig:Scale_COSGANSN_CelebA_k_1} shows sample images of the \textcolor[RGB]{145,102,191}{\textbf{COS-GAN-SN}}.
\end{itemize}

\subsection{Samples of the Degenerate Loss Function}
Figure \ref{fig:degenerate_loss_functions} shows the sample images of WGAN-SN and some extremely degenerate loss functions, which corresponds to Table 5 in the main paper.
It can be observed that all loss functions have similar performance.


\begin{table*}[t!]
    \caption{Network architectures. BN: batch normalization; SN: spectral normalization; GP: gradient penalty.}
    \vspace{-4mm}
    \small
    \begin{center}
        \begin{subtable}{0.49\columnwidth}
        \centering
            \begin{tabular}{c}
            \toprule
            $z \in \mathbb{R}^{100} \sim \mathcal{N}(0, 1) $ \\
            \hline
            Reshape $\to 1 \times 1 \times 100$\\
            \hline
            $4 \times 4$, stride=1, deconv. BN 512 ReLU\\
            \hline
            $4 \times 4$, stride=2, deconv. BN 256 ReLU\\
            \hline
            $4 \times 4$, stride=2, deconv. BN 128 ReLU\\
            \hline
            $4 \times 4$, stride=2, deconv. BN 64 ReLU\\
            \hline
            $4 \times 4$, stride=2, deconv. 3 Tanh\\
            \bottomrule
            \end{tabular}
            \subcaption*{Generator}
        \end{subtable}
        \begin{subtable}{0.49\columnwidth}
        \centering
            \vspace*{2mm}
            \begin{tabular}{c}
            \toprule
            RGB image $x \in [-1,1]^{64 \times 64 \times 3}$\\
            \hline
            $4 \times 4$, stride=2, conv 64 BN lReLU(0.2)\\
            \hline
            $4 \times 4$, stride=2, conv 128 BN lReLU(0.2)\\
            \hline
            $4 \times 4$, stride=2, conv 256 BN lReLU(0.2)\\
            \hline
            $4 \times 4$, stride=2, conv 512 BN lReLU(0.2)\\
            \hline
            $4 \times 4$, stride=1, conv 1\\
            \bottomrule
            \end{tabular}
            \vspace*{2mm}
            \subcaption*{Discriminator}
        \end{subtable}
    \caption*{(a) Network architectures with gradient penalty (GP) regularizer ($64 \times 64$ resolution).}
        \begin{subtable}{0.49\columnwidth}
        \centering
            \begin{tabular}{c}
            \toprule
            $z \in \mathbb{R}^{100} \sim \mathcal{N}(0, 1) $ \\
            \hline
            Reshape $\to 1 \times 1 \times 100$\\
            \hline
            $4 \times 4$, stride=1, deconv. BN 512 ReLU\\
            \hline
            $4 \times 4$, stride=2, deconv. BN 256 ReLU\\
            \hline
            $4 \times 4$, stride=2, deconv. BN 128 ReLU\\
            \hline
            $4 \times 4$, stride=2, deconv. BN 64 ReLU\\
            \hline
            $4 \times 4$, stride=2, deconv. 3 Tanh\\
            \bottomrule
            \end{tabular}
            \subcaption*{Generator}
        \end{subtable}
        \begin{subtable}{0.49\columnwidth}
        \centering
            \vspace*{2mm}
            \begin{tabular}{c}
            \toprule
            RGB image $x \in [-1,1]^{64 \times 64 \times 3}$\\
            \hline
            $4 \times 4$, stride=2, conv 64 SN lReLU(0.2)\\
            \hline
            $4 \times 4$, stride=2, conv 128 SN lReLU(0.2)\\
            \hline
            $4 \times 4$, stride=2, conv 256 SN lReLU(0.2)\\
            \hline
            $4 \times 4$, stride=2, conv 512 SN lReLU(0.2)\\
            \hline
            $4 \times 4$, stride=1, conv 1\\
            \bottomrule
            \end{tabular}
            \vspace*{2mm}
            \subcaption*{Discriminator}
        \end{subtable} 
    \caption*{(b) Network architectures with spectral normalization (SN) regularizer ($64 \times 64$ resolution).}
        \begin{subtable}{0.49\columnwidth}
        \centering
            \begin{tabular}{c}
            \toprule
            $z \in \mathbb{R}^{100} \sim \mathcal{N}(0, 1) $ \\
            \hline
            Reshape $\to 1 \times 1 \times 100$\\
            \hline
            $4 \times 4$, stride=1, deconv. BN 256 ReLU\\
            \hline
            $4 \times 4$, stride=2, deconv. BN 128 ReLU\\
            \hline
            $4 \times 4$, stride=2, deconv. BN 64 ReLU\\
            \hline
            $4 \times 4$, stride=2, deconv. 3 Tanh\\
            \bottomrule
            \end{tabular}
            \subcaption*{Generator}
        \end{subtable}
        \begin{subtable}{0.49\columnwidth}
        \centering
            \vspace*{2mm}
            \begin{tabular}{c}
            \toprule
            Image $x \in  [-1,1]^{32 \times 32 \times c}$\\
            \hline
            $4 \times 4$, stride=2, conv 64 SN lReLU(0.2)\\
            \hline
            $4 \times 4$, stride=2, conv 128 SN lReLU(0.2)\\
            \hline
            $4 \times 4$, stride=2, conv 256 SN lReLU(0.2)\\
            \hline
            $4 \times 4$, stride=1, conv 1\\
            \bottomrule
            \end{tabular}
            \vspace*{2mm}
            \subcaption*{Discriminator}
        \end{subtable}
    \caption*{(c) Network architectures with spectral normalization (SN) regularizer ($32 \times 32$ resolution). For input $x\in  [-1,1]^{32 \times 32 \times c}$, $c=3$ for CIFAR10 dataset and $c=1$ for MNIST dataset.}
    \end{center}
    \label{tb:CNN_architecture}
\end{table*}

\begin{table*}[t]
\caption{FID scores v.s. $\alpha$. For the line plots, the $x$-axis shows $\alpha$ (in log scale) and the $y$-axis shows the FID scores.  Lower FID scores are better. The FID scores of the datasets in \textbf{bold} are not included in the main paper.}
\begin{subtable}{0.99\textwidth}
    \begin{center}
        \centering
        \resizebox{\textwidth}{!}{
            \begin{tabular}{l l r r  r  r  r  r  c}
            \toprule
            \multirow{2}{*}{Dataset} & \multirow{2}{*}{GANs} & \multicolumn{6}{c}{FID Scores} & \multirow{2}{*}{Line Plot}\\
            & & $\alpha$ = $1e^{-11}$ & $1e^{-9}$ & $1e^{-7}$ & $1e^{-5}$ & $1e^{-3}$ & $1e^{-1}$ & \\
            \hline
            \multirow{5}{*}{\textbf{MNIST}} & NS-GAN-SN & 3.67 & 3.61 & 3.81 & 33.48 &154.14 & 155.54 & \multirow{5}{*}{\raisebox{-.9\height}{\includegraphics[width=.21\linewidth]{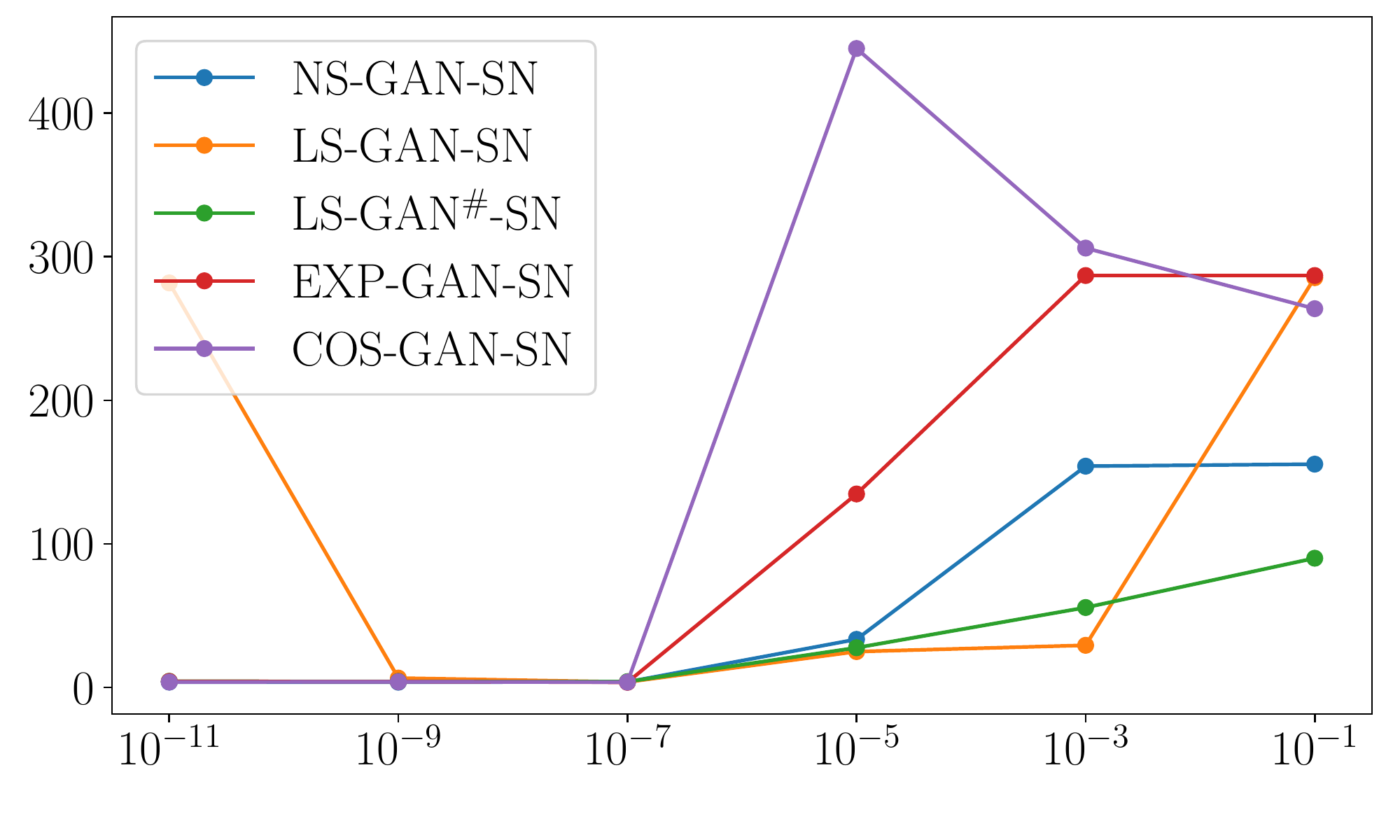}}}\\
            & LS-GAN-SN & 281.82 & 6.48 & 3.74 & 24.93 & 29.31 & 285.46 &\\
            & LS-GAN$^{\#}$-SN & 4.25 & 4.15 & 3.99 & 27.62 & 55.64 & 90.00 &\\
            & COS-GAN-SN & 3.74 & 3.93 & 3.66 & 445.15 & 306.09 & 263.85 &\\
            & EXP-GAN-SN & 4.14 & 4.01 & 3.54 & 134.76 & 286.96 & 286.96 &\\
            \hline
            \multirow{5}{*}{\textbf{CIFAR10}} & NS-GAN-SN & 19.58 & 22.46 & 18.73 & 24.57 & 49.56 & 43.42 &  \multirow{5}{*}{\raisebox{-.9\height}{\includegraphics[width=.21\linewidth]{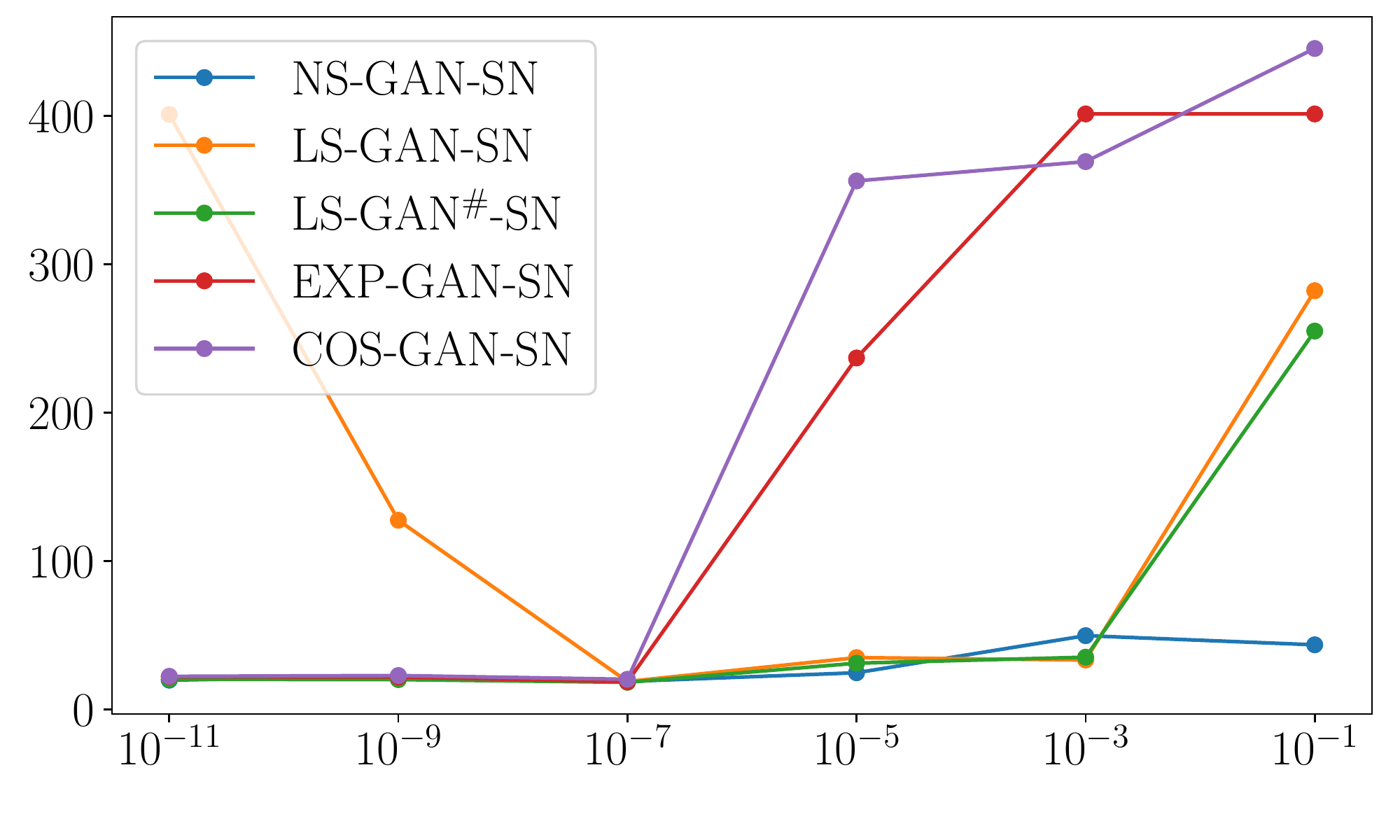}}}\\
            & LS-GAN-SN & 400.91 & 127.38 & 18.68 & 34.78 & 33.17 & 282.11 &\\
            & LS-GAN$^{\#}$-SN & 20.16 & 19.96 & 18.13 & 31.03 & 35.06 & 254.88 &\\
            & COS-GAN-SN & 22.16 & 22.69 & 20.19 & 356.10 & 369.11 & 445.37 &\\
            & EXP-GAN-SN & 21.93 & 21.70 & 18.50 & 236.77 & 401.24 & 401.24 &\\
            \hline
            \multirow{5}{*}{CelebA} & NS-GAN-SN & 9.08 & 7.05 & 7.84 & 18.51 & 18.41 & 242.64 & \multirow{5}{*}{\raisebox{-.9\height}{\includegraphics[width=.21\linewidth]{figures/FID_vs_alpha_ksn_50_celeba.pdf}}}\\
            & LS-GAN-SN & 135.17 & 6.57 & 10.67 & 13.39 & 17.42 & 311.93 &\\
            & LS-GAN$^{\#}$-SN & 6.66 & 5.68 & 8.72 & 11.13 & 14.90 & 383.61 &\\
            & COS-GAN-SN & 8.00 & 6.31 & 300.55 & 280.84 & 373.31 & 318.53&\\
            & EXP-GAN-SN & 8.85 & 6.09 & 264.49 & 375.32 & 375.32 & 375.32 &\\
            \bottomrule
            \end{tabular}
            }
    \end{center}
    \vspace*{-3mm}
    \subcaption{$k_{SN}=50.0$}
    \vspace*{1mm}
\end{subtable}
\begin{subtable}{0.99\textwidth}
    \begin{center}
        \centering
        \resizebox{\textwidth}{!}{
            \begin{tabular}{l l r r  r  r  r  r  c}
            \toprule
            \multirow{2}{*}{Dataset} & \multirow{2}{*}{GANs} & \multicolumn{6}{c}{FID Scores} & \multirow{2}{*}{Line Plot}\\
            & & $\alpha$ = $1e^{1}$ & $1e^{3}$ & $1e^5$ & $1e^{7}$ & $1e^{9}$ & $1e^{11}$ & \\
            \hline
            \multirow{5}{*}{MNIST} & NS-GAN-SN & 6.55 & 148.97 & 134.44 & 133.82 & 130.21 & 131.87 & \multirow{5}{*}{\raisebox{-.9\height}{\includegraphics[width=.21\linewidth]{figures/FID_vs_alpha_ksn_1_mnist.pdf}}}\\
            & LS-GAN-SN & 23.37 & 26.96 & 260.05 & 255.73 & 256.96 & 265.76 & \\
            & LS-GAN$^{\#}$-SN & 13.43 & 26.51 & 271.85 & 212.74 & 274.63 & 269.96 &\\
            & COS-GAN-SN & 11.79 & 377.62 & 375.72 & 363.45 & 401.12 & 376.39 &\\
            & EXP-GAN-SN & 11.02 & 286.96 & 286.96 & 286.96 & 286.96 & 286.96 &\\
            \hline
            \multirow{5}{*}{\textbf{CIFAR10}} & NS-GAN-SN & 17.63 & 47.31 & 46.85 & 45.44 & 45.67 & 39.90 &  \multirow{5}{*}{\raisebox{-.9\height}{\includegraphics[width=.21\linewidth]{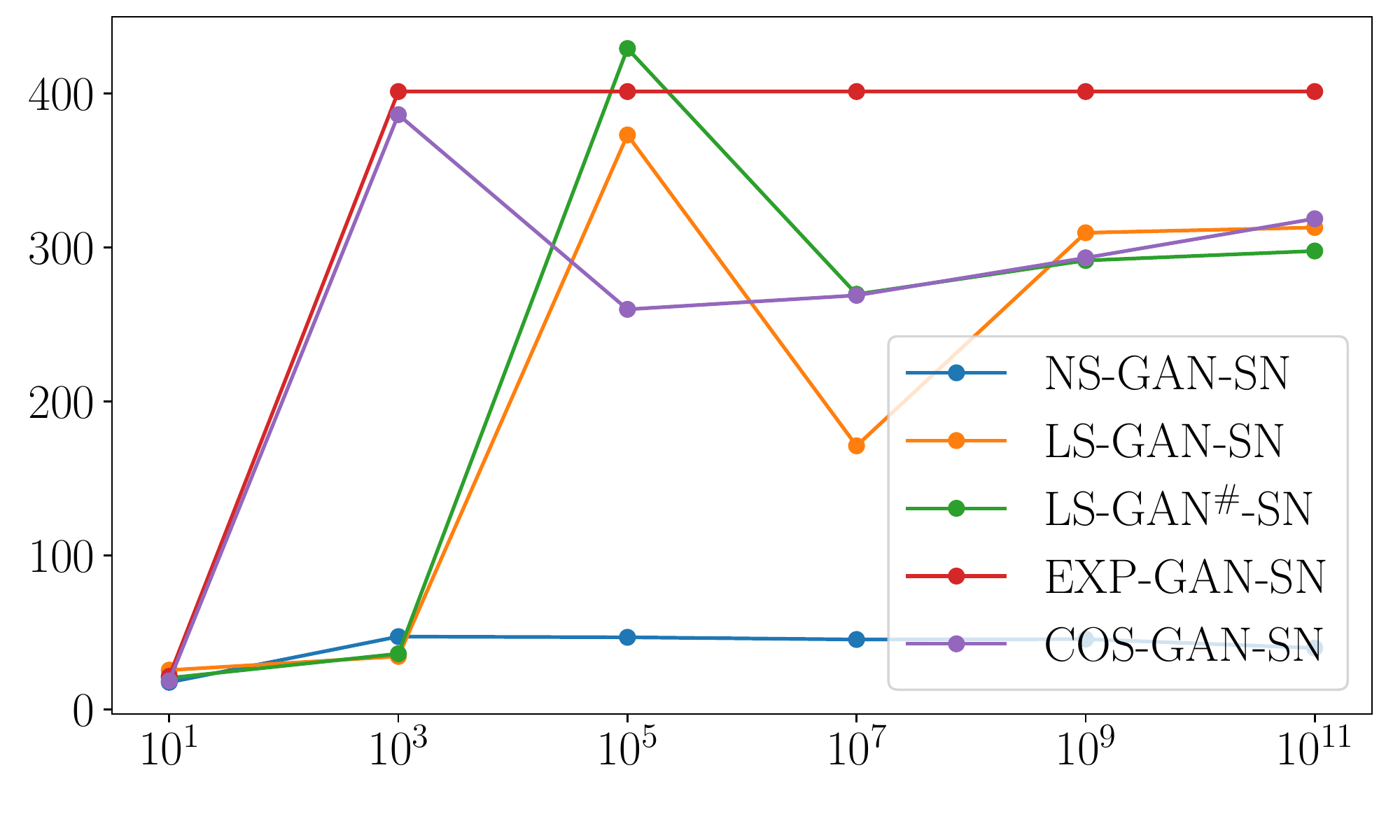}}}\\
            & LS-GAN-SN & 25.55 & 34.44 & 373.07 & 171.18 & 309.55 & 312.96 & \\
            & LS-GAN$^{\#}$-SN & 20.45 & 36.18 & 429.21 & 269.63 & 291.55 & 297.71 &\\
            & COS-GAN-SN & 18.59 & 386.24 & 259.83 & 268.89 & 293.29 & 318.65 &\\
            & EXP-GAN-SN & 21.56 & 401.24 & 401.24 & 401.24 & 401.24 & 401.24 &\\
            \hline
            \multirow{5}{*}{\textbf{CelebA}} & NS-GAN-SN & 5.88 & 16.14 & 17.75 & 17.67 & 16.87 & 18.81 &  \multirow{5}{*}{\raisebox{-.9\height}{\includegraphics[width=.21\linewidth]{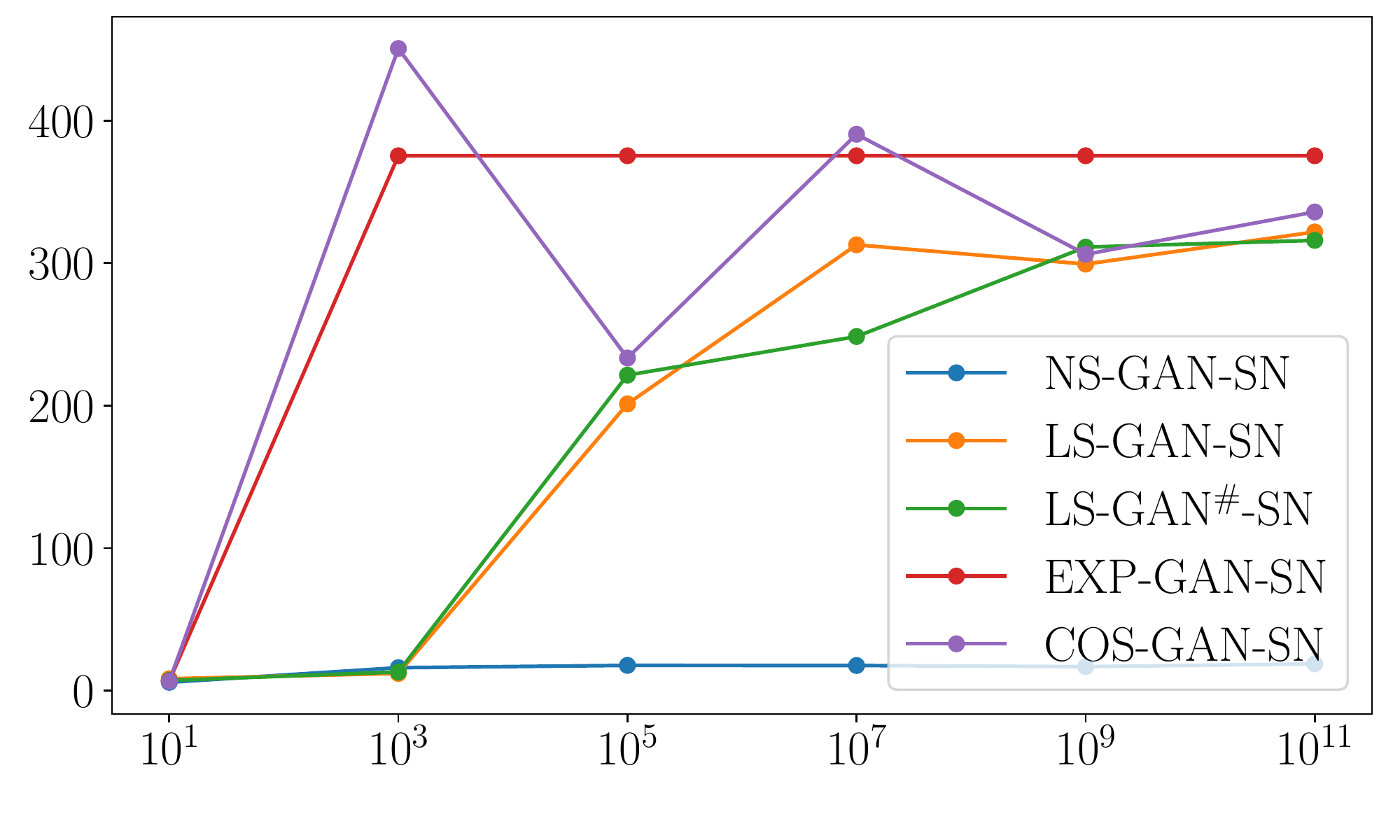}}}\\
            & LS-GAN-SN & 8.41 & 12.09 & 201.22 & 312.83 & 299.30 & 321.84 & \\
            & LS-GAN$^{\#}$-SN & 7.21 & 13.13 & 221.41 & 248.48 & 311.21 & 315.94 &\\
            & COS-GAN-SN & 6.62 & 450.57 & 233.42 & 390.40 & 306.17 & 335.87 &\\
            & EXP-GAN-SN & 6.91 & 375.32 & 375.32 & 375.32 & 375.32 & 375.32 &\\
            \bottomrule
            \end{tabular}
            }
    \end{center}
    \vspace*{-3mm}
    \subcaption{$k_{SN}=1.0$}
\end{subtable}
\label{tab:FID_vs_alpha_supplement}
\end{table*}

\begin{table*}[t]
\caption{Inception Scores (IS) v.s. $\alpha$ on CIFAR10. Higher IS are better.}
\begin{subtable}{0.99\textwidth}
    \begin{center}
        \centering
        \resizebox{\textwidth}{!}{
            \begin{tabular}{l r r r r r r r r r r r}
            \toprule
            \multirow{2}{*}{GANs} & \multicolumn{11}{c}{Inception Scores}\\
            & $\alpha$ = $1e^{-9}$ & $1e^{-7}$ & $1e^{-5}$ & $1e^{-3}$ & $1e^{-1}$ & $1e^{0}$ & $1e^{1}$ & $1e^{3}$ & $1e^{5}$ & $1e^{7}$ & $1e^{9}$\\
            \hline
            NS-GAN & 4.32 & 4.04 & 4.03 & 4.00 & 4.47 & 4.48 & 4.64 & 3.60 & 3.23 & 3.31 & 3.49 \\
            LS-GAN & 1.00 & 1.00 & 1.46 & 2.13 & 4.36 & 4.30 & 4.19 & 3.66 & 1.02 & 1.00 & 1.33 \\
            LS-GAN$^{\#}$ & 3.95 & 4.13 & 4.28 & 4.28 & 4.37 & 4.36 & 4.40 & 3.91 & 1.04 & 1.30 & 1.03 \\
            EXP-GAN & 4.20 & 4.01 & 3.97 & 4.24 & 4.34 & 4.31 & 4.30 & 1.00 & 1.00 & 1.00 & 1.00 \\
            COS-GAN & 4.20 & 4.10 & 4.24 & 4.13 & 4.32 & 4.35 & 4.40 & 2.00 & 1.00 & 1.05 & 1.35\\
            \bottomrule
            \end{tabular}
            }
    \end{center}
    \vspace*{-3mm}
    \subcaption{Lipschitz regularizer: SN}
    \vspace*{1mm}
\end{subtable}
\begin{subtable}{0.99\textwidth}
    \begin{center}
        \centering
        \resizebox{\textwidth}{!}{
            \begin{tabular}{l r r r r r r r r r r r}
            \toprule
            \multirow{2}{*}{GANs} & \multicolumn{11}{c}{Inception Scores}\\
            & $\alpha$ = $1e^{-9}$ & $1e^{-7}$ & $1e^{-5}$ & $1e^{-3}$ & $1e^{-1}$ & $1e^{0}$ & $1e^{1}$ & $1e^{3}$ & $1e^{5}$ & $1e^{7}$ & $1e^{9}$\\
            \hline
            NS-GAN & 3.80 & 4.03 & 4.00 & 4.29 & 4.37 & 4.29 & 4.48 & 3.36 & 3.51 & 3.55 & 3.52\\
            LS-GAN & 1.00 & 1.00 & 1.53 & 2.24 & 4.37 & 4.59 & 4.25 & 3.60 & 1.21 & 1.37 & 1.27\\
            LS-GAN$^{\#}$ & 4.05 & 3.88 & 3.91 & 4.02 & 4.31 & 4.54 & 4.44 & 3.92 & 1.64 & 1.32 & 1.00\\
            EXP-GAN & 3.99 & 3.94 & 3.98 & 4.22 & 4.42 & 4.34 & 4.43 & 1.00 & 1.00 & 1.00 & 1.00\\
            COS-GAN & 3.98 & 4.21 & 3.98 & 3.71 & 4.25 & 4.60 & 4.43 & 1.04 & 1.61 & 1.29 & 1.66\\
            \bottomrule
            \end{tabular}
            }
    \end{center}
    \vspace*{-3mm}
    \subcaption{Lipschitz regularizer: SRN+SN}
    \vspace*{1mm}
\end{subtable}
\label{tab:IS_vs_alpha}
\end{table*}

\begin{table*}[t]
\caption{Neural Divergence (ND) v.s. $\alpha$ on CIFAR10. Lower ND are better.}
\begin{subtable}{0.99\textwidth}
    \begin{center}
        \centering
        \resizebox{\textwidth}{!}{
            \begin{tabular}{l r r r r r r r r r r r}
            \toprule
            \multirow{2}{*}{GANs} & \multicolumn{11}{c}{Neural Divergence}\\
            & $\alpha$ = $1e^{-9}$ & $1e^{-7}$ & $1e^{-5}$ & $1e^{-3}$ & $1e^{-1}$ & $1e^{0}$ & $1e^{1}$ & $1e^{3}$ & $1e^{5}$ & $1e^{7}$ & $1e^{9}$\\
            \hline
            NS-GAN & 44.76 & 36.04 & 39.01 & 37.31 & 36.39 & 32.98 & 24.47 & 52.44 & 46.27 & 52.33 & 51.26\\
            LS-GAN & 157.69 & 222.71 & 82.76 & 118.53 & 37.16 & 25.94 & 42.15 & 38.47 & 72.27 & 155.82 & 286.46\\
            LS-GAN$^{\#}$ & 36.71 & 35.07 & 40.61 & 43.93 & 32.97 & 28.01 & 32.62 & 41.48 & 72.47 & 259.10 & 167.50\\
            EXP-GAN & 36.22 & 43.71 & 41.67 & 35.37 & 34.23 & 32.49 & 31.79 & nan & nan & nan & nan\\
            COS-GAN & 40.53 & 35.66 & 36.33 & 39.48 & 36.39 & 31.73 & 33.21 & 477.07 & 61.40 & 69.79 & 81.15\\
            \bottomrule
            \end{tabular}
            }
    \end{center}
    \vspace*{-3mm}
    \subcaption{Lipschitz regularizer: SN}
    \vspace*{1mm}
\end{subtable}
\begin{subtable}{0.99\textwidth}
    \begin{center}
        \centering
        \resizebox{\textwidth}{!}{
            \begin{tabular}{l r r r r r r r r r r r}
            \toprule
            \multirow{2}{*}{GANs} & \multicolumn{11}{c}{Neural Divergence}\\
            & $\alpha$ = $1e^{-9}$ & $1e^{-7}$ & $1e^{-5}$ & $1e^{-3}$ & $1e^{-1}$ & $1e^{0}$ & $1e^{1}$ & $1e^{3}$ & $1e^{5}$ & $1e^{7}$ & $1e^{9}$\\
            \hline
            NS-GAN & 39.28 & 37.12 & 31.60 & 37.82 & 33.56 & 23.33 & 26.57 & 34.30 & 49.68 & 45.96 & 52.72\\
            LS-GAN & 149.16 & 144.51 & 75.44 & 106.35 & 32.28 & 26.55 & 31.18 & 35.09 & 81.23 & 314.82 & 273.76\\
            LS-GAN$^{\#}$ & 45.73 & 36.29 & 41.87 & 39.78 & 22.94 & 28.94 & 28.52 & 42.42 & 60.87 & 240.42 & 142.73\\
            EXP-GAN & 38.78 & 39.59 & 31.07 & 33.98 & 27.61 & 29.34 & 24.56 & nan & nan & nan & nan\\
            COS-GAN & 35.85 & 40.54 & 31.28 & 41.60 & 30.50 & 39.72 & 25.62 & 104.24 & 107.67 & 263.35 & 158.75\\
            \bottomrule
            \end{tabular}
            }
    \end{center}
    \vspace*{-3mm}
    \subcaption{Lipschitz regularizer: SRN+SN}
    \vspace*{1mm}
\end{subtable}
\label{tab:ND_vs_alpha}
\end{table*}

\FloatBarrier
\clearpage
\newpage

\vspace*{\fill}
\begin{table*}[h!]
\centering
\captionsetup{justification=centering}
\caption*{{\LARGE FID scores v.s. $k_{SN}$ of Different Loss Functions\\ \vspace*{5mm}
- MNIST -}}
\end{table*}
\vspace*{\fill}

\FloatBarrier
\newpage

\begin{figure*}[h]
\begin{center}
    \begin{subfigure}{0.32\textwidth}
        \centering
        \includegraphics[width=0.99\linewidth]{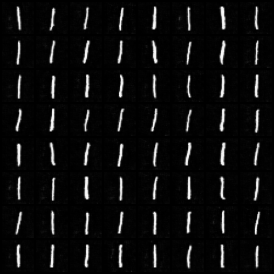}
        \subcaption{$k_{SN}$=50.0, $\mathrm{FID}$=155.41}
    \end{subfigure}
    \begin{subfigure}{0.32\textwidth}
        \centering
        \includegraphics[width=0.99\linewidth]{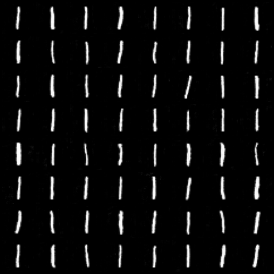}
        \subcaption{$k_{SN}$=10.0, $\mathrm{FID}$=156.60}
    \end{subfigure}
    \begin{subfigure}{0.32\textwidth}
        \centering
        \includegraphics[width=0.99\linewidth]{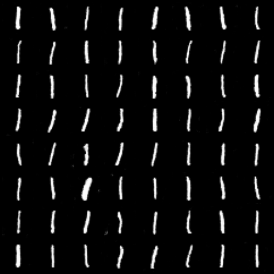}
        \subcaption{$k_{SN}$=5.0, $\mathrm{FID}$=144.28}
    \end{subfigure}
    \begin{subfigure}{0.32\textwidth}
        \centering
        \includegraphics[width=0.99\linewidth]{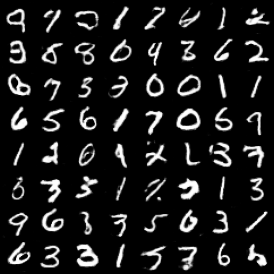}
        \subcaption{$k_{SN}$=1.0, $\mathrm{FID}$=3.90}
    \end{subfigure}
    \begin{subfigure}{0.32\textwidth}
        \centering
        \includegraphics[width=0.99\linewidth]{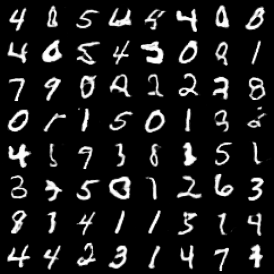}
        \subcaption{$k_{SN}$=0.5, $\mathrm{FID}$=4.20}
    \end{subfigure}
    \begin{subfigure}{0.32\textwidth}
        \centering
        \includegraphics[width=0.99\linewidth]{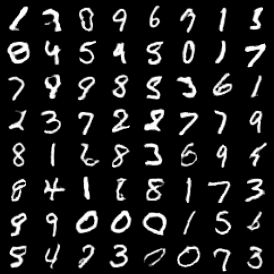}
        \subcaption{$k_{SN}$=0.25, $\mathrm{FID}$=3.99}
    \end{subfigure}
    \begin{subfigure}{0.32\textwidth}
        \centering
        \includegraphics[width=0.99\linewidth]{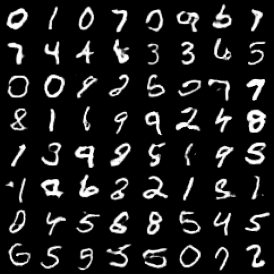}
        \subcaption{$k_{SN}$=0.2, $\mathrm{FID}$=5.41}
    \end{subfigure}
    \begin{subfigure}{0.6\textwidth}
        \includegraphics[width=.9\linewidth]{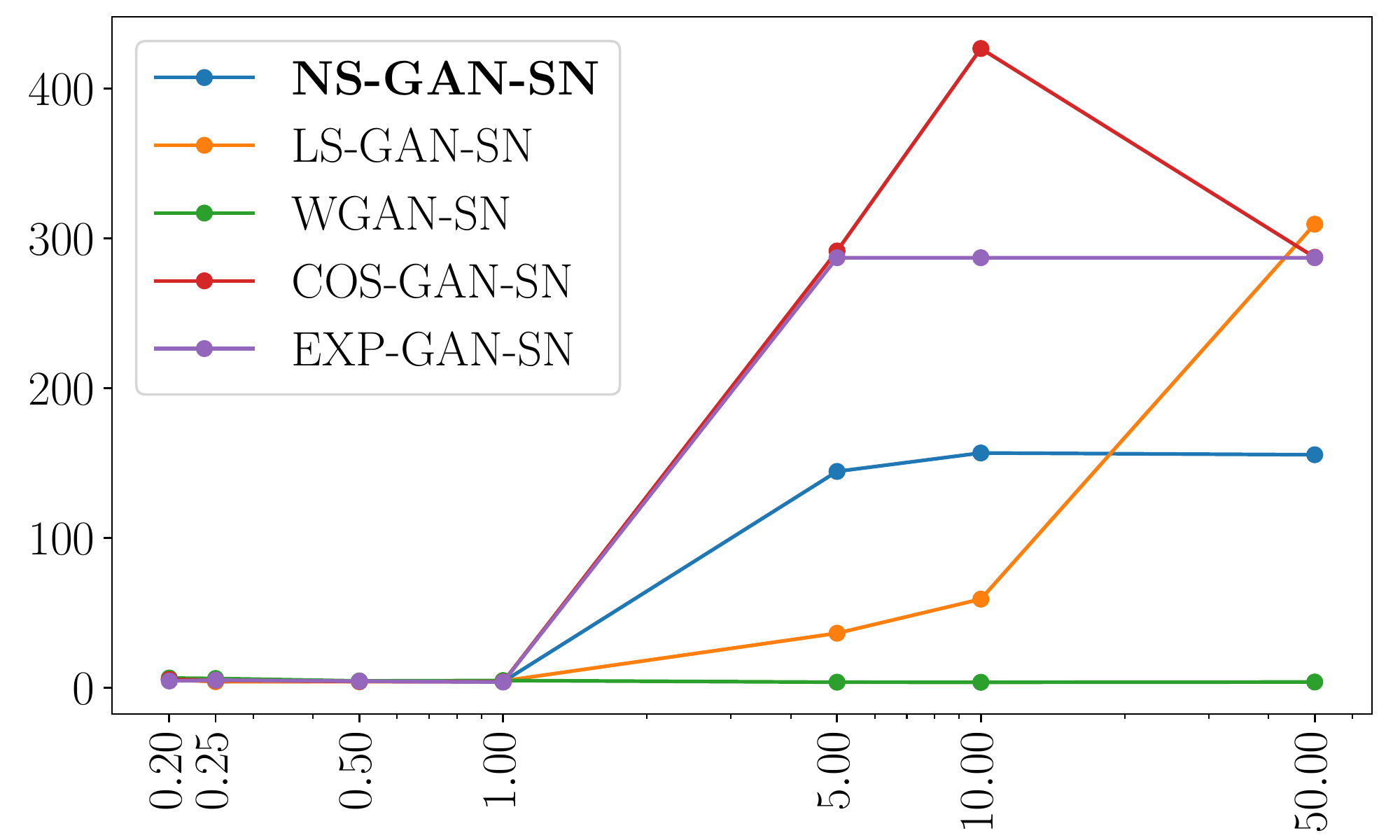}
    \end{subfigure}
\end{center}
  \caption{Samples of randomly generated images with NS-GAN-SN of varying $k_{SN}$ (MNIST). For the line plot, $x$-axis shows $k_{SN}$ (in log scale) and $y$-axis shows the FID scores. }
\label{fig:NSGANSN_MNIST}
\end{figure*}

\begin{figure*}[h]
\begin{center}
    \begin{subfigure}{0.32\textwidth}
        \centering
        \includegraphics[width=0.99\linewidth]{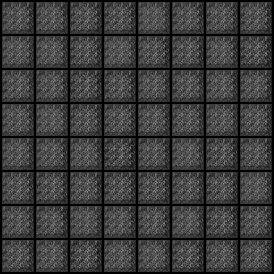}
        \subcaption{$k_{SN}$=50.0, $\mathrm{FID}$=309.35}
    \end{subfigure}
    \begin{subfigure}{0.32\textwidth}
        \centering
        \includegraphics[width=0.99\linewidth]{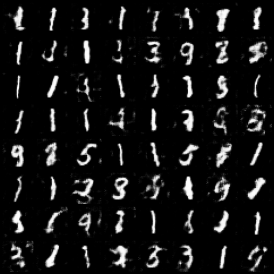}
        \subcaption{$k_{SN}$=10.0, $\mathrm{FID}$=59.04}
    \end{subfigure}
    \begin{subfigure}{0.32\textwidth}
        \centering
        \includegraphics[width=0.99\linewidth]{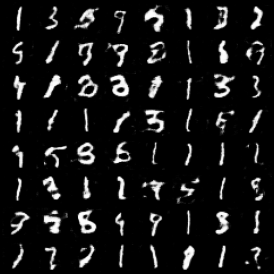}
        \subcaption{$k_{SN}$=5.0, $\mathrm{FID}$=36.26}
    \end{subfigure}
    \begin{subfigure}{0.32\textwidth}
        \centering
        \includegraphics[width=0.99\linewidth]{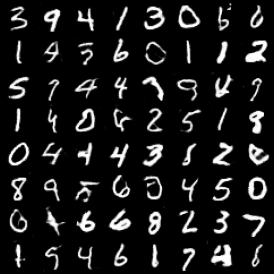}
        \subcaption{$k_{SN}$=1.0, $\mathrm{FID}$=4.42}
    \end{subfigure}
    \begin{subfigure}{0.32\textwidth}
        \centering
        \includegraphics[width=0.99\linewidth]{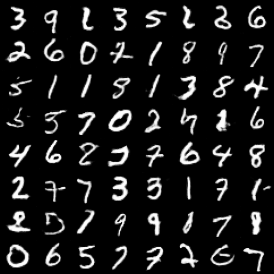}
        \subcaption{$k_{SN}$=0.5, $\mathrm{FID}$=3.90}
    \end{subfigure}
    \begin{subfigure}{0.32\textwidth}
        \centering
        \includegraphics[width=0.99\linewidth]{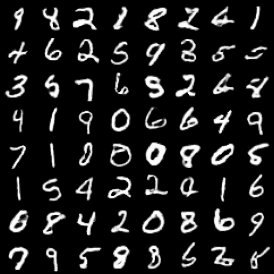}
        \subcaption{$k_{SN}$=0.25, $\mathrm{FID}$=3.96}
    \end{subfigure}
    \begin{subfigure}{0.32\textwidth}
        \centering
        \includegraphics[width=0.99\linewidth]{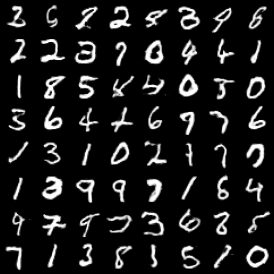}
        \subcaption{$k_{SN}$=0.2, $\mathrm{FID}$=5.14}
    \end{subfigure}
    \begin{subfigure}{0.6\textwidth}
        \includegraphics[width=.9\linewidth]{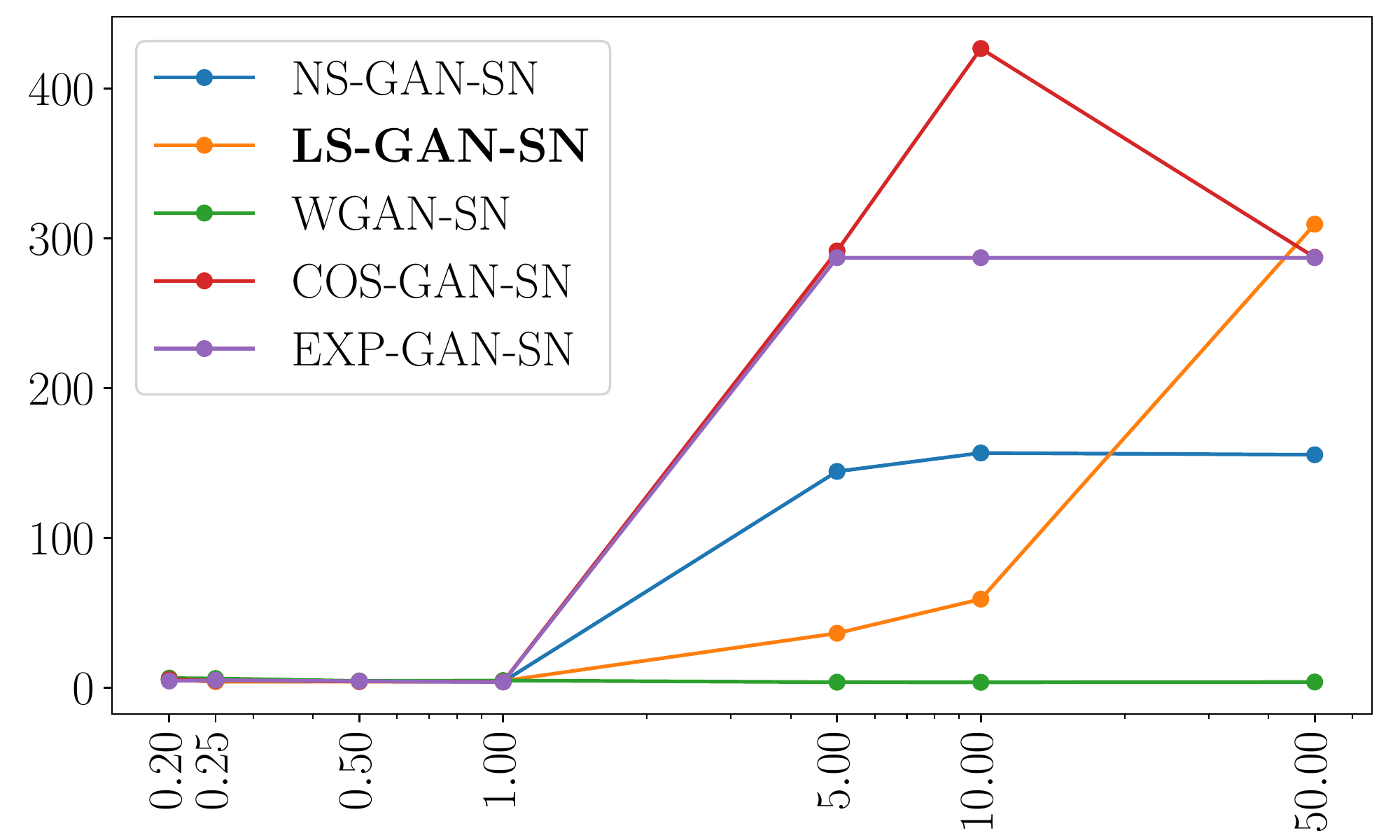}
    \end{subfigure}
\end{center}
  \caption{Samples of randomly generated images with LS-GAN-SN of varying $k_{SN}$ (MNIST). For the line plot, $x$-axis shows $k_{SN}$ (in log scale) and $y$-axis shows the FID scores.}
\label{fig:LSGANSN_MNIST}
\end{figure*}

\begin{figure*}[h]
\begin{center}
    \begin{subfigure}{0.32\textwidth}
        \centering
        \includegraphics[width=0.99\linewidth]{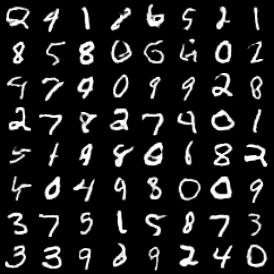}
        \subcaption{$k_{SN}=50.0$, $\mathrm{FID}=3.71$}
    \end{subfigure}
    \begin{subfigure}{0.32\textwidth}
        \centering
        \includegraphics[width=0.99\linewidth]{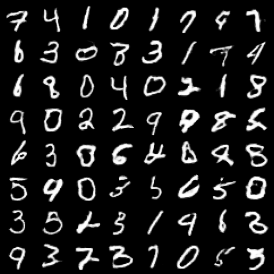}
        \subcaption{$k_{SN}=10.0$, $\mathrm{FID}=3.50$}
    \end{subfigure}
    \begin{subfigure}{0.32\textwidth}
        \centering
        \includegraphics[width=0.99\linewidth]{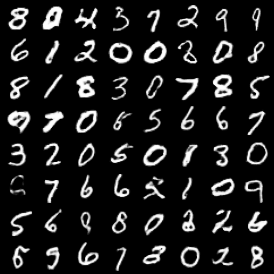}
        \subcaption{$k_{SN}=5.0$, $\mathrm{FID}=3.58$}
    \end{subfigure}
    \begin{subfigure}{0.32\textwidth}
        \centering
        \includegraphics[width=0.99\linewidth]{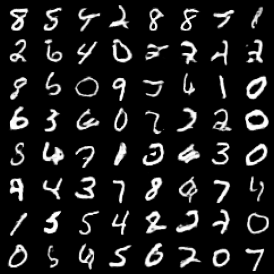}
        \subcaption{$k_{SN}=1.0$, $\mathrm{FID}=4.70$}
    \end{subfigure}
    \begin{subfigure}{0.32\textwidth}
        \centering
        \includegraphics[width=0.99\linewidth]{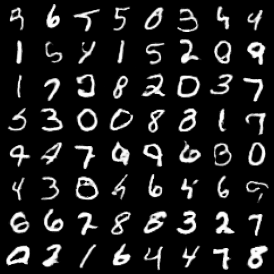}
        \subcaption{$k_{SN}=0.5$, $\mathrm{FID}=4.44$}
    \end{subfigure}
    \begin{subfigure}{0.32\textwidth}
        \centering
        \includegraphics[width=0.99\linewidth]{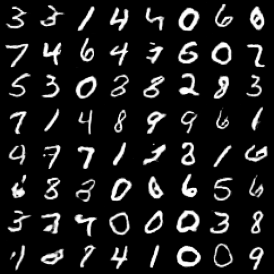}
        \subcaption{$k_{SN}=0.25$, $\mathrm{FID}=6.06$}
    \end{subfigure}
    \begin{subfigure}{0.32\textwidth}
        \centering
        \includegraphics[width=0.99\linewidth]{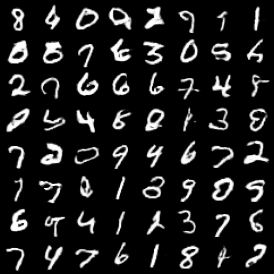}
        \subcaption{$k_{SN}=0.2$, $\mathrm{FID}=6.35$}
    \end{subfigure}
    \begin{subfigure}{0.6\textwidth}
        \includegraphics[width=.9\linewidth]{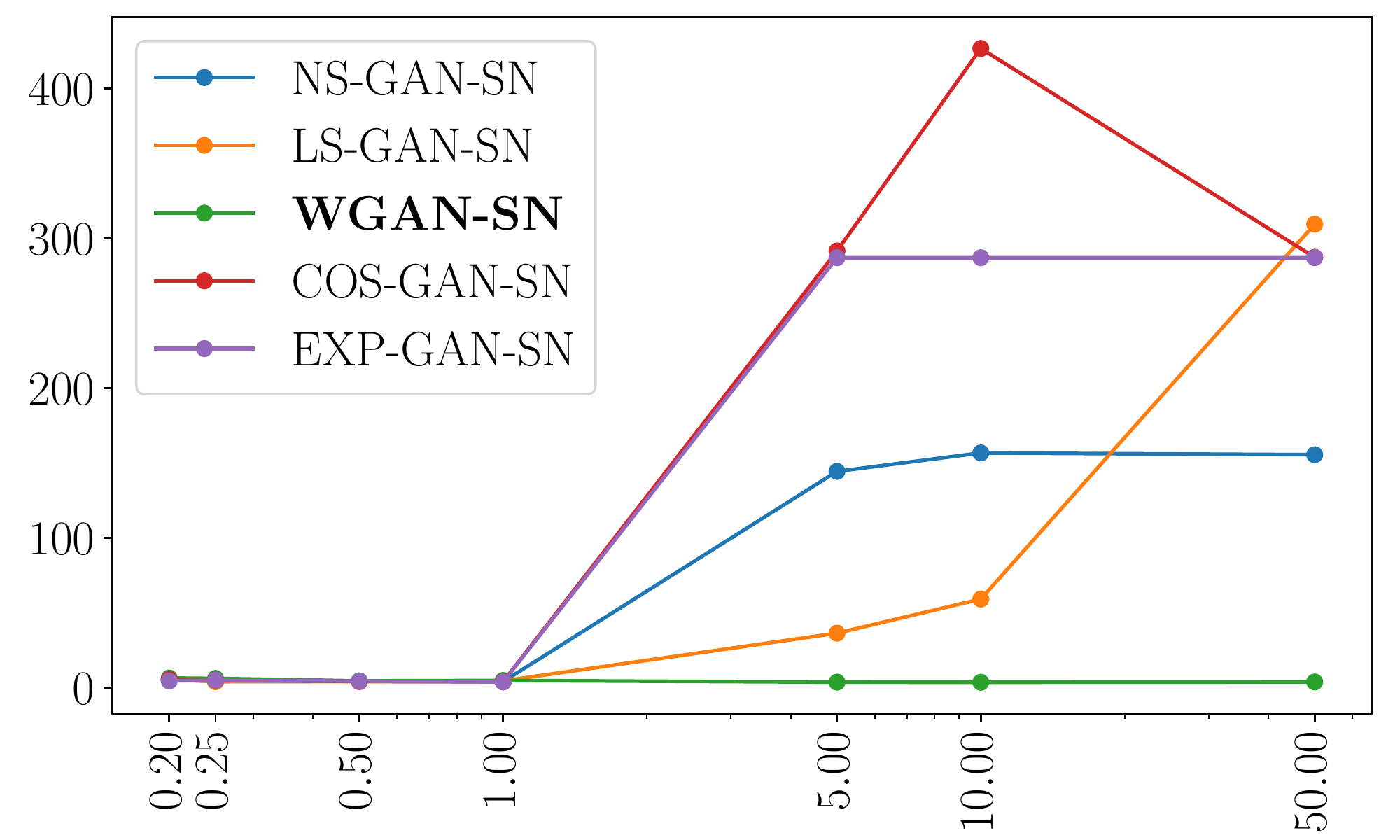}
    \end{subfigure}
\end{center}
  \caption{Samples of randomly generated images with WGAN-SN of varying $k_{SN}$ (MNIST). For the line plot, $x$-axis shows $k_{SN}$ (in log scale) and $y$-axis shows the FID scores.}
\label{fig:WGANSN_MNIST}
\end{figure*}

\begin{figure*}[h]
\begin{center}
    \begin{subfigure}{0.32\textwidth}
        \centering
        \includegraphics[width=0.99\linewidth]{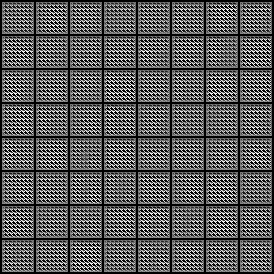}
        \subcaption{$k_{SN}$=50.0, $\mathrm{FID}$=287.23}
    \end{subfigure}
    \begin{subfigure}{0.32\textwidth}
        \centering
        \includegraphics[width=0.99\linewidth]{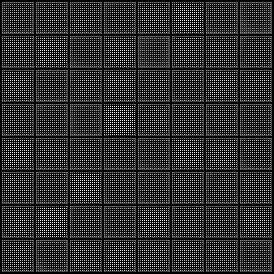}
        \subcaption{$k_{SN}$=10.0, $\mathrm{FID}$=426.62}
    \end{subfigure}
    \begin{subfigure}{0.32\textwidth}
        \centering
        \includegraphics[width=0.99\linewidth]{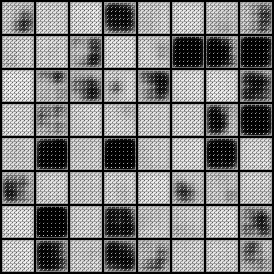}
        \subcaption{$k_{SN}$=5.0, $\mathrm{FID}$=291.44}
    \end{subfigure}
    \begin{subfigure}{0.32\textwidth}
        \centering
        \includegraphics[width=0.99\linewidth]{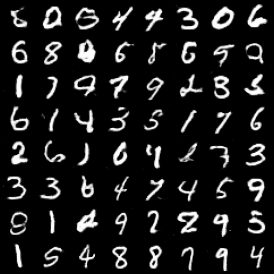}
        \subcaption{$k_{SN}$=1.0, $\mathrm{FID}$=3.86}
    \end{subfigure}
    \begin{subfigure}{0.32\textwidth}
        \centering
        \includegraphics[width=0.99\linewidth]{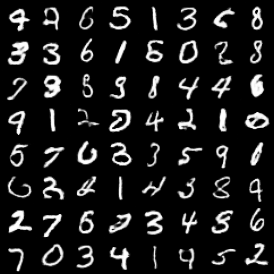}
        \subcaption{$k_{SN}$=0.5, $\mathrm{FID}$=4.05}
    \end{subfigure}
    \begin{subfigure}{0.32\textwidth}
        \centering
        \includegraphics[width=0.99\linewidth]{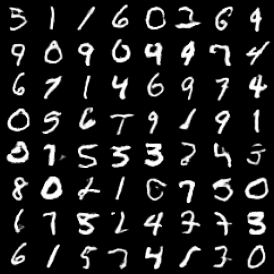}
        \subcaption{$k_{SN}$=0.25, $\mathrm{FID}$=4.83}
    \end{subfigure}
    \begin{subfigure}{0.32\textwidth}
        \centering
        \includegraphics[width=0.99\linewidth]{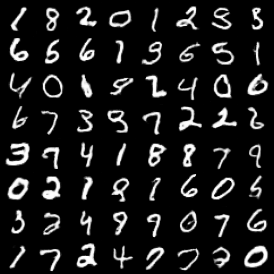}
        \subcaption{$k_{SN}$=0.2, $\mathrm{FID}$=5.41}
    \end{subfigure}
    \begin{subfigure}{0.6\textwidth}
        \includegraphics[width=.9\linewidth]{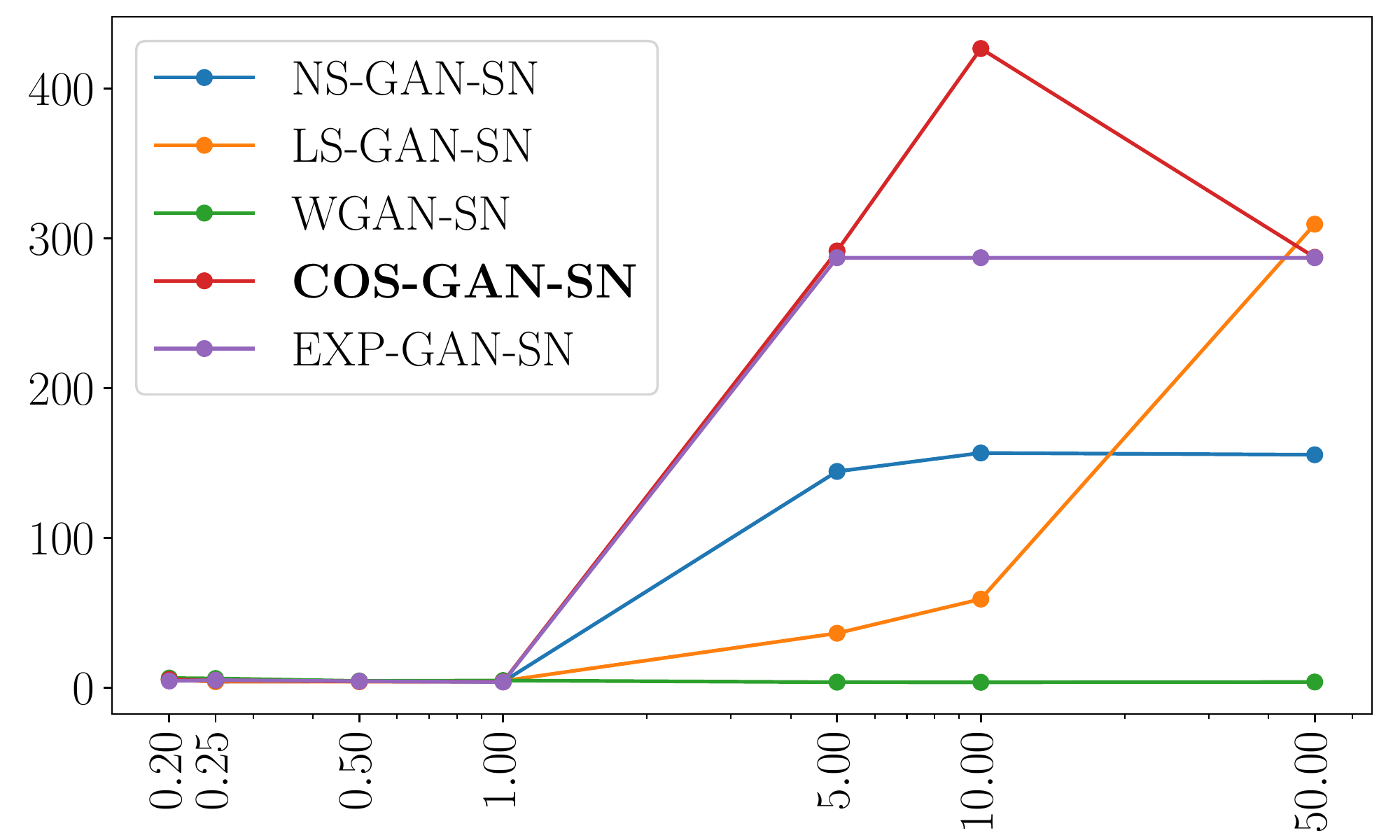}
    \end{subfigure}
\end{center}
  \caption{Samples of randomly generated images with COS-GAN-SN of varying $k_{SN}$ (MNIST). For the line plot, $x$-axis shows $k_{SN}$ (in log scale) and $y$-axis shows the FID scores.}
\label{fig:COSGANSN_MNIST}
\end{figure*}

\begin{figure*}[h]
\begin{center}
    \begin{subfigure}{0.32\textwidth}
        \centering
        \includegraphics[width=0.99\linewidth]{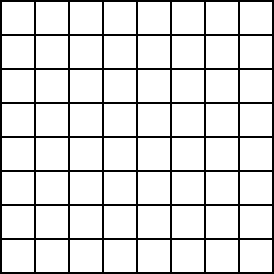}
        \subcaption{$k_{SN}$=50.0, $\mathrm{FID}$=286.96}
    \end{subfigure}
    \begin{subfigure}{0.32\textwidth}
        \centering
        \includegraphics[width=0.99\linewidth]{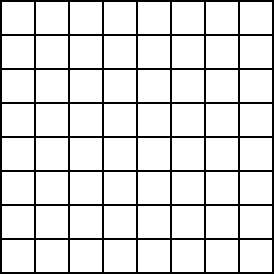}
        \subcaption{$k_{SN}$=10.0, $\mathrm{FID}$=286.96}
    \end{subfigure}
    \begin{subfigure}{0.32\textwidth}
        \centering
        \includegraphics[width=0.99\linewidth]{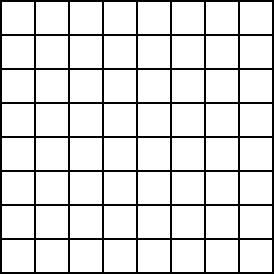}
        \subcaption{$k_{SN}$=5.0, $\mathrm{FID}$=286.96}
    \end{subfigure}
    \begin{subfigure}{0.32\textwidth}
        \centering
        \includegraphics[width=0.99\linewidth]{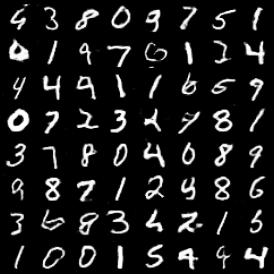}
        \subcaption{$k_{SN}$=1.0, $\mathrm{FID}$=3.69}
    \end{subfigure}
    \begin{subfigure}{0.32\textwidth}
        \centering
        \includegraphics[width=0.99\linewidth]{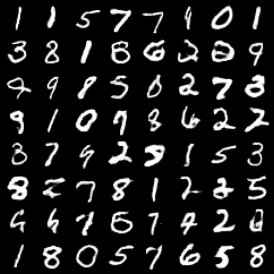}
        \subcaption{$k_{SN}$=0.5, $\mathrm{FID}$=4.25}
    \end{subfigure}
    \begin{subfigure}{0.32\textwidth}
        \centering
        \includegraphics[width=0.99\linewidth]{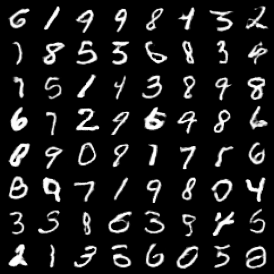}
        \subcaption{$k_{SN}$=0.25, $\mathrm{FID}$=4.93}
    \end{subfigure}
    \begin{subfigure}{0.32\textwidth}
        \centering
        \includegraphics[width=0.99\linewidth]{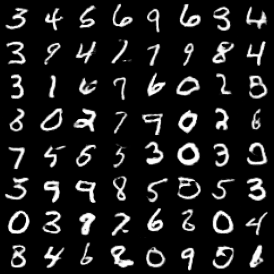}
        \subcaption{$k_{SN}$=0.2, $\mathrm{FID}$=4.38}
    \end{subfigure}
    \begin{subfigure}{0.6\textwidth}
        \includegraphics[width=.9\linewidth]{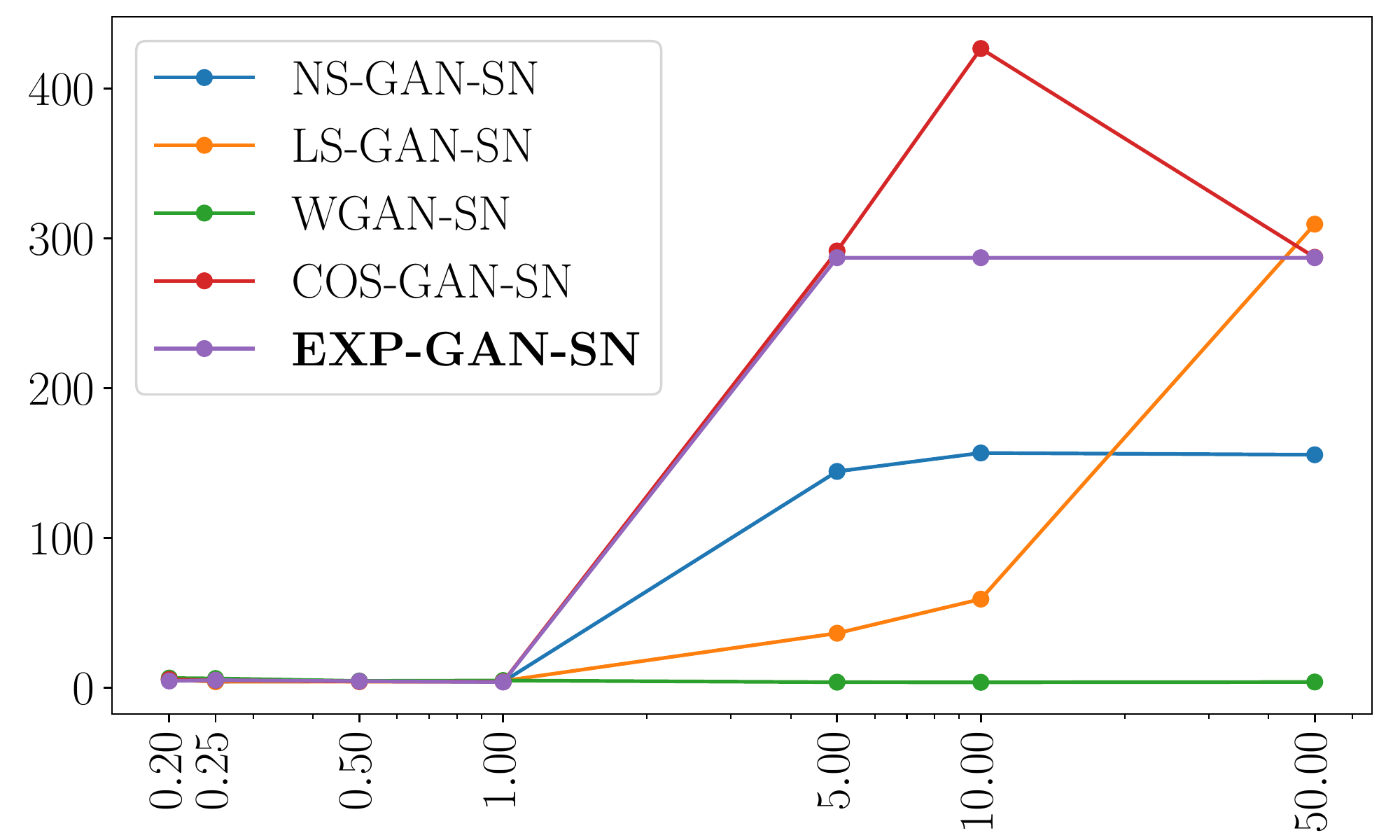}
    \end{subfigure}
\end{center}
  \caption{Samples of randomly generated images with EXP-GAN-SN of varying $k_{SN}$ (MNIST). For the line plot, $x$-axis shows $k_{SN}$ (in log scale) and $y$-axis shows the FID scores.}
\label{fig:EXPGANSN_MNIST}
\end{figure*}

\FloatBarrier
\newpage

\vspace*{\fill}
\begin{table*}[h!]
\centering
\captionsetup{justification=centering}
\caption*{{\LARGE FID scores v.s. $k_{SN}$ of Different Loss Functions\\ \vspace*{5mm}
- CIFAR10 -}}
\end{table*}
\vspace*{\fill}

\FloatBarrier
\newpage

\begin{figure*}[h]
\begin{center}
    \begin{subfigure}{0.32\textwidth}
        \centering
        \includegraphics[width=0.99\linewidth]{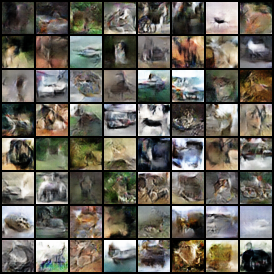}
        \subcaption{$k_{SN}$=50.0, $\mathrm{FID}$=48.03}
    \end{subfigure}
    \begin{subfigure}{0.32\textwidth}
        \centering
        \includegraphics[width=0.99\linewidth]{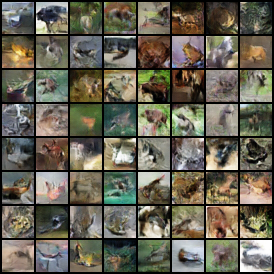}
        \subcaption{$k_{SN}$=10.0, $\mathrm{FID}$=49.67}
    \end{subfigure}
    \begin{subfigure}{0.32\textwidth}
        \centering
        \includegraphics[width=0.99\linewidth]{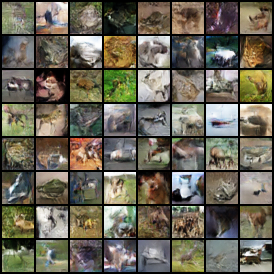}
        \subcaption{$k_{SN}$=5.0, $\mathrm{FID}$=41.04}
    \end{subfigure}
    \begin{subfigure}{0.32\textwidth}
        \centering
        \includegraphics[width=0.99\linewidth]{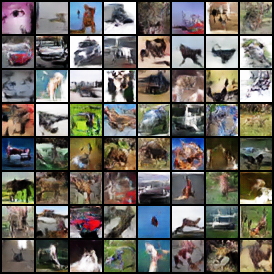}
        \subcaption{$k_{SN}$=1.0, $\mathrm{FID}$=15.81}
    \end{subfigure}
    \begin{subfigure}{0.32\textwidth}
        \centering
        \includegraphics[width=0.99\linewidth]{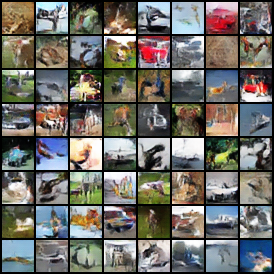}
        \subcaption{$k_{SN}$=0.5, $\mathrm{FID}$=23.29}
    \end{subfigure}
    \begin{subfigure}{0.32\textwidth}
        \centering
        \includegraphics[width=0.99\linewidth]{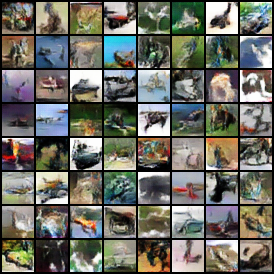}
        \subcaption{$k_{SN}$=0.25, $\mathrm{FID}$=24.37}
    \end{subfigure}
    \begin{subfigure}{0.32\textwidth}
        \centering
        \includegraphics[width=0.99\linewidth]{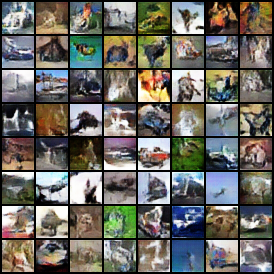}
        \subcaption{$k_{SN}$=0.2, $\mathrm{FID}$=29.23}
    \end{subfigure}
    \begin{subfigure}{0.6\textwidth}
        \includegraphics[width=.9\linewidth]{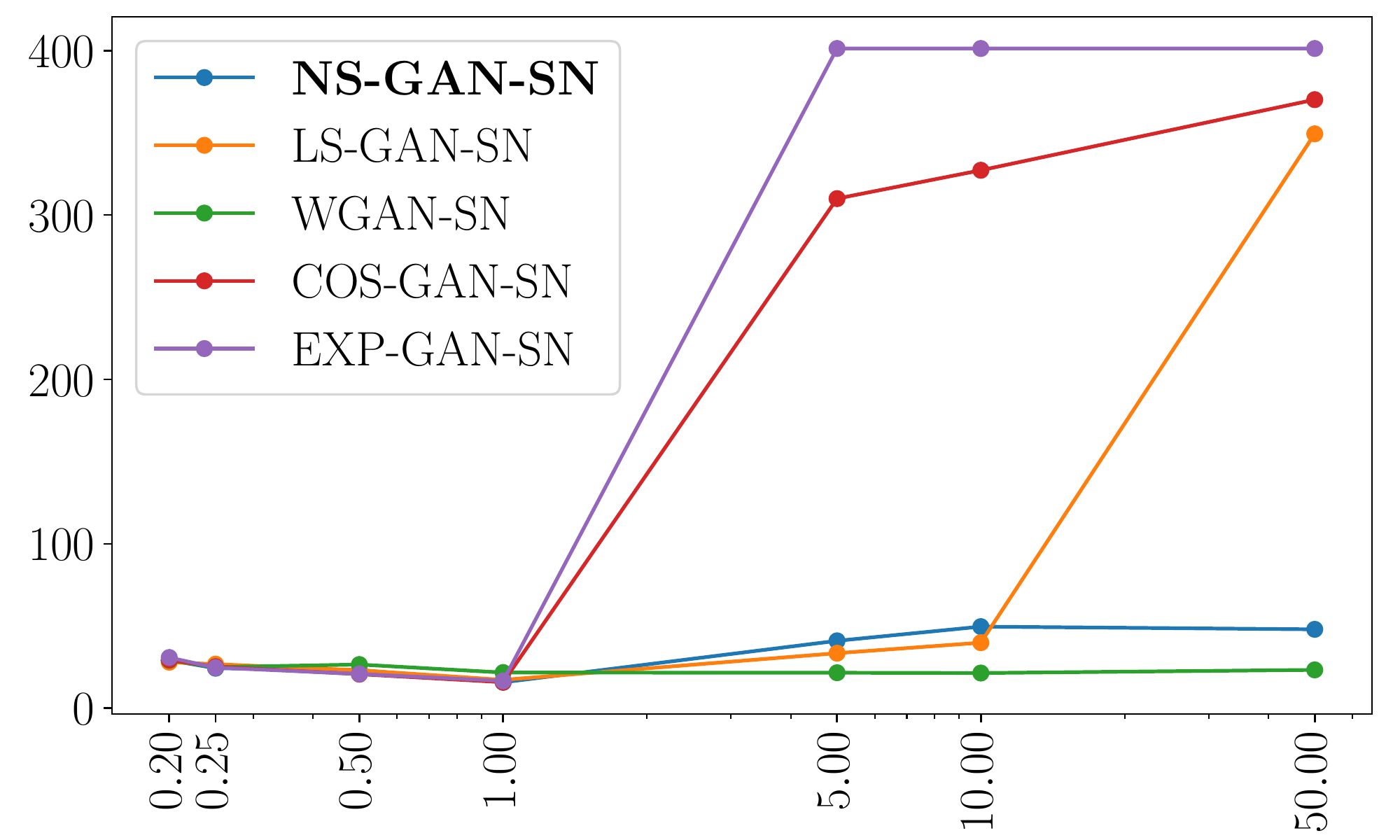}
    \end{subfigure}
\end{center}
  \caption{Samples of randomly generated images with NS-GAN-SN of varying $k_{SN}$ (CIFAR10). For the line plot, $x$-axis shows $k_{SN}$ (in log scale) and $y$-axis shows the FID scores. }
\label{fig:NSGANSN_CIFAR10}
\end{figure*}

\begin{figure*}[h]
\begin{center}
    \begin{subfigure}{0.32\textwidth}
        \centering
        \includegraphics[width=0.99\linewidth]{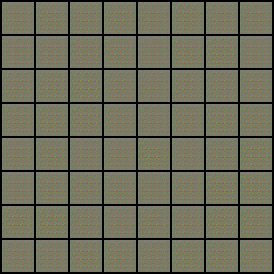}
        \subcaption{$k_{SN}$=50.0, $\mathrm{FID}$= 349.35}
    \end{subfigure}
    \begin{subfigure}{0.32\textwidth}
        \centering
        \includegraphics[width=0.99\linewidth]{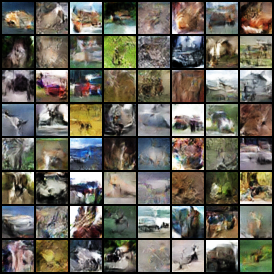}
        \subcaption{$k_{SN}$=10.0, $\mathrm{FID}$= 39.90}
    \end{subfigure}
    \begin{subfigure}{0.32\textwidth}
        \centering
        \includegraphics[width=0.99\linewidth]{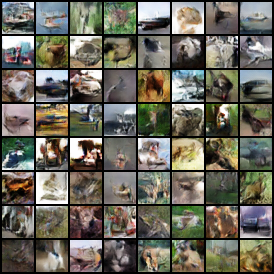}
        \subcaption{$k_{SN}$=5.0, $\mathrm{FID}$= 33.53}
    \end{subfigure}
    \begin{subfigure}{0.32\textwidth}
        \centering
        \includegraphics[width=0.99\linewidth]{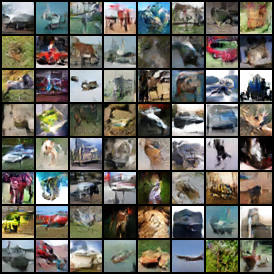}
        \subcaption{$k_{SN}$=1.0, $\mathrm{FID}$= 17.30}
    \end{subfigure}
    \begin{subfigure}{0.32\textwidth}
        \centering
        \includegraphics[width=0.99\linewidth]{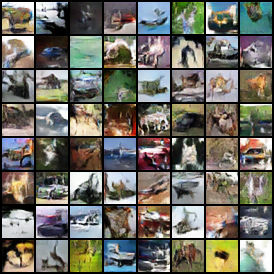}
        \subcaption{$k_{SN}$=0.5, $\mathrm{FID}$= 23.14}
    \end{subfigure}
    \begin{subfigure}{0.32\textwidth}
        \centering
        \includegraphics[width=0.99\linewidth]{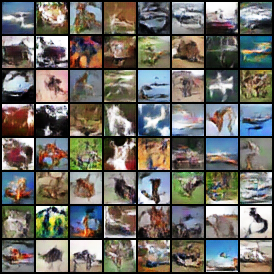}
        \subcaption{$k_{SN}$=0.25, $\mathrm{FID}$= 26.85}
    \end{subfigure}
    \begin{subfigure}{0.32\textwidth}
        \centering
        \includegraphics[width=0.99\linewidth]{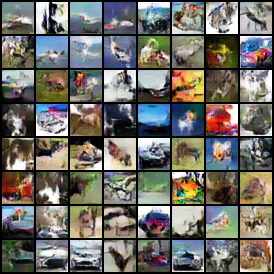}
        \subcaption{$k_{SN}$=0.2, $\mathrm{FID}$= 28.04}
    \end{subfigure}
    \begin{subfigure}{0.6\textwidth}
        \includegraphics[width=.9\linewidth]{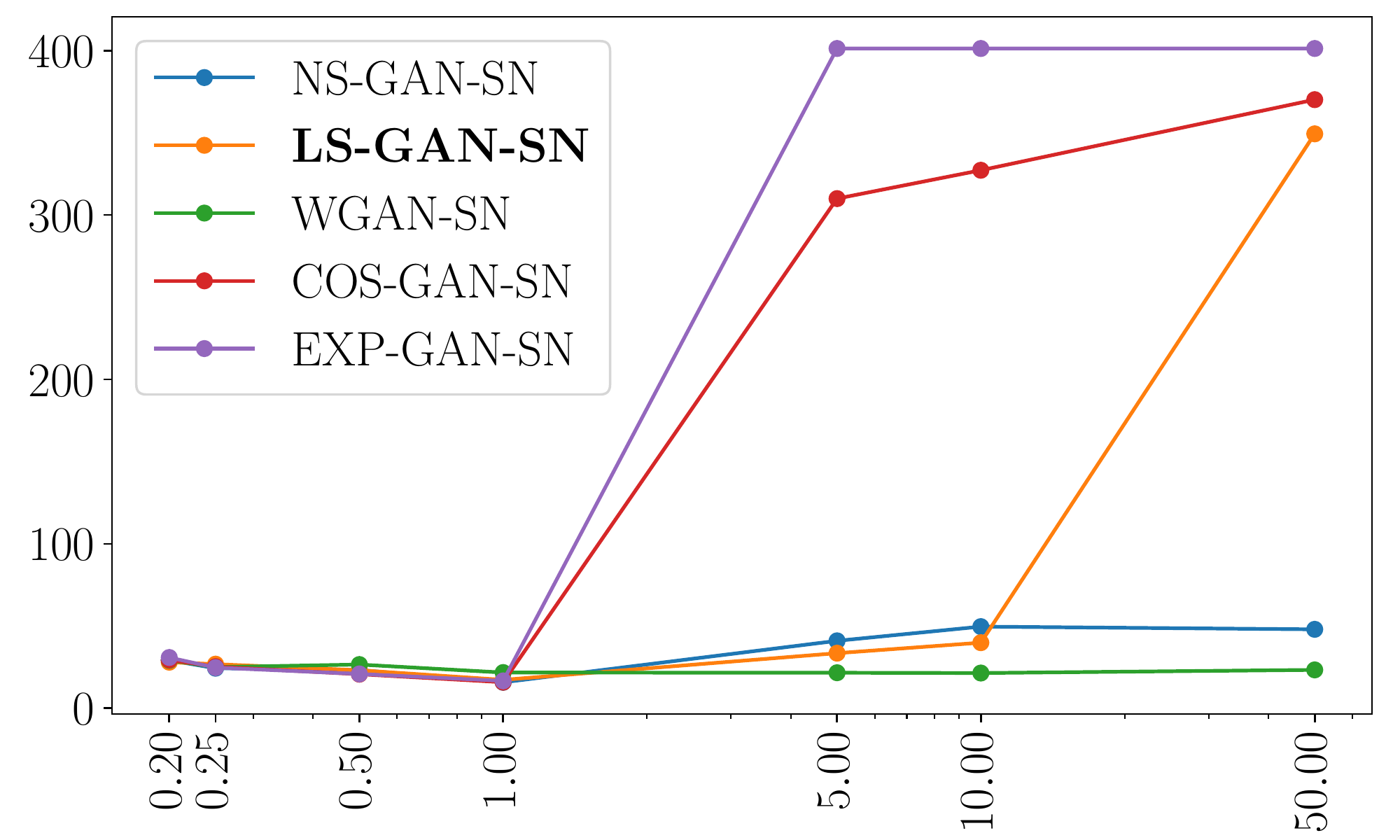}
    \end{subfigure}
\end{center}
  \caption{Samples of randomly generated images with LS-GAN-SN of varying $k_{SN}$ (CIFAR10). For the line plot, $x$-axis shows $k_{SN}$ (in log scale) and $y$-axis shows the FID scores. }
\label{fig:LSGANSN_CIFAR10}
\end{figure*}

\begin{figure*}[h]
\begin{center}
    \begin{subfigure}{0.32\textwidth}
        \centering
        \includegraphics[width=0.99\linewidth]{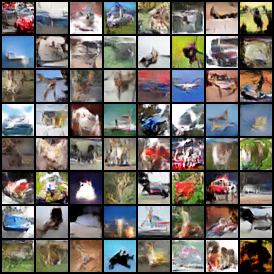}
        \subcaption{$k_{SN}$=50.0, $\mathrm{FID}$= 23.36}
    \end{subfigure}
    \begin{subfigure}{0.32\textwidth}
        \centering
        \includegraphics[width=0.99\linewidth]{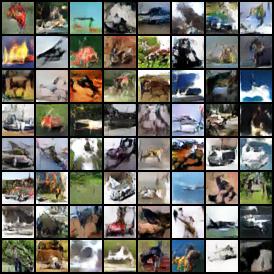}
        \subcaption{$k_{SN}$=10.0, $\mathrm{FID}$= 21.45}
    \end{subfigure}
    \begin{subfigure}{0.32\textwidth}
        \centering
        \includegraphics[width=0.99\linewidth]{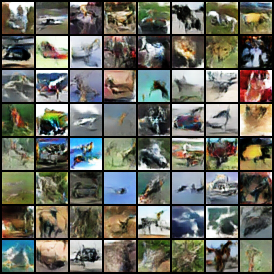}
        \subcaption{$k_{SN}$=5.0, $\mathrm{FID}$= 21.63}
    \end{subfigure}
    \begin{subfigure}{0.32\textwidth}
        \centering
        \includegraphics[width=0.99\linewidth]{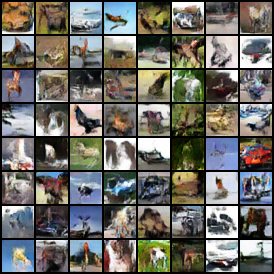}
        \subcaption{$k_{SN}$=1.0, $\mathrm{FID}$= 21.75}
    \end{subfigure}
    \begin{subfigure}{0.32\textwidth}
        \centering
        \includegraphics[width=0.99\linewidth]{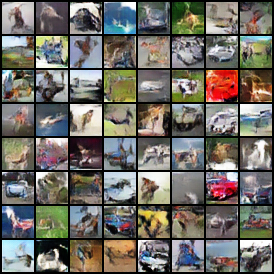}
        \subcaption{$k_{SN}$=0.5, $\mathrm{FID}$= 26.61}
    \end{subfigure}
    \begin{subfigure}{0.32\textwidth}
        \centering
        \includegraphics[width=0.99\linewidth]{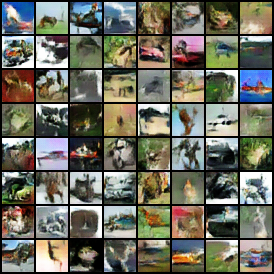}
        \subcaption{$k_{SN}$=0.25, $\mathrm{FID}$= 25.07}
    \end{subfigure}
    \begin{subfigure}{0.32\textwidth}
        \centering
        \includegraphics[width=0.99\linewidth]{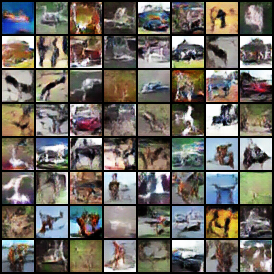}
        \subcaption{$k_{SN}$=0.2, $\mathrm{FID}$= 29.20}
    \end{subfigure}
    \begin{subfigure}{0.6\textwidth}
        \includegraphics[width=.9\linewidth]{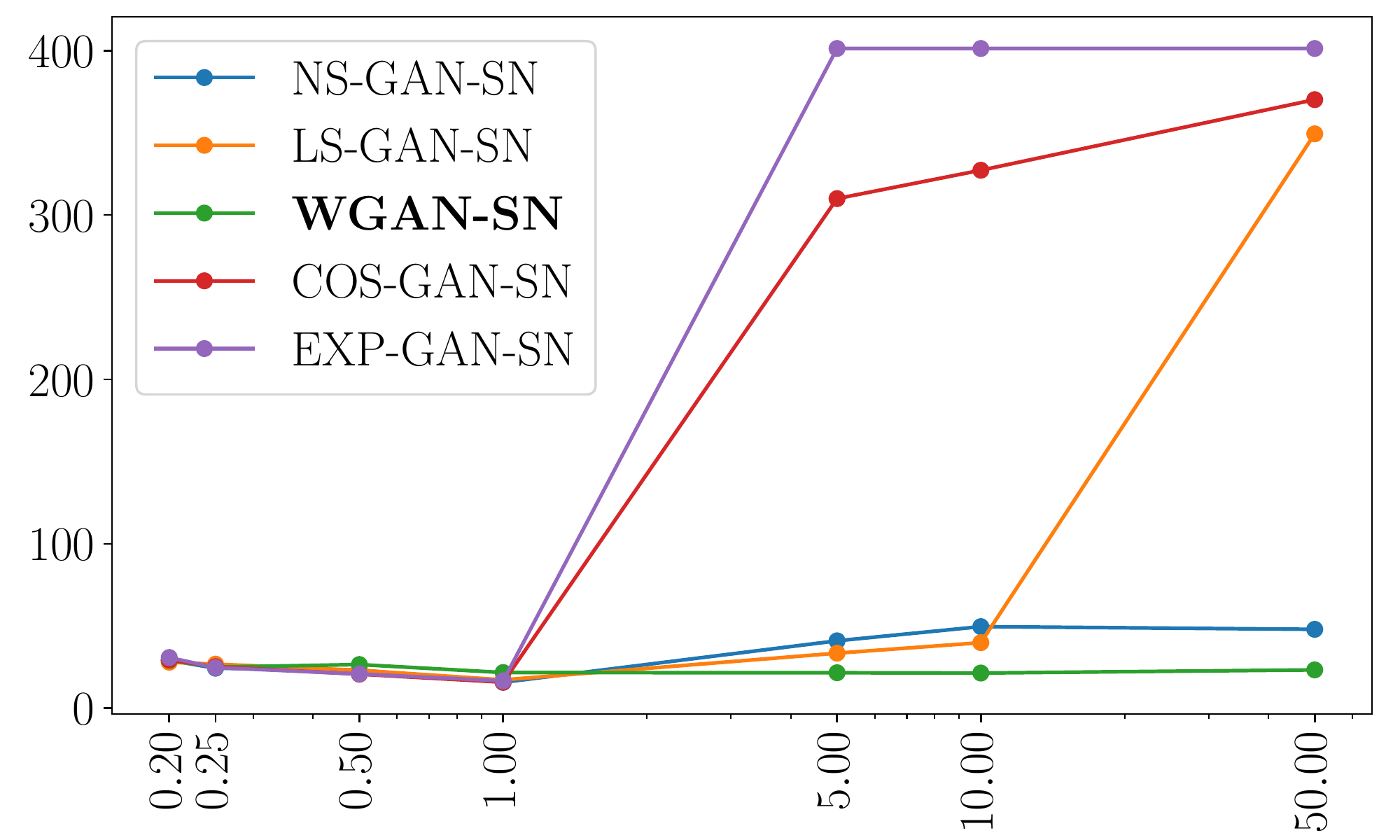}
    \end{subfigure}
\end{center}
  \caption{Samples of randomly generated images with WGAN-SN of varying $k_{SN}$ (CIFAR10). For the line plot, $x$-axis shows $k_{SN}$ (in log scale) and $y$-axis shows the FID scores. }
\label{fig:WGANSN_CIFAR10}
\end{figure*}

\begin{figure*}[h]
\begin{center}
    \begin{subfigure}{0.32\textwidth}
        \centering
        \includegraphics[width=0.99\linewidth]{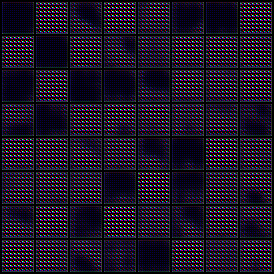}
        \subcaption{$k_{SN}$=50.0, $\mathrm{FID}$= 370.13}
    \end{subfigure}
    \begin{subfigure}{0.32\textwidth}
        \centering
        \includegraphics[width=0.99\linewidth]{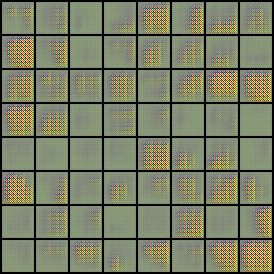}
        \subcaption{$k_{SN}$=10.0, $\mathrm{FID}$= 327.20}
    \end{subfigure}
    \begin{subfigure}{0.32\textwidth}
        \centering
        \includegraphics[width=0.99\linewidth]{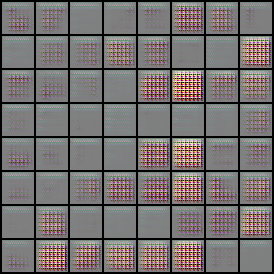}
        \subcaption{$k_{SN}$=5.0, $\mathrm{FID}$= 309.96}
    \end{subfigure}
    \begin{subfigure}{0.32\textwidth}
        \centering
        \includegraphics[width=0.99\linewidth]{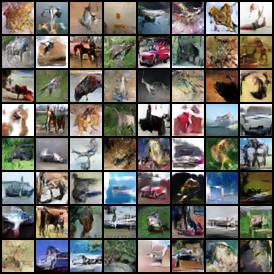}
        \subcaption{$k_{SN}$=1.0, $\mathrm{FID}$= 15.88}
    \end{subfigure}
    \begin{subfigure}{0.32\textwidth}
        \centering
        \includegraphics[width=0.99\linewidth]{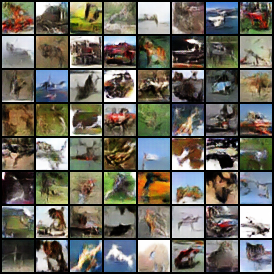}
        \subcaption{$k_{SN}$=0.5, $\mathrm{FID}$= 20.73}
    \end{subfigure}
    \begin{subfigure}{0.32\textwidth}
        \centering
        \includegraphics[width=0.99\linewidth]{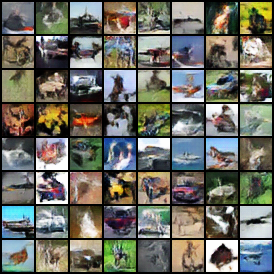}
        \subcaption{$k_{SN}$=0.25, $\mathrm{FID}$= 25.31}
    \end{subfigure}
    \begin{subfigure}{0.32\textwidth}
        \centering
        \includegraphics[width=0.99\linewidth]{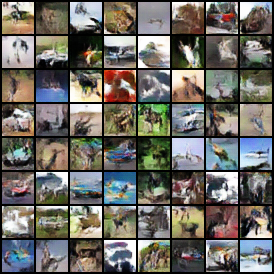}
        \subcaption{$k_{SN}$=0.2, $\mathrm{FID}$= 29.45}
    \end{subfigure}
    \begin{subfigure}{0.6\textwidth}
        \includegraphics[width=.9\linewidth]{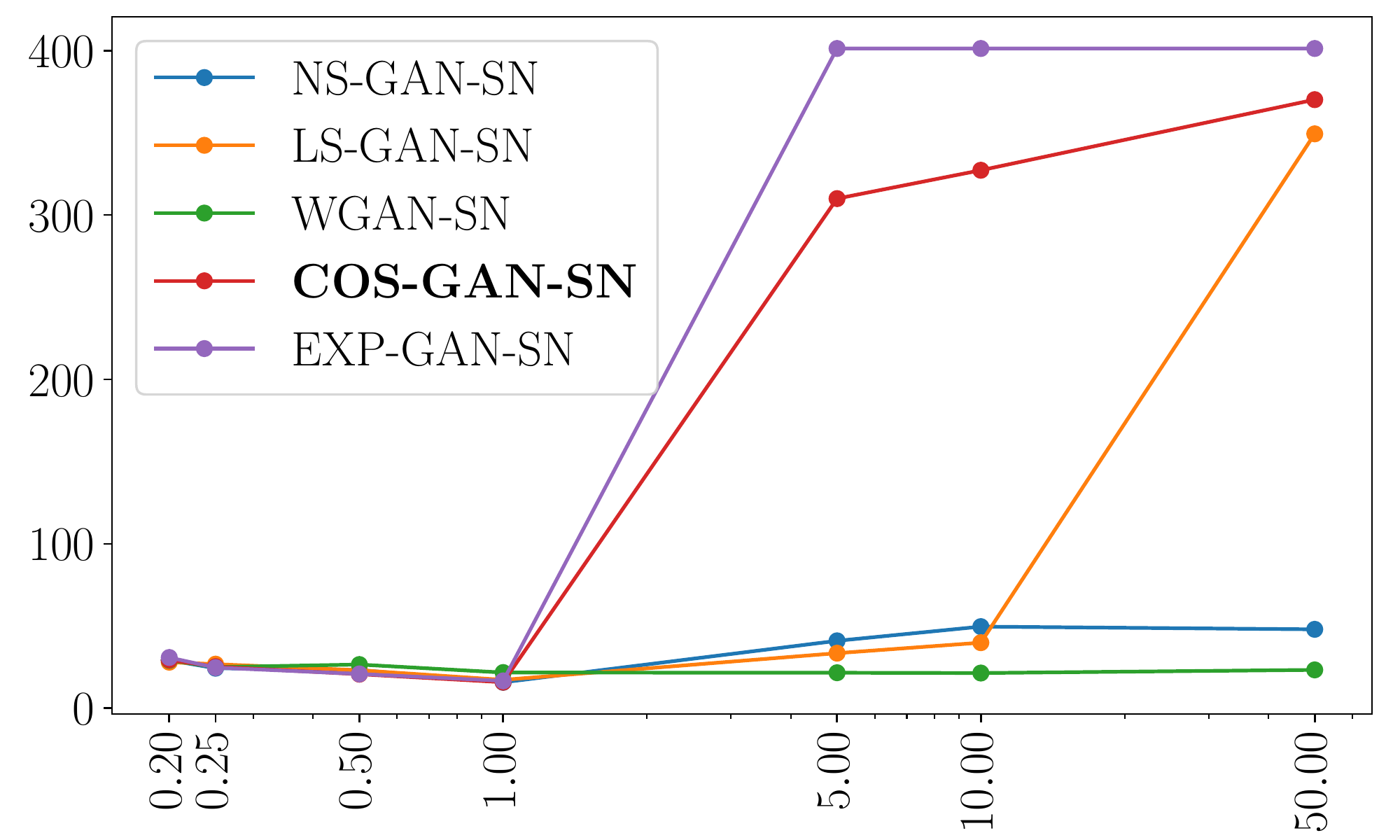}
    \end{subfigure}
\end{center}
  \caption{Samples of randomly generated images with COS-GAN-SN of varying $k_{SN}$ (CIFAR10). For the line plot, $x$-axis shows $k_{SN}$ (in log scale) and $y$-axis shows the FID scores. }
\label{fig:COSGANSN_CIFAR10}
\end{figure*}

\begin{figure*}[h]
\begin{center}
    \begin{subfigure}{0.32\textwidth}
        \centering
        \includegraphics[width=0.99\linewidth]{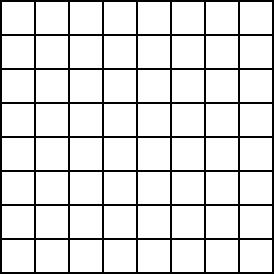}
        \subcaption{$k_{SN}$=50.0, $\mathrm{FID}$= 401.24}
    \end{subfigure}
    \begin{subfigure}{0.32\textwidth}
        \centering
        \includegraphics[width=0.99\linewidth]{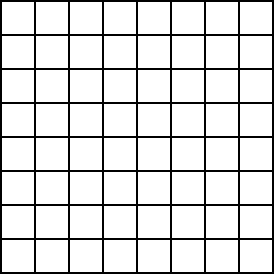}
        \subcaption{$k_{SN}$=10.0, $\mathrm{FID}$= 401.24}
    \end{subfigure}
    \begin{subfigure}{0.32\textwidth}
        \centering
        \includegraphics[width=0.99\linewidth]{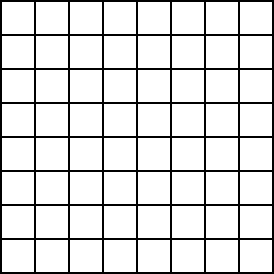}
        \subcaption{$k_{SN}$=5.0, $\mathrm{FID}$= 401.24}
    \end{subfigure}
    \begin{subfigure}{0.32\textwidth}
        \centering
        \includegraphics[width=0.99\linewidth]{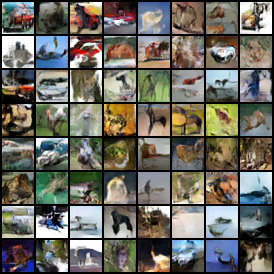}
        \subcaption{$k_{SN}$=1.0, $\mathrm{FID}$= 16.66}
    \end{subfigure}
    \begin{subfigure}{0.32\textwidth}
        \centering
        \includegraphics[width=0.99\linewidth]{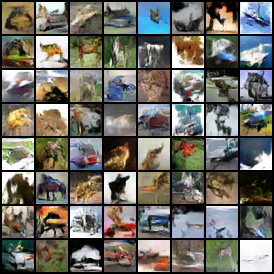}
        \subcaption{$k_{SN}$=0.5, $\mathrm{FID}$= 20.90}
    \end{subfigure}
    \begin{subfigure}{0.32\textwidth}
        \centering
        \includegraphics[width=0.99\linewidth]{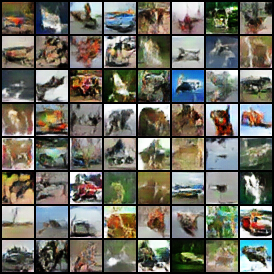}
        \subcaption{$k_{SN}$=0.25, $\mathrm{FID}$= 24.74}
    \end{subfigure}
    \begin{subfigure}{0.32\textwidth}
        \centering
        \includegraphics[width=0.99\linewidth]{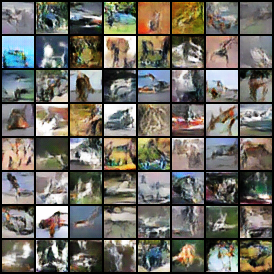}
        \subcaption{$k_{SN}$=0.2, $\mathrm{FID}$= 30.89}
    \end{subfigure}
    \begin{subfigure}{0.6\textwidth}
        \includegraphics[width=.9\linewidth]{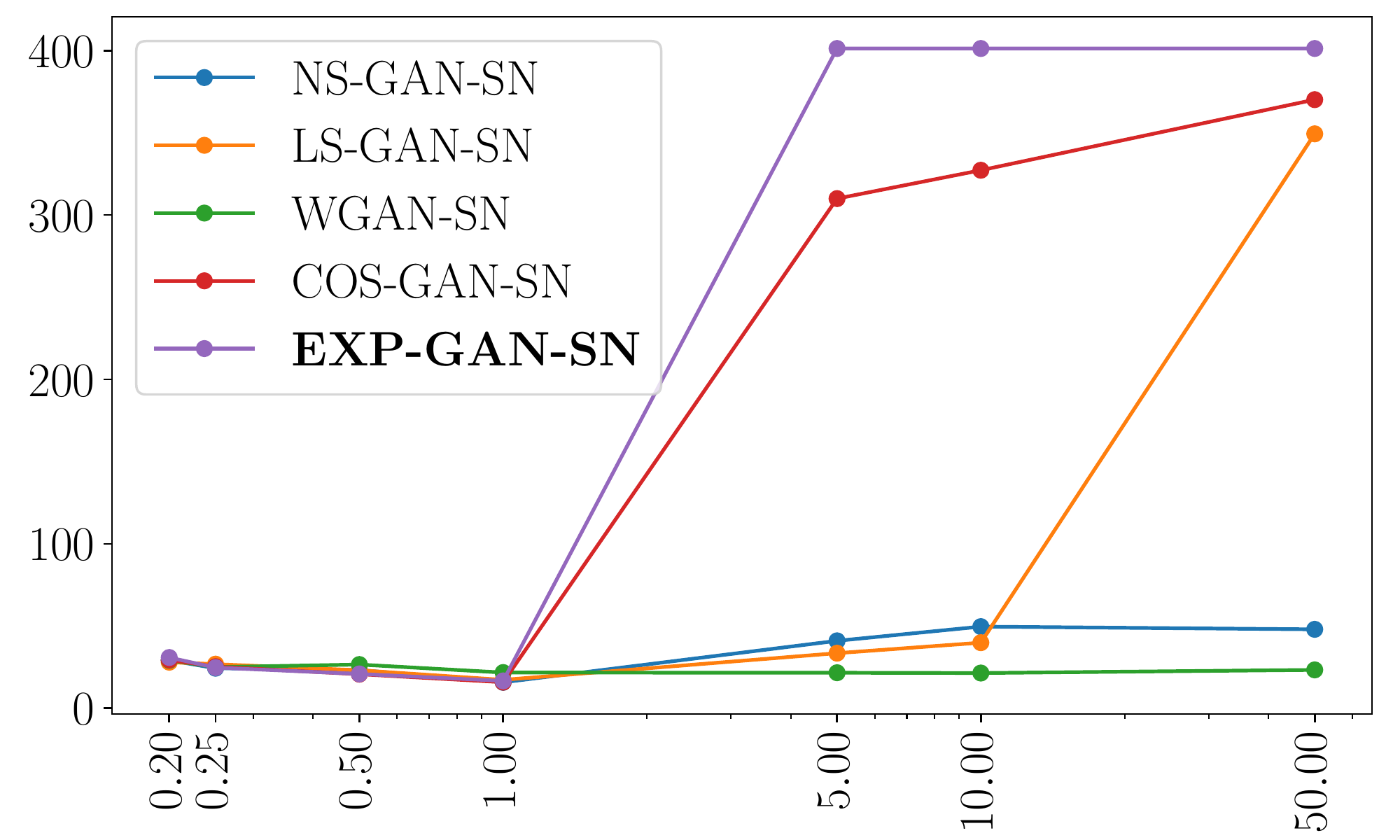}
    \end{subfigure}
\end{center}
  \caption{Samples of randomly generated images with EXP-GAN-SN of varying $k_{SN}$ (CIFAR10). For the line plot, $x$-axis shows $k_{SN}$ (in log scale) and $y$-axis shows the FID scores. }
\label{fig:EXPGANSN_CIFAR10}
\end{figure*}

\FloatBarrier
\newpage

\vspace*{\fill}
\begin{table*}[h!]
\centering
\captionsetup{justification=centering}
\caption*{{\LARGE FID scores v.s. $k_{SN}$ of Different Loss Functions\\ \vspace*{5mm}
- CelebA -}}
\end{table*}
\vspace*{\fill}

\FloatBarrier
\newpage

\begin{figure*}[h]
\begin{center}
    \begin{subfigure}{0.32\textwidth}
        \centering
        \includegraphics[width=0.99\linewidth]{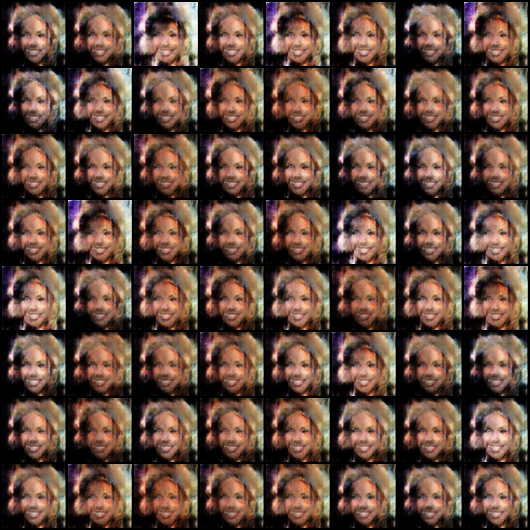}
        \subcaption{$k_{SN}$=50.0, $\mathrm{FID}$= 184.06}
    \end{subfigure}
    \begin{subfigure}{0.32\textwidth}
        \centering
        \includegraphics[width=0.99\linewidth]{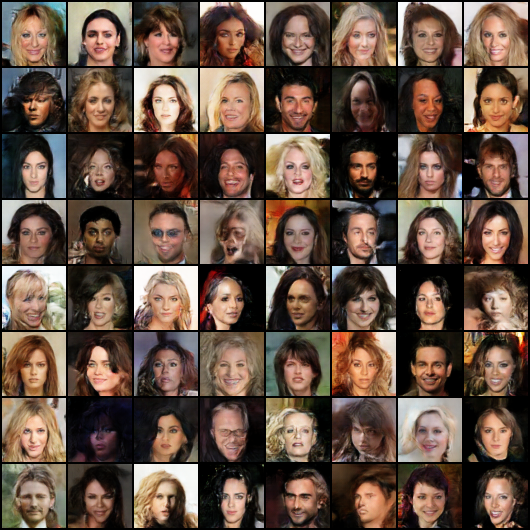}
        \subcaption{$k_{SN}$=10.0, $\mathrm{FID}$= 17.04}
    \end{subfigure}
    \begin{subfigure}{0.32\textwidth}
        \centering
        \includegraphics[width=0.99\linewidth]{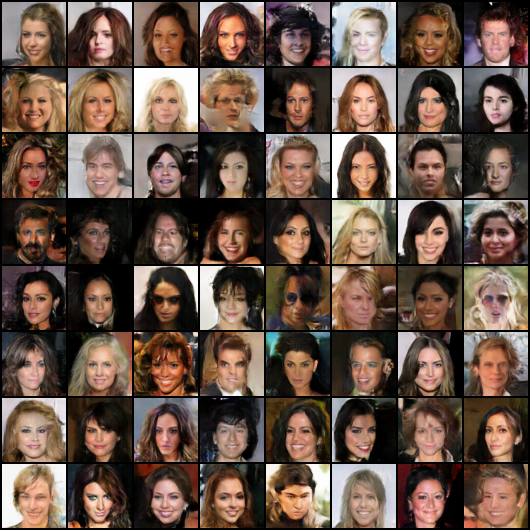}
        \subcaption{$k_{SN}$=5.0, $\mathrm{FID}$= 18.95}
    \end{subfigure}
    \begin{subfigure}{0.32\textwidth}
        \centering
        \includegraphics[width=0.99\linewidth]{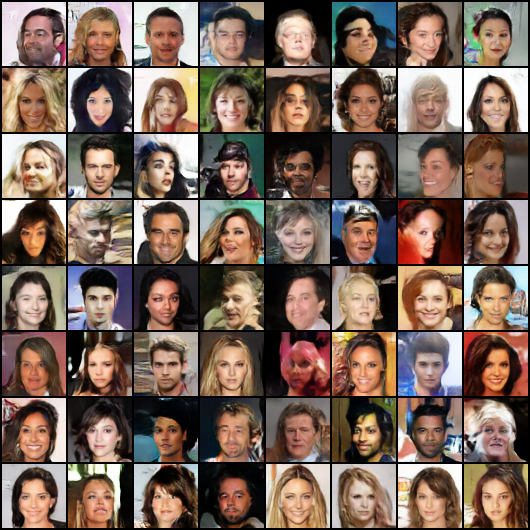}
        \subcaption{$k_{SN}$=1.0, $\mathrm{FID}$= 6.11}
    \end{subfigure}
    \begin{subfigure}{0.32\textwidth}
        \centering
        \includegraphics[width=0.99\linewidth]{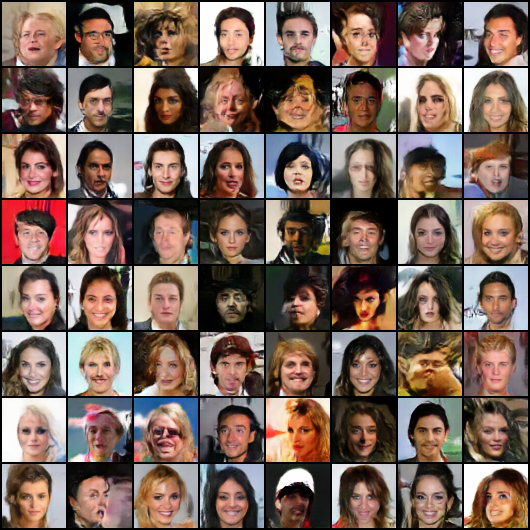}
        \subcaption{$k_{SN}$=0.5, $\mathrm{FID}$= 8.04}
    \end{subfigure}
    \begin{subfigure}{0.32\textwidth}
        \centering
        \includegraphics[width=0.99\linewidth]{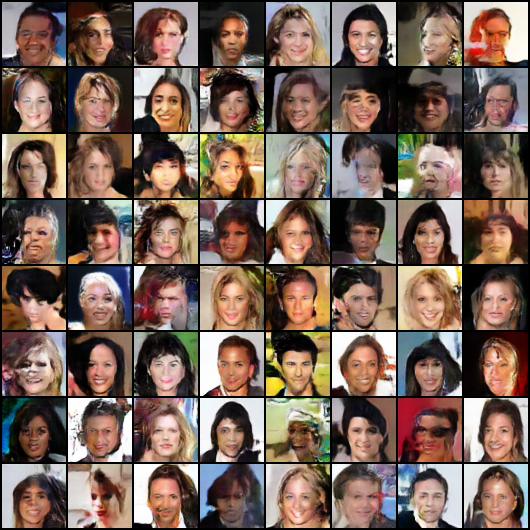}
        \subcaption{$k_{SN}$=0.25, $\mathrm{FID}$= 12.71}
    \end{subfigure}
    \begin{subfigure}{0.32\textwidth}
        \centering
        \includegraphics[width=0.99\linewidth]{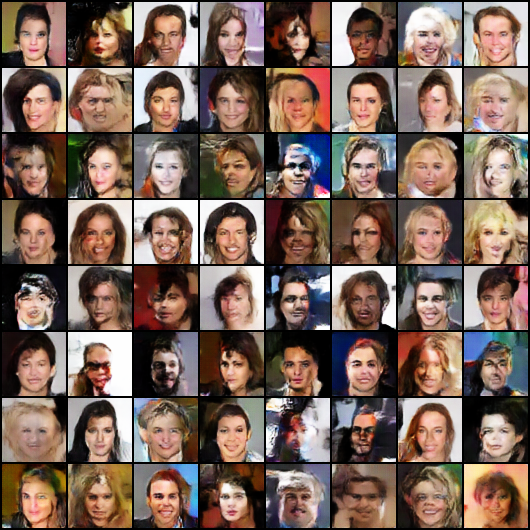}
        \subcaption{$k_{SN}$=0.2, $\mathrm{FID}$= 18.59}
    \end{subfigure}
    \begin{subfigure}{0.6\textwidth}
        \includegraphics[width=.9\linewidth]{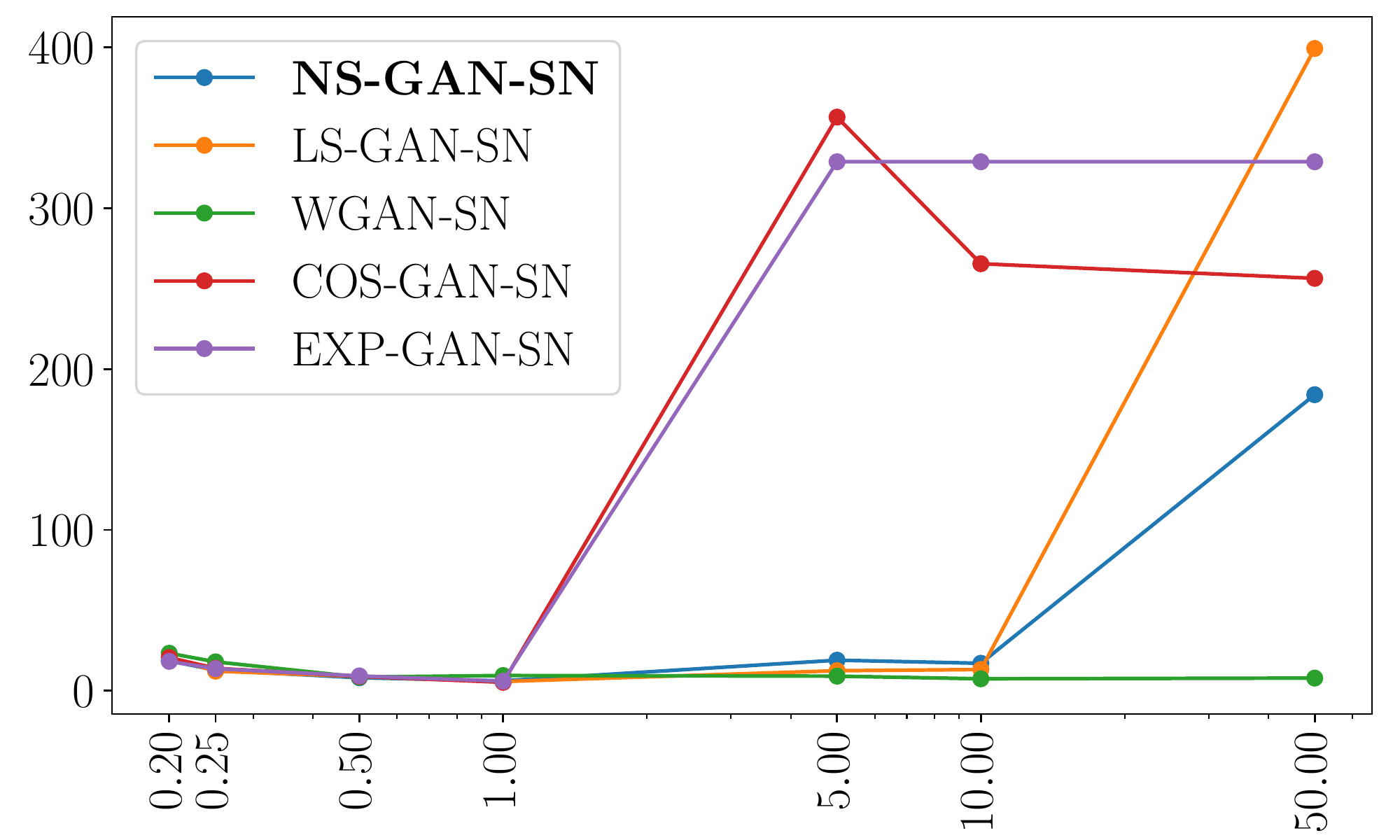}
    \end{subfigure}
\end{center}
  \caption{Samples of randomly generated images with NS-GAN-SN of varying $k_{SN}$ (CelebA). For the line plot, $x$-axis shows $k_{SN}$ (in log scale) and $y$-axis shows the FID scores. }
\label{fig:NSGANSN_CelebA}
\end{figure*}

\begin{figure*}[h]
\begin{center}
    \begin{subfigure}{0.32\textwidth}
        \centering
        \includegraphics[width=0.99\linewidth]{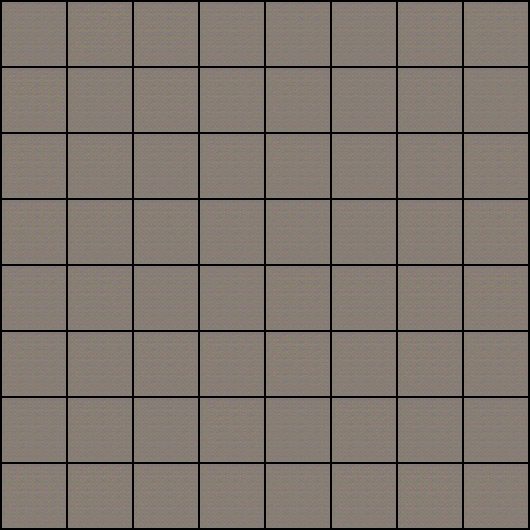}
        \subcaption{$k_{SN}$=50.0, $\mathrm{FID}$= 399.39}
    \end{subfigure}
    \begin{subfigure}{0.32\textwidth}
        \centering
        \includegraphics[width=0.99\linewidth]{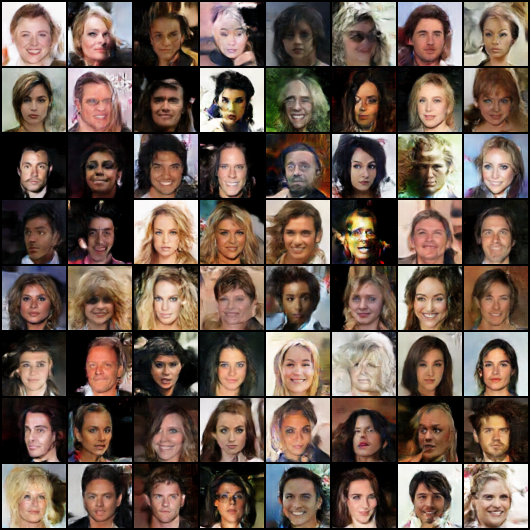}
        \subcaption{$k_{SN}$=10.0, $\mathrm{FID}$= 13.14}
    \end{subfigure}
    \begin{subfigure}{0.32\textwidth}
        \centering
        \includegraphics[width=0.99\linewidth]{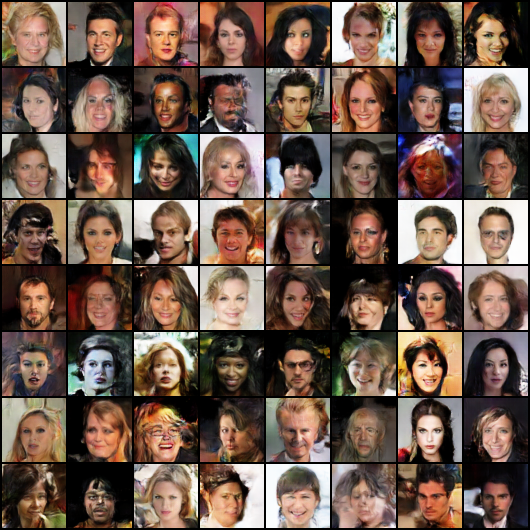}
        \subcaption{$k_{SN}$=5.0, $\mathrm{FID}$= 12.40}
    \end{subfigure}
    \begin{subfigure}{0.32\textwidth}
        \centering
        \includegraphics[width=0.99\linewidth]{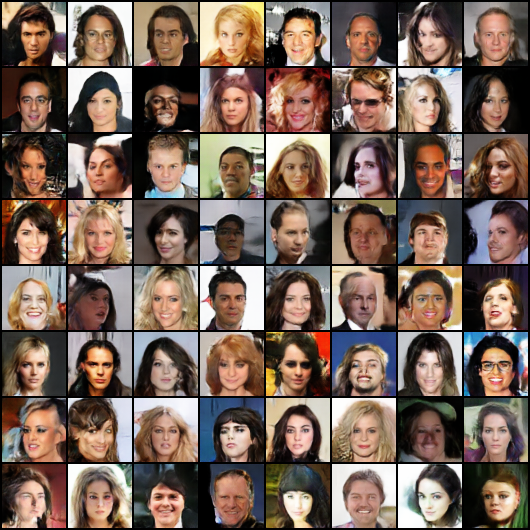}
        \subcaption{$k_{SN}$=1.0, $\mathrm{FID}$= 5.69}
    \end{subfigure}
    \begin{subfigure}{0.32\textwidth}
        \centering
        \includegraphics[width=0.99\linewidth]{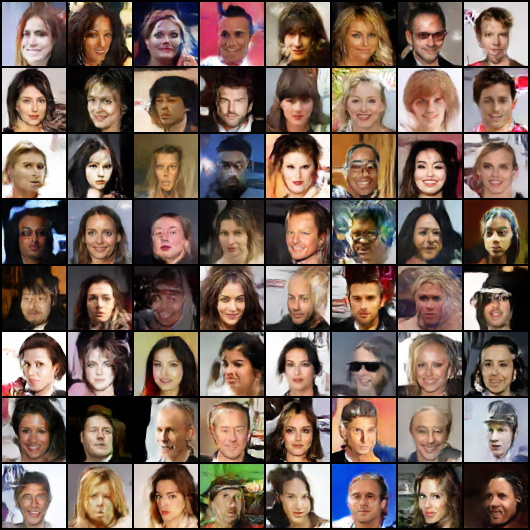}
        \subcaption{$k_{SN}$=0.5, $\mathrm{FID}$= 8.85}
    \end{subfigure}
    \begin{subfigure}{0.32\textwidth}
        \centering
        \includegraphics[width=0.99\linewidth]{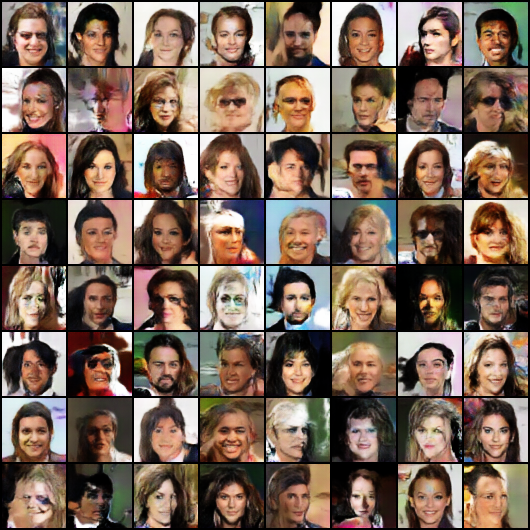}
        \subcaption{$k_{SN}$=0.25, $\mathrm{FID}$= 12.14}
    \end{subfigure}
    \begin{subfigure}{0.32\textwidth}
        \centering
        \includegraphics[width=0.99\linewidth]{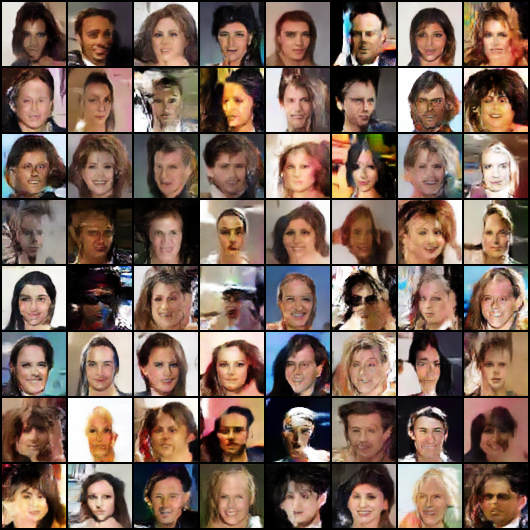}
        \subcaption{$k_{SN}$=0.2, $\mathrm{FID}$= 20.34}
    \end{subfigure}
    \begin{subfigure}{0.6\textwidth}
        \includegraphics[width=.9\linewidth]{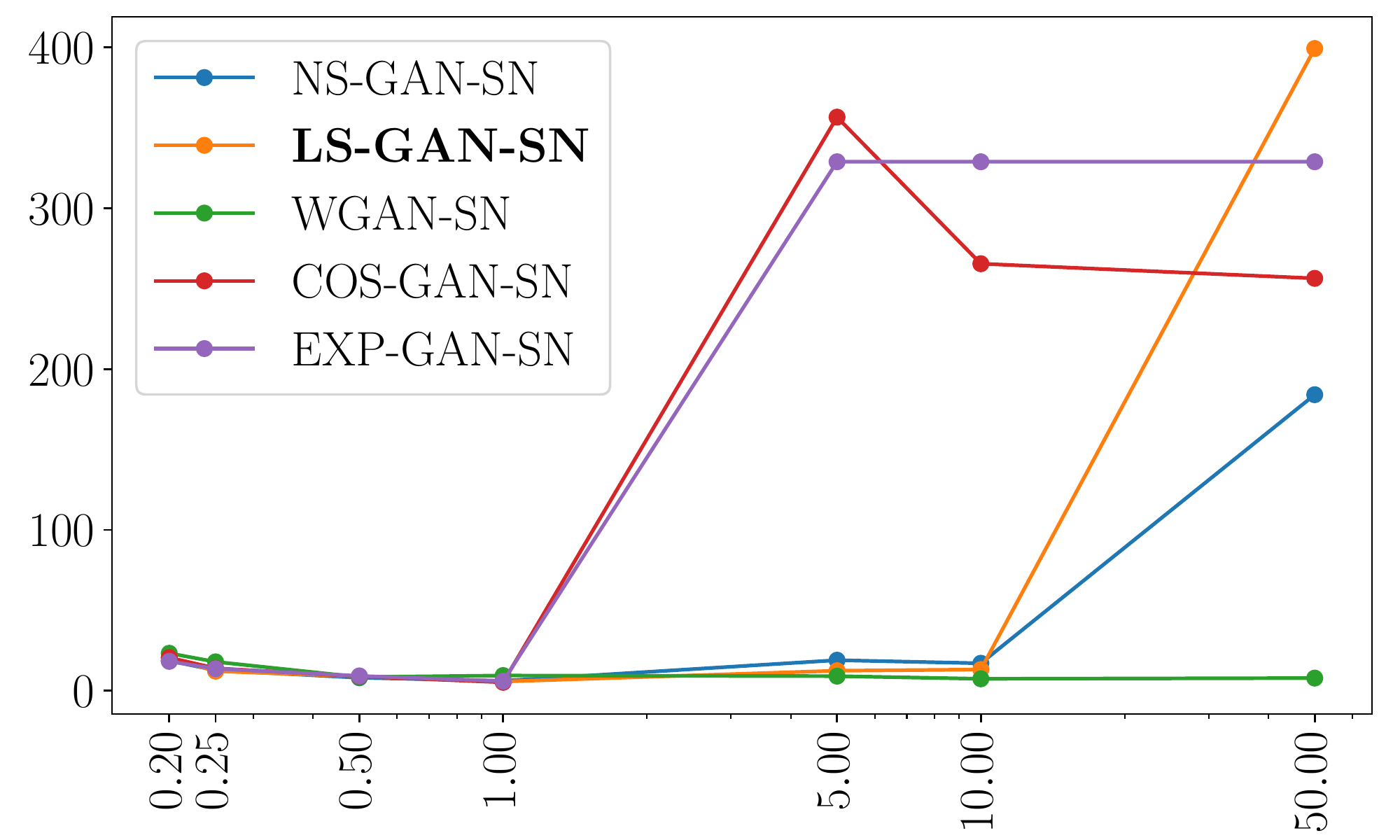}
    \end{subfigure}
\end{center}
  \caption{Samples of randomly generated images with LS-GAN-SN of varying $k_{SN}$ (CelebA). For the line plot, $x$-axis shows $k_{SN}$ (in log scale) and $y$-axis shows the FID scores. }
\label{fig:LSGANSN_CelebA}
\end{figure*}

\begin{figure*}[h]
\begin{center}
    \begin{subfigure}{0.32\textwidth}
        \centering
        \includegraphics[width=0.99\linewidth]{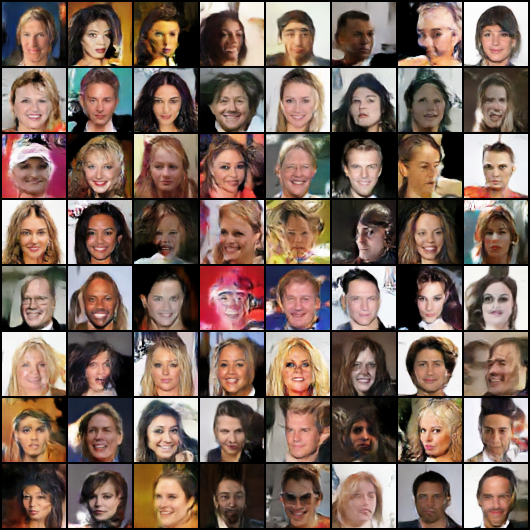}
        \subcaption{$k_{SN}$=50.0, $\mathrm{FID}$= 7.82}
    \end{subfigure}
    \begin{subfigure}{0.32\textwidth}
        \centering
        \includegraphics[width=0.99\linewidth]{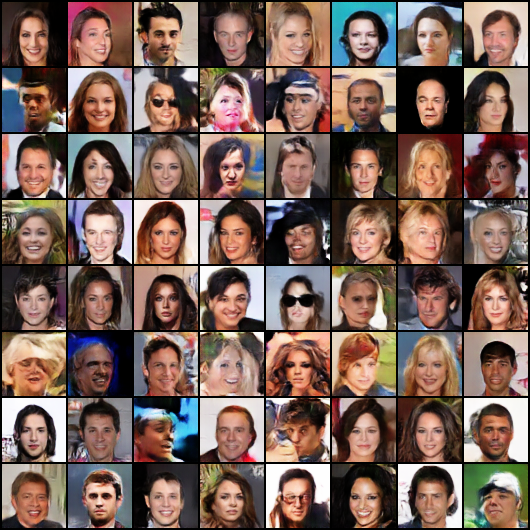}
        \subcaption{$k_{SN}$=10.0, $\mathrm{FID}$= 7.37}
    \end{subfigure}
    \begin{subfigure}{0.32\textwidth}
        \centering
        \includegraphics[width=0.99\linewidth]{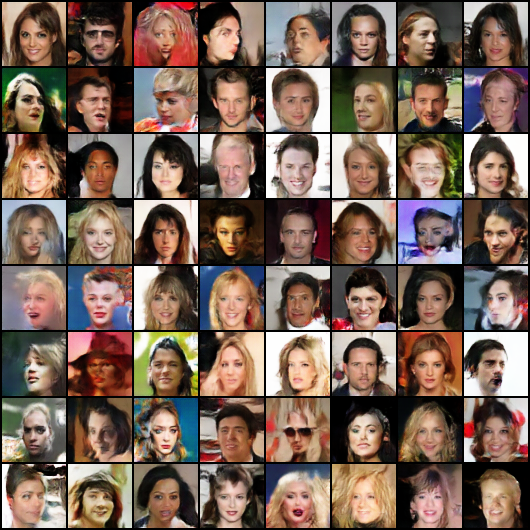}
        \subcaption{$k_{SN}$=5.0, $\mathrm{FID}$= 9.03}
    \end{subfigure}
    \begin{subfigure}{0.32\textwidth}
        \centering
        \includegraphics[width=0.99\linewidth]{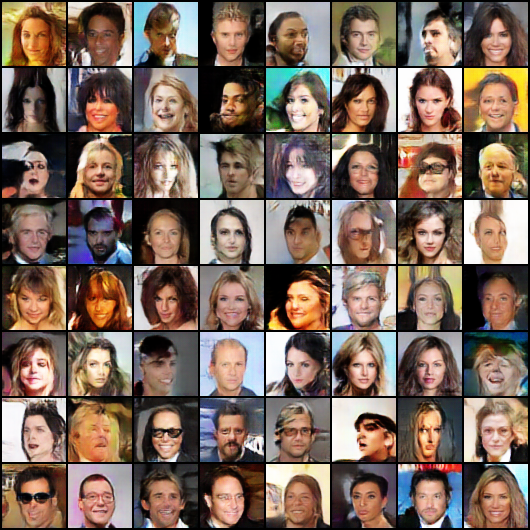}
        \subcaption{$k_{SN}$=1.0, $\mathrm{FID}$= 9.41}
    \end{subfigure}
    \begin{subfigure}{0.32\textwidth}
        \centering
        \includegraphics[width=0.99\linewidth]{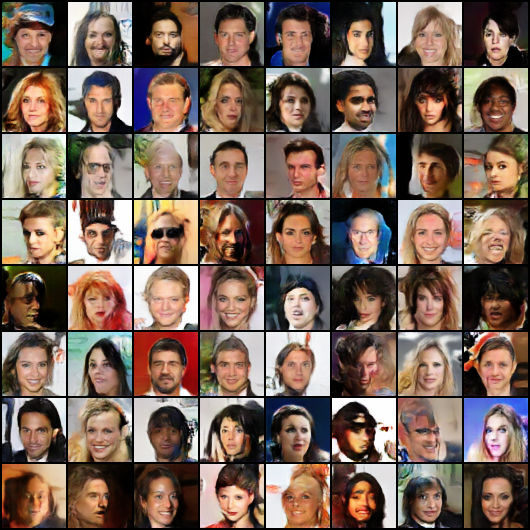}
        \subcaption{$k_{SN}$=0.5, $\mathrm{FID}$= 8.48}
    \end{subfigure}
    \begin{subfigure}{0.32\textwidth}
        \centering
        \includegraphics[width=0.99\linewidth]{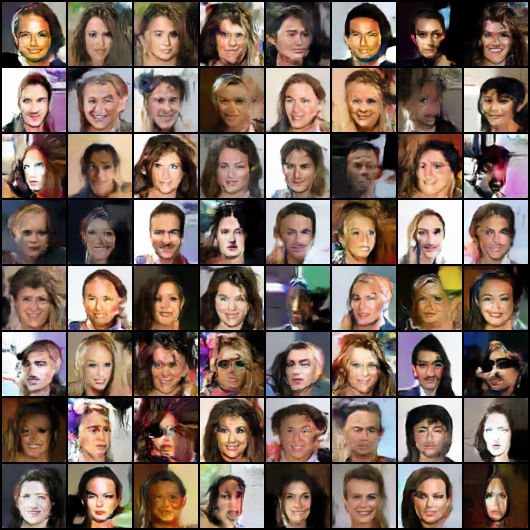}
        \subcaption{$k_{SN}$=0.25, $\mathrm{FID}$= 17.93}
    \end{subfigure}
    \begin{subfigure}{0.32\textwidth}
        \centering
        \includegraphics[width=0.99\linewidth]{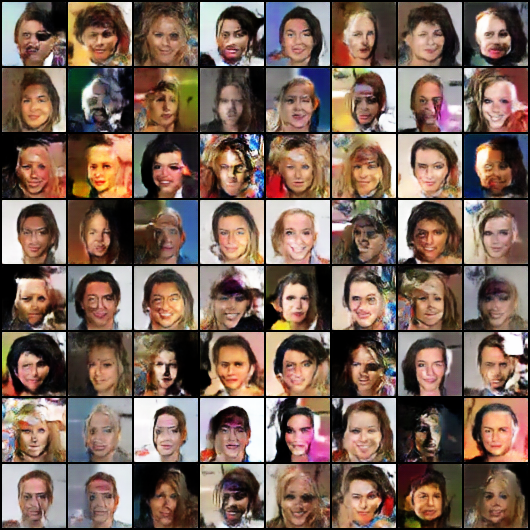}
        \subcaption{$k_{SN}$=0.2, $\mathrm{FID}$= 23.26}
    \end{subfigure}
    \begin{subfigure}{0.6\textwidth}
        \includegraphics[width=.9\linewidth]{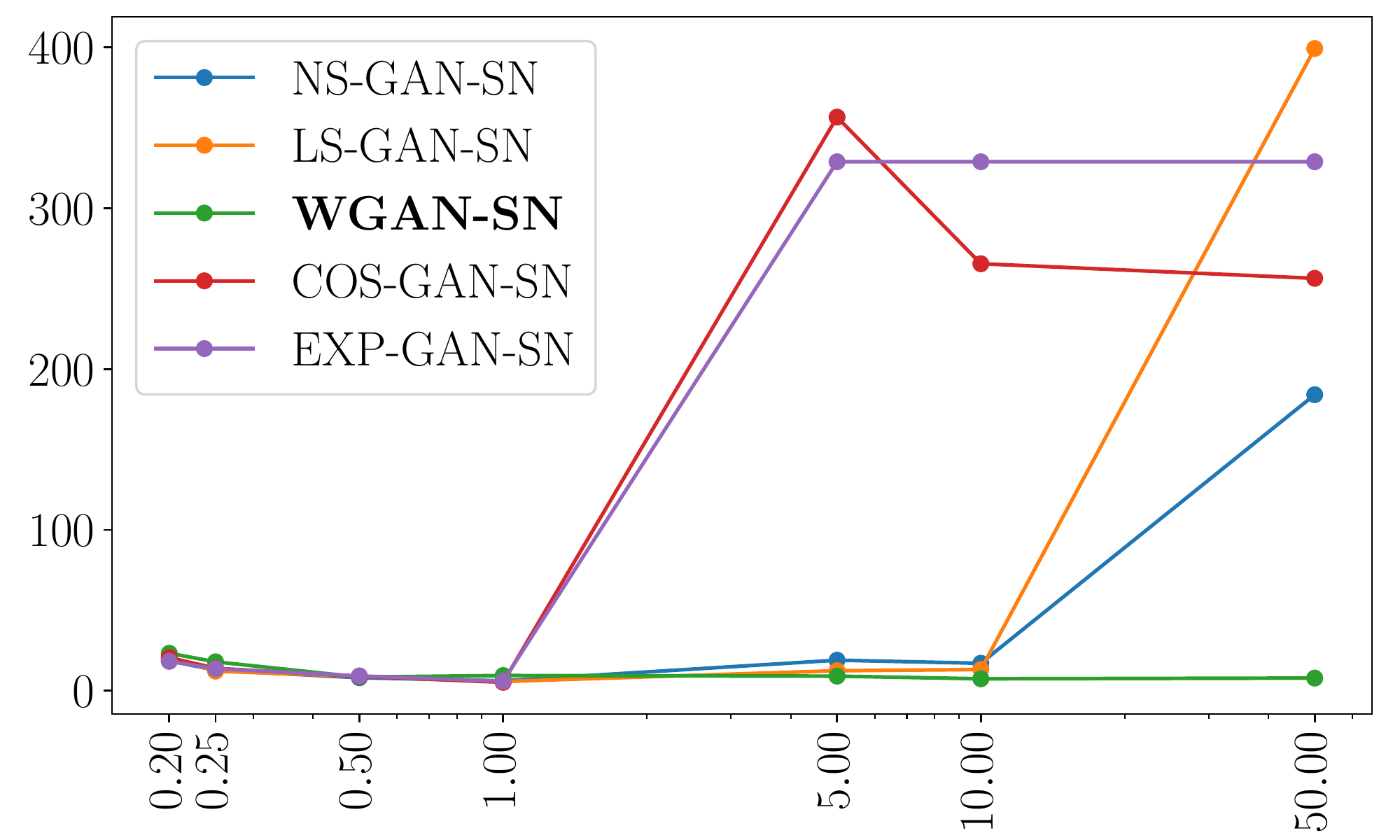}
    \end{subfigure}
\end{center}
  \caption{Samples of randomly generated images with WGAN-SN of varying $k_{SN}$ (CelebA). For the line plot, $x$-axis shows $k_{SN}$ (in log scale) and $y$-axis shows the FID scores. }
\label{fig:WGANSN_CelebA}
\end{figure*}

\begin{figure*}[h]
\begin{center}
    \begin{subfigure}{0.32\textwidth}
        \centering
        \includegraphics[width=0.99\linewidth]{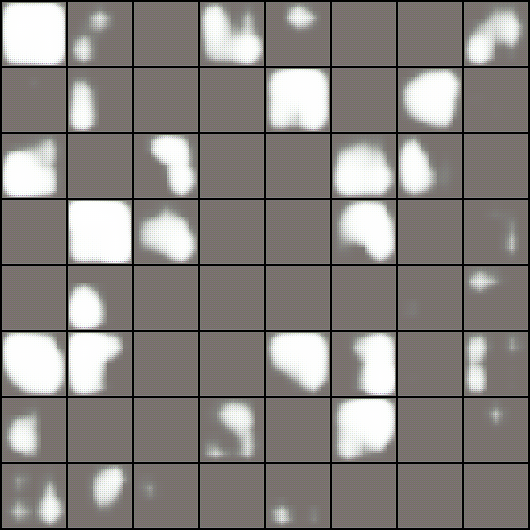}
        \subcaption{$k_{SN}$=50.0, $\mathrm{FID}$= 256.44}
    \end{subfigure}
    \begin{subfigure}{0.32\textwidth}
        \centering
        \includegraphics[width=0.99\linewidth]{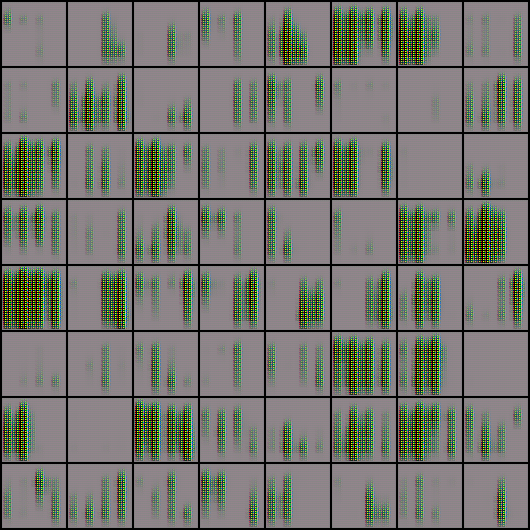}
        \subcaption{$k_{SN}$=10.0, $\mathrm{FID}$= 265.53}
    \end{subfigure}
    \begin{subfigure}{0.32\textwidth}
        \centering
        \includegraphics[width=0.99\linewidth]{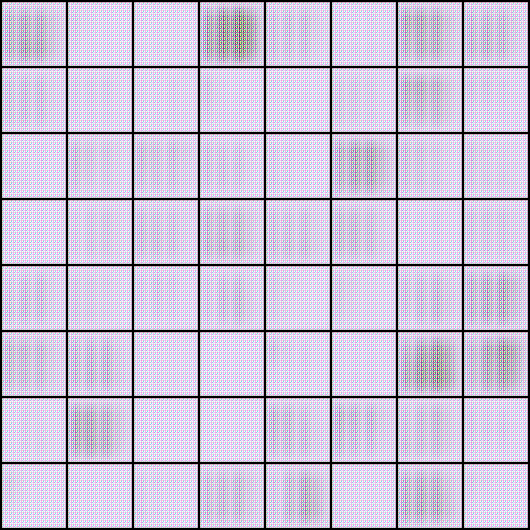}
        \subcaption{$k_{SN}$=5.0, $\mathrm{FID}$= 356.70}
    \end{subfigure}
    \begin{subfigure}{0.32\textwidth}
        \centering
        \includegraphics[width=0.99\linewidth]{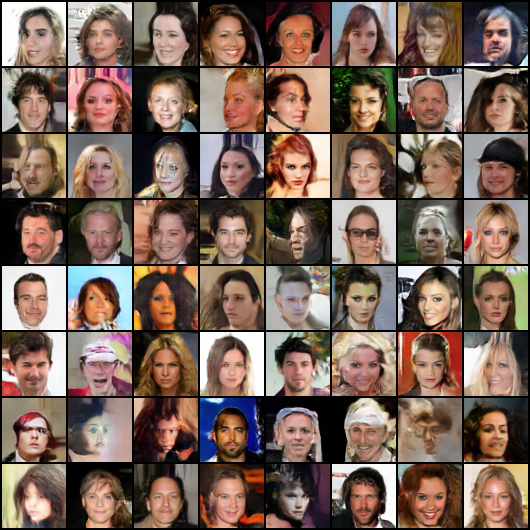}
        \subcaption{$k_{SN}$=1.0, $\mathrm{FID}$= 5.20}
    \end{subfigure}
    \begin{subfigure}{0.32\textwidth}
        \centering
        \includegraphics[width=0.99\linewidth]{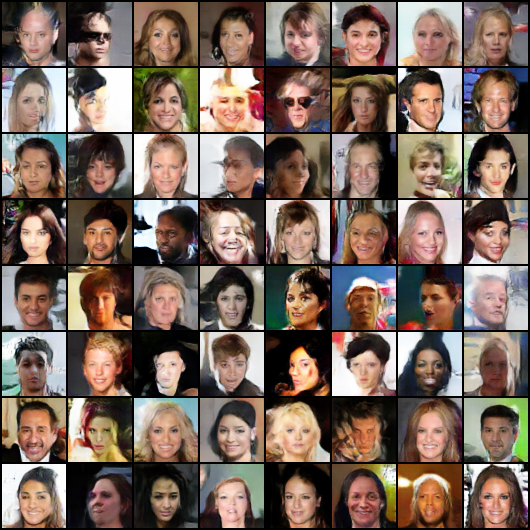}
        \subcaption{$k_{SN}$=0.5, $\mathrm{FID}$= 8.88}
    \end{subfigure}
    \begin{subfigure}{0.32\textwidth}
        \centering
        \includegraphics[width=0.99\linewidth]{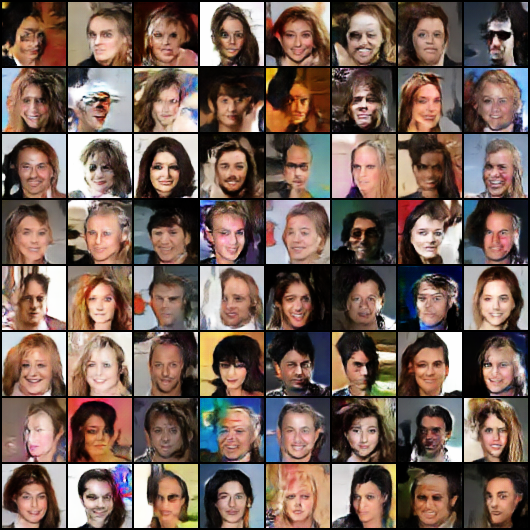}
        \subcaption{$k_{SN}$=0.25, $\mathrm{FID}$= 13.93}
    \end{subfigure}
    \begin{subfigure}{0.32\textwidth}
        \centering
        \includegraphics[width=0.99\linewidth]{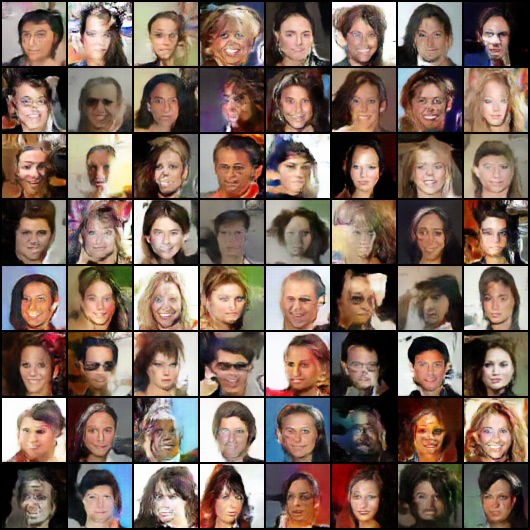}
        \subcaption{$k_{SN}$=0.2, $\mathrm{FID}$= 20.59}
    \end{subfigure}
    \begin{subfigure}{0.6\textwidth}
        \includegraphics[width=.9\linewidth]{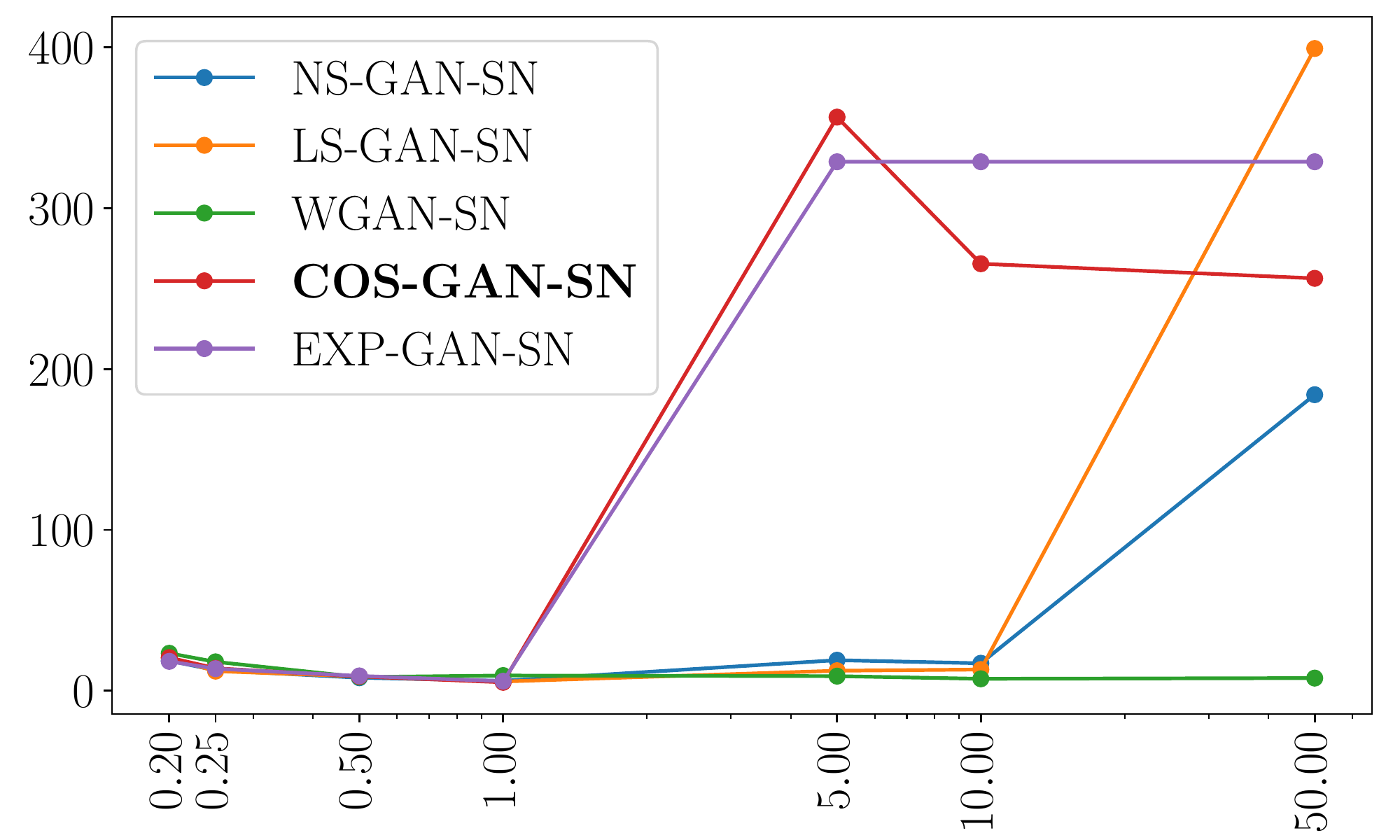}
    \end{subfigure}
\end{center}
  \caption{Samples of randomly generated images with COS-GAN-SN of varying $k_{SN}$ (CelebA). For the line plot, $x$-axis shows $k_{SN}$ (in log scale) and $y$-axis shows the FID scores. }
\label{fig:COSGANSN_CelebA}
\end{figure*}

\begin{figure*}[h]
\begin{center}
    \begin{subfigure}{0.32\textwidth}
        \centering
        \includegraphics[width=0.99\linewidth]{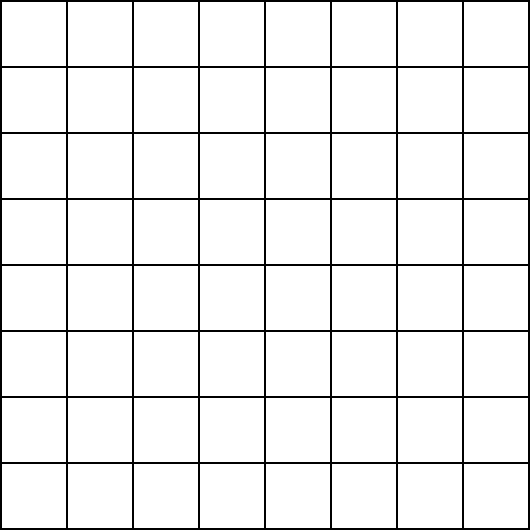}
        \subcaption{$k_{SN}$=50.0, $\mathrm{FID}$= 328.94}
    \end{subfigure}
    \begin{subfigure}{0.32\textwidth}
        \centering
        \includegraphics[width=0.99\linewidth]{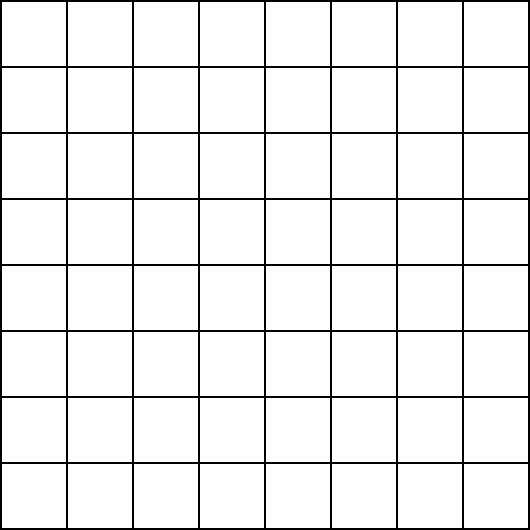}
        \subcaption{$k_{SN}$=10.0, $\mathrm{FID}$= 328.94}
    \end{subfigure}
    \begin{subfigure}{0.32\textwidth}
        \centering
        \includegraphics[width=0.99\linewidth]{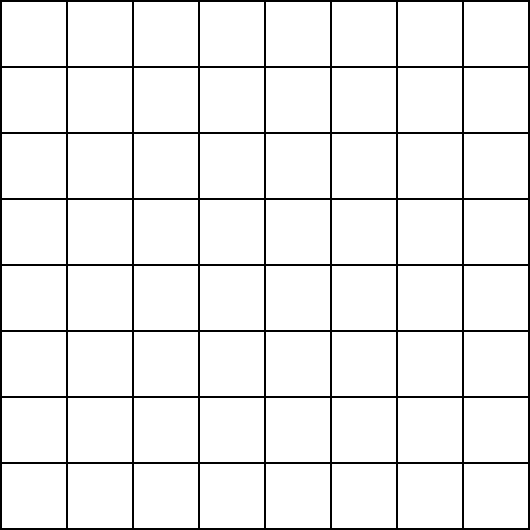}
        \subcaption{$k_{SN}$=5.0, $\mathrm{FID}$= 328.94}
    \end{subfigure}
    \begin{subfigure}{0.32\textwidth}
        \centering
        \includegraphics[width=0.99\linewidth]{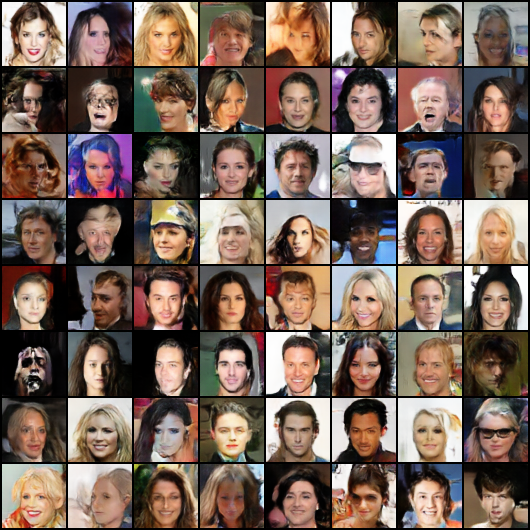}
        \subcaption{$k_{SN}$=1.0, $\mathrm{FID}$= 5.88}
    \end{subfigure}
    \begin{subfigure}{0.32\textwidth}
        \centering
        \includegraphics[width=0.99\linewidth]{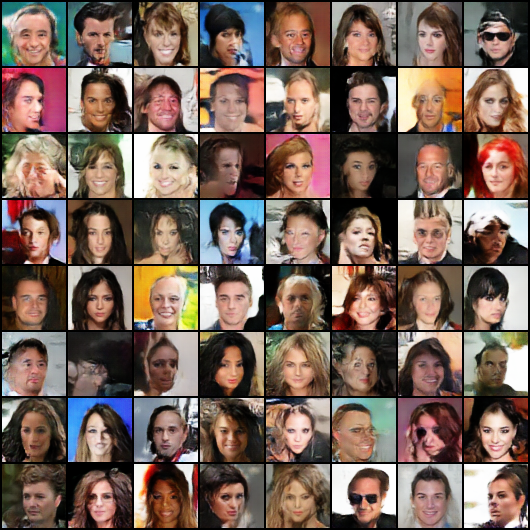}
        \subcaption{$k_{SN}$=0.5, $\mathrm{FID}$= 9.18}
    \end{subfigure}
    \begin{subfigure}{0.32\textwidth}
        \centering
        \includegraphics[width=0.99\linewidth]{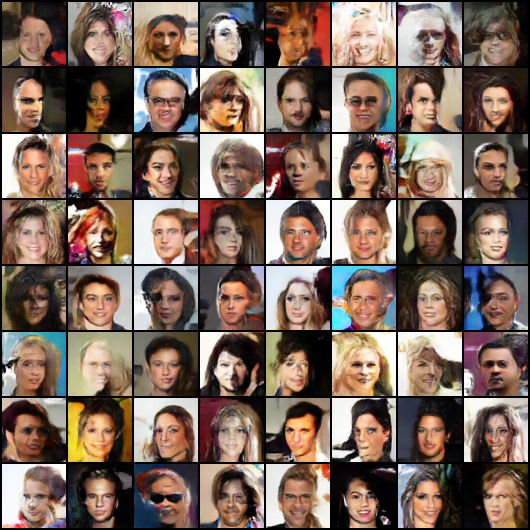}
        \subcaption{$k_{SN}$=0.25, $\mathrm{FID}$= 13.65}
    \end{subfigure}
    \begin{subfigure}{0.32\textwidth}
        \centering
        \includegraphics[width=0.99\linewidth]{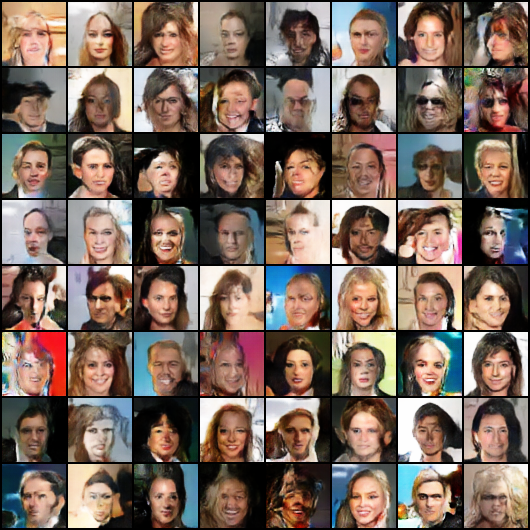}
        \subcaption{$k_{SN}$=0.2, $\mathrm{FID}$= 18.23}
    \end{subfigure}
    \begin{subfigure}{0.6\textwidth}
        \includegraphics[width=.9\linewidth]{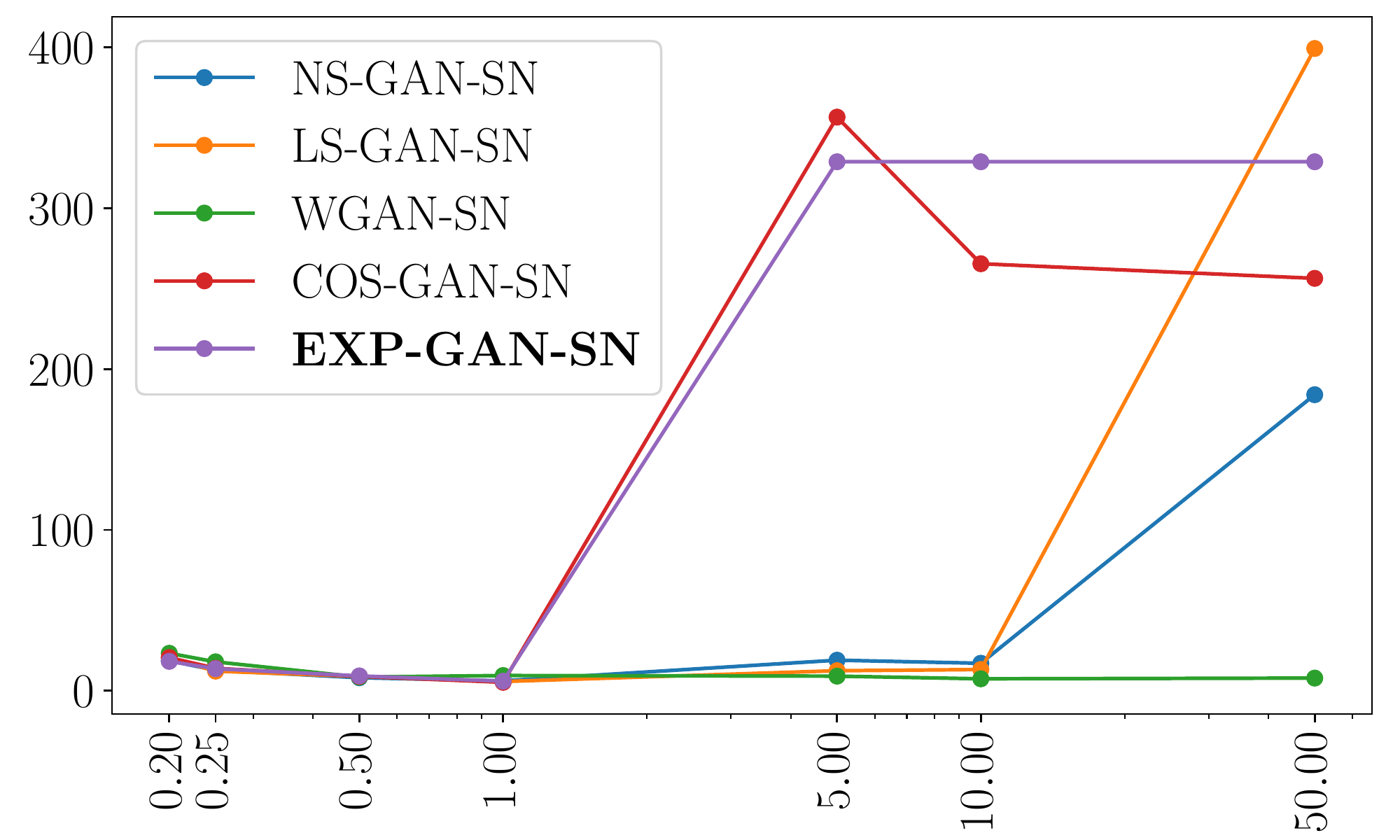}
    \end{subfigure}
\end{center}
  \caption{Samples of randomly generated images with EXP-GAN-SN of varying $k_{SN}$ (CelebA). For the line plot, $x$-axis shows $k_{SN}$ (in log scale) and $y$-axis shows the FID scores. }
\label{fig:EXPGANSN_CelebA}
\end{figure*}

\FloatBarrier
\newpage

\vspace*{\fill}
\begin{table*}[h!]
\centering
\captionsetup{justification=centering}
\caption*{{\LARGE FID scores v.s. $\alpha$ ($k_{SN}=50.0$) of Different Loss Functions\\ \vspace*{5mm}
- MNIST -}}
\end{table*}
\vspace*{\fill}

\FloatBarrier
\newpage

\begin{figure*}[h]
\begin{center}
    \begin{subfigure}{0.32\textwidth}
        \centering
        \includegraphics[width=0.99\linewidth]{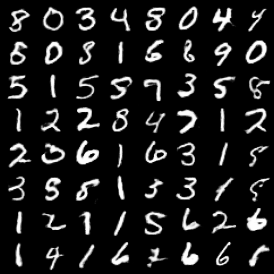}
        \subcaption{$\alpha=1e^{-11}$, $\mathrm{FID}=3.67$}
    \end{subfigure}
    \begin{subfigure}{0.32\textwidth}
        \centering
        \includegraphics[width=0.99\linewidth]{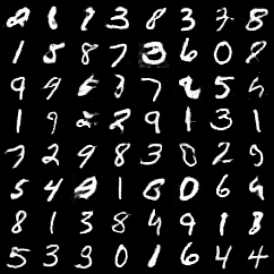}
        \subcaption{$\alpha=1e^{-9}$, $\mathrm{FID}=3.61$}
    \end{subfigure}
    \begin{subfigure}{0.32\textwidth}
        \centering
        \includegraphics[width=0.99\linewidth]{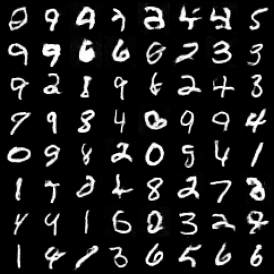}
        \subcaption{$\alpha=1e^{-7}$, $\mathrm{FID}=3.81$}
    \end{subfigure}
    \begin{subfigure}{0.32\textwidth}
        \centering
        \includegraphics[width=0.99\linewidth]{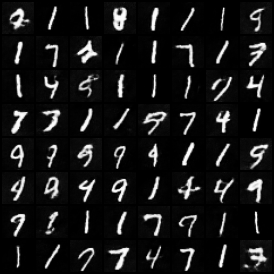}
        \subcaption{$\alpha=1e^{-5}$, $\mathrm{FID}=33.48$}
    \end{subfigure}
    \begin{subfigure}{0.32\textwidth}
        \centering
        \includegraphics[width=0.99\linewidth]{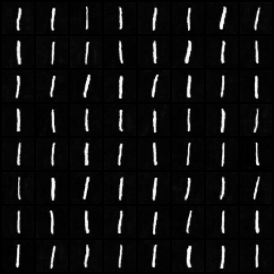}
        \subcaption{$\alpha=1e^{-3}$, $\mathrm{FID}=154.14$}
    \end{subfigure}
    \begin{subfigure}{0.32\textwidth}
        \centering
        \includegraphics[width=0.99\linewidth]{Scale_NSGANSN_MNIST_k_50/a_e1.png}
        \subcaption{$\alpha=1e^{-1}$, $\mathrm{FID}=155.54$}
    \end{subfigure}
    \begin{subfigure}{0.6\textwidth}
        \includegraphics[width=.9\linewidth]{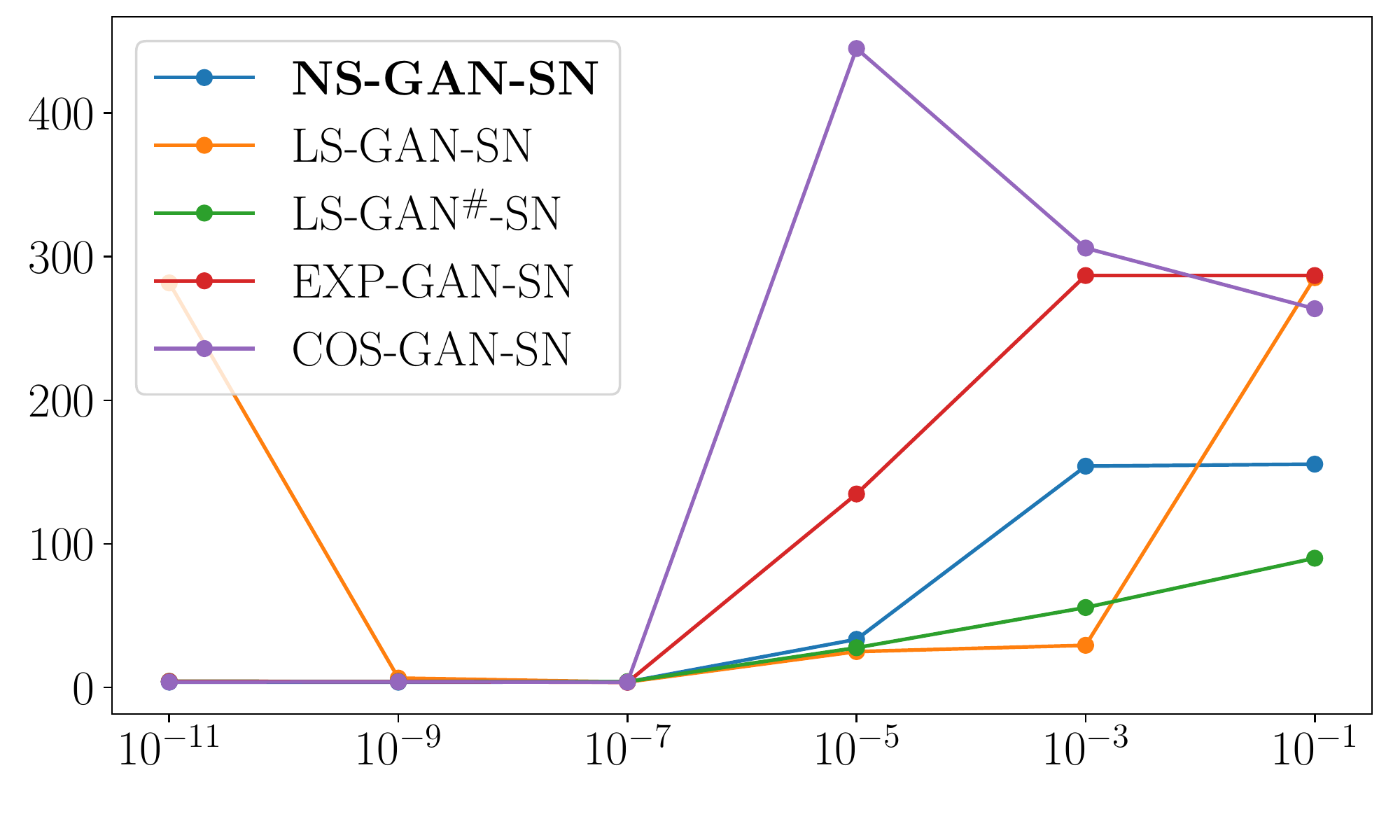}
    \end{subfigure}
\end{center}
  \caption{Samples of randomly generated images with NS-GAN-SN of varying $\alpha$ ($k_{SN}=50.0$, MNIST). For the line plot, $x$-axis shows $\alpha$ (in log scale) and $y$-axis shows the FID scores.}
\label{fig:Scale_NSGANSN_MNIST_k_50}
\end{figure*}

\begin{figure*}[h]
\begin{center}
    \begin{subfigure}{0.32\textwidth}
        \centering
        \includegraphics[width=0.99\linewidth]{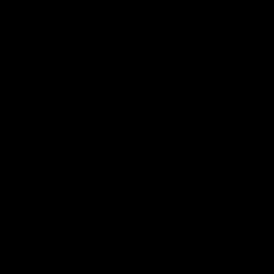}
        \subcaption{$\alpha=1e^{-11}$, $\mathrm{FID}$=281.82}
    \end{subfigure}
    \begin{subfigure}{0.32\textwidth}
        \centering
        \includegraphics[width=0.99\linewidth]{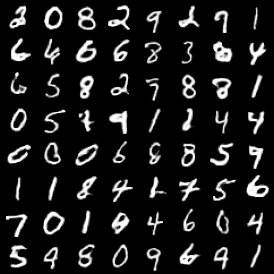}
        \subcaption{$\alpha=1e^{-9}$, $\mathrm{FID}=6.48$}
    \end{subfigure}
    \begin{subfigure}{0.32\textwidth}
        \centering
        \includegraphics[width=0.99\linewidth]{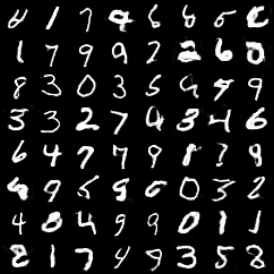}
        \subcaption{$\alpha=1e^{-7}$, $\mathrm{FID}=3.74$}
    \end{subfigure}
    \begin{subfigure}{0.32\textwidth}
        \centering
        \includegraphics[width=0.99\linewidth]{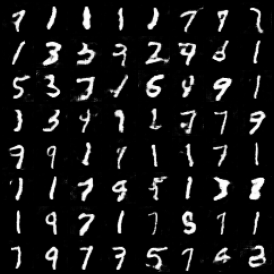}
        \subcaption{$\alpha=1e^{-5}$, $\mathrm{FID}=24.93$}
    \end{subfigure}
    \begin{subfigure}{0.32\textwidth}
        \centering
        \includegraphics[width=0.99\linewidth]{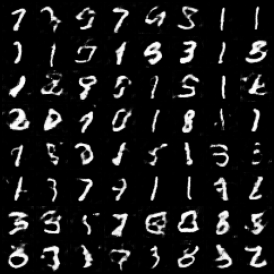}
        \subcaption{$\alpha=1e^{-3}$, $\mathrm{FID}=29.31$}
    \end{subfigure}
    \begin{subfigure}{0.32\textwidth}
        \centering
        \includegraphics[width=0.99\linewidth]{Scale_LSGANSN_MNIST_k_50/a_e1.png}
        \subcaption{$\alpha=1e^{-1}$, $\mathrm{FID}=285.46$}
    \end{subfigure}
    \begin{subfigure}{0.6\textwidth}
        \includegraphics[width=.9\linewidth]{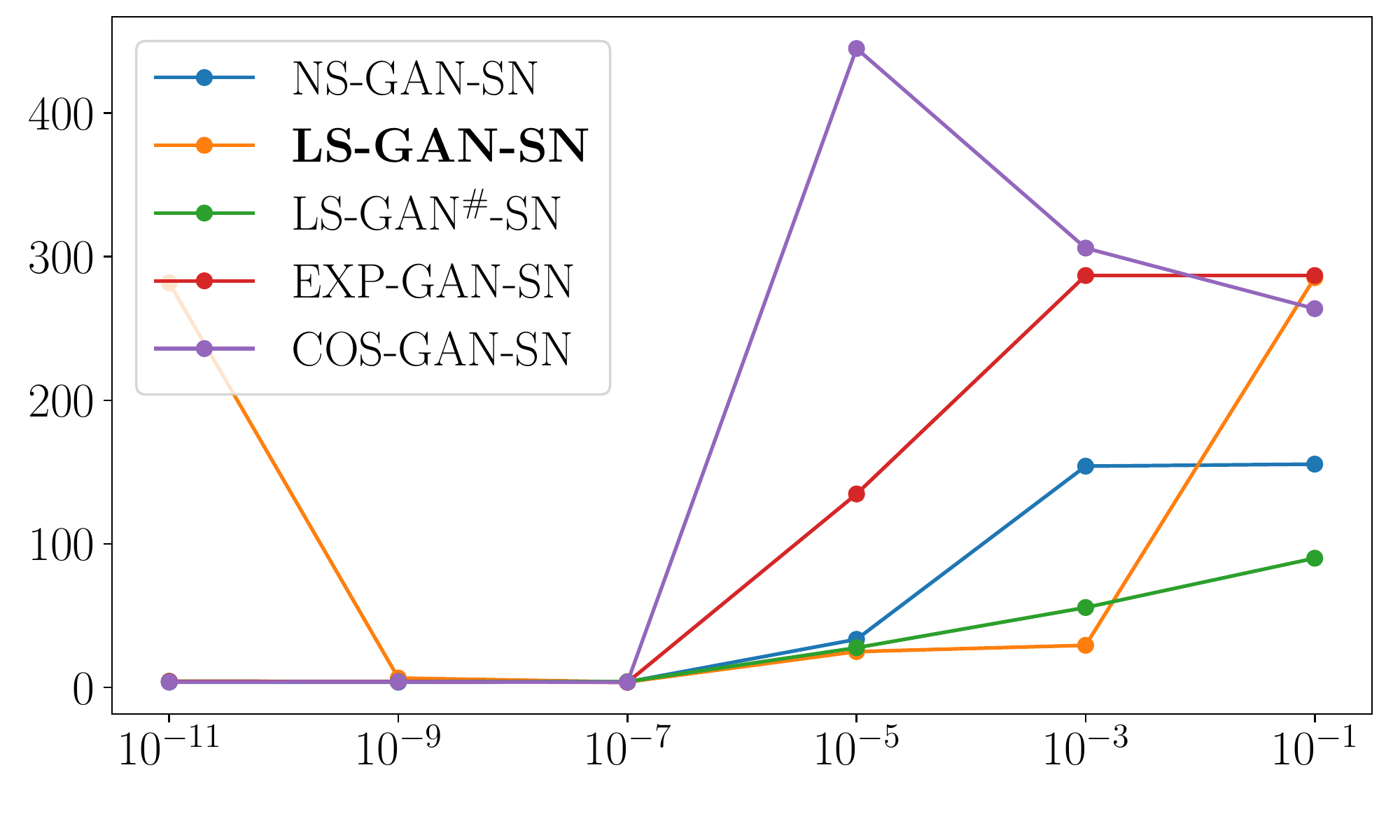}
    \end{subfigure}
\end{center}
  \caption{Samples of randomly generated images with LS-GAN-SN of varying $\alpha$ ($k_{SN}=50.0$, MNIST). For the line plot, $x$-axis shows $\alpha$ (in log scale) and $y$-axis shows the FID scores.}
\label{fig:Scale_LSGANSN_MNIST_k_50}
\end{figure*}

\begin{figure*}[h]
\begin{center}
    \begin{subfigure}{0.32\textwidth}
        \centering
        \includegraphics[width=0.99\linewidth]{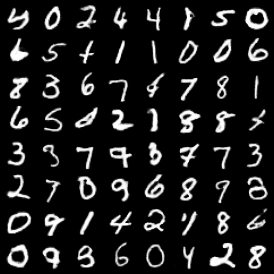}
        \subcaption{$\alpha=1e^{-11}$, $\mathrm{FID}=4.25$}
    \end{subfigure}
    \begin{subfigure}{0.32\textwidth}
        \centering
        \includegraphics[width=0.99\linewidth]{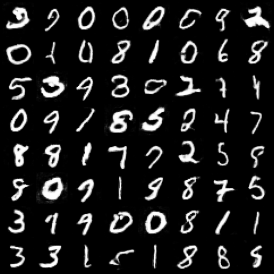}
        \subcaption{$\alpha=1e^{-9}$, $\mathrm{FID}=4.15$}
    \end{subfigure}
    \begin{subfigure}{0.32\textwidth}
        \centering
        \includegraphics[width=0.99\linewidth]{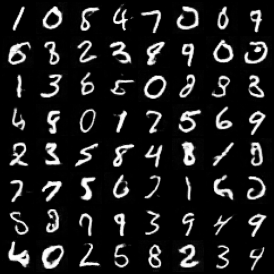}
        \subcaption{$\alpha=1e^{-7}$, $\mathrm{FID}=3.99$}
    \end{subfigure}
    \begin{subfigure}{0.32\textwidth}
        \centering
        \includegraphics[width=0.99\linewidth]{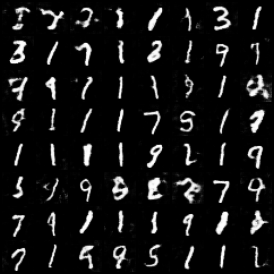}
        \subcaption{$\alpha=1e^{-5}$, $\mathrm{FID}=27.62$}
    \end{subfigure}
    \begin{subfigure}{0.32\textwidth}
        \centering
        \includegraphics[width=0.99\linewidth]{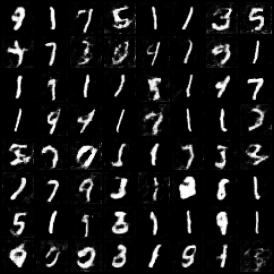}
        \subcaption{$\alpha=1e^{-3}$, $\mathrm{FID}=55.64$}
    \end{subfigure}
    \begin{subfigure}{0.32\textwidth}
        \centering
        \includegraphics[width=0.99\linewidth]{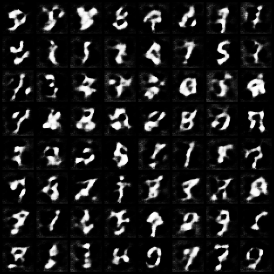}
        \subcaption{$\alpha=1e^{-1}$, $\mathrm{FID}=90.00$}
    \end{subfigure}
    \begin{subfigure}{0.6\textwidth}
        \includegraphics[width=.9\linewidth]{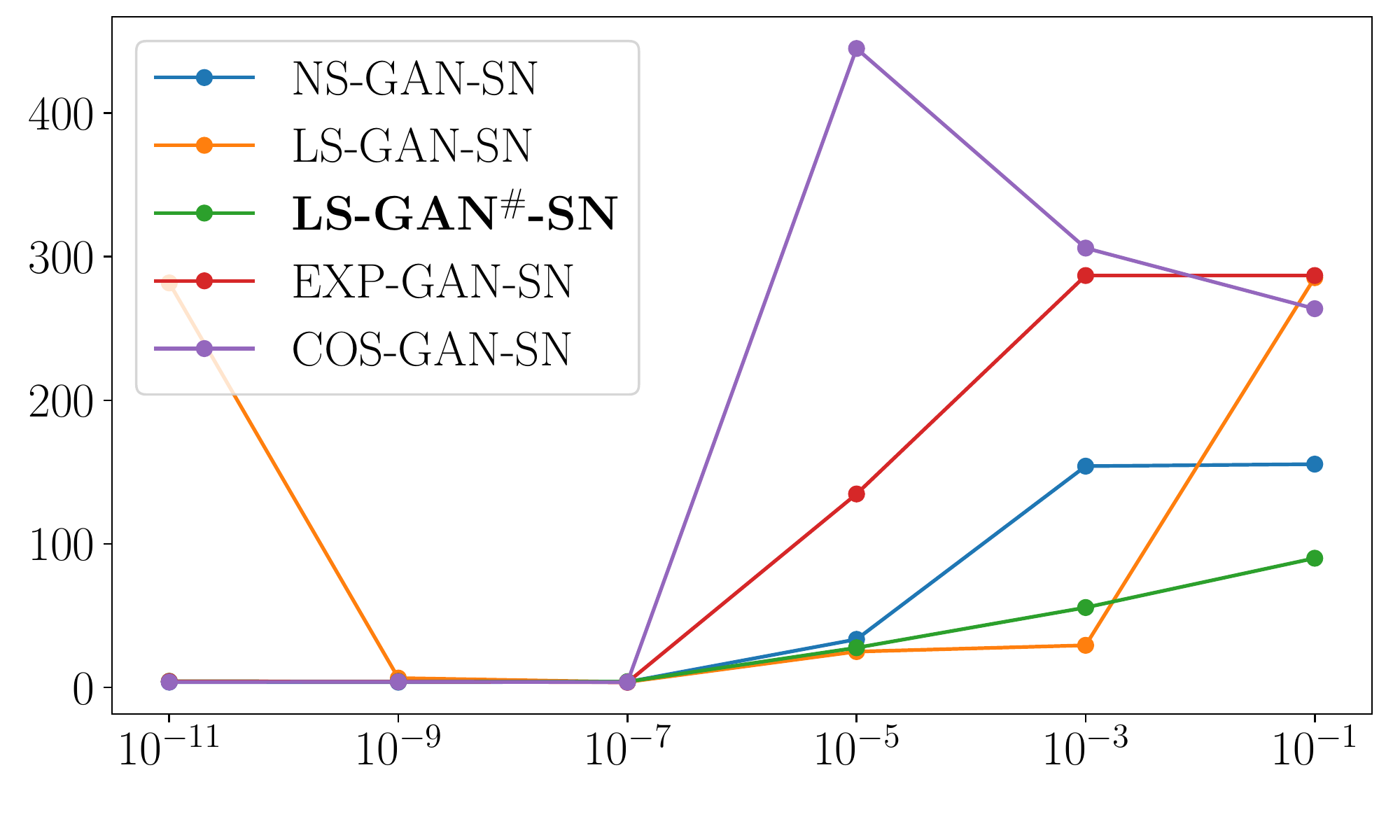}
    \end{subfigure}
\end{center}
  \caption{Samples of randomly generated images with LS-GAN$^\#$-SN of varying $\alpha$ ($k_{SN}=50.0$, MNIST). For the line plot, $x$-axis shows $\alpha$ (in log scale) and $y$-axis shows the FID scores.}
\label{fig:Scale_LSGANSN_zero_centered_MNIST_k_50}
\end{figure*}

\begin{figure*}[h]
\begin{center}
    \begin{subfigure}{0.32\textwidth}
        \centering
        \includegraphics[width=0.99\linewidth]{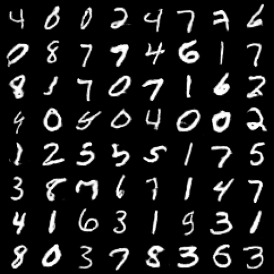}
        \subcaption{$\alpha=1e^{-11}$, $\mathrm{FID}=4.14$}
    \end{subfigure}
    \begin{subfigure}{0.32\textwidth}
        \centering
        \includegraphics[width=0.99\linewidth]{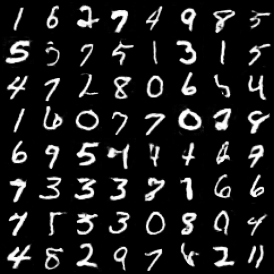}
        \subcaption{$\alpha=1e^{-9}$, $\mathrm{FID}=4.01$}
    \end{subfigure}
    \begin{subfigure}{0.32\textwidth}
        \centering
        \includegraphics[width=0.99\linewidth]{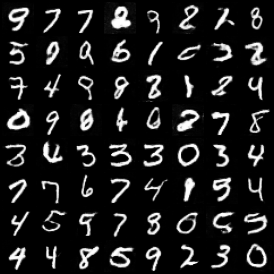}
        \subcaption{$\alpha=1e^{-7}$, $\mathrm{FID}=3.54$}
    \end{subfigure}
    \begin{subfigure}{0.32\textwidth}
        \centering
        \includegraphics[width=0.99\linewidth]{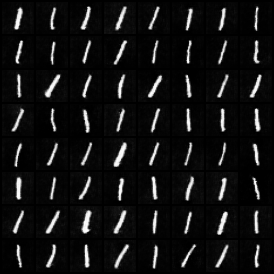}
        \subcaption{$\alpha=1e^{-5}$, $\mathrm{FID}=134.76$}
    \end{subfigure}
    \begin{subfigure}{0.32\textwidth}
        \centering
        \includegraphics[width=0.99\linewidth]{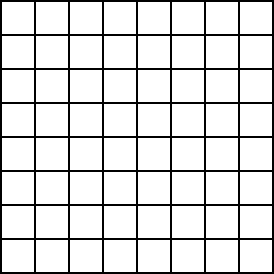}
        \subcaption{$\alpha=1e^{-3}$, $\mathrm{FID}=286.96$}
    \end{subfigure}
    \begin{subfigure}{0.32\textwidth}
        \centering
        \includegraphics[width=0.99\linewidth]{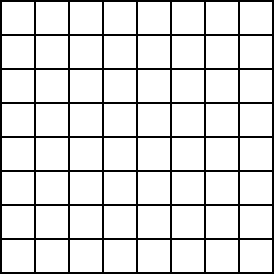}
        \subcaption{$\alpha=1e^{-1}$, $\mathrm{FID}=286.96$}
    \end{subfigure}
    \begin{subfigure}{0.6\textwidth}
        \includegraphics[width=.9\linewidth]{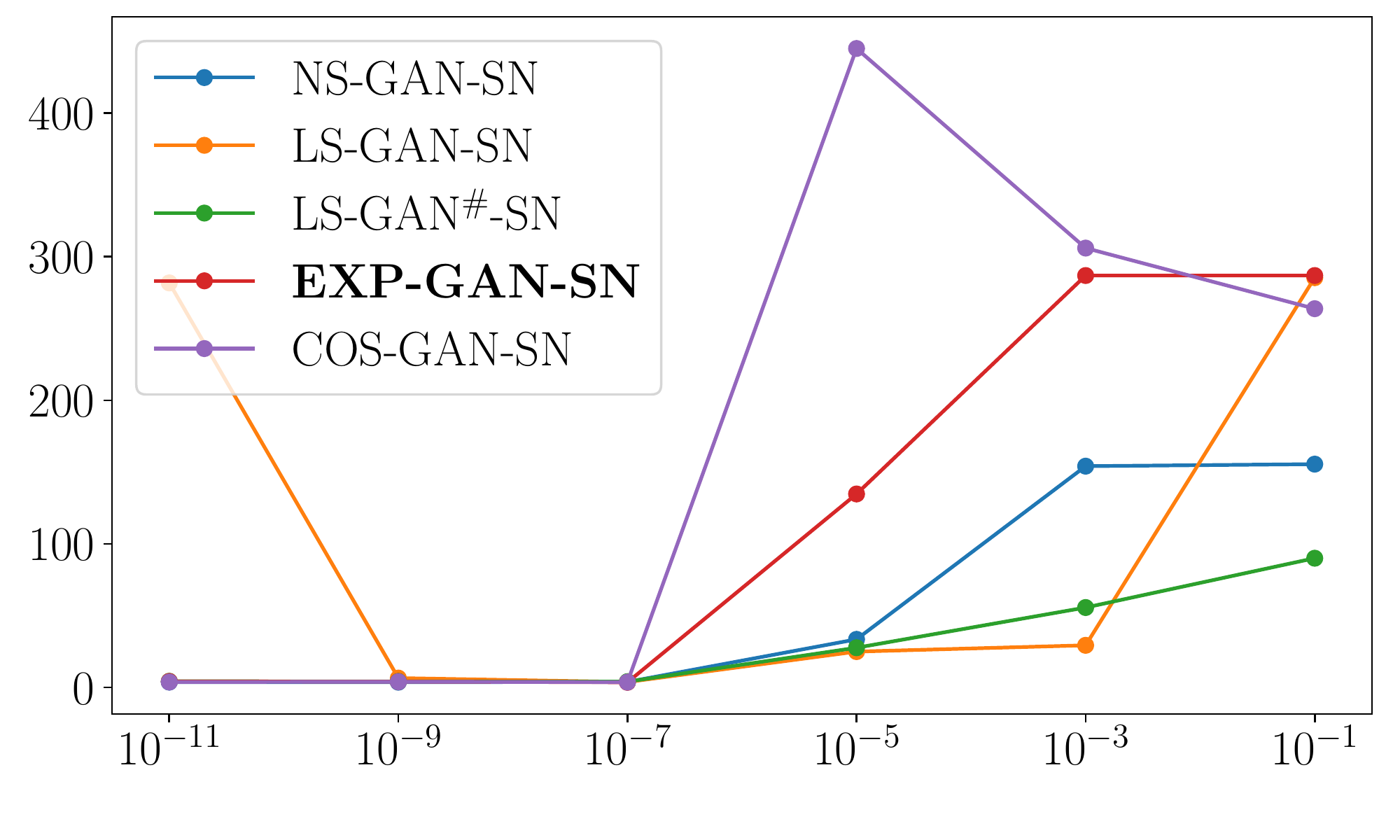}
    \end{subfigure}
\end{center}
  \caption{Samples of randomly generated images with EXP-GAN-SN of varying $\alpha$ ($k_{SN}=50.0$, MNIST). For the line plot, $x$-axis shows $\alpha$ (in log scale) and $y$-axis shows the FID scores.}
\label{fig:Scale_EXPGANSN_MNIST_k_50}
\end{figure*}

\begin{figure*}[h]
\begin{center}
    \begin{subfigure}{0.32\textwidth}
        \centering
        \includegraphics[width=0.99\linewidth]{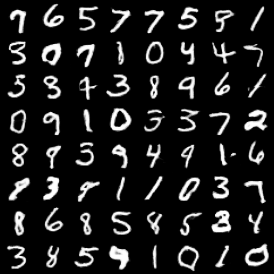}
        \subcaption{$\alpha=1e^{-11}$, $\mathrm{FID}=3.74$}
    \end{subfigure}
    \begin{subfigure}{0.32\textwidth}
        \centering
        \includegraphics[width=0.99\linewidth]{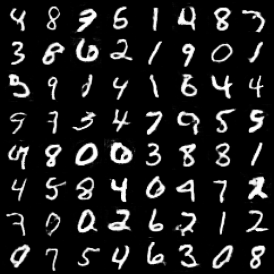}
        \subcaption{$\alpha=1e^{-9}$, $\mathrm{FID}=3.93$}
    \end{subfigure}
    \begin{subfigure}{0.32\textwidth}
        \centering
        \includegraphics[width=0.99\linewidth]{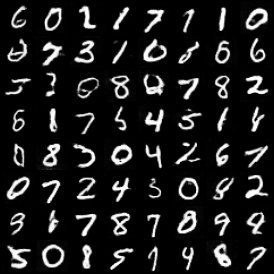}
        \subcaption{$\alpha=1e^{-7}$, $\mathrm{FID}=3.66$}
    \end{subfigure}
    \begin{subfigure}{0.32\textwidth}
        \centering
        \includegraphics[width=0.99\linewidth]{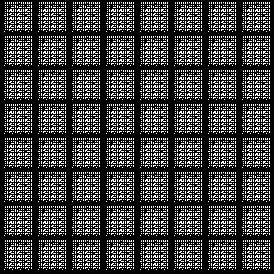}
        \subcaption{$\alpha=1e^{-5}$, $\mathrm{FID}=445.15$}
    \end{subfigure}
    \begin{subfigure}{0.32\textwidth}
        \centering
        \includegraphics[width=0.99\linewidth]{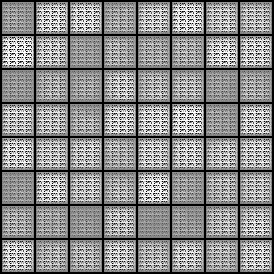}
        \subcaption{$\alpha=1e^{-3}$, $\mathrm{FID}=306.09$}
    \end{subfigure}
    \begin{subfigure}{0.32\textwidth}
        \centering
        \includegraphics[width=0.99\linewidth]{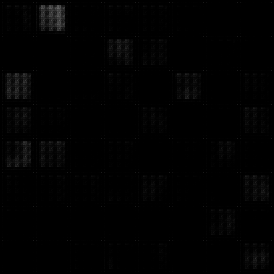}
        \subcaption{$\alpha=1e^{-1}$, $\mathrm{FID}=263.85$}
    \end{subfigure}
    \begin{subfigure}{0.6\textwidth}
        \includegraphics[width=.9\linewidth]{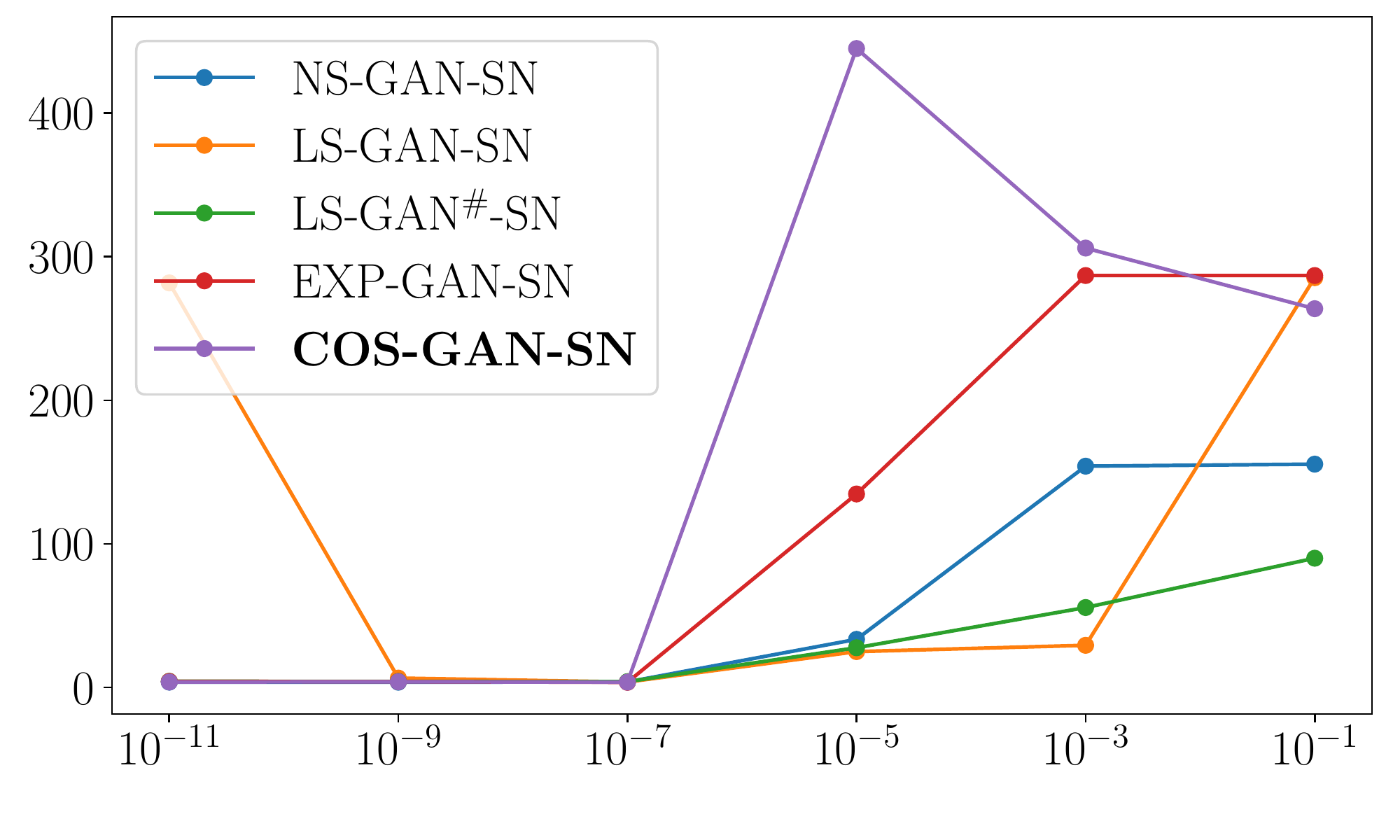}
    \end{subfigure}
\end{center}
  \caption{Samples of randomly generated images with COS-GAN-SN of varying $\alpha$ ($k_{SN}=50.0$, MNIST). For the line plot, $x$-axis shows $\alpha$ (in log scale) and $y$-axis shows the FID scores.}
\label{fig:Scale_COSGANSN_MNIST_k_50}
\end{figure*}

\FloatBarrier
\newpage

\vspace*{\fill}
\begin{table*}[h!]
\centering
\captionsetup{justification=centering}
\caption*{{\LARGE FID scores v.s. $\alpha$ ($k_{SN}=50.0$) of Different Loss Functions\\ \vspace*{5mm}
- CIFAR10 -}}
\end{table*}
\vspace*{\fill}

\FloatBarrier
\newpage

\begin{figure*}[h]
\begin{center}
    \begin{subfigure}{0.32\textwidth}
        \centering
        \includegraphics[width=0.99\linewidth]{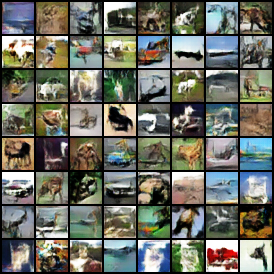}
        \subcaption{$\alpha=1e^{-11}$, $\mathrm{FID}=19.58$}
    \end{subfigure}
    \begin{subfigure}{0.32\textwidth}
        \centering
        \includegraphics[width=0.99\linewidth]{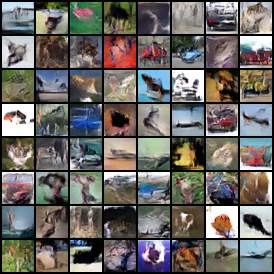}
        \subcaption{$\alpha=1e^{-9}$, $\mathrm{FID}=22.46$}
    \end{subfigure}
    \begin{subfigure}{0.32\textwidth}
        \centering
        \includegraphics[width=0.99\linewidth]{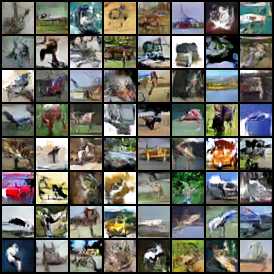}
        \subcaption{$\alpha=1e^{-7}$, $\mathrm{FID}=18.73$}
    \end{subfigure}
    \begin{subfigure}{0.32\textwidth}
        \centering
        \includegraphics[width=0.99\linewidth]{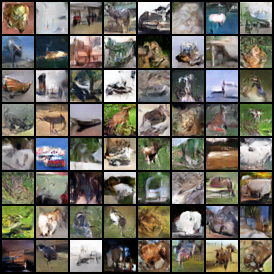}
        \subcaption{$\alpha=1e^{-5}$, $\mathrm{FID}=24.57$}
    \end{subfigure}
    \begin{subfigure}{0.32\textwidth}
        \centering
        \includegraphics[width=0.99\linewidth]{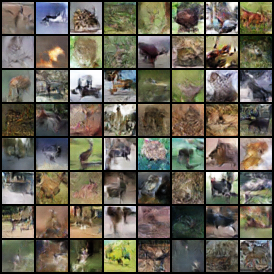}
        \subcaption{$\alpha=1e^{-3}$, $\mathrm{FID}=49.56$}
    \end{subfigure}
    \begin{subfigure}{0.32\textwidth}
        \centering
        \includegraphics[width=0.99\linewidth]{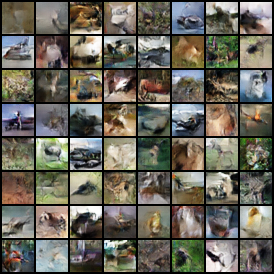}
        \subcaption{$\alpha=1e^{-1}$, $\mathrm{FID}=43.42$}
    \end{subfigure}
    \begin{subfigure}{0.6\textwidth}
        \includegraphics[width=.9\linewidth]{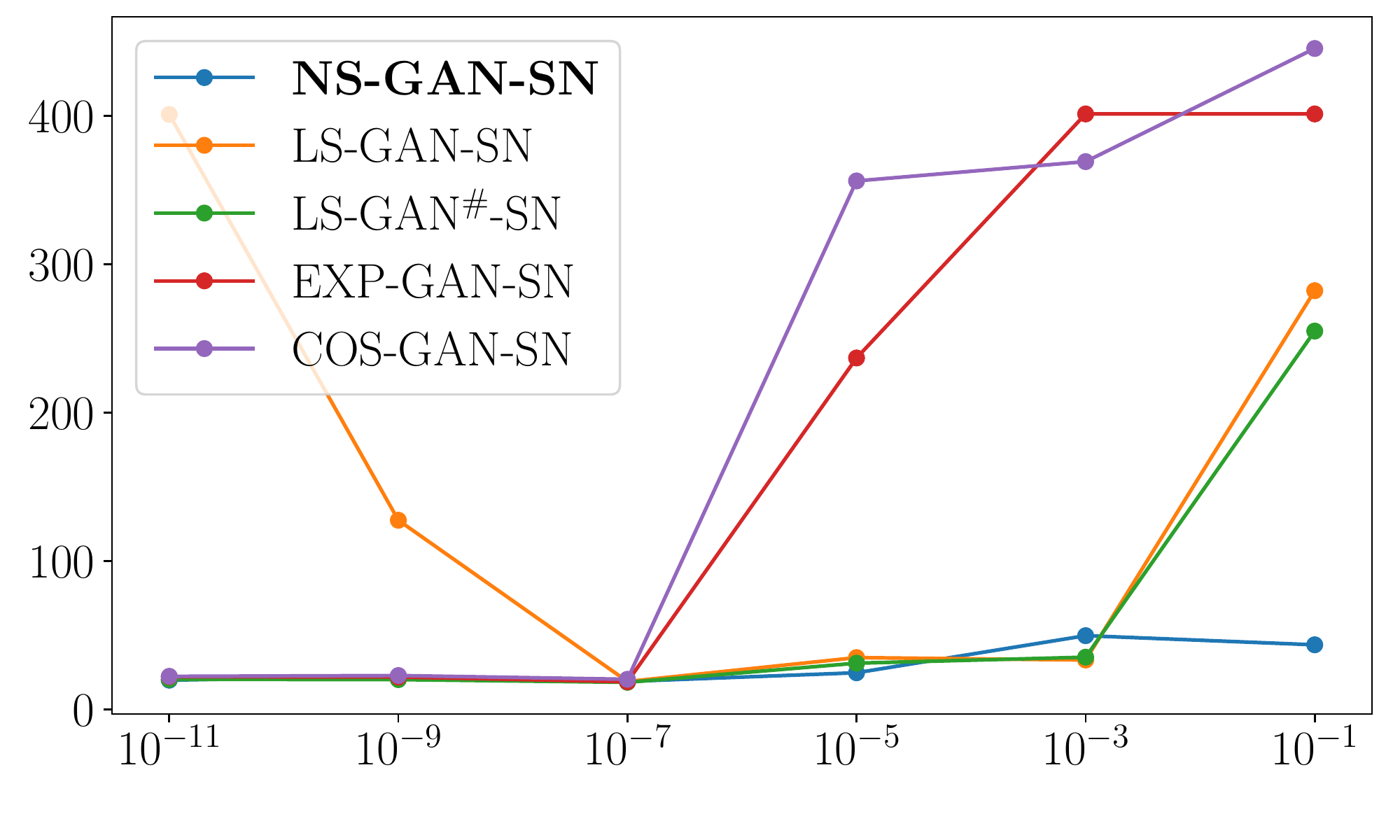}
    \end{subfigure}
\end{center}
  \caption{Samples of randomly generated images with NS-GAN-SN of varying $\alpha$ ($k_{SN}=50.0$, CIFAR10). For the line plot, $x$-axis shows $\alpha$ (in log scale) and $y$-axis shows the FID scores.}
\label{fig:Scale_NSGANSN_CIFAR10_k_50}
\end{figure*}

\begin{figure*}[h]
\begin{center}
    \begin{subfigure}{0.32\textwidth}
        \centering
        \includegraphics[width=0.99\linewidth]{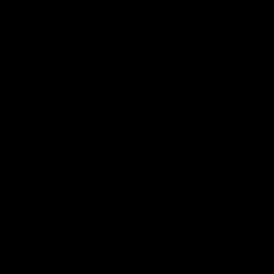}
        \subcaption{$\alpha=1e^{-11}$, $\mathrm{FID}$=400.91}
    \end{subfigure}
    \begin{subfigure}{0.32\textwidth}
        \centering
        \includegraphics[width=0.99\linewidth]{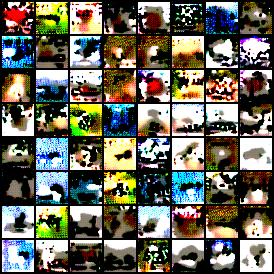}
        \subcaption{$\alpha=1e^{-9}$, $\mathrm{FID}=127.38$}
    \end{subfigure}
    \begin{subfigure}{0.32\textwidth}
        \centering
        \includegraphics[width=0.99\linewidth]{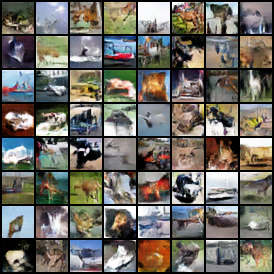}
        \subcaption{$\alpha=1e^{-7}$, $\mathrm{FID}=18.68$}
    \end{subfigure}
    \begin{subfigure}{0.32\textwidth}
        \centering
        \includegraphics[width=0.99\linewidth]{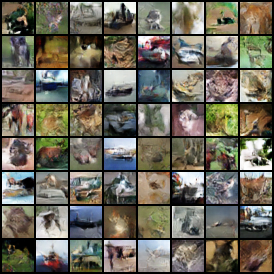}
        \subcaption{$\alpha=1e^{-5}$, $\mathrm{FID}=34.78$}
    \end{subfigure}
    \begin{subfigure}{0.32\textwidth}
        \centering
        \includegraphics[width=0.99\linewidth]{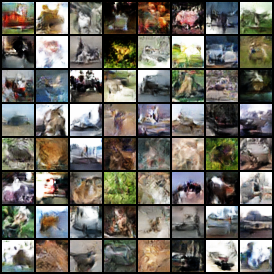}
        \subcaption{$\alpha=1e^{-3}$, $\mathrm{FID}=33.17$}
    \end{subfigure}
    \begin{subfigure}{0.32\textwidth}
        \centering
        \includegraphics[width=0.99\linewidth]{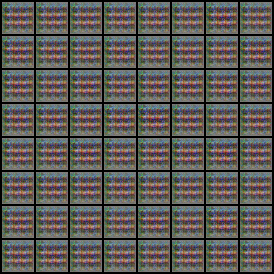}
        \subcaption{$\alpha=1e^{-1}$, $\mathrm{FID}=282.11$}
    \end{subfigure}
    \begin{subfigure}{0.6\textwidth}
        \includegraphics[width=.9\linewidth]{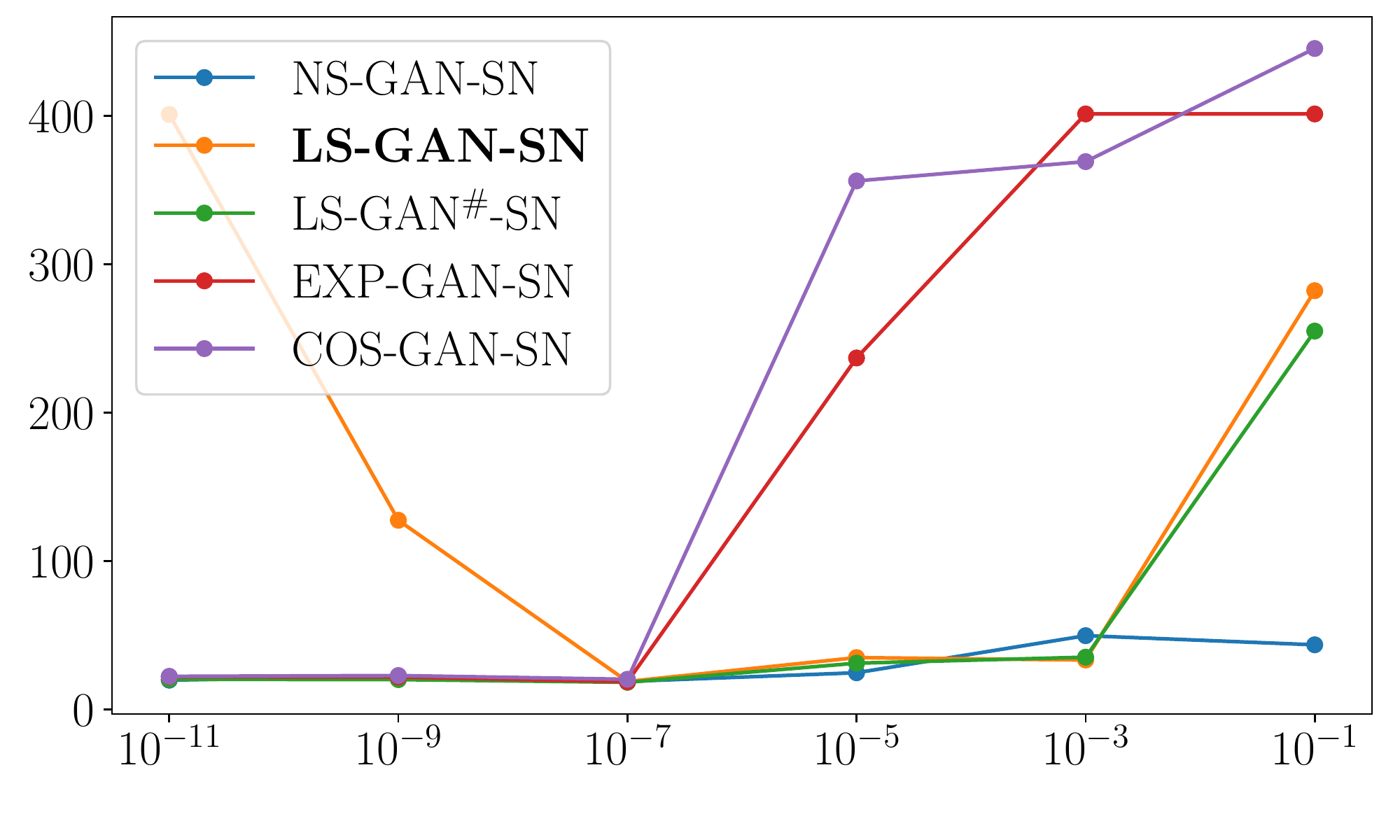}
    \end{subfigure}
\end{center}
  \caption{Samples of randomly generated images with LS-GAN-SN of varying $\alpha$ ($k_{SN}=50.0$, CIFAR10). For the line plot, $x$-axis shows $\alpha$ (in log scale) and $y$-axis shows the FID scores.}
\label{fig:Scale_LSGANSN_CIFAR10_k_50}
\end{figure*}

\begin{figure*}[h]
\begin{center}
    \begin{subfigure}{0.32\textwidth}
        \centering
        \includegraphics[width=0.99\linewidth]{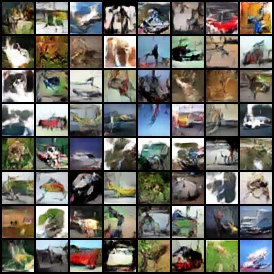}
        \subcaption{$\alpha=1e^{-11}$, $\mathrm{FID}=20.16$}
    \end{subfigure}
    \begin{subfigure}{0.32\textwidth}
        \centering
        \includegraphics[width=0.99\linewidth]{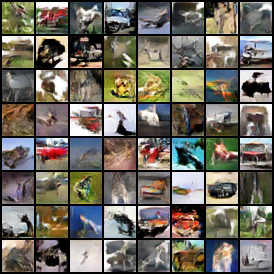}
        \subcaption{$\alpha=1e^{-9}$, $\mathrm{FID}=19.96$}
    \end{subfigure}
    \begin{subfigure}{0.32\textwidth}
        \centering
        \includegraphics[width=0.99\linewidth]{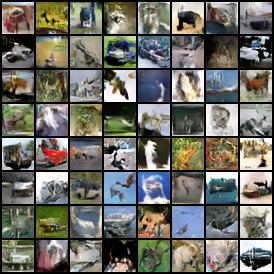}
        \subcaption{$\alpha=1e^{-7}$, $\mathrm{FID}=18.13$}
    \end{subfigure}
    \begin{subfigure}{0.32\textwidth}
        \centering
        \includegraphics[width=0.99\linewidth]{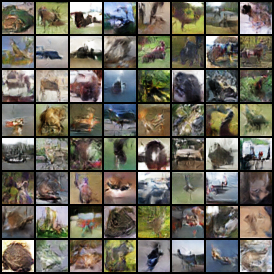}
        \subcaption{$\alpha=1e^{-5}$, $\mathrm{FID}=31.03$}
    \end{subfigure}
    \begin{subfigure}{0.32\textwidth}
        \centering
        \includegraphics[width=0.99\linewidth]{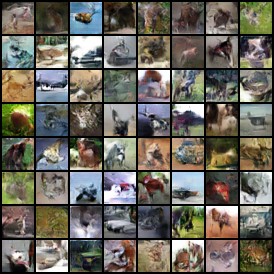}
        \subcaption{$\alpha=1e^{-3}$, $\mathrm{FID}=35.06$}
    \end{subfigure}
    \begin{subfigure}{0.32\textwidth}
        \centering
        \includegraphics[width=0.99\linewidth]{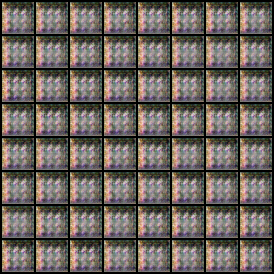}
        \subcaption{$\alpha=1e^{-1}$, $\mathrm{FID}=254.88$}
    \end{subfigure}
    \begin{subfigure}{0.6\textwidth}
        \includegraphics[width=.9\linewidth]{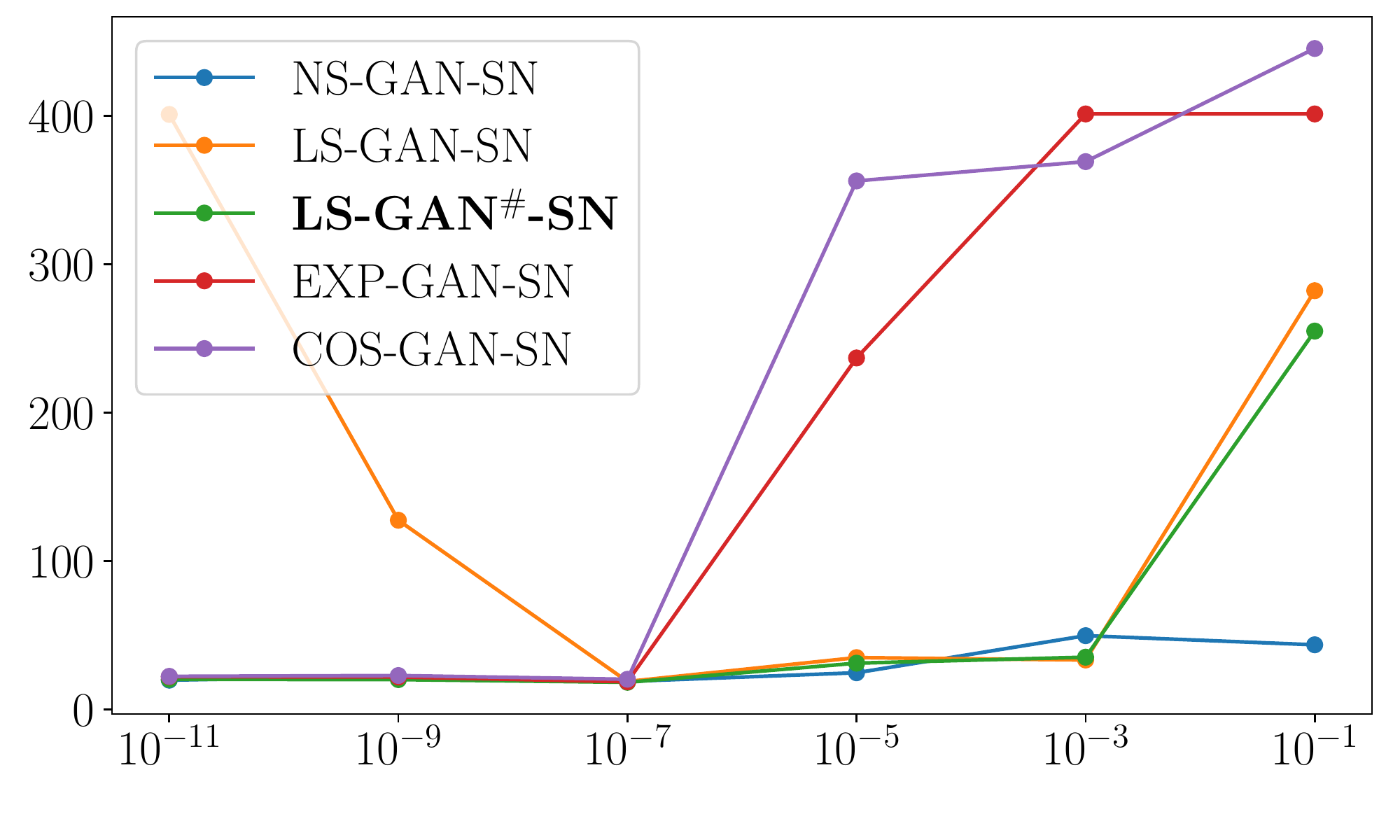}
    \end{subfigure}
\end{center}
  \caption{Samples of randomly generated images with LS-GAN$^\#$-SN of varying $\alpha$ ($k_{SN}=50.0$, CIFAR10). For the line plot, $x$-axis shows $\alpha$ (in log scale) and $y$-axis shows the FID scores.}
\label{fig:Scale_LSGANSN_zero_centered_CIFAR10_k_50}
\end{figure*}

\begin{figure*}[h]
\begin{center}
    \begin{subfigure}{0.32\textwidth}
        \centering
        \includegraphics[width=0.99\linewidth]{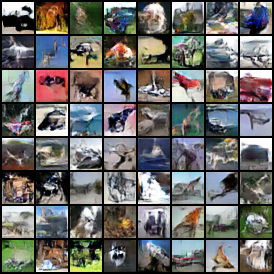}
        \subcaption{$\alpha=1e^{-11}$, $\mathrm{FID}=21.93$}
    \end{subfigure}
    \begin{subfigure}{0.32\textwidth}
        \centering
        \includegraphics[width=0.99\linewidth]{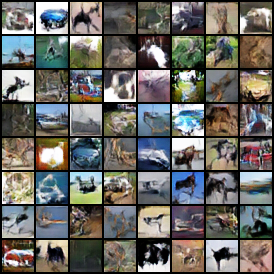}
        \subcaption{$\alpha=1e^{-9}$, $\mathrm{FID}=21.70$}
    \end{subfigure}
    \begin{subfigure}{0.32\textwidth}
        \centering
        \includegraphics[width=0.99\linewidth]{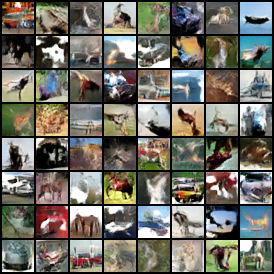}
        \subcaption{$\alpha=1e^{-7}$, $\mathrm{FID}=18.50$}
    \end{subfigure}
    \begin{subfigure}{0.32\textwidth}
        \centering
        \includegraphics[width=0.99\linewidth]{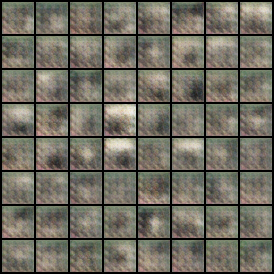}
        \subcaption{$\alpha=1e^{-5}$, $\mathrm{FID}=236.77$}
    \end{subfigure}
    \begin{subfigure}{0.32\textwidth}
        \centering
        \includegraphics[width=0.99\linewidth]{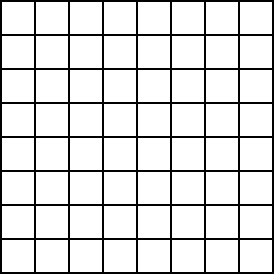}
        \subcaption{$\alpha=1e^{-3}$, $\mathrm{FID}=401.24$}
    \end{subfigure}
    \begin{subfigure}{0.32\textwidth}
        \centering
        \includegraphics[width=0.99\linewidth]{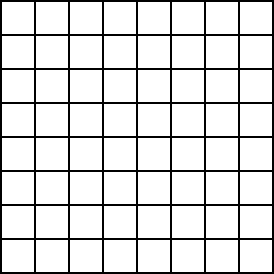}
        \subcaption{$\alpha=1e^{-1}$, $\mathrm{FID}=401.24$}
    \end{subfigure}
    \begin{subfigure}{0.6\textwidth}
        \includegraphics[width=.9\linewidth]{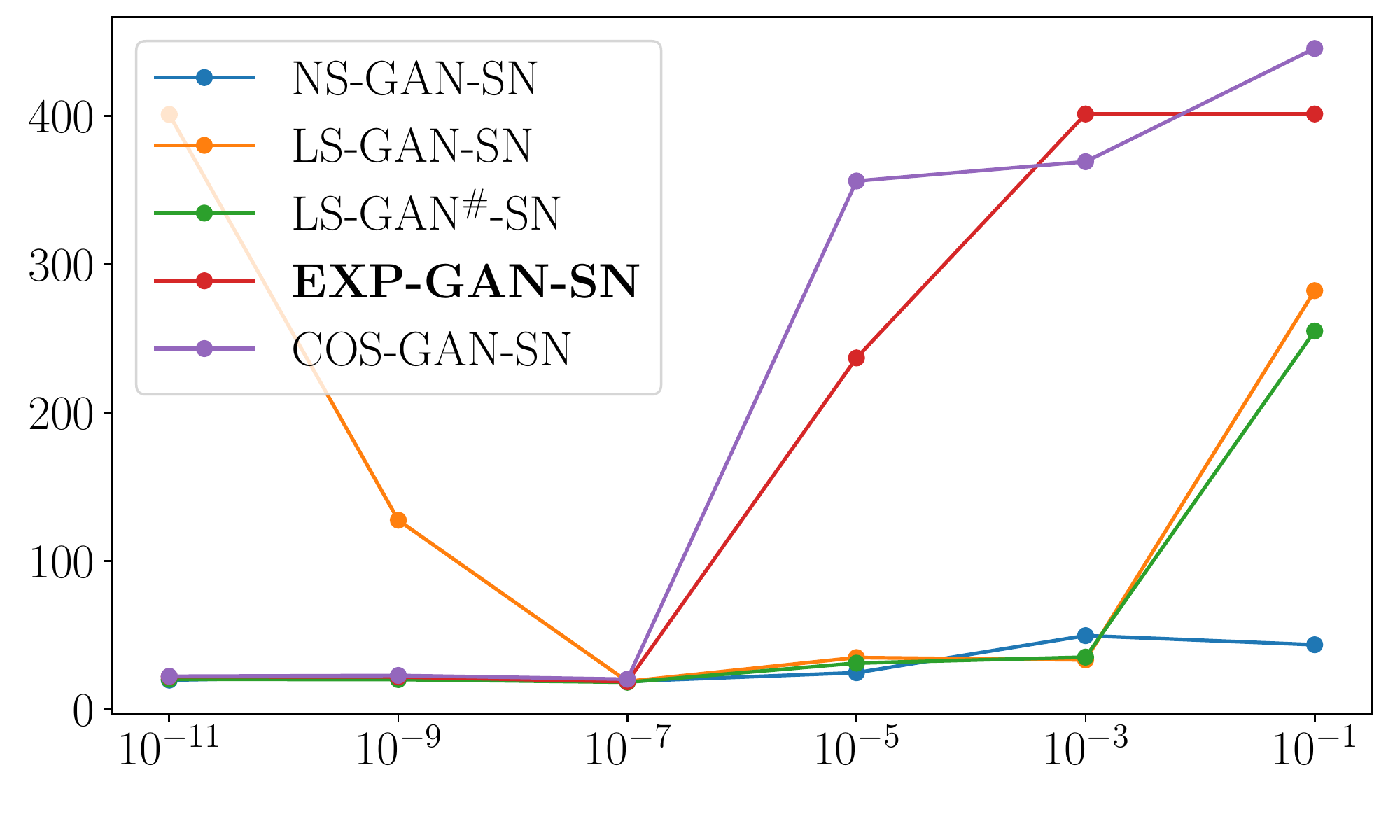}
    \end{subfigure}
\end{center}
  \caption{Samples of randomly generated images with EXP-GAN-SN of varying $\alpha$ ($k_{SN}=50.0$, CIFAR10). For the line plot, $x$-axis shows $\alpha$ (in log scale) and $y$-axis shows the FID scores.}
\label{fig:Scale_EXPGANSN_CIFAR10_k_50}
\end{figure*}

\begin{figure*}[h]
\begin{center}
    \begin{subfigure}{0.32\textwidth}
        \centering
        \includegraphics[width=0.99\linewidth]{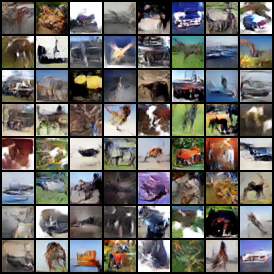}
        \subcaption{$\alpha=1e^{-11}$, $\mathrm{FID}=22.16$}
    \end{subfigure}
    \begin{subfigure}{0.32\textwidth}
        \centering
        \includegraphics[width=0.99\linewidth]{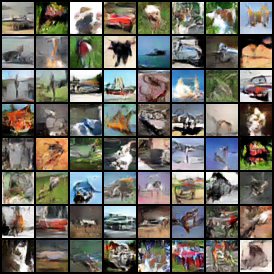}
        \subcaption{$\alpha=1e^{-9}$, $\mathrm{FID}=22.69$}
    \end{subfigure}
    \begin{subfigure}{0.32\textwidth}
        \centering
        \includegraphics[width=0.99\linewidth]{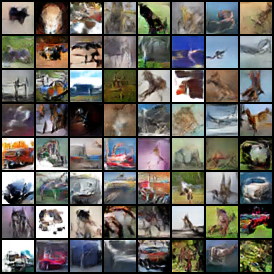}
        \subcaption{$\alpha=1e^{-7}$, $\mathrm{FID}=20.19$}
    \end{subfigure}
    \begin{subfigure}{0.32\textwidth}
        \centering
        \includegraphics[width=0.99\linewidth]{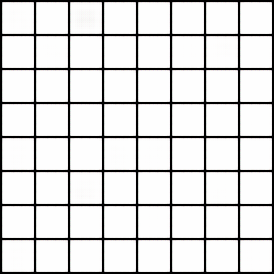}
        \subcaption{$\alpha=1e^{-5}$, $\mathrm{FID}=356.10$}
    \end{subfigure}
    \begin{subfigure}{0.32\textwidth}
        \centering
        \includegraphics[width=0.99\linewidth]{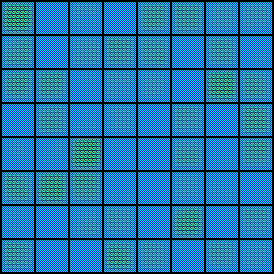}
        \subcaption{$\alpha=1e^{-3}$, $\mathrm{FID}=369.11$}
    \end{subfigure}
    \begin{subfigure}{0.32\textwidth}
        \centering
        \includegraphics[width=0.99\linewidth]{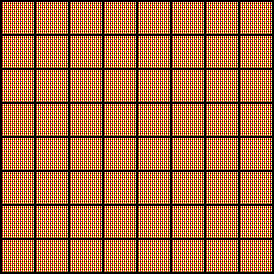}
        \subcaption{$\alpha=1e^{-1}$, $\mathrm{FID}=445.37$}
    \end{subfigure}
    \begin{subfigure}{0.6\textwidth}
        \includegraphics[width=.9\linewidth]{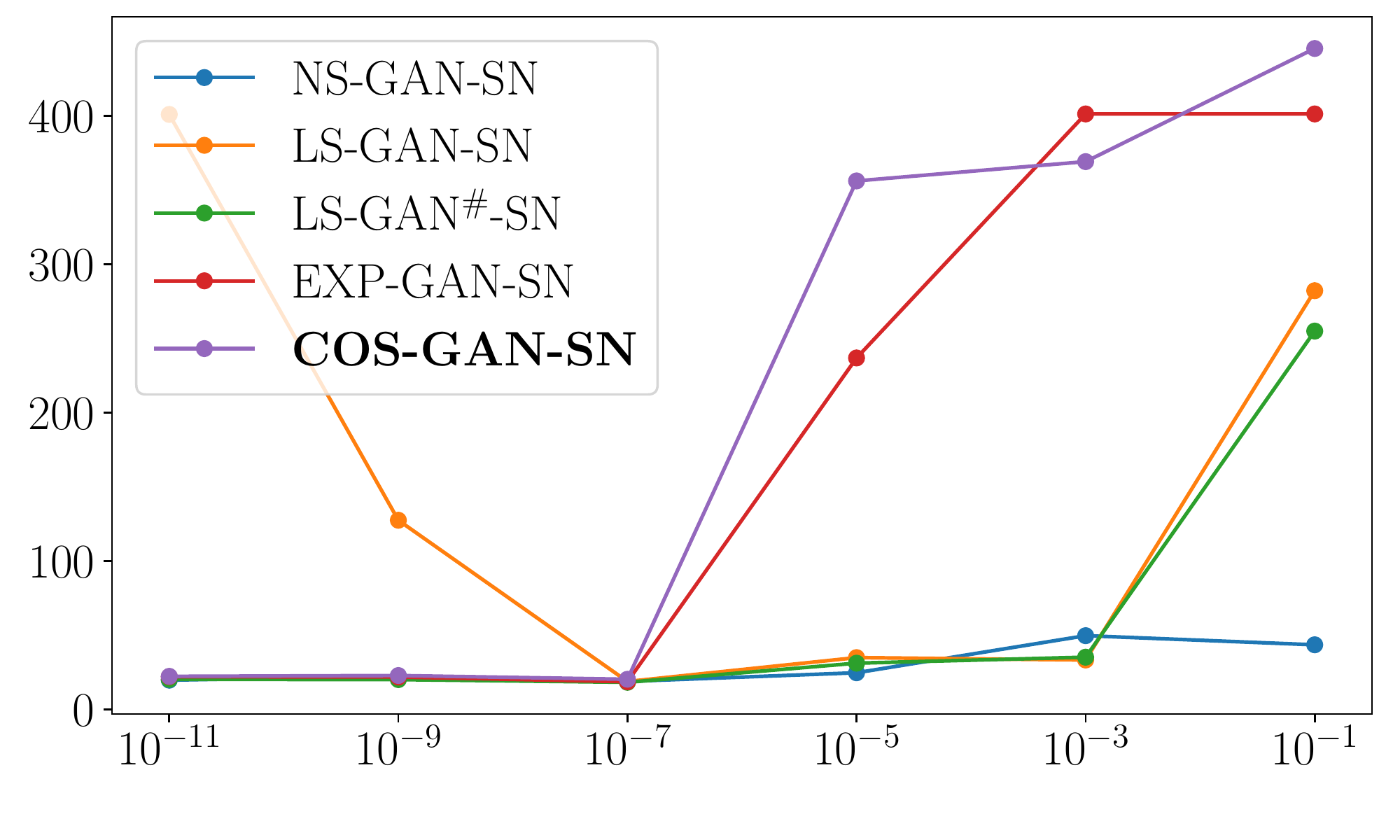}
    \end{subfigure}
\end{center}
  \caption{Samples of randomly generated images with COS-GAN-SN of varying $\alpha$ ($k_{SN}=50.0$, CIFAR10). For the line plot, $x$-axis shows $\alpha$ (in log scale) and $y$-axis shows the FID scores.}
\label{fig:Scale_COSGANSN_CIFAR10_k_50}
\end{figure*}

\FloatBarrier
\newpage

\vspace*{\fill}
\begin{table*}[h!]
\centering
\captionsetup{justification=centering}
\caption*{{\LARGE FID scores v.s. $\alpha$ ($k_{SN}=50.0$) of Different Loss Functions\\ \vspace*{5mm}
- CelebA -}}
\end{table*}
\vspace*{\fill}

\FloatBarrier
\newpage

\begin{figure*}[h]
\begin{center}
    \begin{subfigure}{0.32\textwidth}
        \centering
        \includegraphics[width=0.99\linewidth]{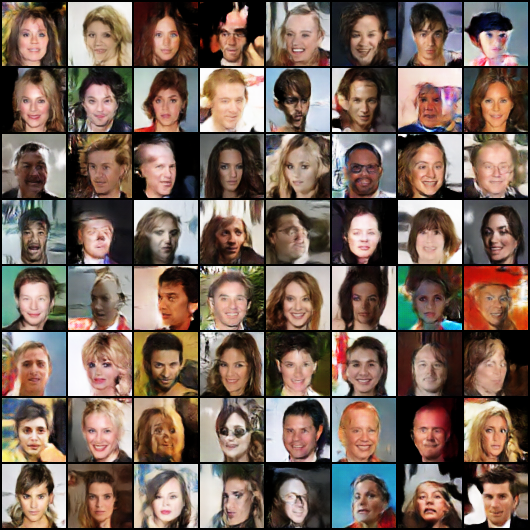}
        \subcaption{$\alpha=1e^{-11}$, $\mathrm{FID}=9.08$}
    \end{subfigure}
    \begin{subfigure}{0.32\textwidth}
        \centering
        \includegraphics[width=0.99\linewidth]{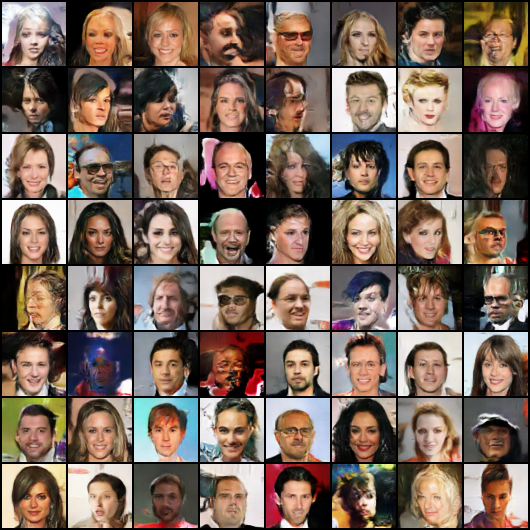}
        \subcaption{$\alpha=1e^{-9}$, $\mathrm{FID}=7.05$}
    \end{subfigure}
    \begin{subfigure}{0.32\textwidth}
        \centering
        \includegraphics[width=0.99\linewidth]{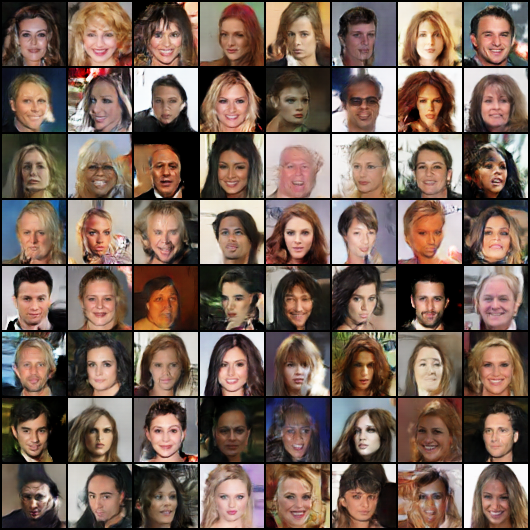}
        \subcaption{$\alpha=1e^{-7}$, $\mathrm{FID}=7.84$}
    \end{subfigure}
    \begin{subfigure}{0.32\textwidth}
        \centering
        \includegraphics[width=0.99\linewidth]{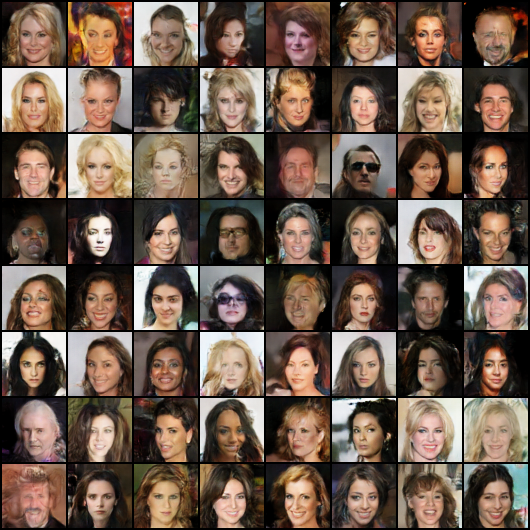}
        \subcaption{$\alpha=1e^{-5}$, $\mathrm{FID}=18.51$}
    \end{subfigure}
    \begin{subfigure}{0.32\textwidth}
        \centering
        \includegraphics[width=0.99\linewidth]{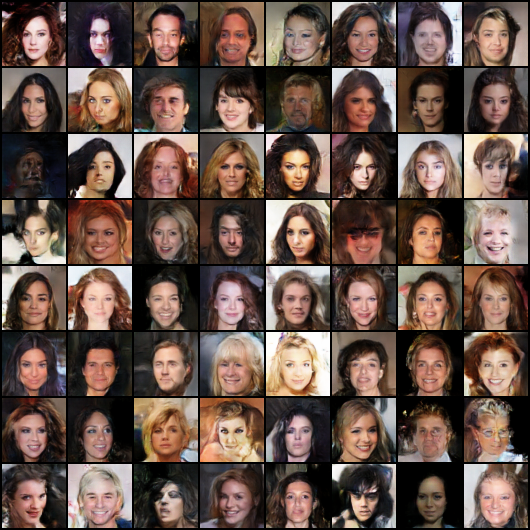}
        \subcaption{$\alpha=1e^{-3}$, $\mathrm{FID}=18.41$}
    \end{subfigure}
    \begin{subfigure}{0.32\textwidth}
        \centering
        \includegraphics[width=0.99\linewidth]{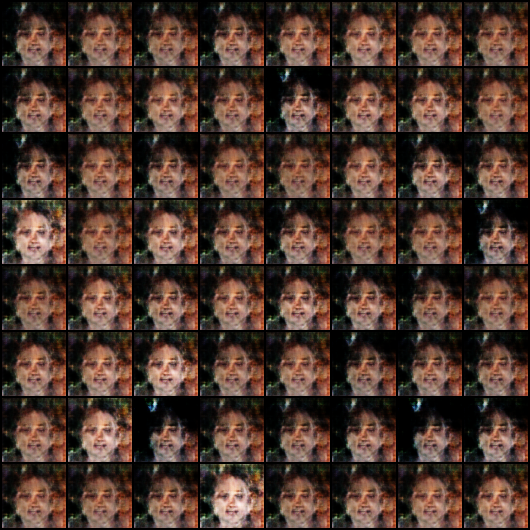}
        \subcaption{$\alpha=1e^{-1}$, $\mathrm{FID}=242.64$}
    \end{subfigure}
    \begin{subfigure}{0.6\textwidth}
        \includegraphics[width=.9\linewidth]{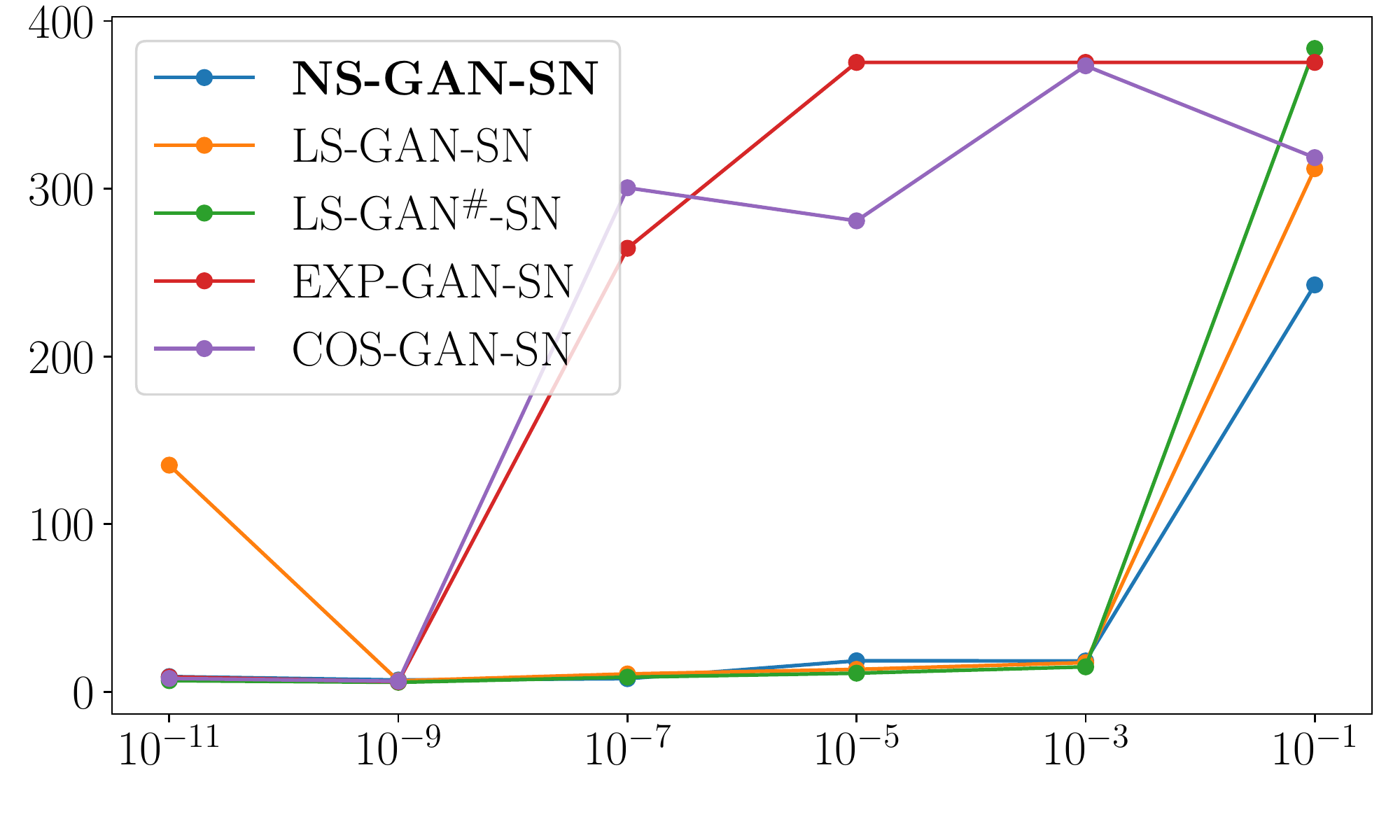}
    \end{subfigure}
\end{center}
  \caption{Samples of randomly generated images with NS-GAN-SN of varying $\alpha$ ($k_{SN}=50.0$, CelebA). For the line plot, $x$-axis shows $\alpha$ (in log scale) and $y$-axis shows the FID scores.}
\label{fig:Scale_NSGANSN_CelebA_k_50}
\end{figure*}

\begin{figure*}[h]
\begin{center}
    \begin{subfigure}{0.32\textwidth}
        \centering
        \includegraphics[width=0.99\linewidth]{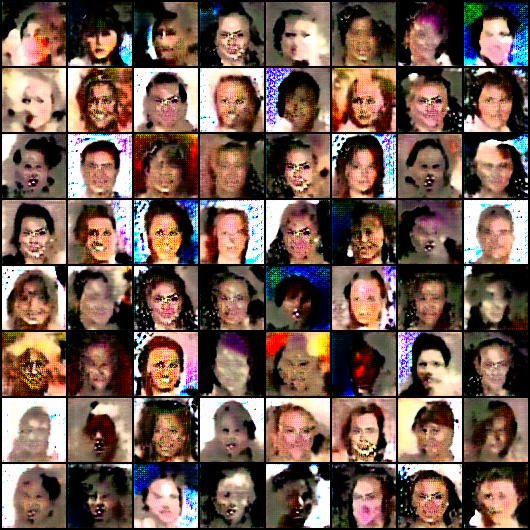}
        \subcaption{$\alpha=1e^{-11}$, $\mathrm{FID}$=135.17}
    \end{subfigure}
    \begin{subfigure}{0.32\textwidth}
        \centering
        \includegraphics[width=0.99\linewidth]{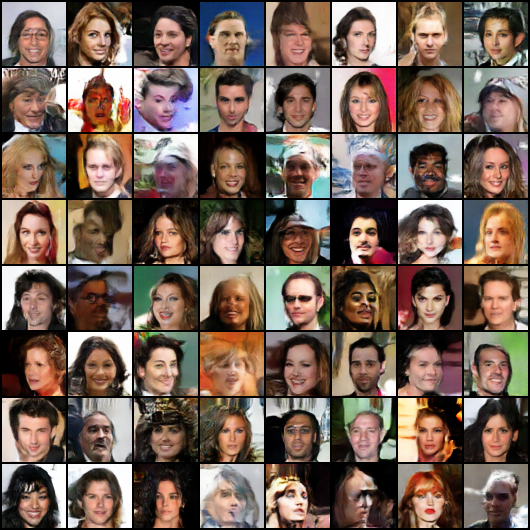}
        \subcaption{$\alpha=1e^{-9}$, $\mathrm{FID}=6.57$}
    \end{subfigure}
    \begin{subfigure}{0.32\textwidth}
        \centering
        \includegraphics[width=0.99\linewidth]{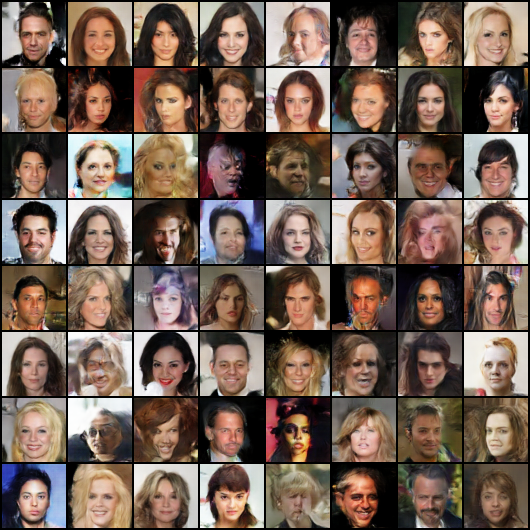}
        \subcaption{$\alpha=1e^{-7}$, $\mathrm{FID}=10.67$}
    \end{subfigure}
    \begin{subfigure}{0.32\textwidth}
        \centering
        \includegraphics[width=0.99\linewidth]{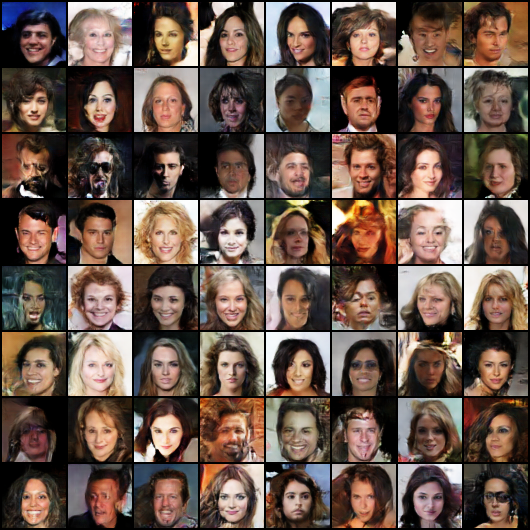}
        \subcaption{$\alpha=1e^{-5}$, $\mathrm{FID}=13.39$}
    \end{subfigure}
    \begin{subfigure}{0.32\textwidth}
        \centering
        \includegraphics[width=0.99\linewidth]{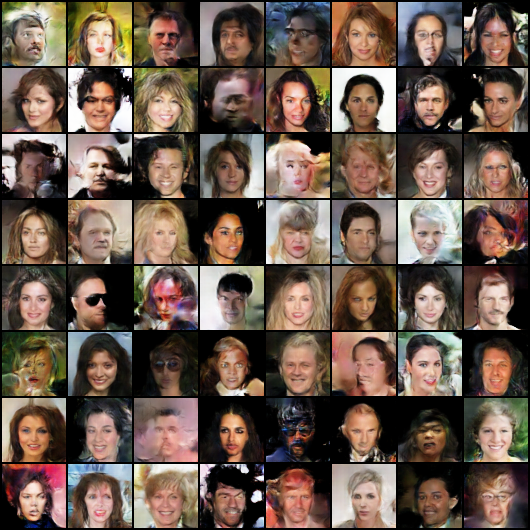}
        \subcaption{$\alpha=1e^{-3}$, $\mathrm{FID}=17.42$}
    \end{subfigure}
    \begin{subfigure}{0.32\textwidth}
        \centering
        \includegraphics[width=0.99\linewidth]{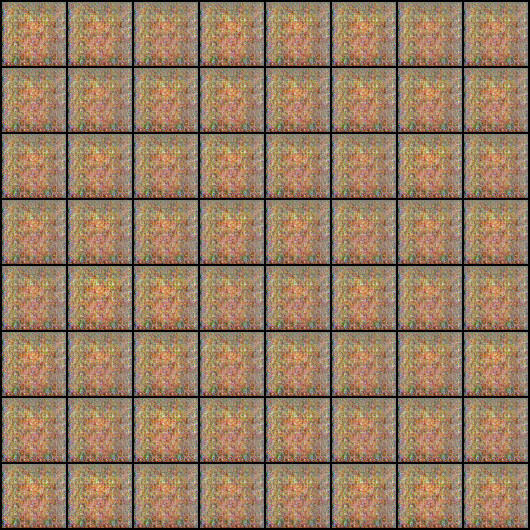}
        \subcaption{$\alpha=1e^{-1}$, $\mathrm{FID}=311.93$}
    \end{subfigure}
    \begin{subfigure}{0.6\textwidth}
        \includegraphics[width=.9\linewidth]{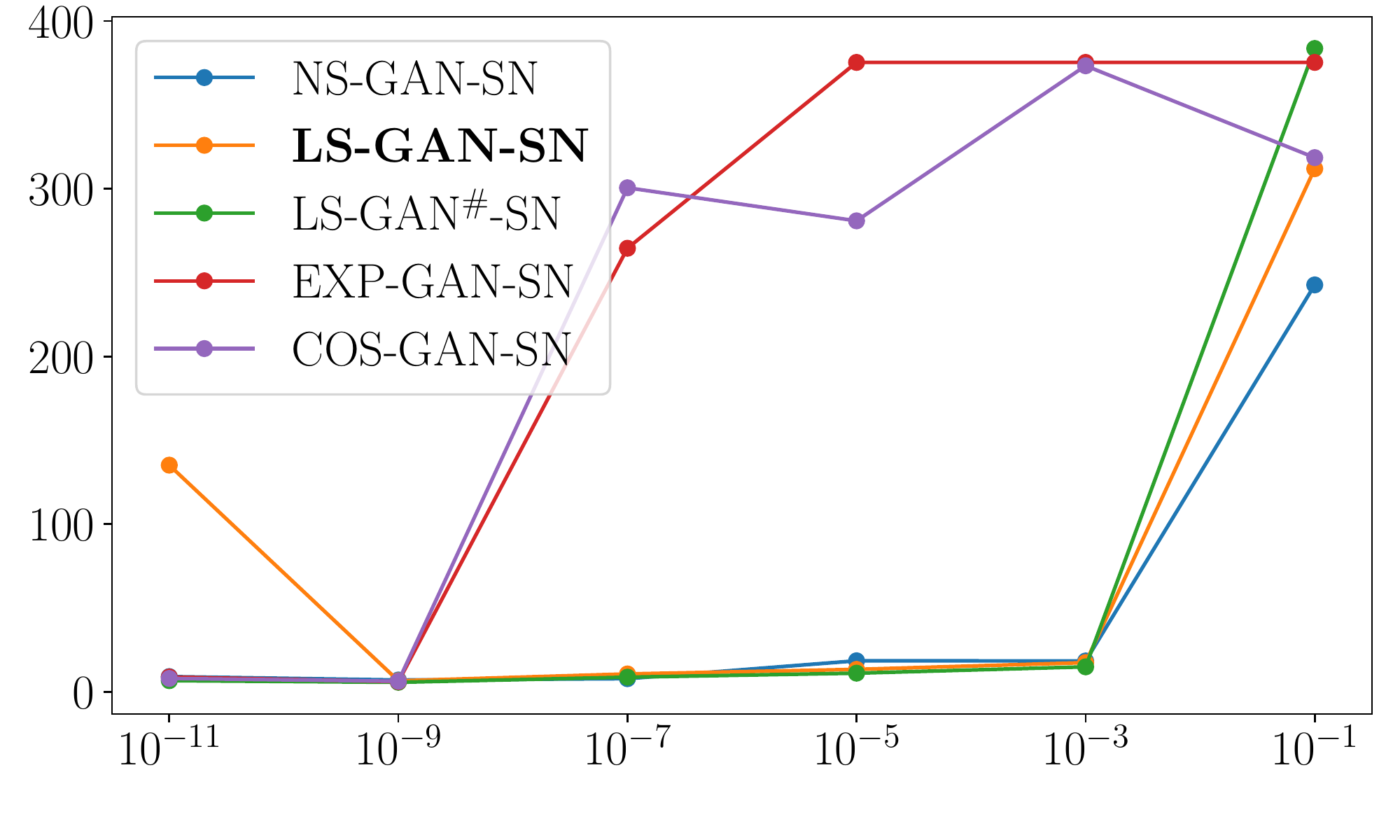}
    \end{subfigure}
\end{center}
  \caption{Samples of randomly generated images with LS-GAN-SN of varying $\alpha$ ($k_{SN}=50.0$, CelebA). For the line plot, $x$-axis shows $\alpha$ (in log scale) and $y$-axis shows the FID scores.}
\label{fig:Scale_LSGANSN_CelebA_k_50}
\end{figure*}

\begin{figure*}[h]
\begin{center}
    \begin{subfigure}{0.32\textwidth}
        \centering
        \includegraphics[width=0.99\linewidth]{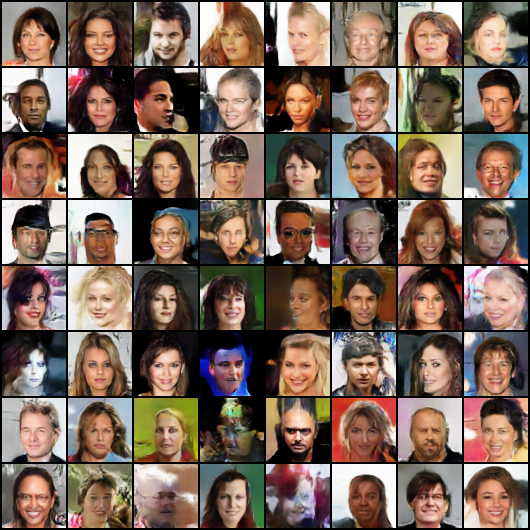}
        \subcaption{$\alpha=1e^{-11}$, $\mathrm{FID}=6.66$}
    \end{subfigure}
    \begin{subfigure}{0.32\textwidth}
        \centering
        \includegraphics[width=0.99\linewidth]{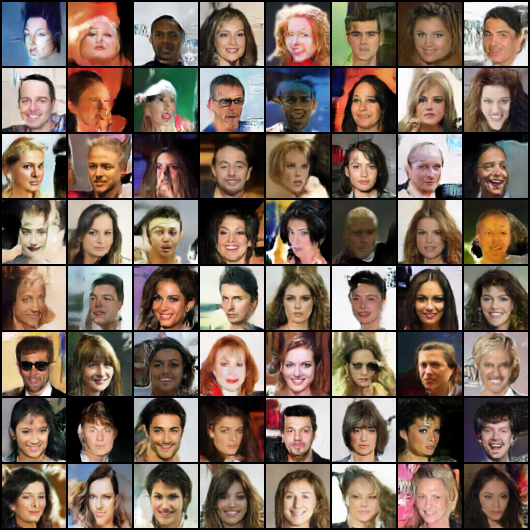}
        \subcaption{$\alpha=1e^{-9}$, $\mathrm{FID}=5.68$}
    \end{subfigure}
    \begin{subfigure}{0.32\textwidth}
        \centering
        \includegraphics[width=0.99\linewidth]{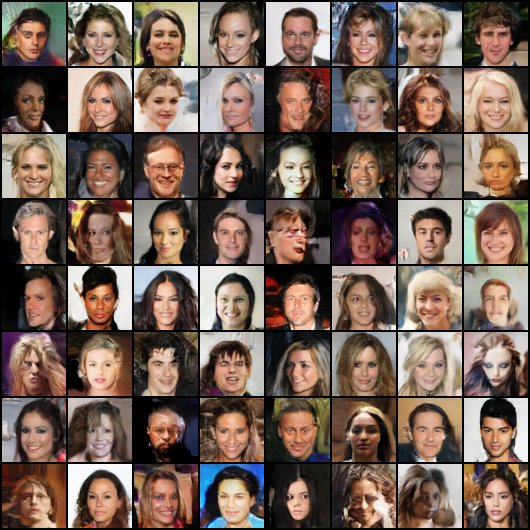}
        \subcaption{$\alpha=1e^{-7}$, $\mathrm{FID}=8.72$}
    \end{subfigure}
    \begin{subfigure}{0.32\textwidth}
        \centering
        \includegraphics[width=0.99\linewidth]{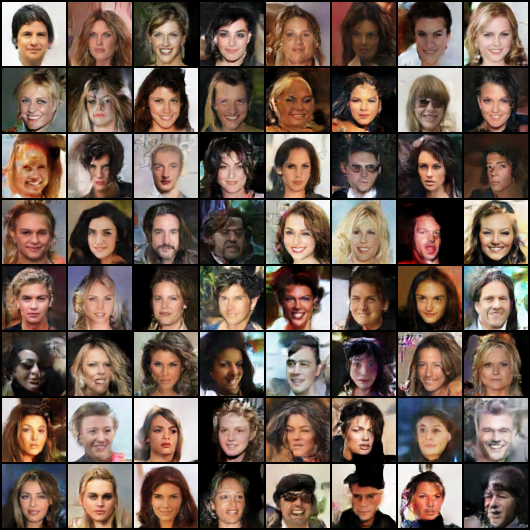}
        \subcaption{$\alpha=1e^{-5}$, $\mathrm{FID}=11.13$}
    \end{subfigure}
    \begin{subfigure}{0.32\textwidth}
        \centering
        \includegraphics[width=0.99\linewidth]{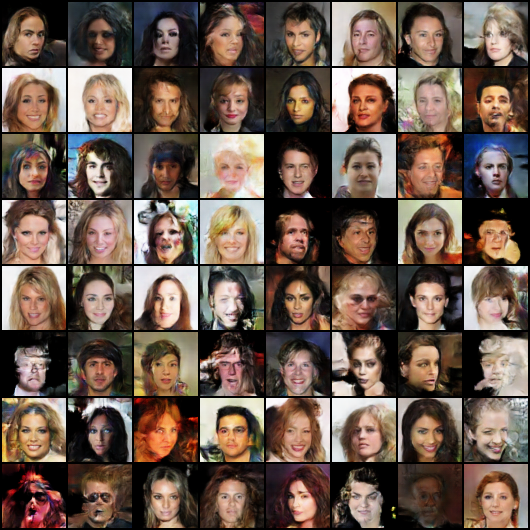}
        \subcaption{$\alpha=1e^{-3}$, $\mathrm{FID}=14.90$}
    \end{subfigure}
    \begin{subfigure}{0.32\textwidth}
        \centering
        \includegraphics[width=0.99\linewidth]{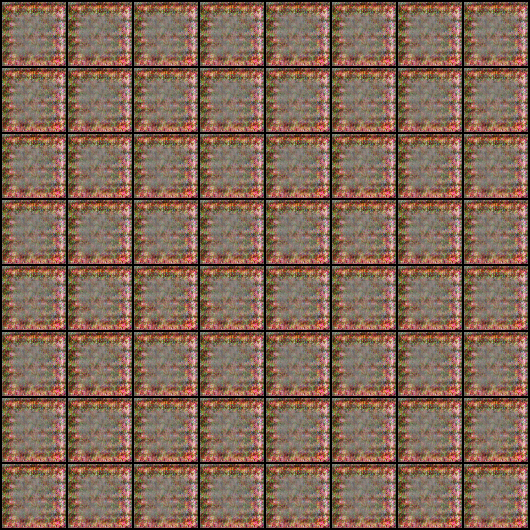}
        \subcaption{$\alpha=1e^{-1}$, $\mathrm{FID}=383.61$}
    \end{subfigure}
    \begin{subfigure}{0.6\textwidth}
        \includegraphics[width=.9\linewidth]{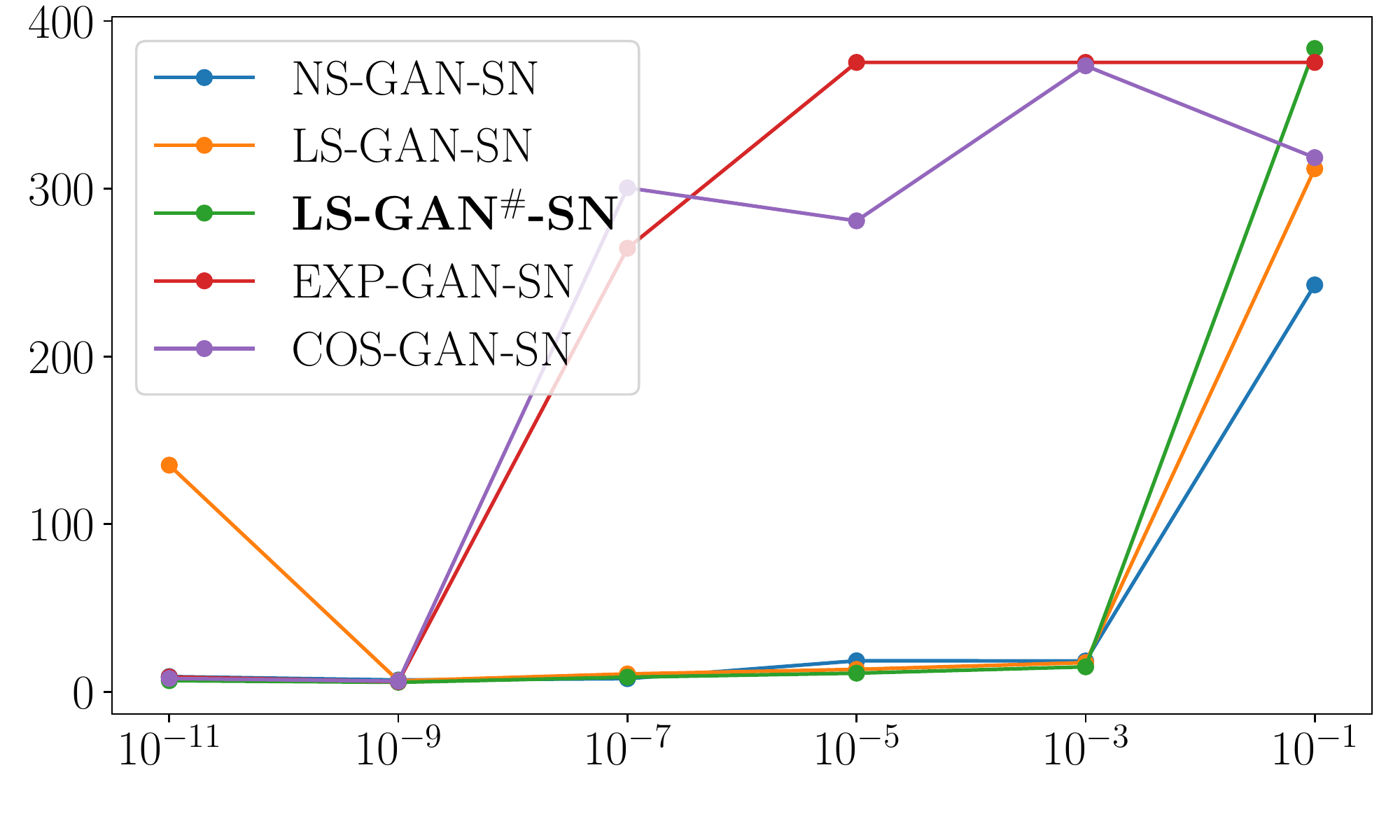}
    \end{subfigure}
\end{center}
  \caption{Samples of randomly generated images with LS-GAN$^\#$-SN of varying $\alpha$ ($k_{SN}=50.0$, CelebA). For the line plot, $x$-axis shows $\alpha$ (in log scale) and $y$-axis shows the FID scores.}
\label{fig:Scale_LSGANSN_zero_centered_CelebA_k_50}
\end{figure*}

\begin{figure*}[h]
\begin{center}
    \begin{subfigure}{0.32\textwidth}
        \centering
        \includegraphics[width=0.99\linewidth]{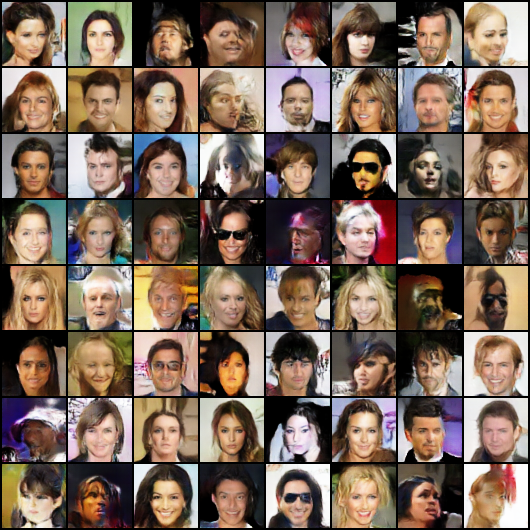}
        \subcaption{$\alpha=1e^{-11}$, $\mathrm{FID}=8.85$}
    \end{subfigure}
    \begin{subfigure}{0.32\textwidth}
        \centering
        \includegraphics[width=0.99\linewidth]{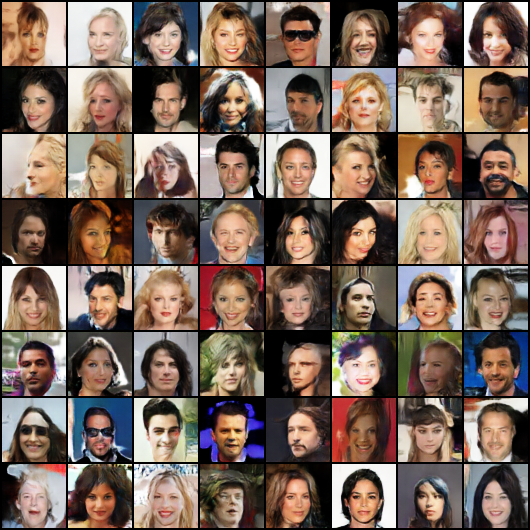}
        \subcaption{$\alpha=1e^{-9}$, $\mathrm{FID}=6.09$}
    \end{subfigure}
    \begin{subfigure}{0.32\textwidth}
        \centering
        \includegraphics[width=0.99\linewidth]{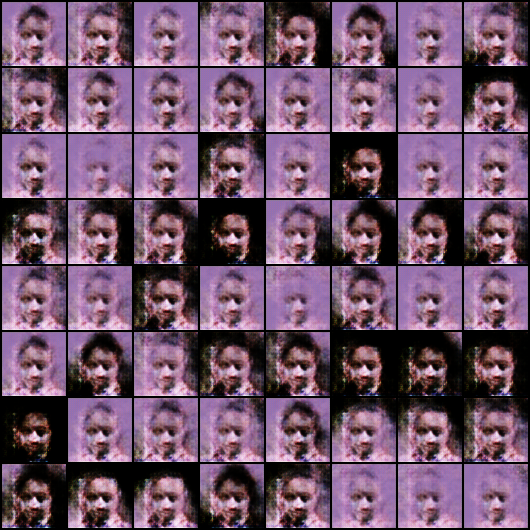}
        \subcaption{$\alpha=1e^{-7}$, $\mathrm{FID}=264.49$}
    \end{subfigure}
    \begin{subfigure}{0.32\textwidth}
        \centering
        \includegraphics[width=0.99\linewidth]{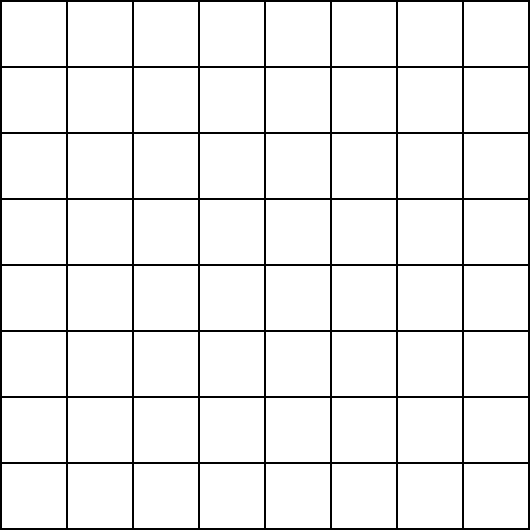}
        \subcaption{$\alpha=1e^{-5}$, $\mathrm{FID}=375.32$}
    \end{subfigure}
    \begin{subfigure}{0.32\textwidth}
        \centering
        \includegraphics[width=0.99\linewidth]{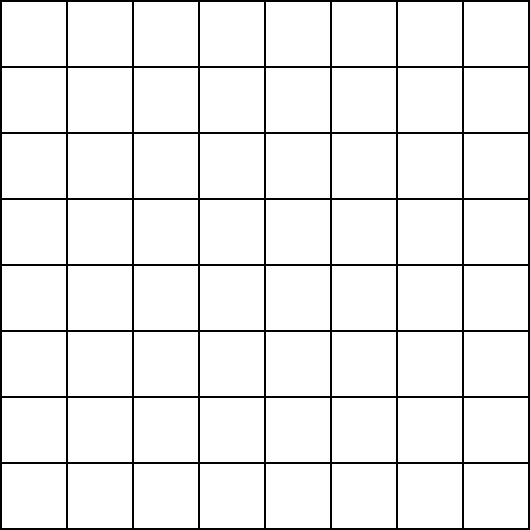}
        \subcaption{$\alpha=1e^{-3}$, $\mathrm{FID}=375.32$}
    \end{subfigure}
    \begin{subfigure}{0.32\textwidth}
        \centering
        \includegraphics[width=0.99\linewidth]{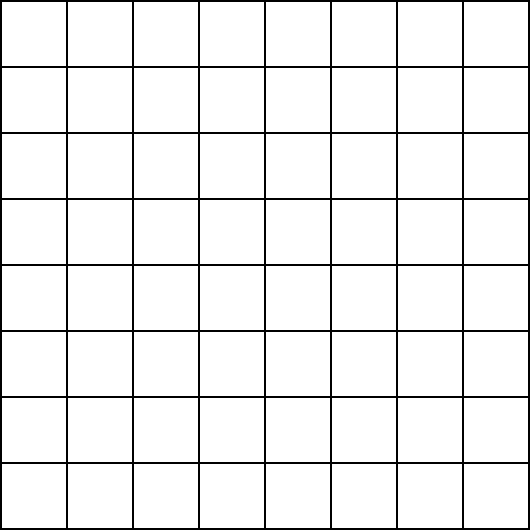}
        \subcaption{$\alpha=1e^{-1}$, $\mathrm{FID}=375.32$}
    \end{subfigure}
    \begin{subfigure}{0.6\textwidth}
        \includegraphics[width=.9\linewidth]{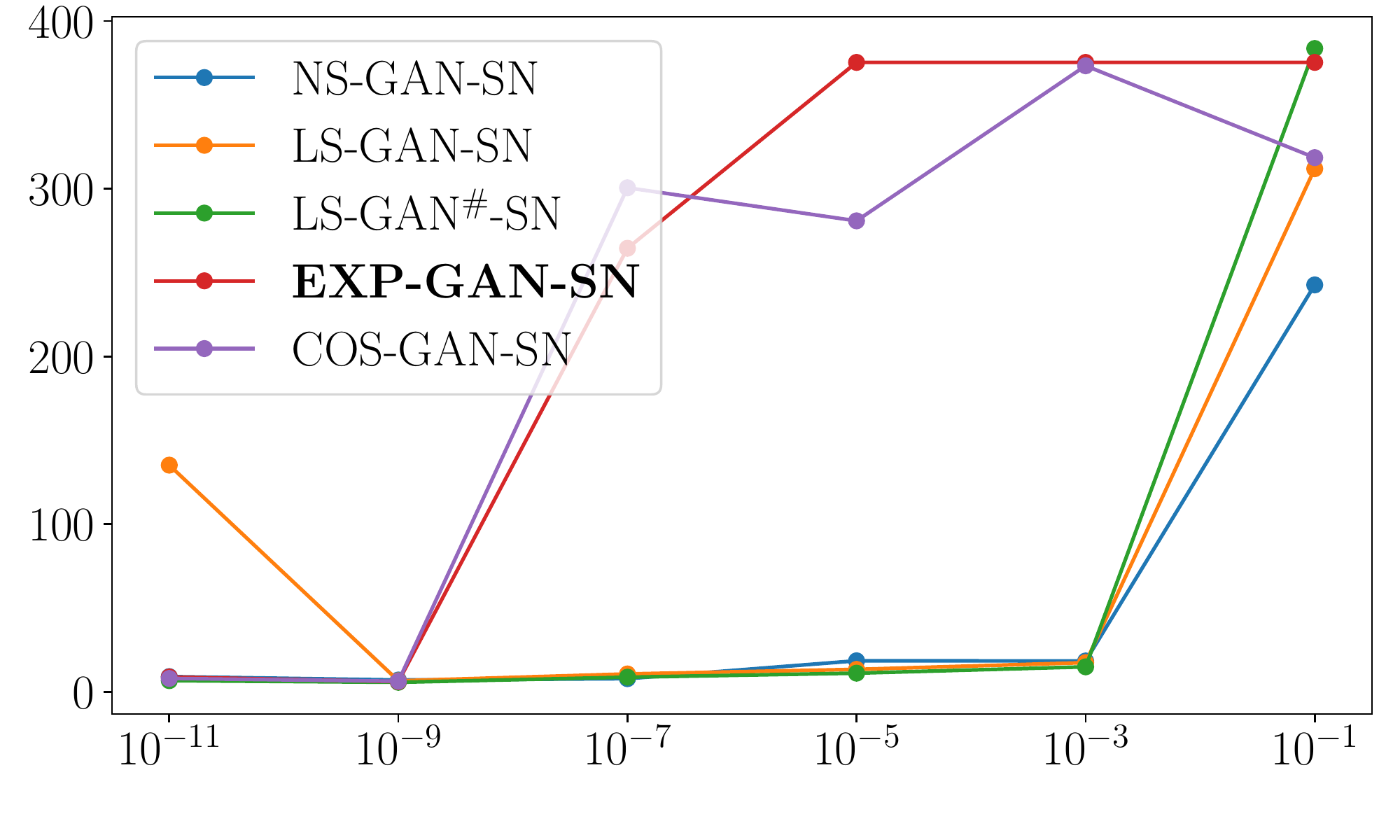}
    \end{subfigure}
\end{center}
  \caption{Samples of randomly generated images with EXP-GAN-SN of varying $\alpha$ ($k_{SN}=50.0$, CelebA). For the line plot, $x$-axis shows $\alpha$ (in log scale) and $y$-axis shows the FID scores.}
\label{fig:Scale_EXPGANSN_CelebA_k_50}
\end{figure*}

\begin{figure*}[h]
\begin{center}
    \begin{subfigure}{0.32\textwidth}
        \centering
        \includegraphics[width=0.99\linewidth]{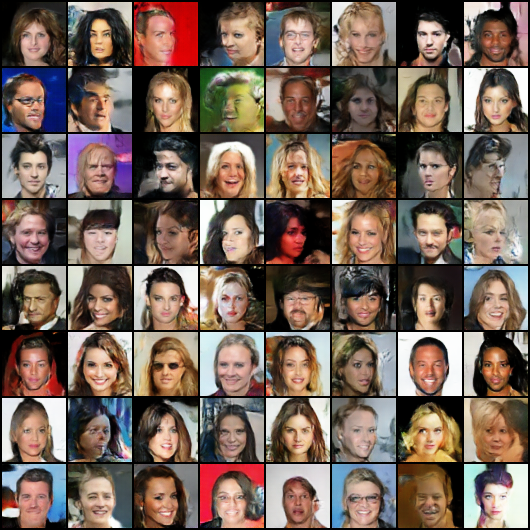}
        \subcaption{$\alpha=1e^{-11}$, $\mathrm{FID}=8.00$}
    \end{subfigure}
    \begin{subfigure}{0.32\textwidth}
        \centering
        \includegraphics[width=0.99\linewidth]{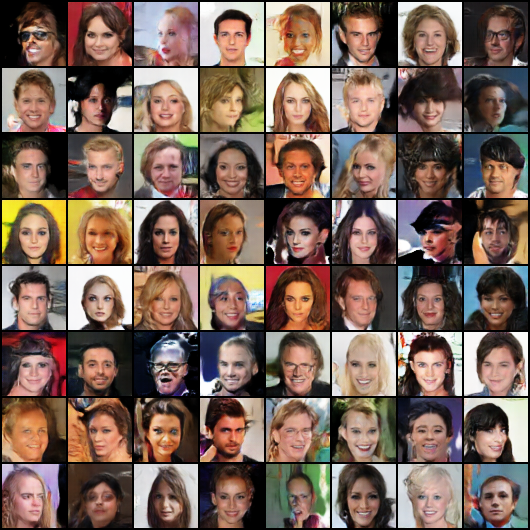}
        \subcaption{$\alpha=1e^{-9}$, $\mathrm{FID}=6.31$}
    \end{subfigure}
    \begin{subfigure}{0.32\textwidth}
        \centering
        \includegraphics[width=0.99\linewidth]{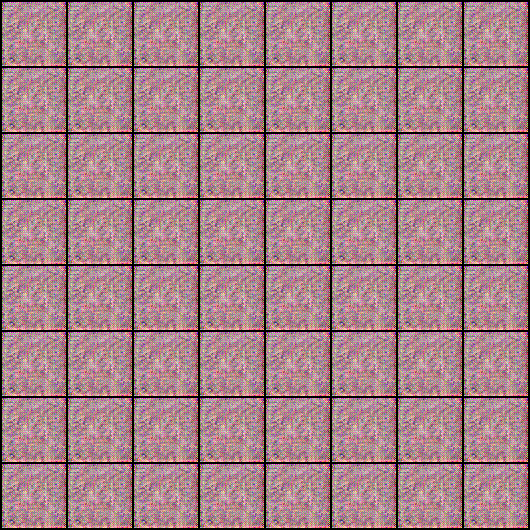}
        \subcaption{$\alpha=1e^{-7}$, $\mathrm{FID}=300.55$}
    \end{subfigure}
    \begin{subfigure}{0.32\textwidth}
        \centering
        \includegraphics[width=0.99\linewidth]{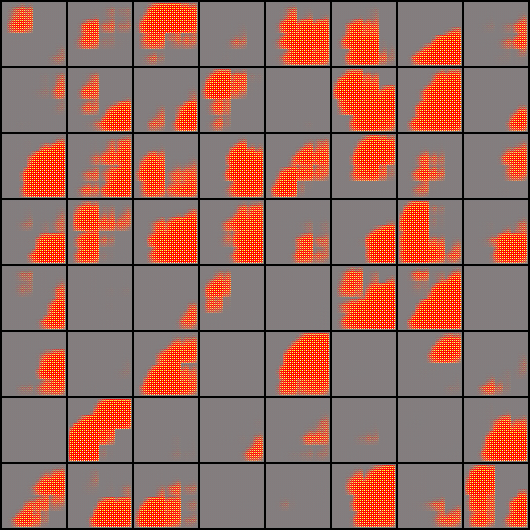}
        \subcaption{$\alpha=1e^{-5}$, $\mathrm{FID}=280.84$}
    \end{subfigure}
    \begin{subfigure}{0.32\textwidth}
        \centering
        \includegraphics[width=0.99\linewidth]{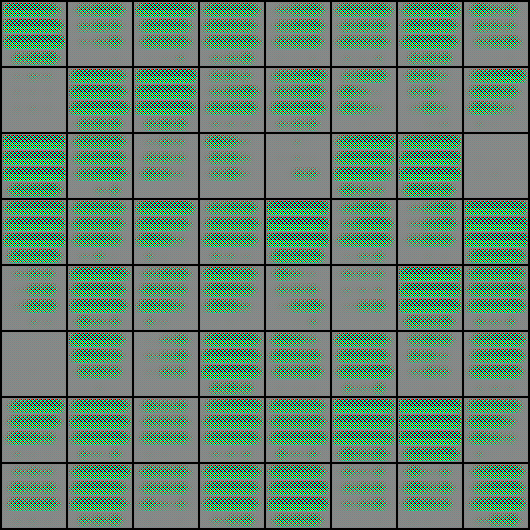}
        \subcaption{$\alpha=1e^{-3}$, $\mathrm{FID}=373.31$}
    \end{subfigure}
    \begin{subfigure}{0.32\textwidth}
        \centering
        \includegraphics[width=0.99\linewidth]{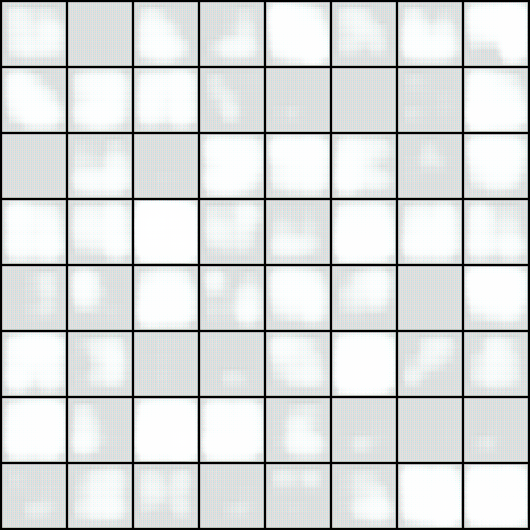}
        \subcaption{$\alpha=1e^{-1}$, $\mathrm{FID}=318.53$}
    \end{subfigure}
    \begin{subfigure}{0.6\textwidth}
        \includegraphics[width=.9\linewidth]{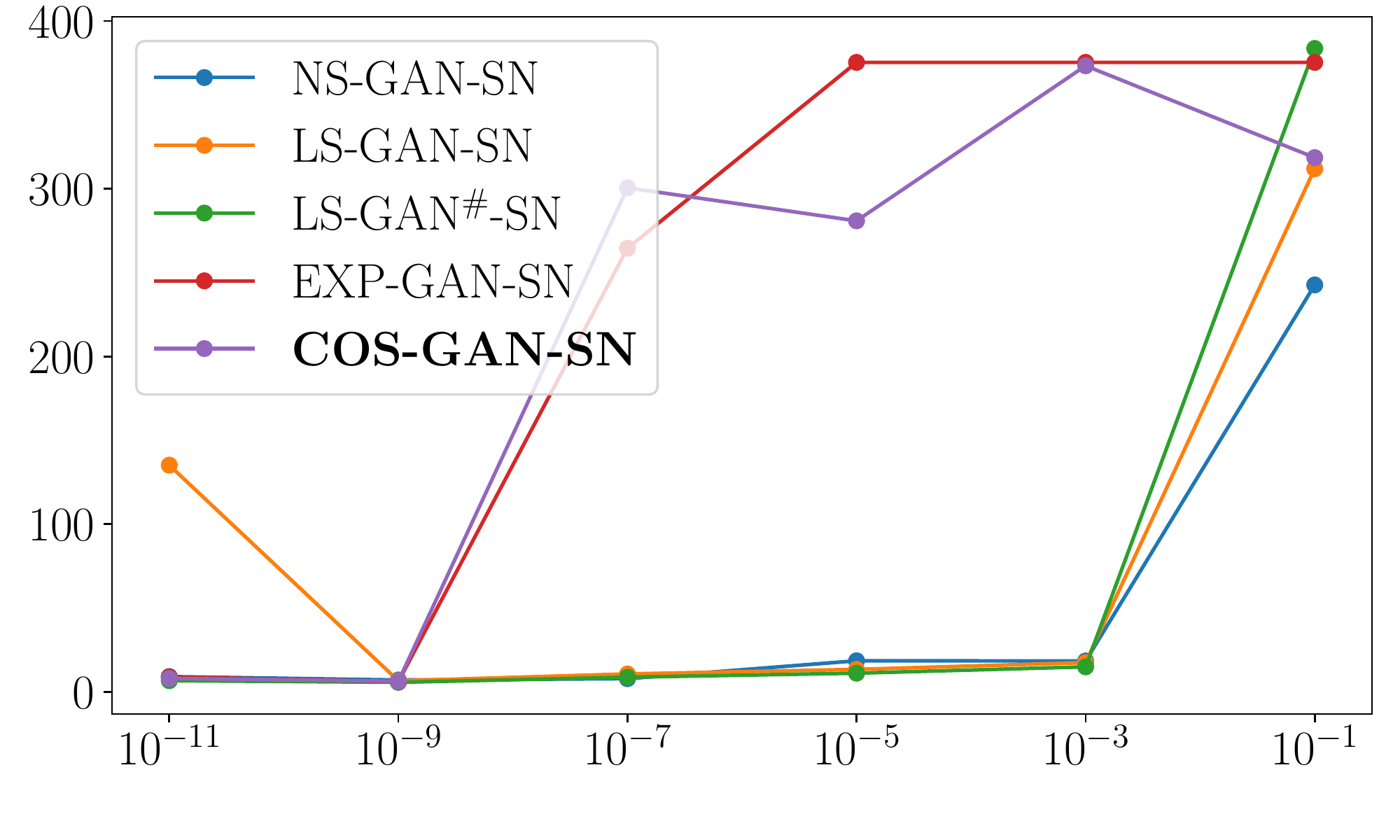}
    \end{subfigure}
\end{center}
  \caption{Samples of randomly generated images with COS-GAN-SN of varying $\alpha$ ($k_{SN}=50.0$, CelebA). For the line plot, $x$-axis shows $\alpha$ (in log scale) and $y$-axis shows the FID scores.}
\label{fig:Scale_COSGANSN_CelebA_k_50}
\end{figure*}

\FloatBarrier
\newpage

\vspace*{\fill}
\begin{table*}[h!]
\centering
\captionsetup{justification=centering}
\caption*{{\LARGE FID scores v.s. $\alpha$ ($k_{SN}=1.0$) of Different Loss Functions\\ \vspace*{5mm}
- MNIST -}}
\end{table*}
\vspace*{\fill}

\FloatBarrier
\newpage

\begin{figure*}[h]
\begin{center}
    \begin{subfigure}{0.32\textwidth}
        \centering
        \includegraphics[width=0.99\linewidth]{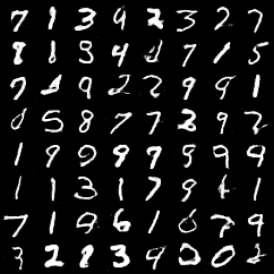}
        \subcaption{$\alpha=1e^{1}$, $\mathrm{FID}=6.55$}
    \end{subfigure}
    \begin{subfigure}{0.32\textwidth}
        \centering
        \includegraphics[width=0.99\linewidth]{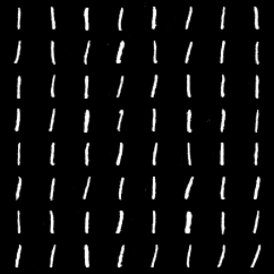}
        \subcaption{$\alpha=1e^{3}$, $\mathrm{FID}=148.97$}
    \end{subfigure}
    \begin{subfigure}{0.32\textwidth}
        \centering
        \includegraphics[width=0.99\linewidth]{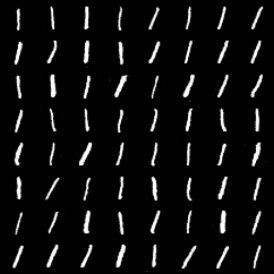}
        \subcaption{$\alpha=1e^{5}$, $\mathrm{FID}=134.44$}
    \end{subfigure}
    \begin{subfigure}{0.32\textwidth}
        \centering
        \includegraphics[width=0.99\linewidth]{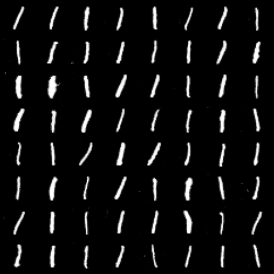}
        \subcaption{$\alpha=1e^{7}$, $\mathrm{FID}=133.82$}
    \end{subfigure}
    \begin{subfigure}{0.32\textwidth}
        \centering
        \includegraphics[width=0.99\linewidth]{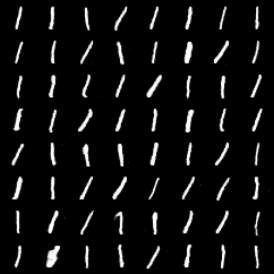}
        \subcaption{$\alpha=1e^{9}$, $\mathrm{FID}=130.21$}
    \end{subfigure}
    \begin{subfigure}{0.32\textwidth}
        \centering
        \includegraphics[width=0.99\linewidth]{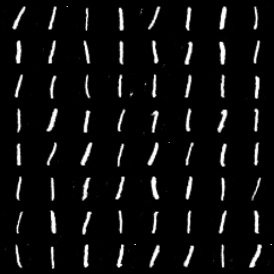}
        \subcaption{$\alpha=1e^{11}$, $\mathrm{FID}=131.87$}
    \end{subfigure}
    \begin{subfigure}{0.6\textwidth}
        \includegraphics[width=.9\linewidth]{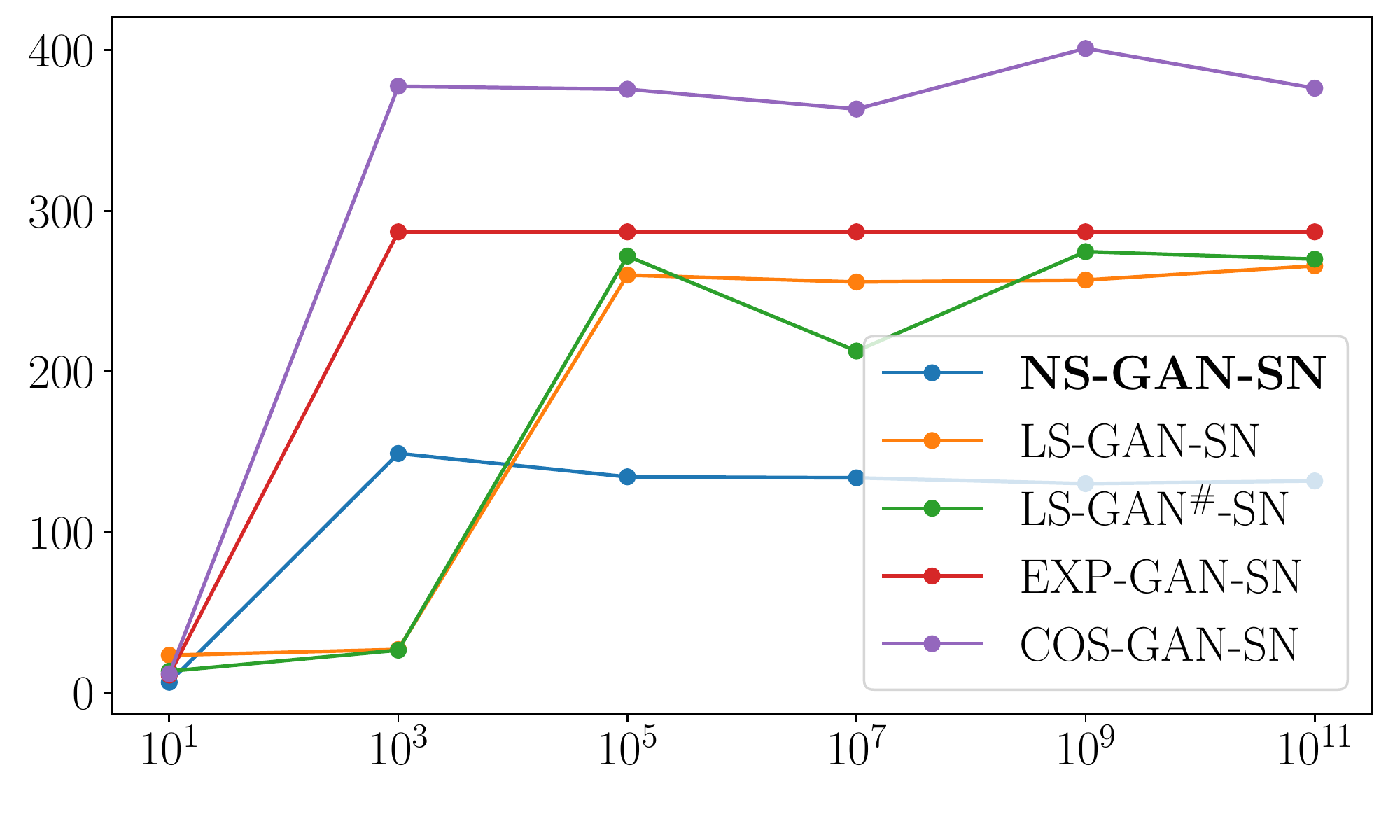}
    \end{subfigure}
\end{center}
  \caption{Samples of randomly generated images with NS-GAN-SN of varying $\alpha$ ($k_{SN}=1.0$, MNIST). For the line plot, $x$-axis shows $\alpha$ (in log scale) and $y$-axis shows the FID scores.}
\label{fig:Scale_NSGANSN_MNIST_k_1}
\end{figure*}

\begin{figure*}[h]
\begin{center}
    \begin{subfigure}{0.32\textwidth}
        \centering
        \includegraphics[width=0.99\linewidth]{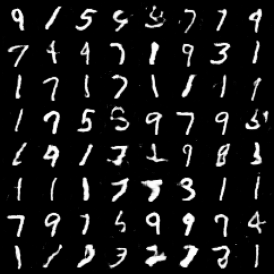}
        \subcaption{$\alpha=1e^{1}$, $\mathrm{FID}=23.37$}
    \end{subfigure}
    \begin{subfigure}{0.32\textwidth}
        \centering
        \includegraphics[width=0.99\linewidth]{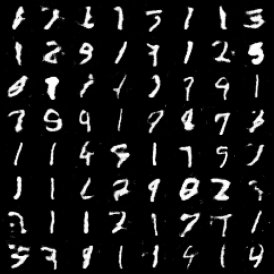}
        \subcaption{$\alpha=1e^{3}$, $\mathrm{FID}=26.96$}
    \end{subfigure}
    \begin{subfigure}{0.32\textwidth}
        \centering
        \includegraphics[width=0.99\linewidth]{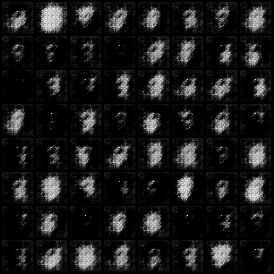}
        \subcaption{$\alpha=1e^{5}$, $\mathrm{FID}=260.05$}
    \end{subfigure}
    \begin{subfigure}{0.32\textwidth}
        \centering
        \includegraphics[width=0.99\linewidth]{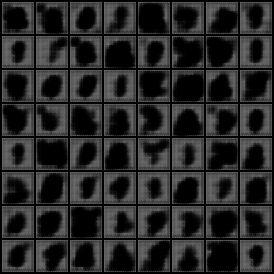}
        \subcaption{$\alpha=1e^{7}$, $\mathrm{FID}=255.73$}
    \end{subfigure}
    \begin{subfigure}{0.32\textwidth}
        \centering
        \includegraphics[width=0.99\linewidth]{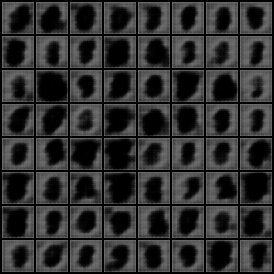}
        \subcaption{$\alpha=1e^{9}$, $\mathrm{FID}=256.96$}
    \end{subfigure}
    \begin{subfigure}{0.32\textwidth}
        \centering
        \includegraphics[width=0.99\linewidth]{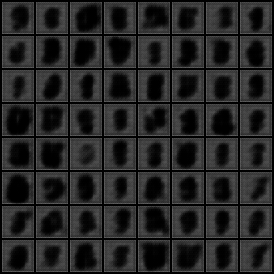}
        \subcaption{$\alpha=1e^{11}$, $\mathrm{FID}=265.76$}
    \end{subfigure}
    \begin{subfigure}{0.6\textwidth}
        \includegraphics[width=.9\linewidth]{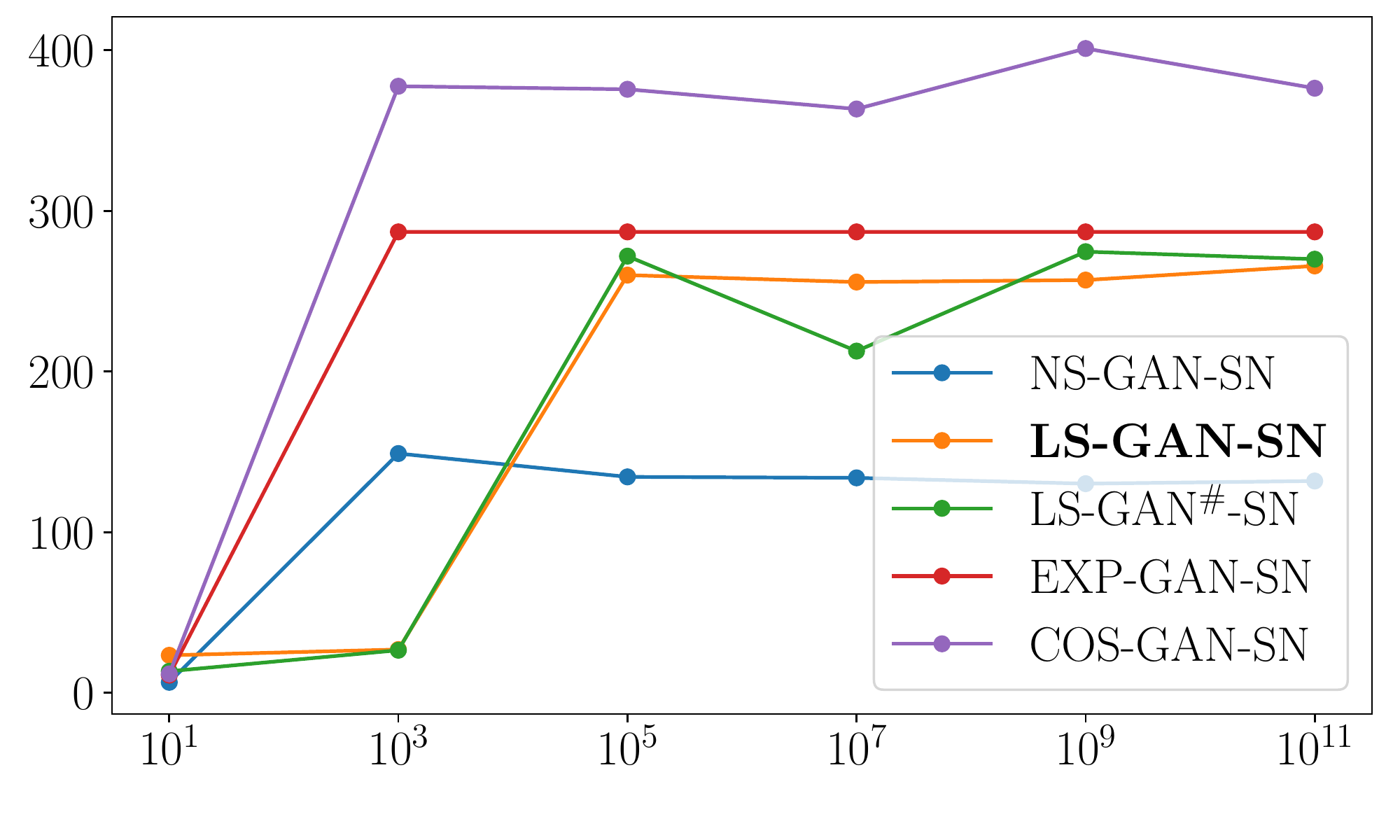}
    \end{subfigure}
\end{center}
  \caption{Samples of randomly generated images with LS-GAN-SN of varying $\alpha$ ($k_{SN}=1.0$, MNIST). For the line plot, $x$-axis shows $\alpha$ (in log scale) and $y$-axis shows the FID scores.}
\label{fig:Scale_LSGANSN_MNIST_k_1}
\end{figure*}

\begin{figure*}[h]
\begin{center}
    \begin{subfigure}{0.32\textwidth}
        \centering
        \includegraphics[width=0.99\linewidth]{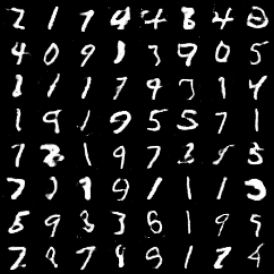}
        \subcaption{$\alpha=1e^{1}$, $\mathrm{FID}=13.43$}
    \end{subfigure}
    \begin{subfigure}{0.32\textwidth}
        \centering
        \includegraphics[width=0.99\linewidth]{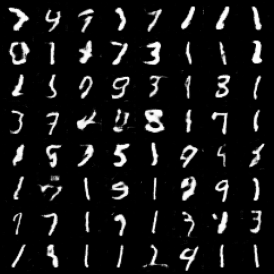}
        \subcaption{$\alpha=1e^{3}$, $\mathrm{FID}=26.51$}
    \end{subfigure}
    \begin{subfigure}{0.32\textwidth}
        \centering
        \includegraphics[width=0.99\linewidth]{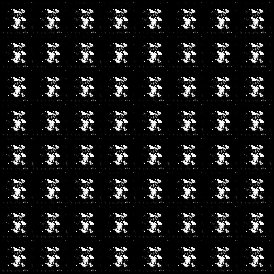}
        \subcaption{$\alpha=1e^{5}$, $\mathrm{FID}=271.85$}
    \end{subfigure}
    \begin{subfigure}{0.32\textwidth}
        \centering
        \includegraphics[width=0.99\linewidth]{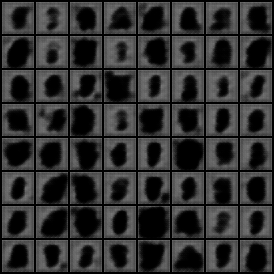}
        \subcaption{$\alpha=1e^{7}$, $\mathrm{FID}=212.74$}
    \end{subfigure}
    \begin{subfigure}{0.32\textwidth}
        \centering
        \includegraphics[width=0.99\linewidth]{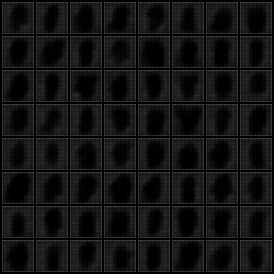}
        \subcaption{$\alpha=1e^{9}$, $\mathrm{FID}=274.63$}
    \end{subfigure}
    \begin{subfigure}{0.32\textwidth}
        \centering
        \includegraphics[width=0.99\linewidth]{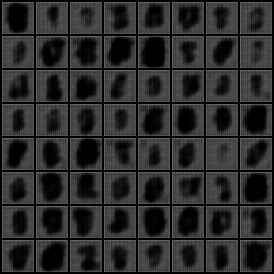}
        \subcaption{$\alpha=1e^{11}$, $\mathrm{FID}=269.96$}
    \end{subfigure}
    \begin{subfigure}{0.6\textwidth}
        \includegraphics[width=.9\linewidth]{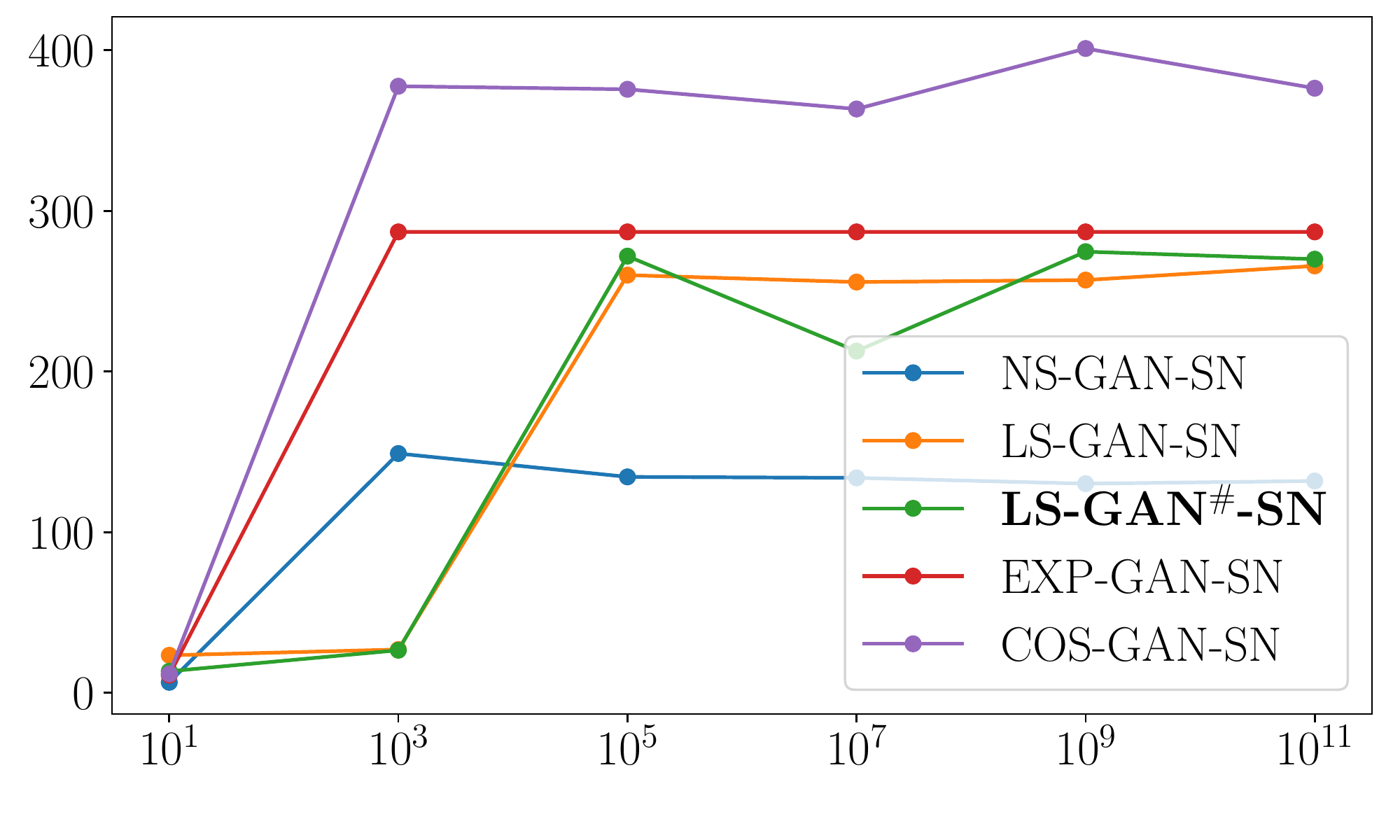}
    \end{subfigure}
\end{center}
  \caption{Samples of randomly generated images with LS-GAN$^\#$-SN of varying $\alpha$ ($k_{SN}=1.0$, MNIST). For the line plot, $x$-axis shows $\alpha$ (in log scale) and $y$-axis shows the FID scores.}
\label{fig:Scale_LSGANSN_zero_centered_MNIST_k_1}
\end{figure*}

\begin{figure*}[h]
\begin{center}
    \begin{subfigure}{0.32\textwidth}
        \centering
        \includegraphics[width=0.99\linewidth]{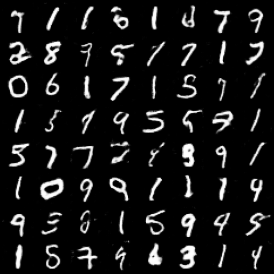}
        \subcaption{$\alpha=1e^{1}$, $\mathrm{FID}=11.02$}
    \end{subfigure}
    \begin{subfigure}{0.32\textwidth}
        \centering
        \includegraphics[width=0.99\linewidth]{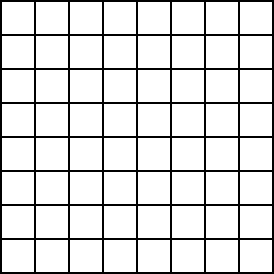}
        \subcaption{$\alpha=1e^{3}$, $\mathrm{FID}=286.96$}
    \end{subfigure}
    \begin{subfigure}{0.32\textwidth}
        \centering
        \includegraphics[width=0.99\linewidth]{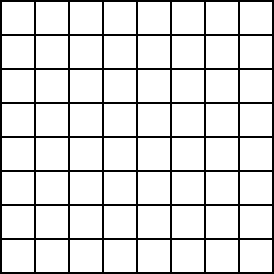}
        \subcaption{$\alpha=1e^{5}$, $\mathrm{FID}=286.96$}
    \end{subfigure}
    \begin{subfigure}{0.32\textwidth}
        \centering
        \includegraphics[width=0.99\linewidth]{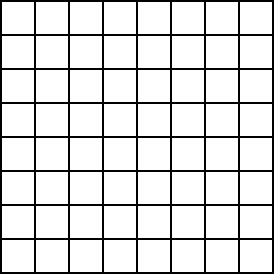}
        \subcaption{$\alpha=1e^{7}$, $\mathrm{FID}=286.96$}
    \end{subfigure}
    \begin{subfigure}{0.32\textwidth}
        \centering
        \includegraphics[width=0.99\linewidth]{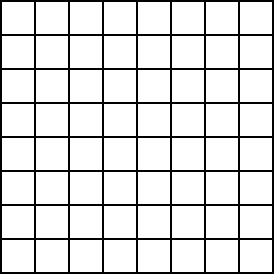}
        \subcaption{$\alpha=1e^{9}$, $\mathrm{FID}=286.96$}
    \end{subfigure}
    \begin{subfigure}{0.32\textwidth}
        \centering
        \includegraphics[width=0.99\linewidth]{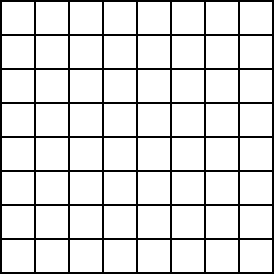}
        \subcaption{$\alpha=1e^{11}$, $\mathrm{FID}=286.96$}
    \end{subfigure}
    \begin{subfigure}{0.6\textwidth}
        \includegraphics[width=.9\linewidth]{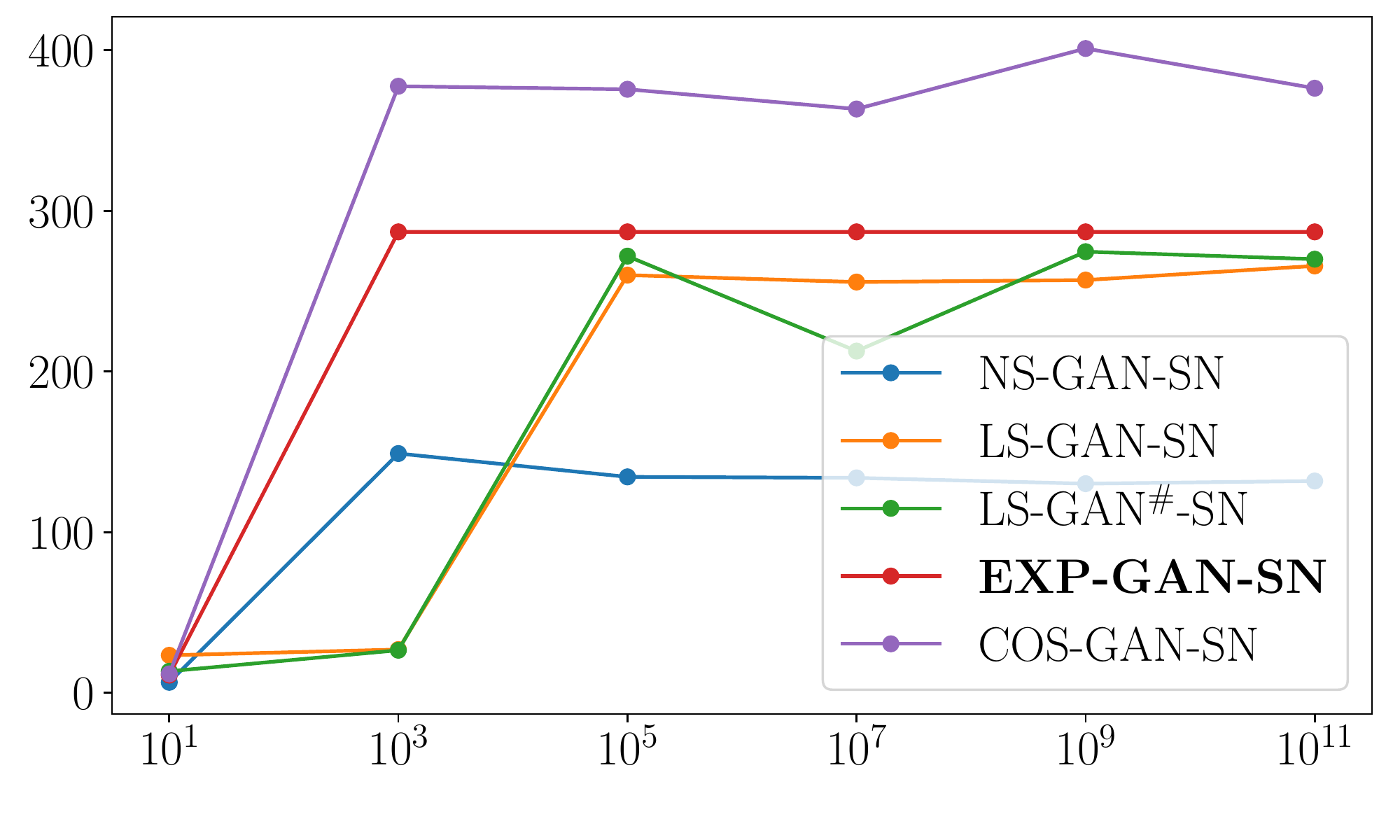}
    \end{subfigure}
\end{center}
  \caption{Samples of randomly generated images with EXP-GAN-SN of varying $\alpha$ ($k_{SN}=1.0$, MNIST). For the line plot, $x$-axis shows $\alpha$ (in log scale) and $y$-axis shows the FID scores.}
\label{fig:Scale_EXPGANSN_MNIST_k_1}
\end{figure*}

\begin{figure*}[h]
\begin{center}
    \begin{subfigure}{0.32\textwidth}
        \centering
        \includegraphics[width=0.99\linewidth]{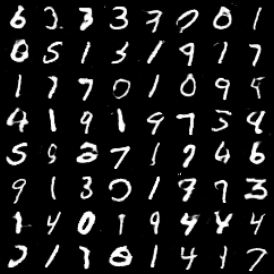}
        \subcaption{$\alpha=1e^{1}$, $\mathrm{FID}=11.79$}
    \end{subfigure}
    \begin{subfigure}{0.32\textwidth}
        \centering
        \includegraphics[width=0.99\linewidth]{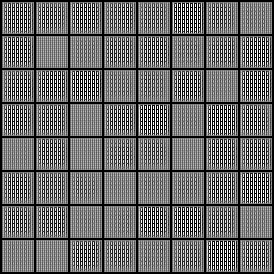}
        \subcaption{$\alpha=1e^{3}$, $\mathrm{FID}=377.62$}
    \end{subfigure}
    \begin{subfigure}{0.32\textwidth}
        \centering
        \includegraphics[width=0.99\linewidth]{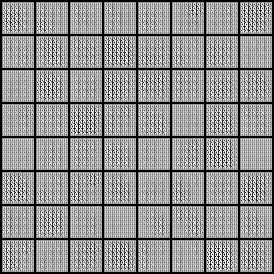}
        \subcaption{$\alpha=1e^{5}$, $\mathrm{FID}=375.72$}
    \end{subfigure}
    \begin{subfigure}{0.32\textwidth}
        \centering
        \includegraphics[width=0.99\linewidth]{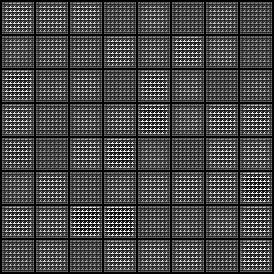}
        \subcaption{$\alpha=1e^{7}$, $\mathrm{FID}=363.45$}
    \end{subfigure}
    \begin{subfigure}{0.32\textwidth}
        \centering
        \includegraphics[width=0.99\linewidth]{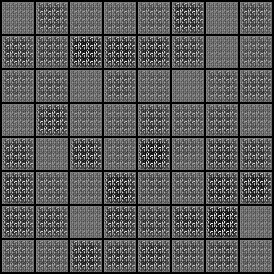}
        \subcaption{$\alpha=1e^{9}$, $\mathrm{FID}=401.12$}
    \end{subfigure}
    \begin{subfigure}{0.32\textwidth}
        \centering
        \includegraphics[width=0.99\linewidth]{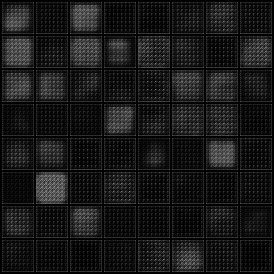}
        \subcaption{$\alpha=1e^{11}$, $\mathrm{FID}=376.39$}
    \end{subfigure}
    \begin{subfigure}{0.6\textwidth}
        \includegraphics[width=.9\linewidth]{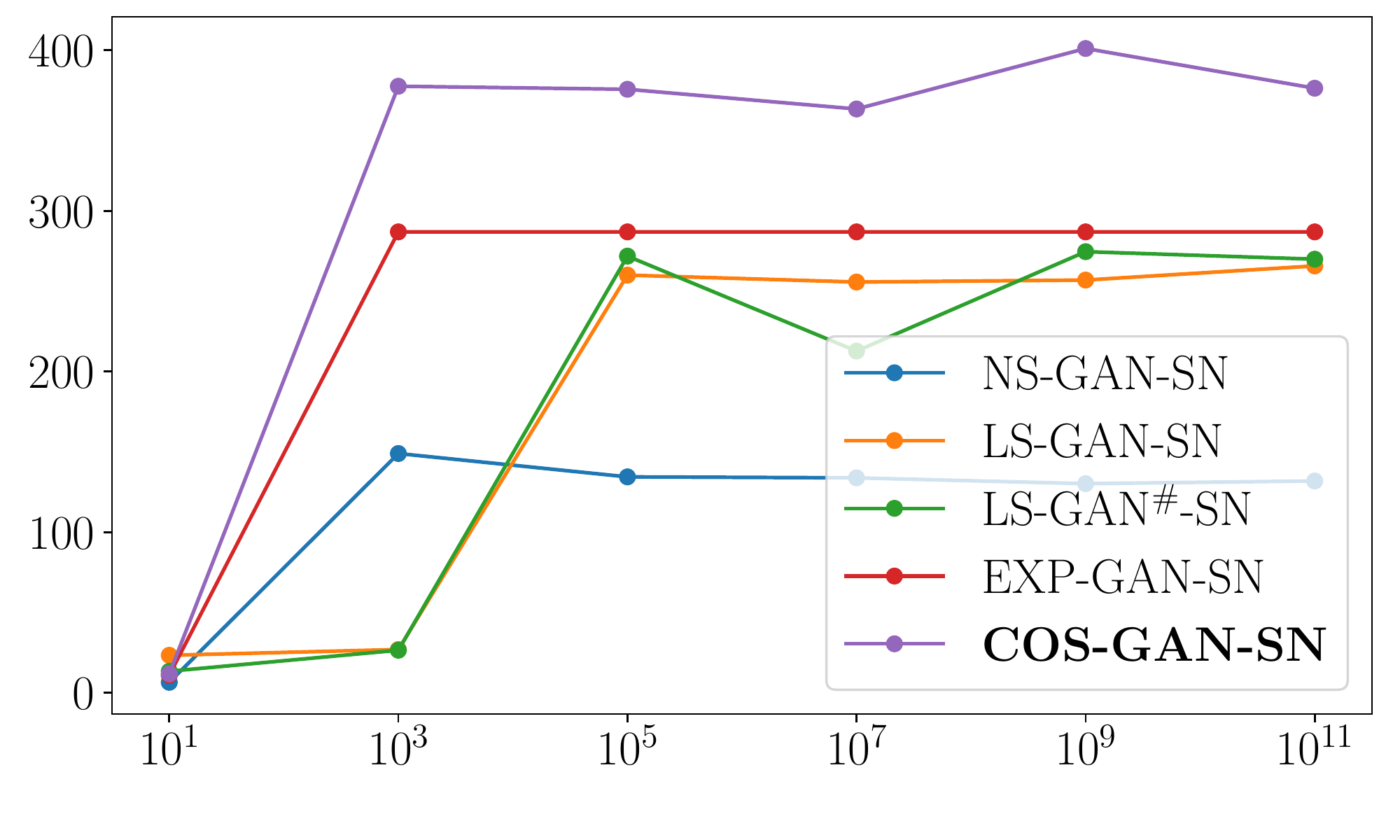}
    \end{subfigure}
\end{center}
  \caption{Samples of randomly generated images with COS-GAN-SN of varying $\alpha$ ($k_{SN}=1.0$, MNIST). For the line plot, $x$-axis shows $\alpha$ (in log scale) and $y$-axis shows the FID scores.}
\label{fig:Scale_COSGANSN_MNIST_k_1}
\end{figure*}

\FloatBarrier
\newpage

\vspace*{\fill}
\begin{table*}[h!]
\centering
\captionsetup{justification=centering}
\caption*{{\LARGE FID scores v.s. $\alpha$ ($k_{SN}=1.0$) of Different Loss Functions\\ \vspace*{5mm}
- CIFAR10 -}}
\end{table*}
\vspace*{\fill}

\FloatBarrier
\newpage

\begin{figure*}[h]
\begin{center}
    \begin{subfigure}{0.32\textwidth}
        \centering
        \includegraphics[width=0.99\linewidth]{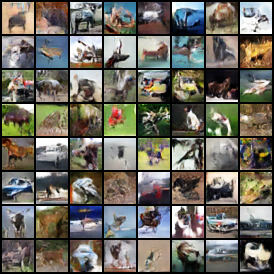}
        \subcaption{$\alpha=1e^{1}$, $\mathrm{FID}=17.63$}
    \end{subfigure}
    \begin{subfigure}{0.32\textwidth}
        \centering
        \includegraphics[width=0.99\linewidth]{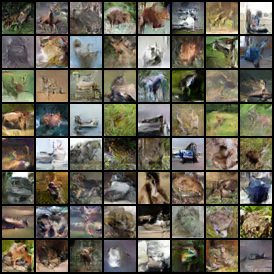}
        \subcaption{$\alpha=1e^{3}$, $\mathrm{FID}=47.31$}
    \end{subfigure}
    \begin{subfigure}{0.32\textwidth}
        \centering
        \includegraphics[width=0.99\linewidth]{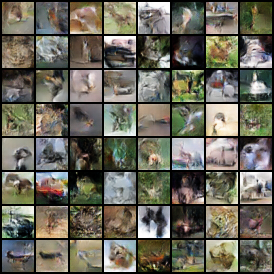}
        \subcaption{$\alpha=1e^{5}$, $\mathrm{FID}=46.85$}
    \end{subfigure}
    \begin{subfigure}{0.32\textwidth}
        \centering
        \includegraphics[width=0.99\linewidth]{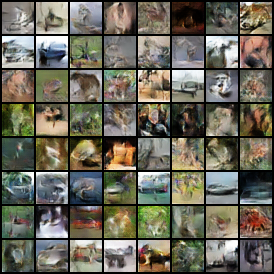}
        \subcaption{$\alpha=1e^{7}$, $\mathrm{FID}=45.44$}
    \end{subfigure}
    \begin{subfigure}{0.32\textwidth}
        \centering
        \includegraphics[width=0.99\linewidth]{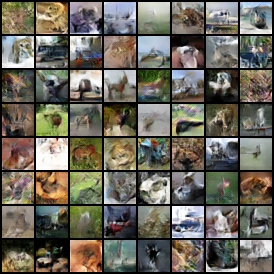}
        \subcaption{$\alpha=1e^{9}$, $\mathrm{FID}=45.67$}
    \end{subfigure}
    \begin{subfigure}{0.32\textwidth}
        \centering
        \includegraphics[width=0.99\linewidth]{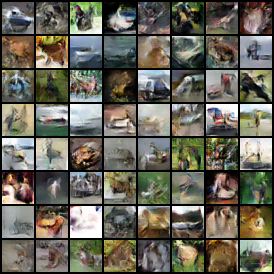}
        \subcaption{$\alpha=1e^{11}$, $\mathrm{FID}=39.90$}
    \end{subfigure}
    \begin{subfigure}{0.6\textwidth}
        \includegraphics[width=.9\linewidth]{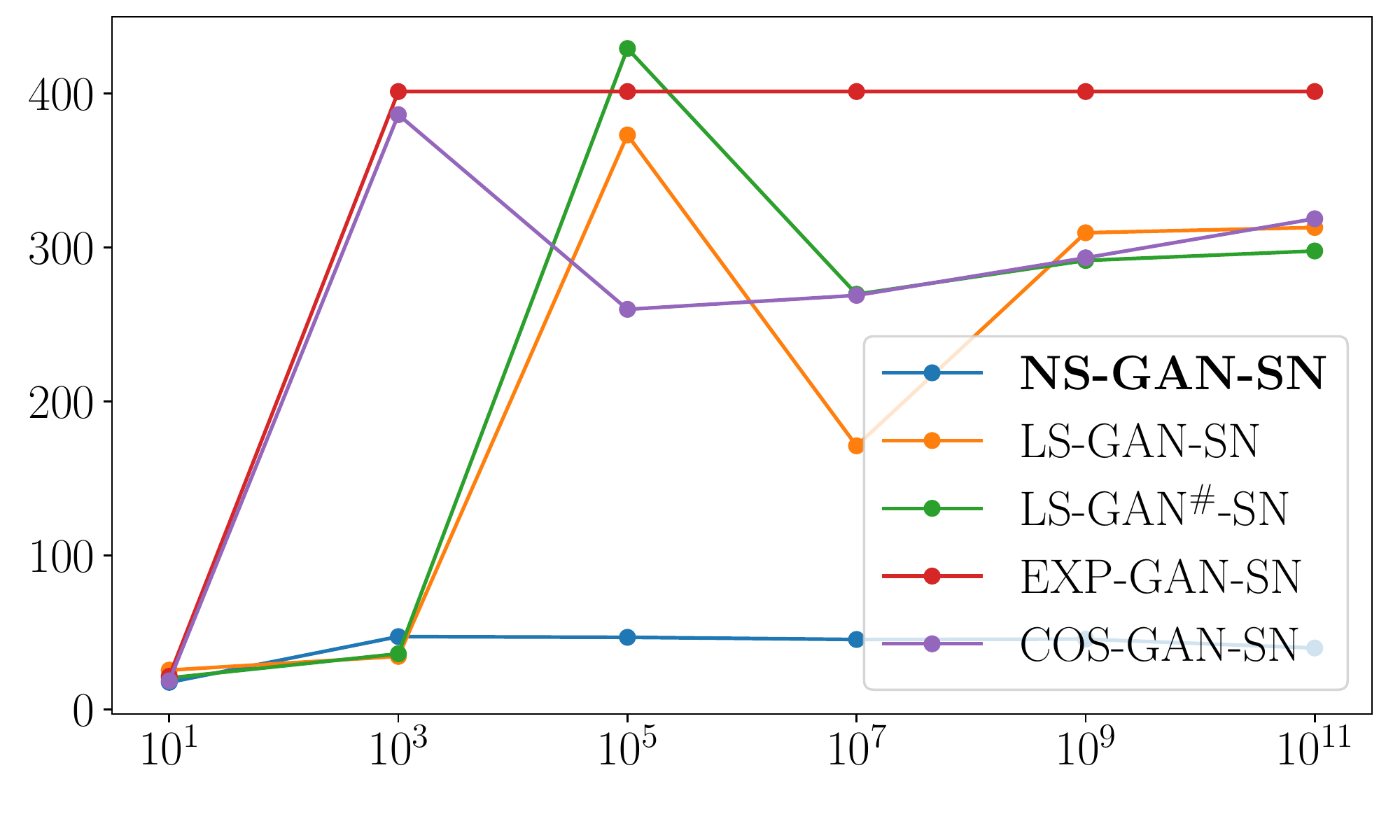}
    \end{subfigure}
\end{center}
  \caption{Samples of randomly generated images with NS-GAN-SN of varying $\alpha$ ($k_{SN}=1.0$, CIFAR10). For the line plot, $x$-axis shows $\alpha$ (in log scale) and $y$-axis shows the FID scores.}
\label{fig:Scale_NSGANSN_CIFAR10_k_1}
\end{figure*}

\begin{figure*}[h]
\begin{center}
    \begin{subfigure}{0.32\textwidth}
        \centering
        \includegraphics[width=0.99\linewidth]{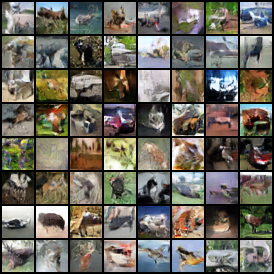}
        \subcaption{$\alpha=1e^{1}$, $\mathrm{FID}=25.55$}
    \end{subfigure}
    \begin{subfigure}{0.32\textwidth}
        \centering
        \includegraphics[width=0.99\linewidth]{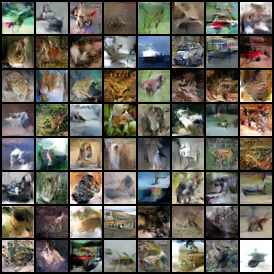}
        \subcaption{$\alpha=1e^{3}$, $\mathrm{FID}=34.44$}
    \end{subfigure}
    \begin{subfigure}{0.32\textwidth}
        \centering
        \includegraphics[width=0.99\linewidth]{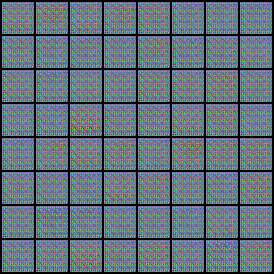}
        \subcaption{$\alpha=1e^{5}$, $\mathrm{FID}=373.07$}
    \end{subfigure}
    \begin{subfigure}{0.32\textwidth}
        \centering
        \includegraphics[width=0.99\linewidth]{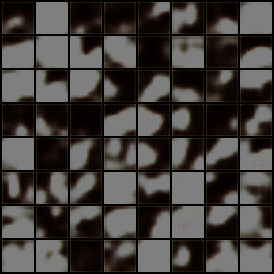}
        \subcaption{$\alpha=1e^{7}$, $\mathrm{FID}=171.18$}
    \end{subfigure}
    \begin{subfigure}{0.32\textwidth}
        \centering
        \includegraphics[width=0.99\linewidth]{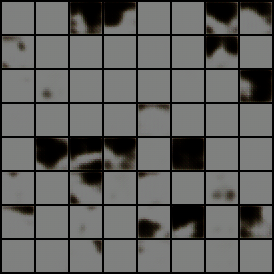}
        \subcaption{$\alpha=1e^{9}$, $\mathrm{FID}=309.55$}
    \end{subfigure}
    \begin{subfigure}{0.32\textwidth}
        \centering
        \includegraphics[width=0.99\linewidth]{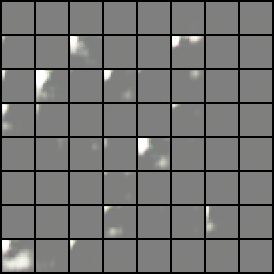}
        \subcaption{$\alpha=1e^{11}$, $\mathrm{FID}=312.96$}
    \end{subfigure}
    \begin{subfigure}{0.6\textwidth}
        \includegraphics[width=.9\linewidth]{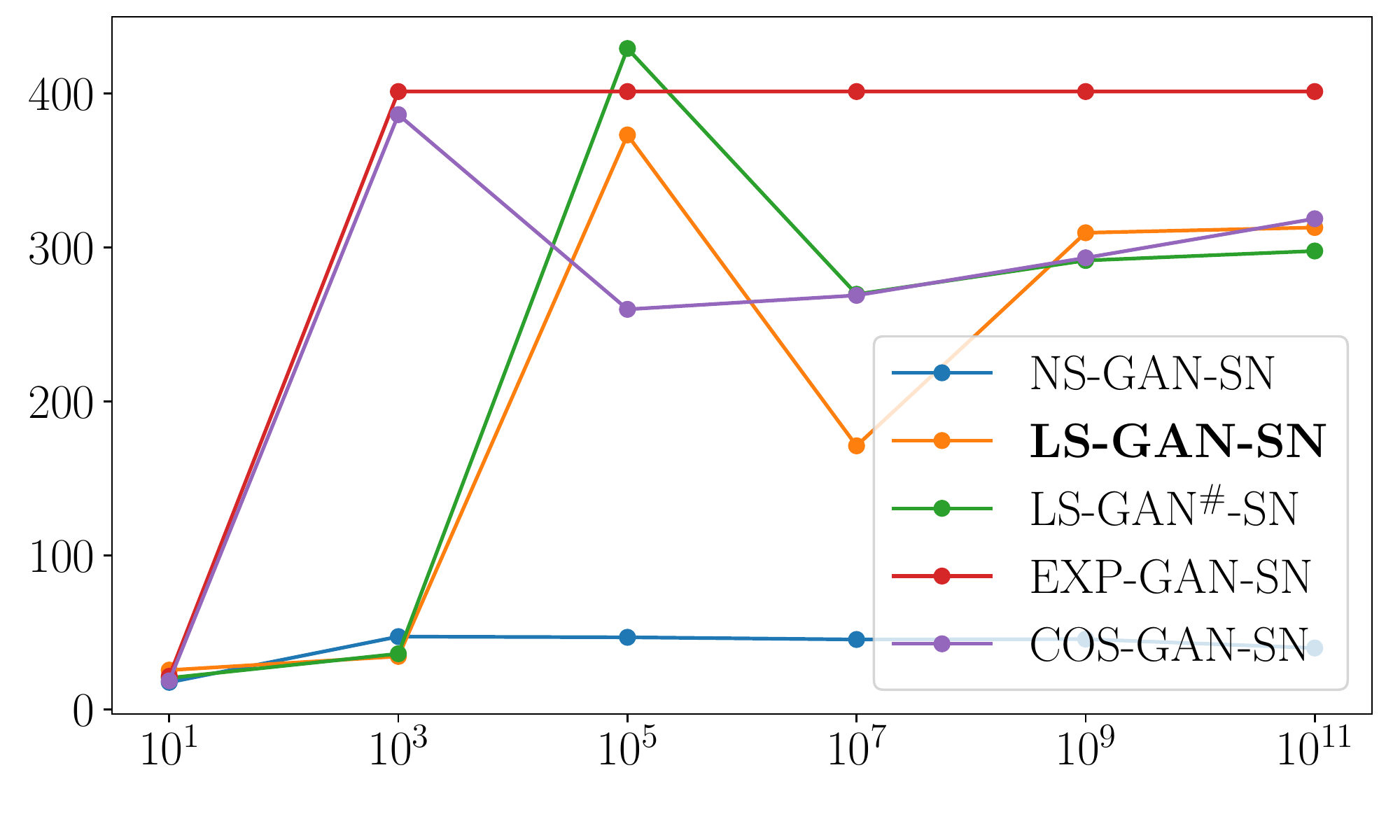}
    \end{subfigure}
\end{center}
  \caption{Samples of randomly generated images with LS-GAN-SN of varying $\alpha$ ($k_{SN}=1.0$, CIFAR10). For the line plot, $x$-axis shows $\alpha$ (in log scale) and $y$-axis shows the FID scores.}
\label{fig:Scale_LSGANSN_CIFAR10_k_1}
\end{figure*}

\begin{figure*}[h]
\begin{center}
    \begin{subfigure}{0.32\textwidth}
        \centering
        \includegraphics[width=0.99\linewidth]{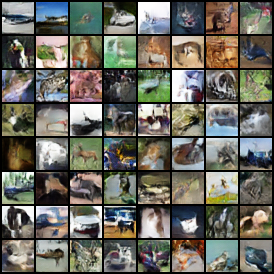}
        \subcaption{$\alpha=1e^{1}$, $\mathrm{FID}=20.45$}
    \end{subfigure}
    \begin{subfigure}{0.32\textwidth}
        \centering
        \includegraphics[width=0.99\linewidth]{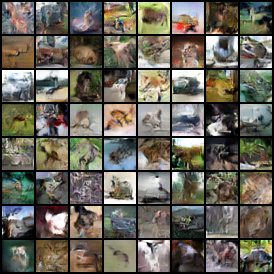}
        \subcaption{$\alpha=1e^{3}$, $\mathrm{FID}=36.18$}
    \end{subfigure}
    \begin{subfigure}{0.32\textwidth}
        \centering
        \includegraphics[width=0.99\linewidth]{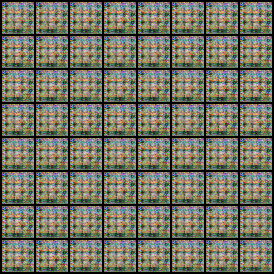}
        \subcaption{$\alpha=1e^{5}$, $\mathrm{FID}=429.21$}
    \end{subfigure}
    \begin{subfigure}{0.32\textwidth}
        \centering
        \includegraphics[width=0.99\linewidth]{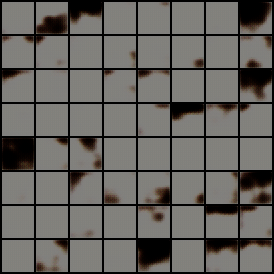}
        \subcaption{$\alpha=1e^{7}$, $\mathrm{FID}=269.63$}
    \end{subfigure}
    \begin{subfigure}{0.32\textwidth}
        \centering
        \includegraphics[width=0.99\linewidth]{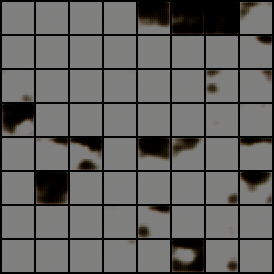}
        \subcaption{$\alpha=1e^{9}$, $\mathrm{FID}=291.55$}
    \end{subfigure}
    \begin{subfigure}{0.32\textwidth}
        \centering
        \includegraphics[width=0.99\linewidth]{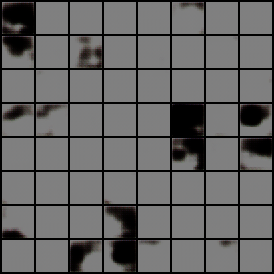}
        \subcaption{$\alpha=1e^{11}$, $\mathrm{FID}=297.71$}
    \end{subfigure}
    \begin{subfigure}{0.6\textwidth}
        \includegraphics[width=.9\linewidth]{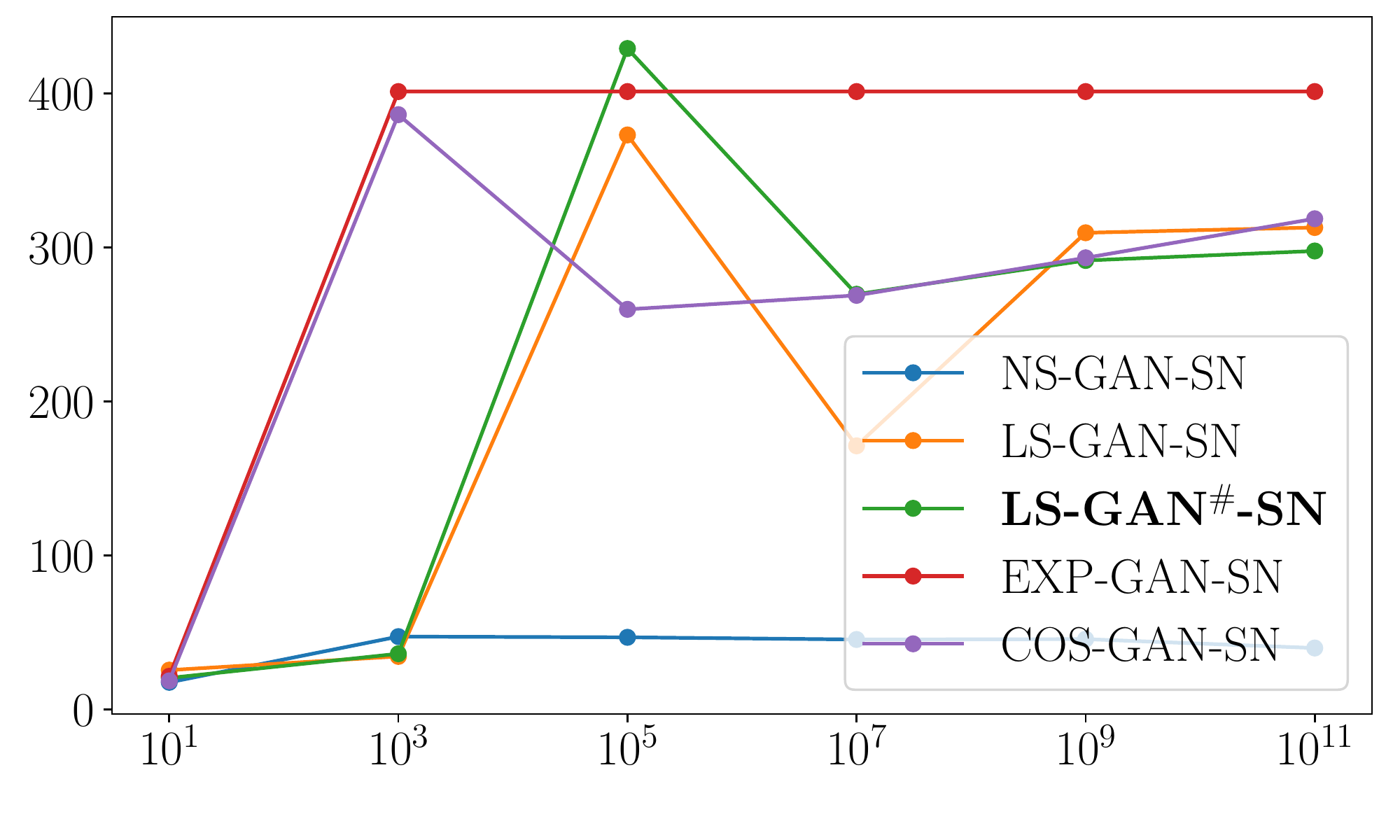}
    \end{subfigure}
\end{center}
  \caption{Samples of randomly generated images with LS-GAN$^\#$-SN of varying $\alpha$ ($k_{SN}=1.0$, CIFAR10). For the line plot, $x$-axis shows $\alpha$ (in log scale) and $y$-axis shows the FID scores.}
\label{fig:Scale_LSGANSN_zero_centered_CIFAR10_k_1}
\end{figure*}

\begin{figure*}[h]
\begin{center}
    \begin{subfigure}{0.32\textwidth}
        \centering
        \includegraphics[width=0.99\linewidth]{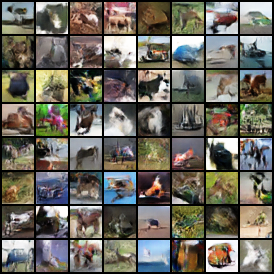}
        \subcaption{$\alpha=1e^{1}$, $\mathrm{FID}=21.56$}
    \end{subfigure}
    \begin{subfigure}{0.32\textwidth}
        \centering
        \includegraphics[width=0.99\linewidth]{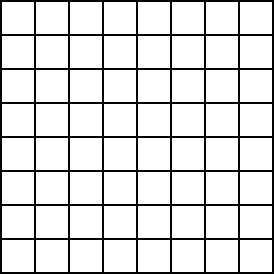}
        \subcaption{$\alpha=1e^{3}$, $\mathrm{FID}=401.24$}
    \end{subfigure}
    \begin{subfigure}{0.32\textwidth}
        \centering
        \includegraphics[width=0.99\linewidth]{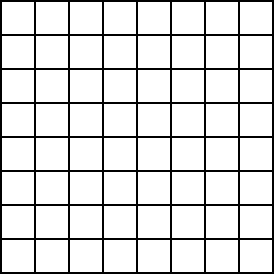}
        \subcaption{$\alpha=1e^{5}$, $\mathrm{FID}=401.24$}
    \end{subfigure}
    \begin{subfigure}{0.32\textwidth}
        \centering
        \includegraphics[width=0.99\linewidth]{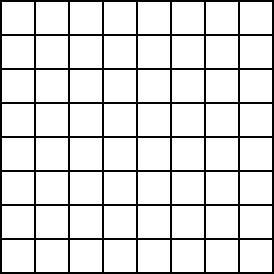}
        \subcaption{$\alpha=1e^{7}$, $\mathrm{FID}=401.24$}
    \end{subfigure}
    \begin{subfigure}{0.32\textwidth}
        \centering
        \includegraphics[width=0.99\linewidth]{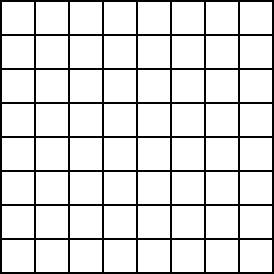}
        \subcaption{$\alpha=1e^{9}$, $\mathrm{FID}=401.24$}
    \end{subfigure}
    \begin{subfigure}{0.32\textwidth}
        \centering
        \includegraphics[width=0.99\linewidth]{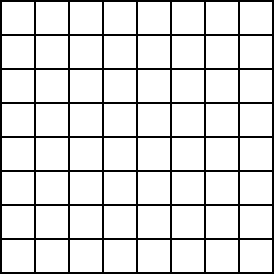}
        \subcaption{$\alpha=1e^{11}$, $\mathrm{FID}=401.24$}
    \end{subfigure}
    \begin{subfigure}{0.6\textwidth}
        \includegraphics[width=.9\linewidth]{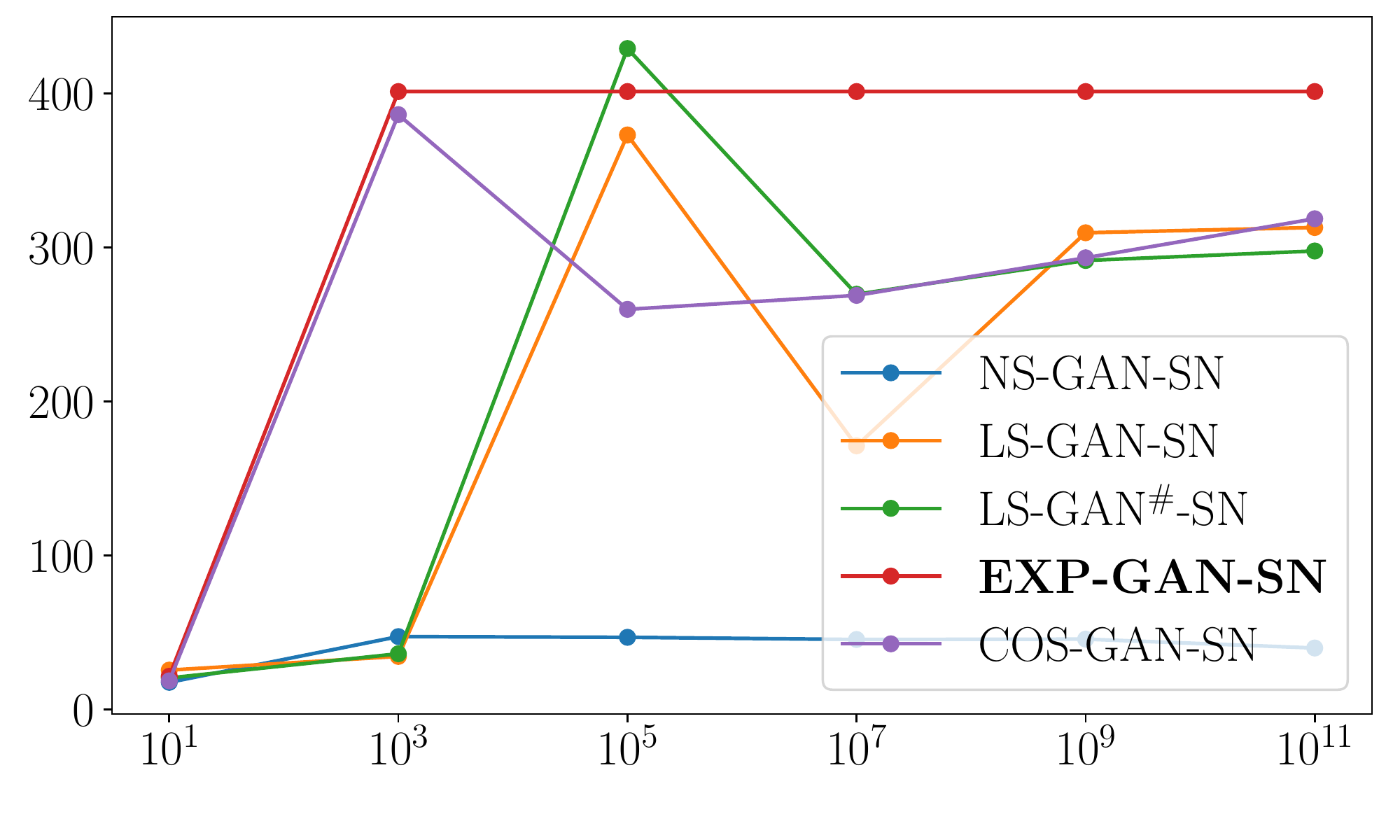}
    \end{subfigure}
\end{center}
  \caption{Samples of randomly generated images with EXP-GAN-SN of varying $\alpha$ ($k_{SN}=1.0$, CIFAR10). For the line plot, $x$-axis shows $\alpha$ (in log scale) and $y$-axis shows the FID scores.}
\label{fig:Scale_EXPGANSN_CIFAR10_k_1}
\end{figure*}

\begin{figure*}[h]
\begin{center}
    \begin{subfigure}{0.32\textwidth}
        \centering
        \includegraphics[width=0.99\linewidth]{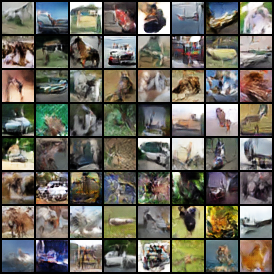}
        \subcaption{$\alpha=1e^{1}$, $\mathrm{FID}=18.59$}
    \end{subfigure}
    \begin{subfigure}{0.32\textwidth}
        \centering
        \includegraphics[width=0.99\linewidth]{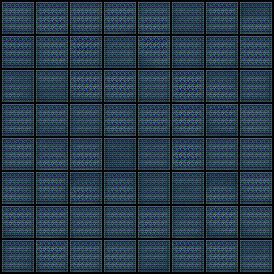}
        \subcaption{$\alpha=1e^{3}$, $\mathrm{FID}=386.24$}
    \end{subfigure}
    \begin{subfigure}{0.32\textwidth}
        \centering
        \includegraphics[width=0.99\linewidth]{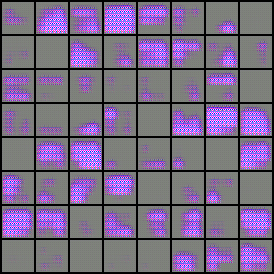}
        \subcaption{$\alpha=1e^{5}$, $\mathrm{FID}=259.83$}
    \end{subfigure}
    \begin{subfigure}{0.32\textwidth}
        \centering
        \includegraphics[width=0.99\linewidth]{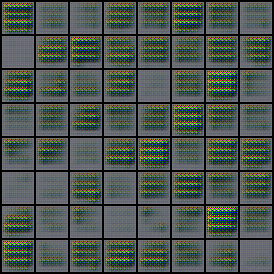}
        \subcaption{$\alpha=1e^{7}$, $\mathrm{FID}=268.89$}
    \end{subfigure}
    \begin{subfigure}{0.32\textwidth}
        \centering
        \includegraphics[width=0.99\linewidth]{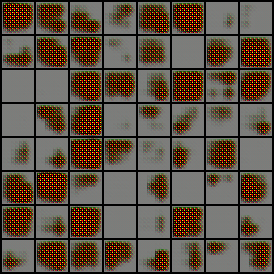}
        \subcaption{$\alpha=1e^{9}$, $\mathrm{FID}=293.29$}
    \end{subfigure}
    \begin{subfigure}{0.32\textwidth}
        \centering
        \includegraphics[width=0.99\linewidth]{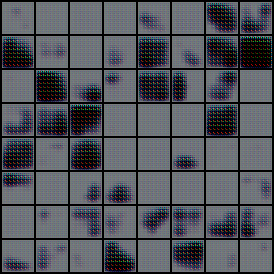}
        \subcaption{$\alpha=1e^{11}$, $\mathrm{FID}=318.65$}
    \end{subfigure}
    \begin{subfigure}{0.6\textwidth}
        \includegraphics[width=.9\linewidth]{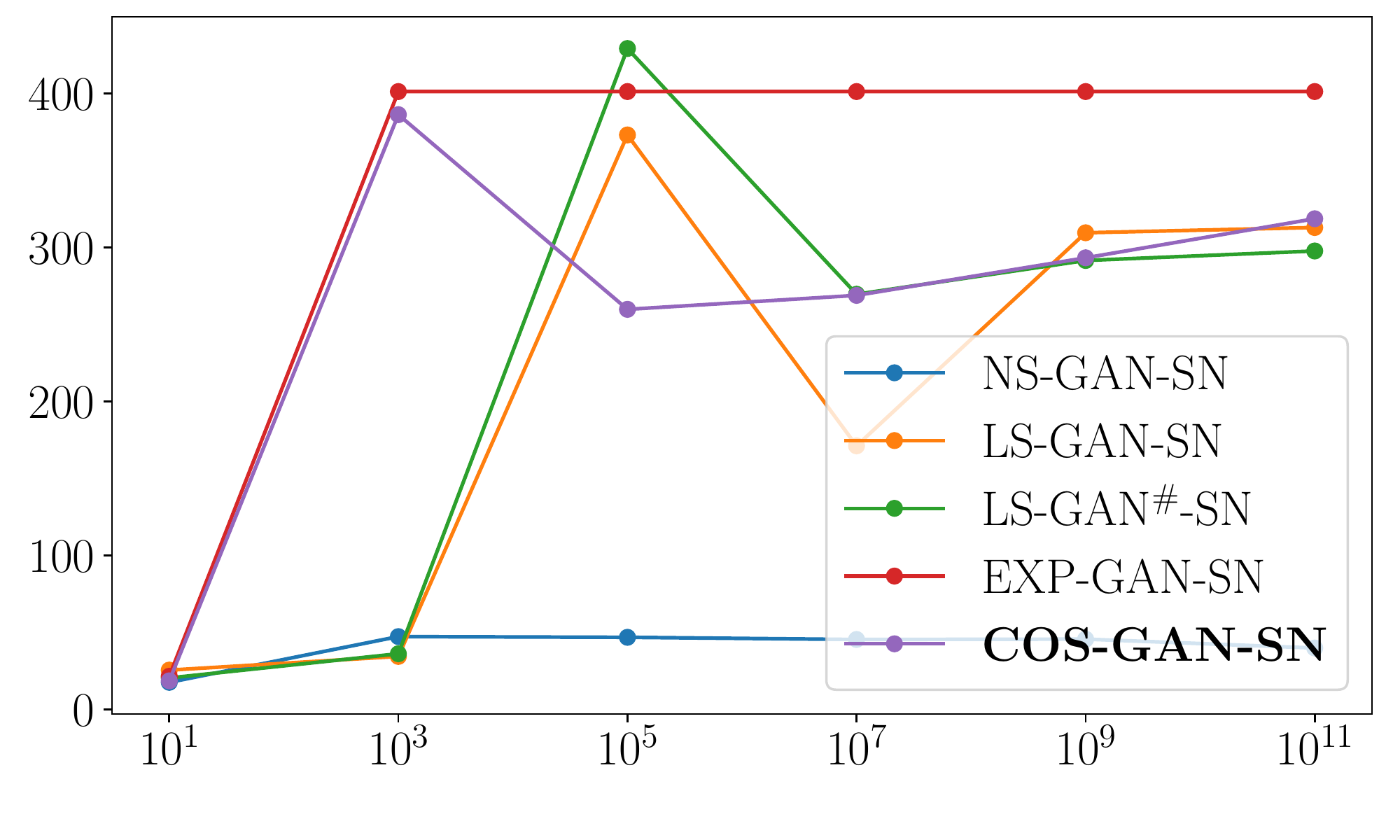}
    \end{subfigure}
\end{center}
  \caption{Samples of randomly generated images with COS-GAN-SN of varying $\alpha$ ($k_{SN}=1.0$, CIFAR10). For the line plot, $x$-axis shows $\alpha$ (in log scale) and $y$-axis shows the FID scores.}
\label{fig:Scale_COSGANSN_CIFAR10_k_1}
\end{figure*}

\FloatBarrier
\newpage

\vspace*{\fill}
\begin{table*}[h!]
\centering
\captionsetup{justification=centering}
\caption*{{\LARGE FID scores v.s. $\alpha$ ($k_{SN}=1.0$) of Different Loss Functions\\ \vspace*{5mm}
- CelebA -}}
\end{table*}
\vspace*{\fill}

\FloatBarrier
\newpage

\begin{figure*}[h]
\begin{center}
    \begin{subfigure}{0.32\textwidth}
        \centering
        \includegraphics[width=0.99\linewidth]{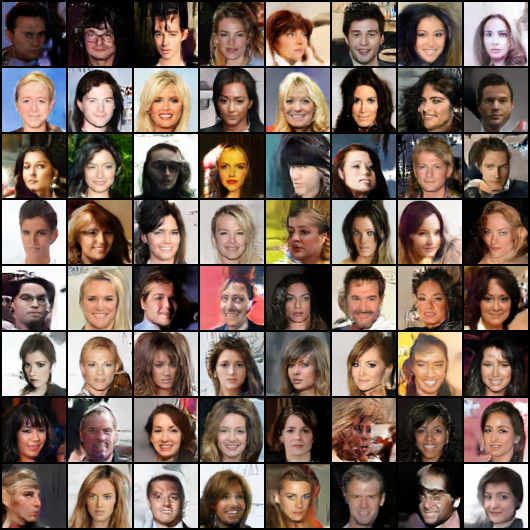}
        \subcaption{$\alpha=1e^{1}$, $\mathrm{FID}=5.88$}
    \end{subfigure}
    \begin{subfigure}{0.32\textwidth}
        \centering
        \includegraphics[width=0.99\linewidth]{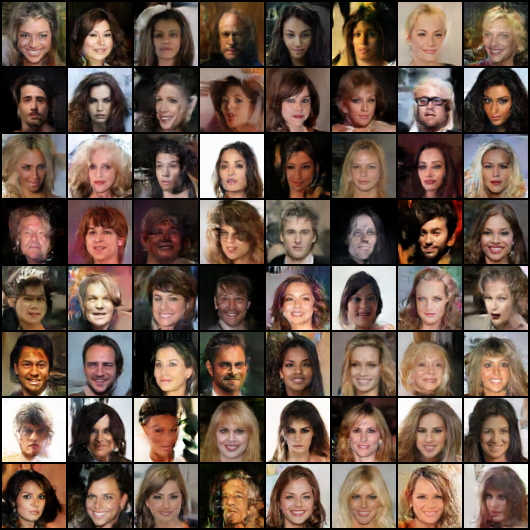}
        \subcaption{$\alpha=1e^{3}$, $\mathrm{FID}=16.14$}
    \end{subfigure}
    \begin{subfigure}{0.32\textwidth}
        \centering
        \includegraphics[width=0.99\linewidth]{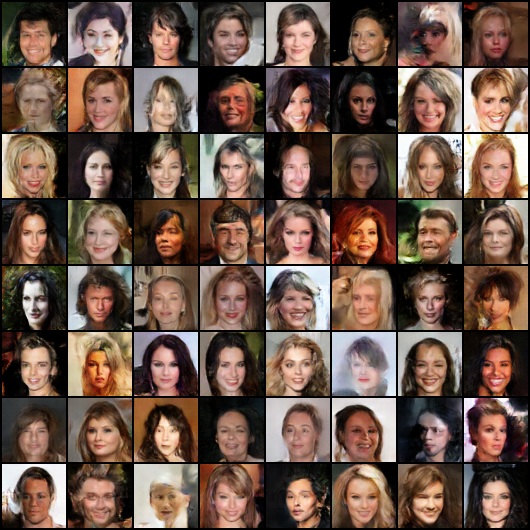}
        \subcaption{$\alpha=1e^{5}$, $\mathrm{FID}=17.75$}
    \end{subfigure}
    \begin{subfigure}{0.32\textwidth}
        \centering
        \includegraphics[width=0.99\linewidth]{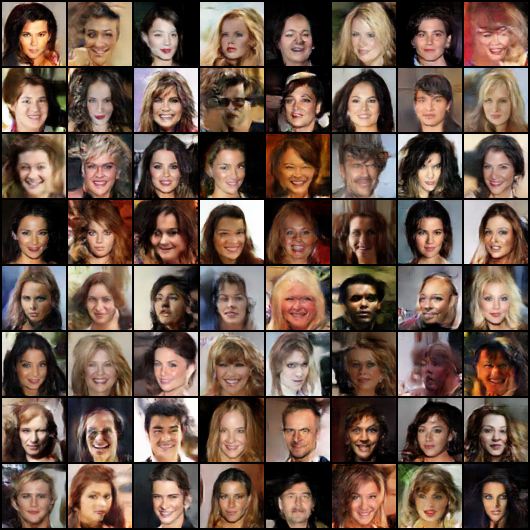}
        \subcaption{$\alpha=1e^{7}$, $\mathrm{FID}=17.67$}
    \end{subfigure}
    \begin{subfigure}{0.32\textwidth}
        \centering
        \includegraphics[width=0.99\linewidth]{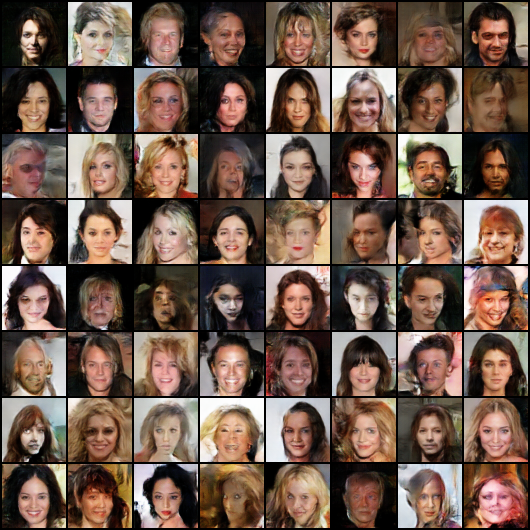}
        \subcaption{$\alpha=1e^{9}$, $\mathrm{FID}=16.87$}
    \end{subfigure}
    \begin{subfigure}{0.32\textwidth}
        \centering
        \includegraphics[width=0.99\linewidth]{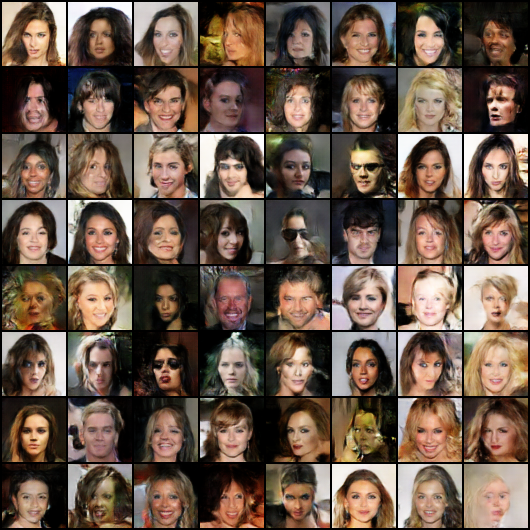}
        \subcaption{$\alpha=1e^{11}$, $\mathrm{FID}=18.81$}
    \end{subfigure}
    \begin{subfigure}{0.6\textwidth}
        \includegraphics[width=.9\linewidth]{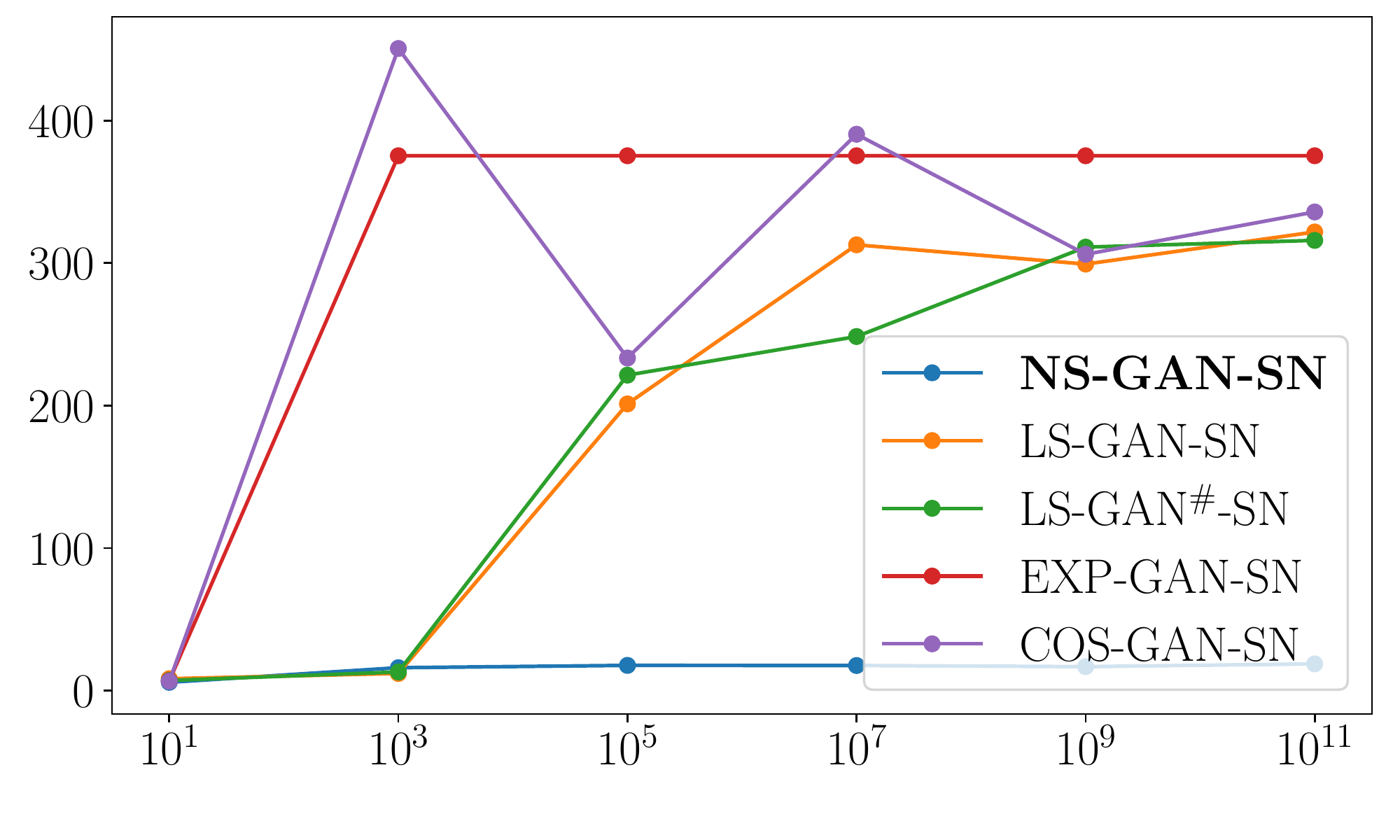}
    \end{subfigure}
\end{center}
  \caption{Samples of randomly generated images with NS-GAN-SN of varying $\alpha$ ($k_{SN}=1.0$, CelebA). For the line plot, $x$-axis shows $\alpha$ (in log scale) and $y$-axis shows the FID scores.}
\label{fig:Scale_NSGANSN_CelebA_k_1}
\end{figure*}

\begin{figure*}[h]
\begin{center}
    \begin{subfigure}{0.32\textwidth}
        \centering
        \includegraphics[width=0.99\linewidth]{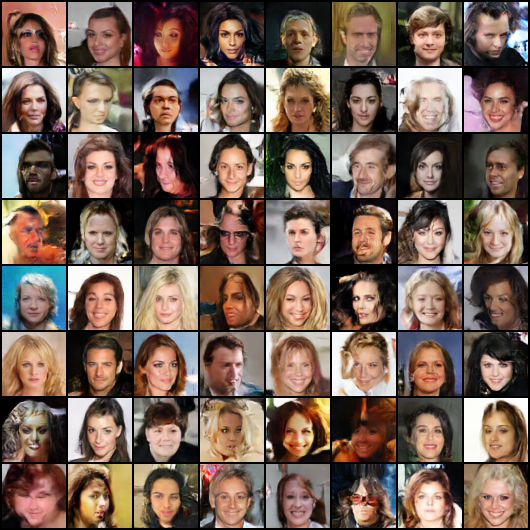}
        \subcaption{$\alpha=1e^{1}$, $\mathrm{FID}=8.41$}
    \end{subfigure}
    \begin{subfigure}{0.32\textwidth}
        \centering
        \includegraphics[width=0.99\linewidth]{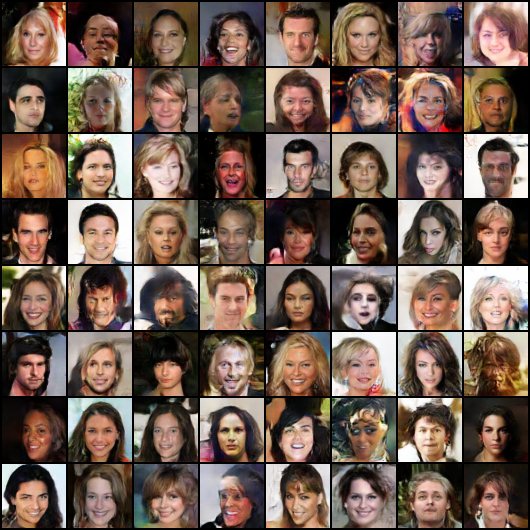}
        \subcaption{$\alpha=1e^{3}$, $\mathrm{FID}=12.09$}
    \end{subfigure}
    \begin{subfigure}{0.32\textwidth}
        \centering
        \includegraphics[width=0.99\linewidth]{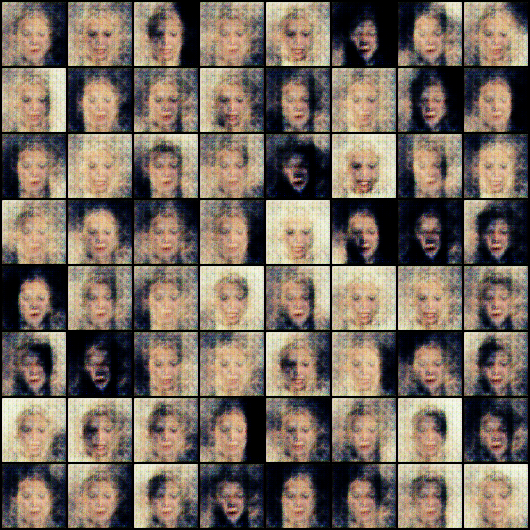}
        \subcaption{$\alpha=1e^{5}$, $\mathrm{FID}=201.22$}
    \end{subfigure}
    \begin{subfigure}{0.32\textwidth}
        \centering
        \includegraphics[width=0.99\linewidth]{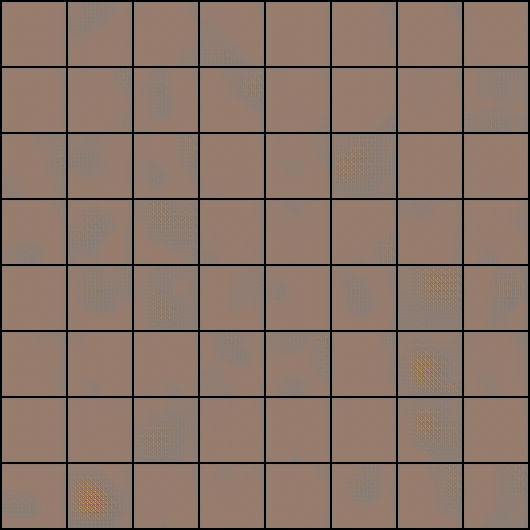}
        \subcaption{$\alpha=1e^{7}$, $\mathrm{FID}=312.83$}
    \end{subfigure}
    \begin{subfigure}{0.32\textwidth}
        \centering
        \includegraphics[width=0.99\linewidth]{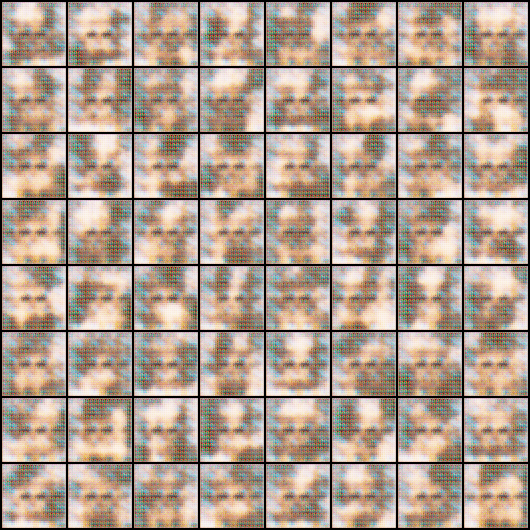}
        \subcaption{$\alpha=1e^{9}$, $\mathrm{FID}=299.30$}
    \end{subfigure}
    \begin{subfigure}{0.32\textwidth}
        \centering
        \includegraphics[width=0.99\linewidth]{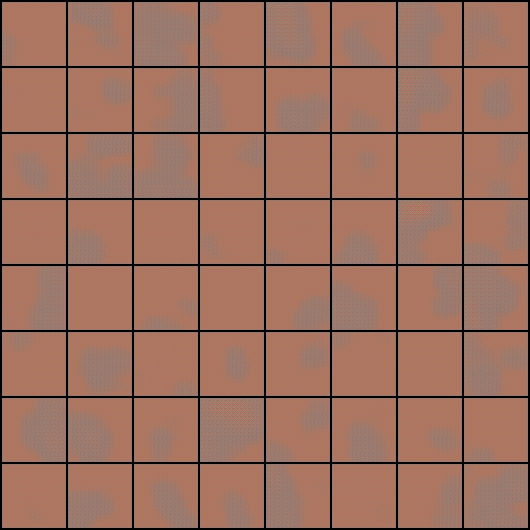}
        \subcaption{$\alpha=1e^{11}$, $\mathrm{FID}=321.84$}
    \end{subfigure}
    \begin{subfigure}{0.6\textwidth}
        \includegraphics[width=.9\linewidth]{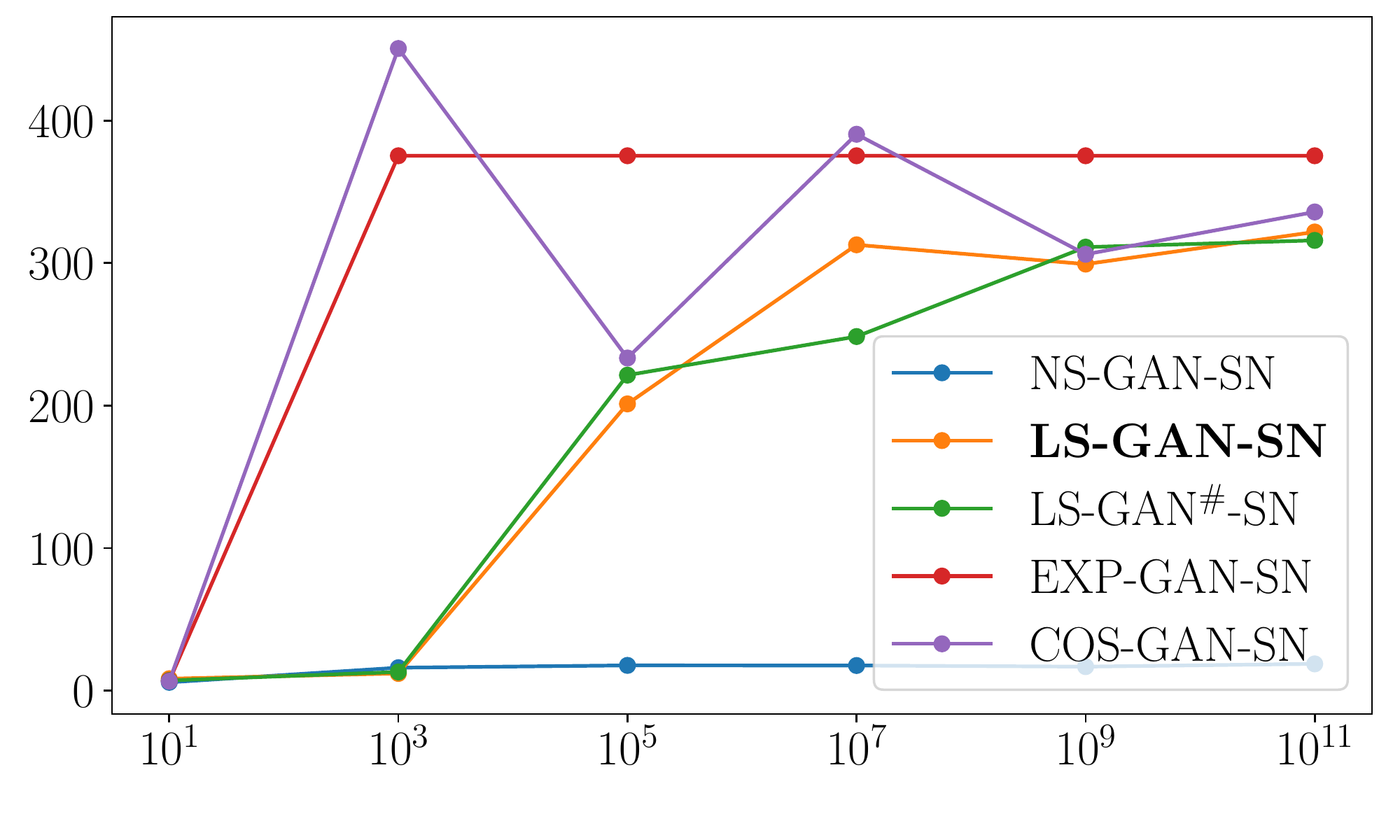}
    \end{subfigure}
\end{center}
  \caption{Samples of randomly generated images with LS-GAN-SN of varying $\alpha$ ($k_{SN}=1.0$, CelebA). For the line plot, $x$-axis shows $\alpha$ (in log scale) and $y$-axis shows the FID scores.}
\label{fig:Scale_LSGANSN_CelebA_k_1}
\end{figure*}

\begin{figure*}[h]
\begin{center}
    \begin{subfigure}{0.32\textwidth}
        \centering
        \includegraphics[width=0.99\linewidth]{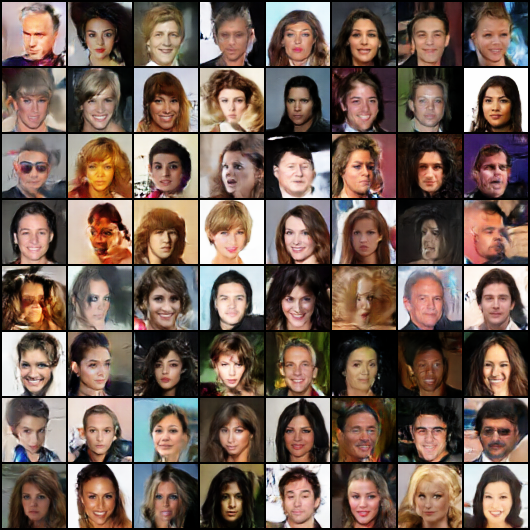}
        \subcaption{$\alpha=1e^{1}$, $\mathrm{FID}=7.21$}
    \end{subfigure}
    \begin{subfigure}{0.32\textwidth}
        \centering
        \includegraphics[width=0.99\linewidth]{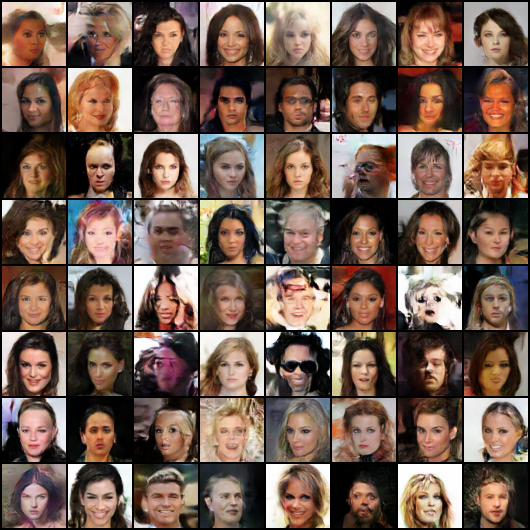}
        \subcaption{$\alpha=1e^{3}$, $\mathrm{FID}=13.13$}
    \end{subfigure}
    \begin{subfigure}{0.32\textwidth}
        \centering
        \includegraphics[width=0.99\linewidth]{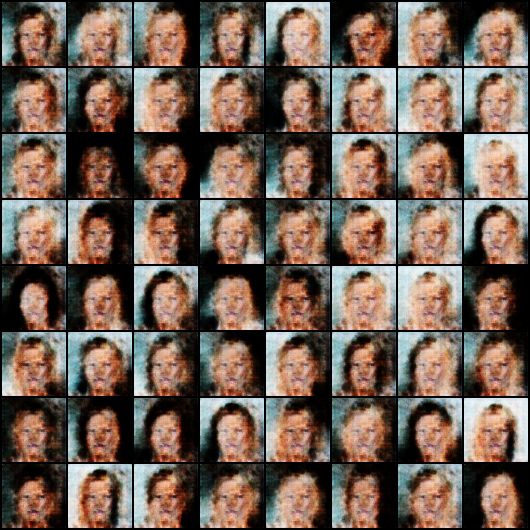}
        \subcaption{$\alpha=1e^{5}$, $\mathrm{FID}=221.41$}
    \end{subfigure}
    \begin{subfigure}{0.32\textwidth}
        \centering
        \includegraphics[width=0.99\linewidth]{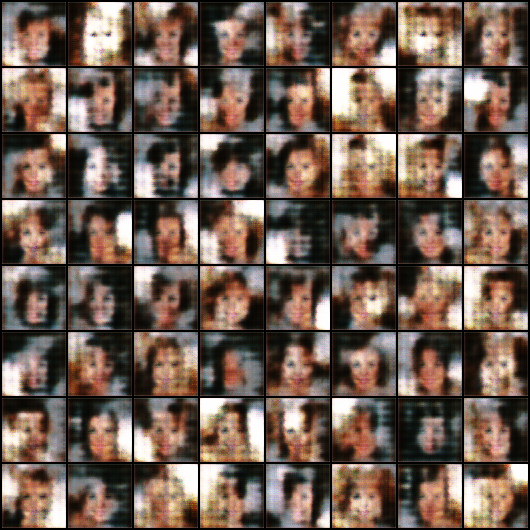}
        \subcaption{$\alpha=1e^{7}$, $\mathrm{FID}=248.48$}
    \end{subfigure}
    \begin{subfigure}{0.32\textwidth}
        \centering
        \includegraphics[width=0.99\linewidth]{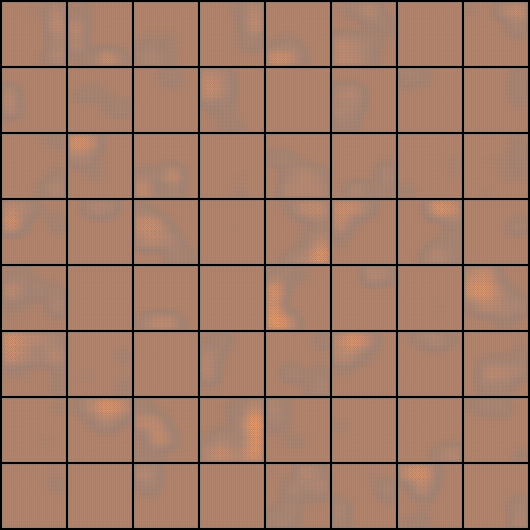}
        \subcaption{$\alpha=1e^{9}$, $\mathrm{FID}=311.21$}
    \end{subfigure}
    \begin{subfigure}{0.32\textwidth}
        \centering
        \includegraphics[width=0.99\linewidth]{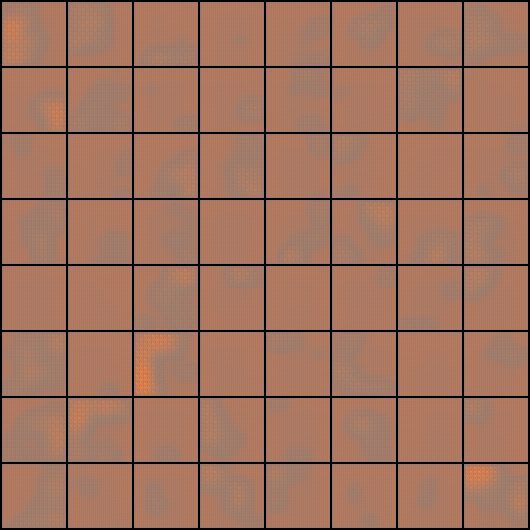}
        \subcaption{$\alpha=1e^{11}$, $\mathrm{FID}=315.94$}
    \end{subfigure}
    \begin{subfigure}{0.6\textwidth}
        \includegraphics[width=.9\linewidth]{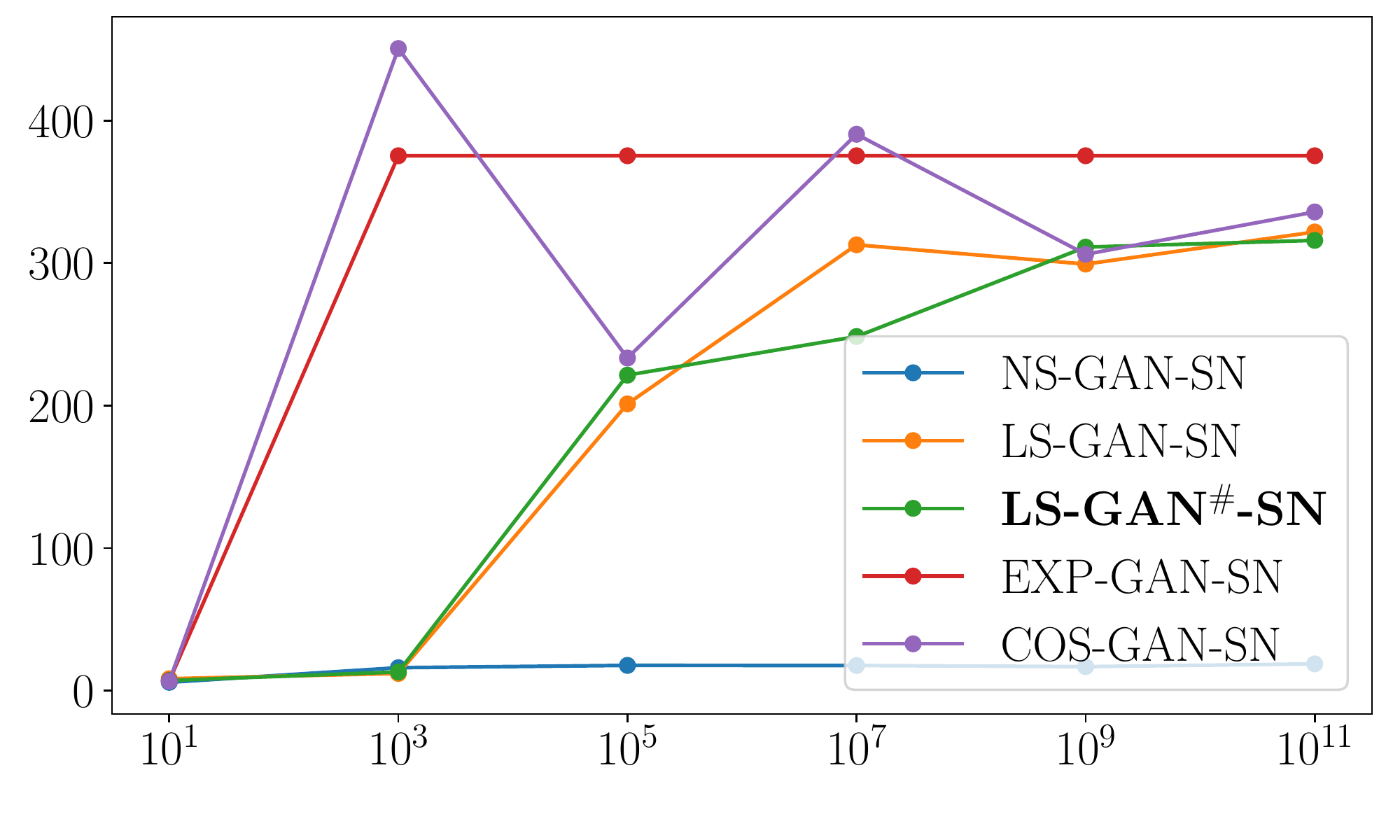}
    \end{subfigure}
\end{center}
  \caption{Samples of randomly generated images with LS-GAN$^\#$-SN of varying $\alpha$ ($k_{SN}=1.0$, CelebA). For the line plot, $x$-axis shows $\alpha$ (in log scale) and $y$-axis shows the FID scores.}
\label{fig:Scale_LSGANSN_zero_centered_CelebA_k_1}
\end{figure*}

\begin{figure*}[h]
\begin{center}
    \begin{subfigure}{0.32\textwidth}
        \centering
        \includegraphics[width=0.99\linewidth]{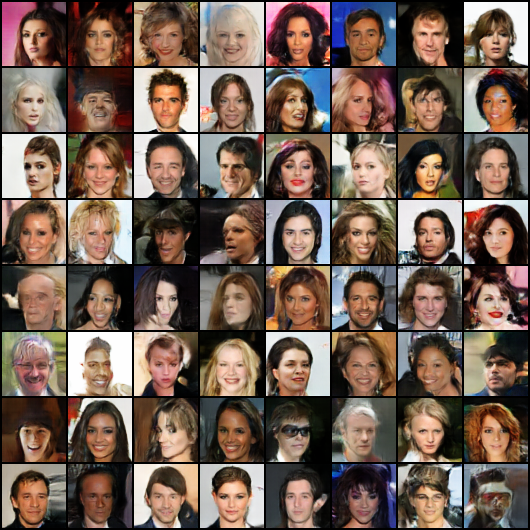}
        \subcaption{$\alpha=1e^{1}$, $\mathrm{FID}=6.91$}
    \end{subfigure}
    \begin{subfigure}{0.32\textwidth}
        \centering
        \includegraphics[width=0.99\linewidth]{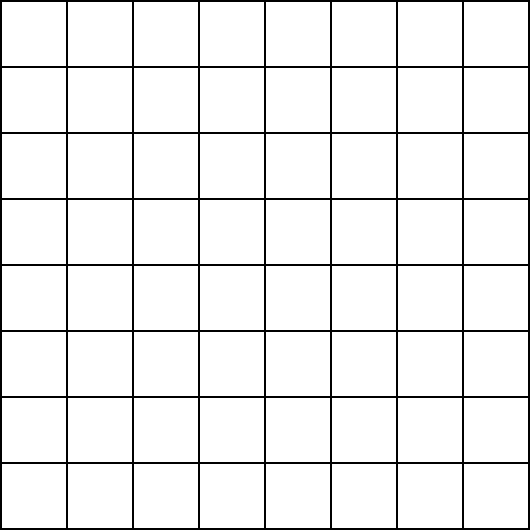}
        \subcaption{$\alpha=1e^{3}$, $\mathrm{FID}=375.32$}
    \end{subfigure}
    \begin{subfigure}{0.32\textwidth}
        \centering
        \includegraphics[width=0.99\linewidth]{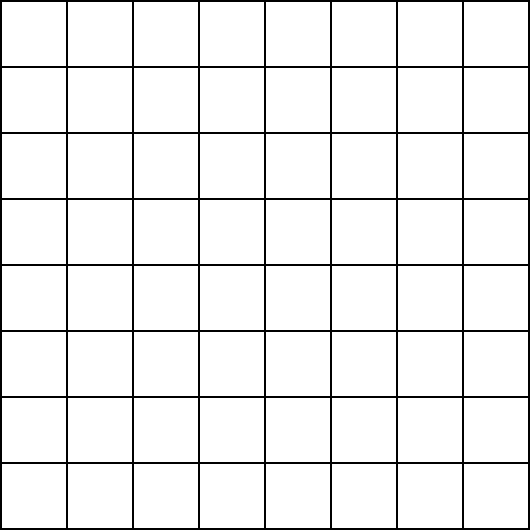}
        \subcaption{$\alpha=1e^{5}$, $\mathrm{FID}=375.32$}
    \end{subfigure}
    \begin{subfigure}{0.32\textwidth}
        \centering
        \includegraphics[width=0.99\linewidth]{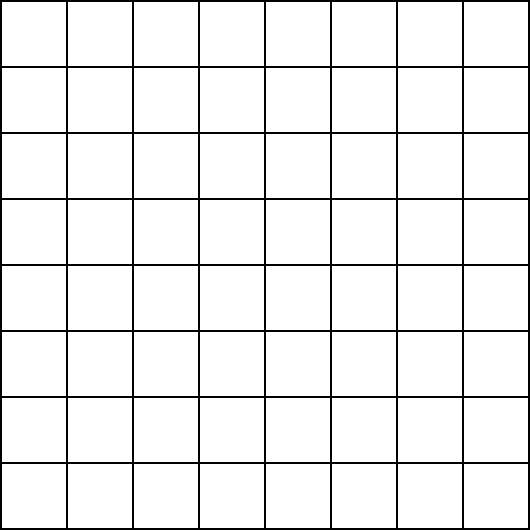}
        \subcaption{$\alpha=1e^{7}$, $\mathrm{FID}=375.32$}
    \end{subfigure}
    \begin{subfigure}{0.32\textwidth}
        \centering
        \includegraphics[width=0.99\linewidth]{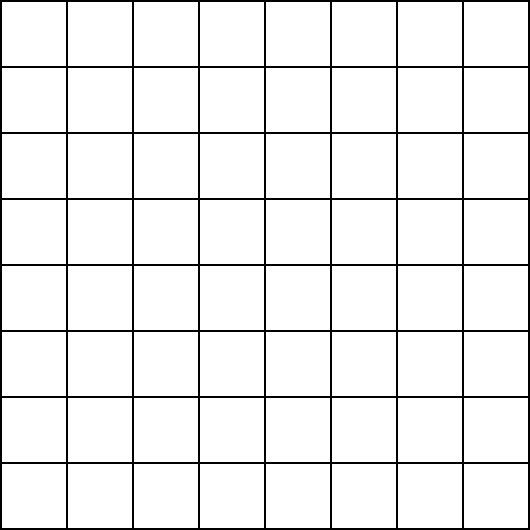}
        \subcaption{$\alpha=1e^{9}$, $\mathrm{FID}=375.32$}
    \end{subfigure}
    \begin{subfigure}{0.32\textwidth}
        \centering
        \includegraphics[width=0.99\linewidth]{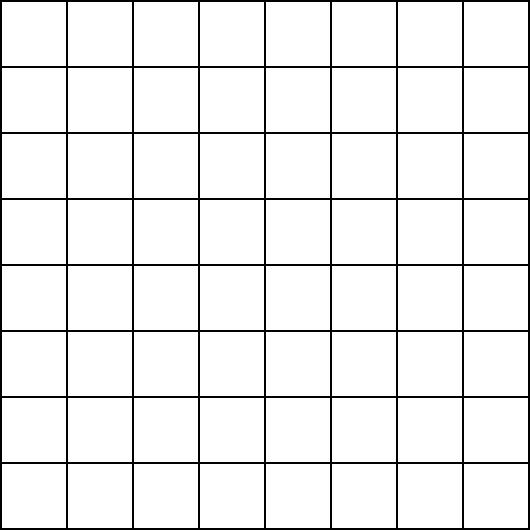}
        \subcaption{$\alpha=1e^{11}$, $\mathrm{FID}=375.32$}
    \end{subfigure}
    \begin{subfigure}{0.6\textwidth}
        \includegraphics[width=.9\linewidth]{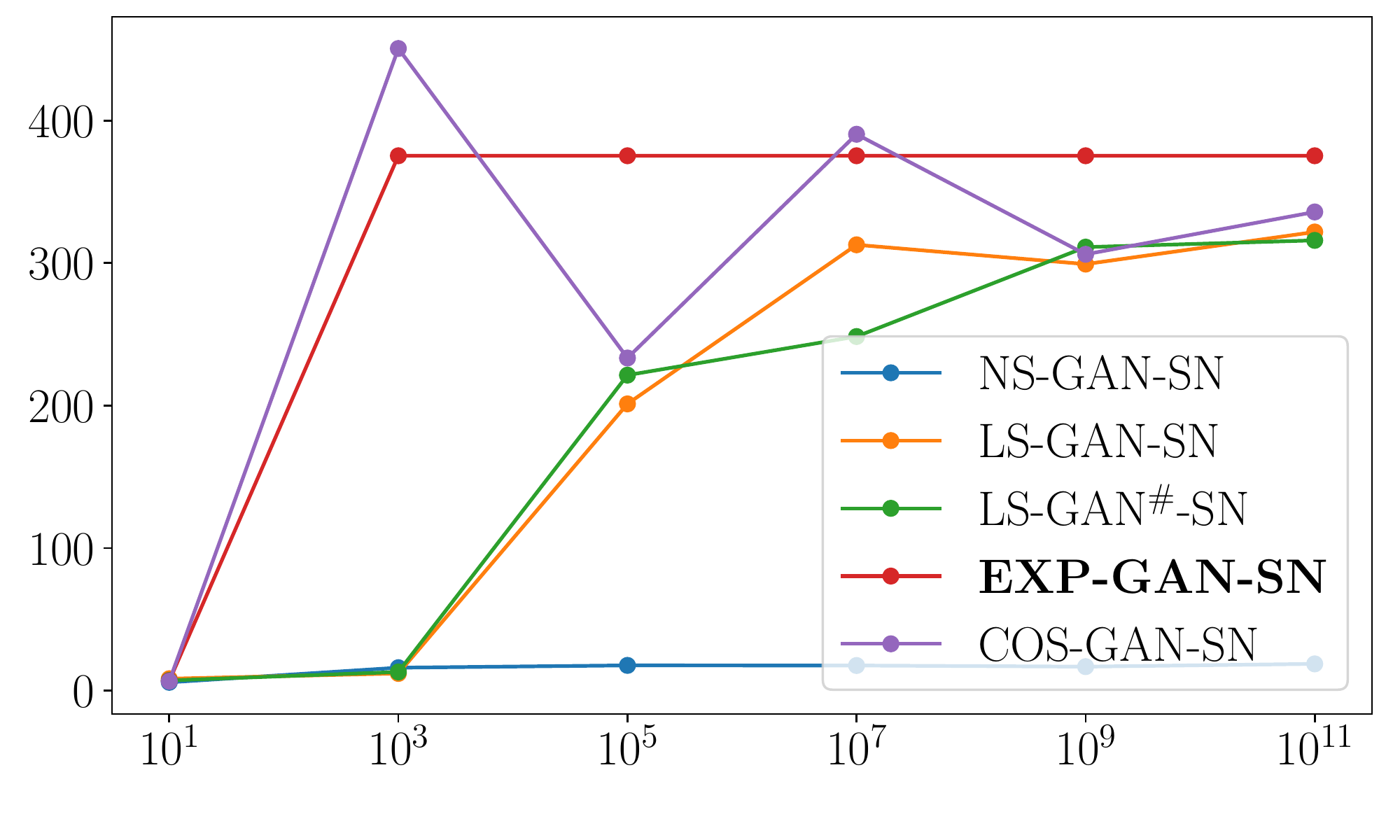}
    \end{subfigure}
\end{center}
  \caption{Samples of randomly generated images with EXP-GAN-SN of varying $\alpha$ ($k_{SN}=1.0$, CelebA). For the line plot, $x$-axis shows $\alpha$ (in log scale) and $y$-axis shows the FID scores.}
\label{fig:Scale_EXPGANSN_CelebA_k_1}
\end{figure*}

\begin{figure*}[h]
\begin{center}
    \begin{subfigure}{0.32\textwidth}
        \centering
        \includegraphics[width=0.99\linewidth]{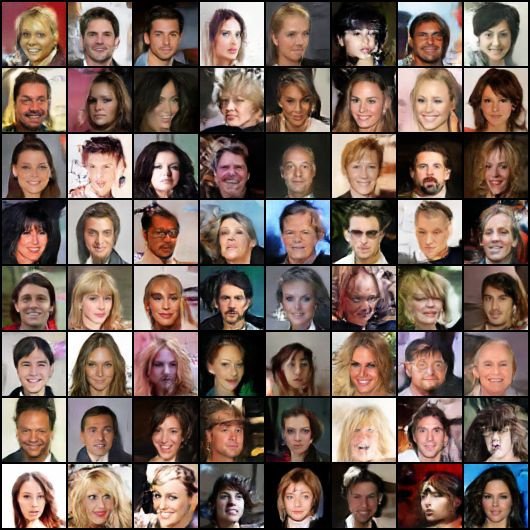}
        \subcaption{$\alpha=1e^{1}$, $\mathrm{FID}=6.62$}
    \end{subfigure}
    \begin{subfigure}{0.32\textwidth}
        \centering
        \includegraphics[width=0.99\linewidth]{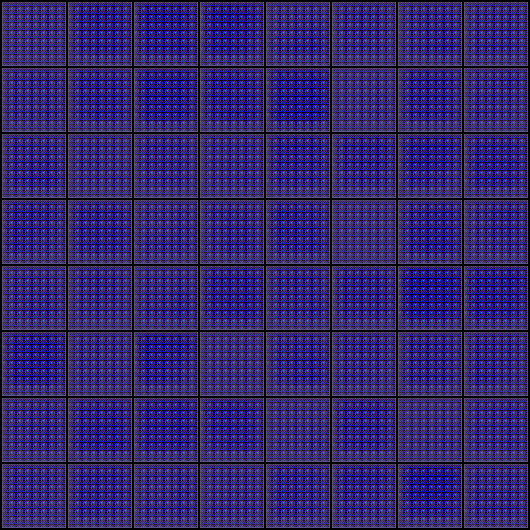}
        \subcaption{$\alpha=1e^{3}$, $\mathrm{FID}=450.57$}
    \end{subfigure}
    \begin{subfigure}{0.32\textwidth}
        \centering
        \includegraphics[width=0.99\linewidth]{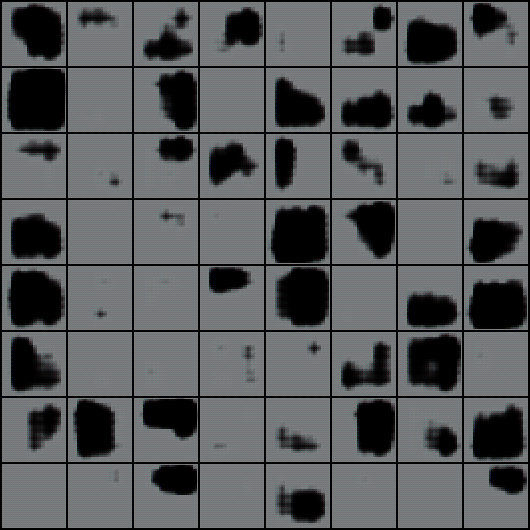}
        \subcaption{$\alpha=1e^{5}$, $\mathrm{FID}=233.42$}
    \end{subfigure}
    \begin{subfigure}{0.32\textwidth}
        \centering
        \includegraphics[width=0.99\linewidth]{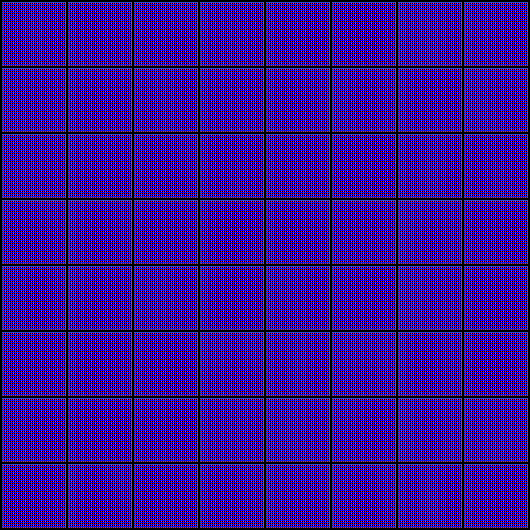}
        \subcaption{$\alpha=1e^{7}$, $\mathrm{FID}=390.40$}
    \end{subfigure}
    \begin{subfigure}{0.32\textwidth}
        \centering
        \includegraphics[width=0.99\linewidth]{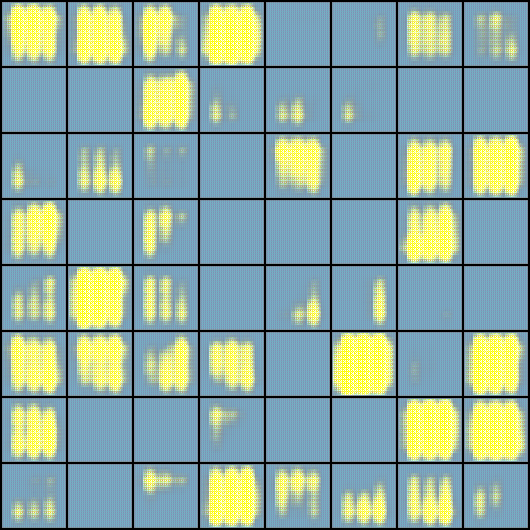}
        \subcaption{$\alpha=1e^{9}$, $\mathrm{FID}=306.17$}
    \end{subfigure}
    \begin{subfigure}{0.32\textwidth}
        \centering
        \includegraphics[width=0.99\linewidth]{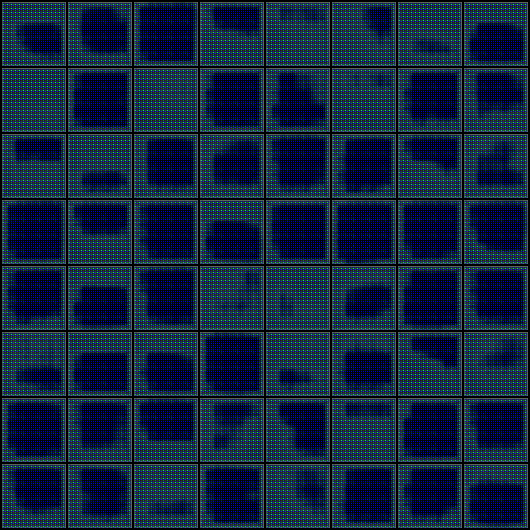}
        \subcaption{$\alpha=1e^{11}$, $\mathrm{FID}=335.87$}
    \end{subfigure}
    \begin{subfigure}{0.6\textwidth}
        \includegraphics[width=.9\linewidth]{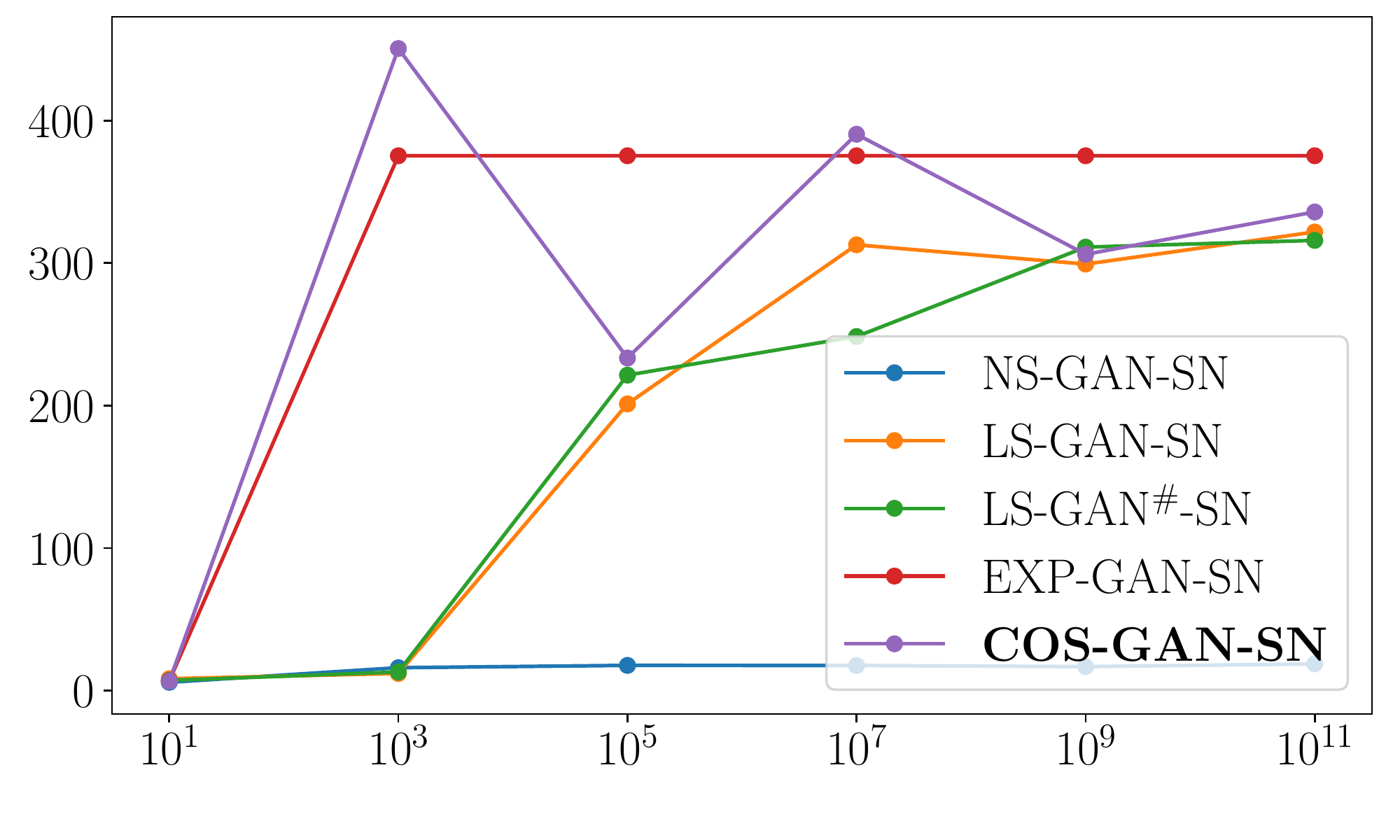}
    \end{subfigure}
\end{center}
  \caption{Samples of randomly generated images with COS-GAN-SN of varying $\alpha$ ($k_{SN}=1.0$, CelebA). For the line plot, $x$-axis shows $\alpha$ (in log scale) and $y$-axis shows the FID scores.}
\label{fig:Scale_COSGANSN_CelebA_k_1}
\end{figure*}

\FloatBarrier
\newpage

\vspace*{\fill}
\begin{table*}[h!]
\centering
\captionsetup{justification=centering}
\caption*{{\LARGE FID scores of WGAN-SN and some extremely degenerate loss functions ($\alpha=1e^{-25}$) \\ on different datasets}}
\end{table*}
\vspace*{\fill}

\FloatBarrier
\newpage

\begin{figure*}[h]
\begin{center}
    \begin{subfigure}{0.32\textwidth}
        \centering
        \includegraphics[width=0.6\linewidth]{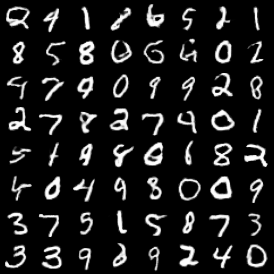}
        \subcaption{WGAN, $\mathrm{FID}=3.71$}
    \end{subfigure}
    \begin{subfigure}{0.32\textwidth}
        \centering
        \includegraphics[width=0.6\linewidth]{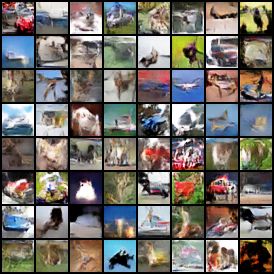}
        \subcaption{WGAN, $\mathrm{FID}=23.36$}
    \end{subfigure}
    \begin{subfigure}{0.32\textwidth}
        \centering
        \includegraphics[width=0.6\linewidth]{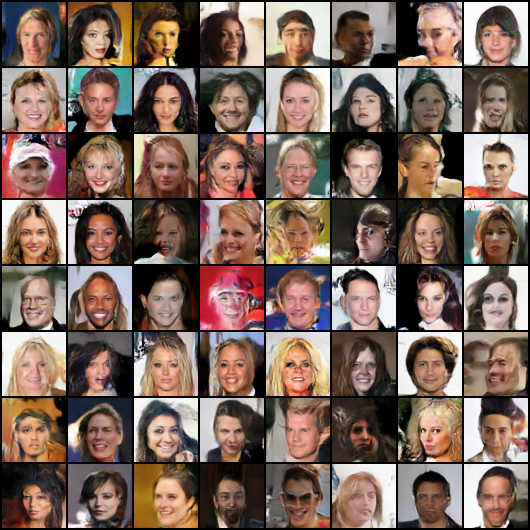}
        \subcaption{WGAN, $\mathrm{FID}=7.82$}
    \end{subfigure}
    
    \begin{subfigure}{0.32\textwidth}
        \centering
        \includegraphics[width=0.6\linewidth]{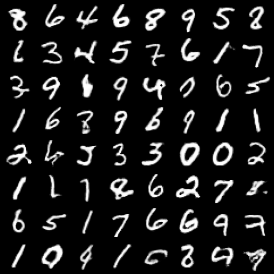}
        \subcaption{NSGAN, $\mathrm{FID}=3.74$}
    \end{subfigure}
    \begin{subfigure}{0.32\textwidth}
        \centering
        \includegraphics[width=0.6\linewidth]{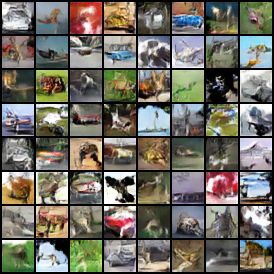}
        \subcaption{NSGAN, $\mathrm{FID}=21.92$}
    \end{subfigure}
    \begin{subfigure}{0.32\textwidth}
        \centering
        \includegraphics[width=0.6\linewidth]{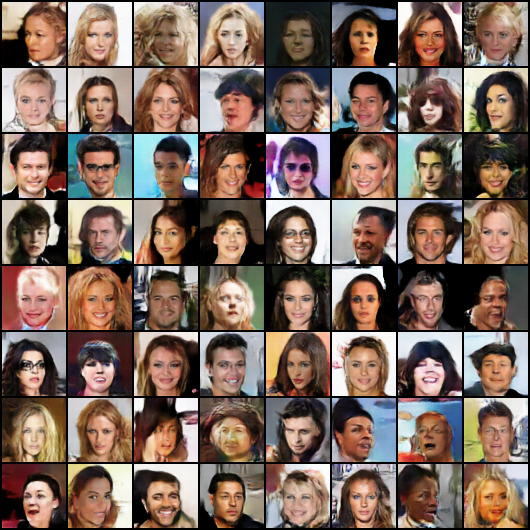}
        \subcaption{NSGAN, $\mathrm{FID}=8.10$}
    \end{subfigure}
    
    \begin{subfigure}{0.32\textwidth}
        \centering
        \includegraphics[width=0.6\linewidth]{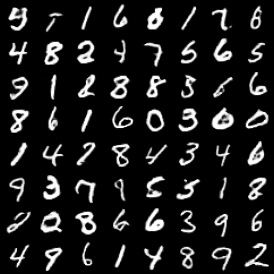}
        \subcaption{LSGAN$^\#$, $\mathrm{FID}=3.81$}
    \end{subfigure}
    \begin{subfigure}{0.32\textwidth}
        \centering
        \includegraphics[width=0.6\linewidth]{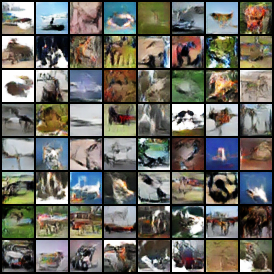}
        \subcaption{LSGAN$^\#$, $\mathrm{FID}=21.47$}
    \end{subfigure}
    \begin{subfigure}{0.32\textwidth}
        \centering
        \includegraphics[width=0.6\linewidth]{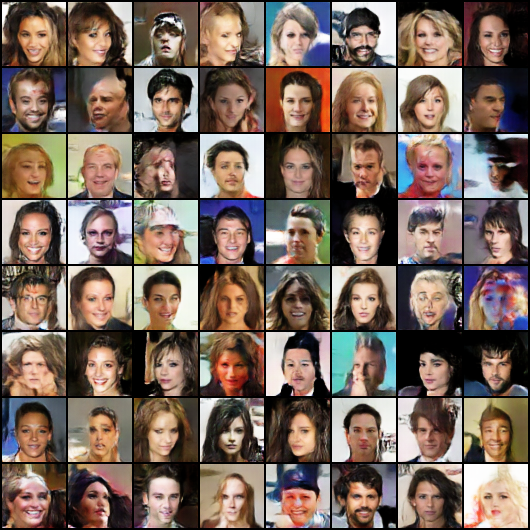}
        \subcaption{LSGAN$^\#$, $\mathrm{FID}=8.51$}
    \end{subfigure}

    \begin{subfigure}{0.32\textwidth}
        \centering
        \includegraphics[width=0.6\linewidth]{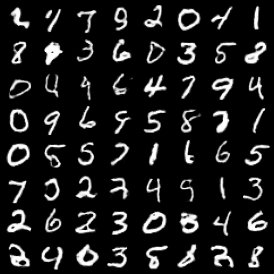}
        \subcaption{COSGAN, $\mathrm{FID}=3.96$}
    \end{subfigure}
    \begin{subfigure}{0.32\textwidth}
        \centering
        \includegraphics[width=0.6\linewidth]{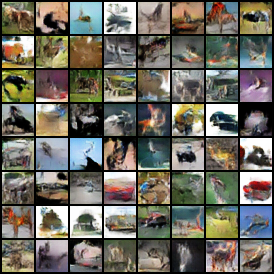}
        \subcaption{COSGAN, $\mathrm{FID}=23.65$}
    \end{subfigure}
    \begin{subfigure}{0.32\textwidth}
        \centering
        \includegraphics[width=0.6\linewidth]{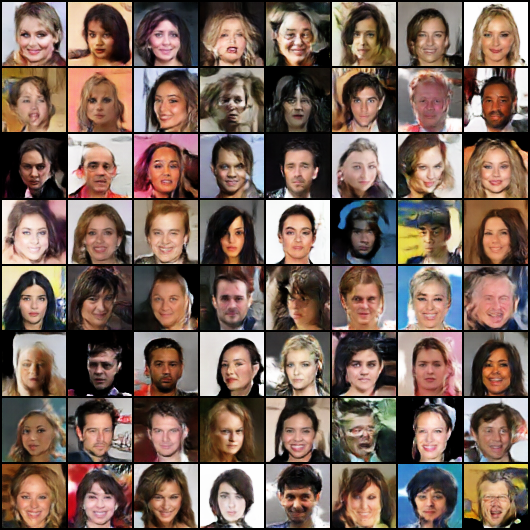}
        \subcaption{COSGAN, $\mathrm{FID}=8.30$}
    \end{subfigure}

    \begin{subfigure}{0.32\textwidth}
        \centering
        \includegraphics[width=0.6\linewidth]{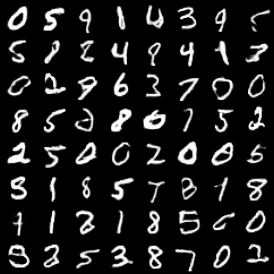}
        \subcaption{EXPGAN, $\mathrm{FID}=3.86$}
    \end{subfigure}
    \begin{subfigure}{0.32\textwidth}
        \centering
        \includegraphics[width=0.6\linewidth]{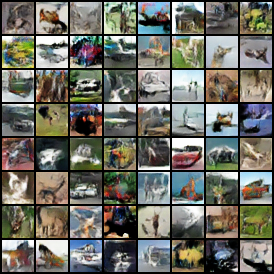}
        \subcaption{EXPGAN, $\mathrm{FID}=21.91$}
    \end{subfigure}
    \begin{subfigure}{0.32\textwidth}
        \centering
        \includegraphics[width=0.6\linewidth]{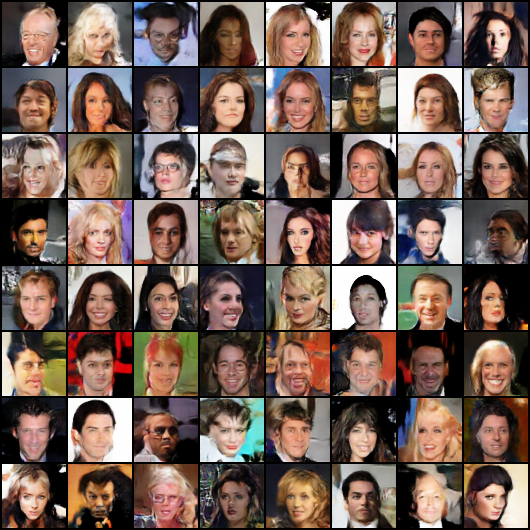}
        \subcaption{EXPGAN, $\mathrm{FID}=8.22$}
    \end{subfigure}
\end{center}
  \caption{Samples of randomly generated images with WGAN-SN and some extremely degenerate loss functions ($\alpha=1e^{-25}$) on different datasets. We use $k_{SN} = 50$ for all our experiments.}
\label{fig:degenerate_loss_functions}
\end{figure*}
\end{document}